\documentclass[imslayout,preprint]{imsart}
\def\journal@name{}
\usepackage[margin=1.5in]{geometry}

\usepackage{amsthm}
\usepackage[numbers,sort&compress]{natbib}

\usepackage[colorlinks,citecolor=blue,urlcolor=black,linkcolor=blue,pdfborder={0 0 0}]{hyperref}
\usepackage{graphicx}
\usepackage{subcaption}

\usepackage{booktabs}
\usepackage{multirow}
\usepackage{float}
\usepackage{algorithmic}
\usepackage[linesnumbered,ruled,vlined]{algorithm2e}
\DontPrintSemicolon

\RequirePackage[usenames,dvipsnames]{color}
\usepackage{jhh-misc2}

\crefname{lemma}{Lemma}{Lemmas}
\crefname{proposition}{Proposition}{Propositions}
\crefname{corollary}{Corollary}{Corollaries}
\crefname{theorem}{Theorem}{Theorems}
\crefname{example}{Example}{Examples}
\crefname{assumption}{Assumption}{Assumptions}
\usepackage{autonum} 
\renewcommand{\Pr}{\mathbb{P}}

\newcommand{\derivop}[1]{\grad^{#1}}
\newcommand{\hessian}{\derivop{2}}

\newcommand{\chain}[1]{v^{(#1)}}

\newcommand{\ineff}{\mcI}

\newcommand{\param}{\theta}  % parameter of interest
\newcommand{\paramrv}{\vartheta}   % parameter random variable
\newcommand{\paramdim}{d}

\newcommand{\targetdensity}{\pi}    % target density

\newcommand{\postdensity}{\targetdensity}     % posterior density 
     
\newcommand{\priordensity}{\targetdensity_{0}}   % prior density 

\newcommand{\lik}{\ell}    % likelihood
\newcommand{\marginallik}{Z}   % marginal likelihood 

\newcommand{\varfamily}{\mcQ}

\newcommand{\approxdensity}{q}

\newcommand{\varparam}{\lambda}   % variational parameter
\newcommand{\optvarparam}{\varparam_{*}}   % optimal variational parameter 
\newcommand{\varparamdim}{m}

\newcommand{\discrepancy}[2]{{\mcD_{#1}(#2)}}

\newcommand{\elbo}[1]{\ensuremath{\mathsf{ELBO}(#1)}\xspace}

\newcommand{\sgditerate}[1]{\varparam^{(#1)}}    % stochastic gradient iterate
\newcommand{\meansgditerate}[1]{\bar\varparam_{#1}}  % mean iterate for step size
\newcommand{\iterateaverage}[1]{\hat{\varparam}_{#1}}  % estimate of mean iterate for step size
\newcommand{\gradsquare}{\nu}    % gradient squared
\newcommand{\graditerate}[1]{m^{(#1)}}    % gradient  iterate 
\newcommand{\gradsquareiterate}[1]{\gradsquare^{(#1)}}    % gradient squared iterate 
\newcommand{\iternum}{k}   % optimization iteration umber 
\newcommand{\objgrad}{g}   % objective gradient
\newcommand{\objgradest}[1]{\hat\objgrad^{(#1)}}  % objective stochastic gradient
\newcommand{\descentdir}[1]{d^{(#1)}}
\newcommand{\learningrate}{\gamma}
\newcommand{\learningrateiterate}[1]{\learningrate^{(#1)}}
\newcommand{\decay}{\beta}
\newcommand{\decayrate}{\alpha}
\newcommand{\windowprop}{\varrho}

\newcommand{\stationarydist}[1]{\mu_{#1}}  % stationary distribution of fixed-step size stochastic optimization

\newcommand{\epochnum}{t}
\newcommand{\totalepochs}{T}
\newcommand{\iterate}{K}
\newcommand{\skl}[2]{\mathrm{SKL}(#1, #2)}
\newcommand{\mcse}{e}
\newcommand{\invariant}{\Delta}
\newcommand{\invariantiterate}[1]{\invariant^{(#1)}}

\newcommand{\biaspower}{\kappa}
\newcommand{\iternumconverged}{\iternum^\mathrm{conv}}
\newcommand{\iternumaverage}{\iternum^\mathrm{avg}}

\newcommand{\approxhessian}[1]{H^{(#1)}} %approximation of inverse of Hessian matrix 
\newcommand{\expectationparam}{m}
\newcommand{\textder}[2]{\dee #1/\dee #2}

\newcommand{\coupling}{\omega}

\newcommand{\pwassSimple}[3]{\mcW_{#1}(#2, #3)}

\newcommand{\mean}[1]{m_{#1}}
\newcommand{\meanvec}[1]{m_{#1}}

\newcommand{\fishinf}[1]{{F}_{#1}}

\def\norm#1{\left\|{#1}\right\|} % A norm with 1 argument
\newcommand{\normarg}[2]{\norm{#2}_{#1}}

\newcommand{\twonorm}[1]{\norm{#1}_2} % L2 norm
\def\staticnorm#1{\|{#1}\|} % A static norm that does not resize with input
\newcommand{\statictwonorm}[1]{\staticnorm{#1}_2} % L2 norm
\newcommand{\inner}[2]{\langle{#1},{#2}\rangle} % Inner product
\theoremstyle{plain}
\newtheorem{theorem}{Theorem}[section]
\newtheorem{proposition}[theorem]{Proposition}

\theoremstyle{definition}
\newtheorem{definition}[theorem]{Definition}

\theoremstyle{remark}
\newtheorem{remark}[theorem]{Remark}

\graphicspath{ {figures/} }

\begin{document}

\begin{frontmatter}

\title{A Framework for Improving the Reliability of Black-box Variational Inference}
\runtitle{~Improving the Reliability of BBVI}
\runauthor{Welandawe et al.~}

\begin{aug}

\author[BU]{\fnms{Manushi}  \snm{Welandawe}\ead[label=mw]{manushiw@bu.edu }},
\author[TU]{\fnms{Michael Riis} \snm{Andersen}\ead[label=mra]{michael.riis@gmail.com}}, \\
\author[AU]{\fnms{Aki} \snm{Vehtari}\ead[label=av]{aki.vehtari@aalto.fi}},
\and
\author[BU,CDS]{\fnms{Jonathan H.} \snm{Huggins}\ead[label=jh]{huggins@bu.edu}}
\address[BU]{Department of Mathematics \& Statistics, Boston University, \printead{mw}}
\address[TU]{DTU Compute, Technical University of Denmark, \printead{mra}}
\address[AU]{Department of Computer Science, Aalto University, \printead{av}}
\address[CDS]{Faculty of Computing \& Data Sciences, Boston University, \printead{jh}}
\end{aug}

\begin{abstract}%
Black-box variational inference (BBVI) now sees widespread use in 
machine learning and statistics as a
fast yet flexible alternative to Markov chain Monte Carlo methods 
for approximate Bayesian inference. 
However, stochastic optimization methods for BBVI remain unreliable 
and require substantial expertise and hand-tuning to apply
effectively. 
In this paper, we propose %
\emph{Robust and Automated Black-box VI} (RABVI), a framework for improving the reliability of BBVI optimization.
RABVI is based on rigorously justified automation techniques,
includes just a small number of intuitive tuning parameters, and
detects inaccurate estimates of the optimal variational approximation. 
RABVI adaptively decreases the learning rate by detecting %
convergence of the fixed--learning-rate iterates, then
estimates the symmetrized Kullback--Leibler (KL) divergence between the current variational approximation and the optimal one. 
It also employs a novel optimization termination criterion that enables the user to balance desired accuracy against computational cost
by comparing  (i) the predicted relative decrease in the symmetrized KL divergence if a smaller learning were used and 
(ii) the predicted computation required to converge with the smaller learning rate.
We validate the robustness and accuracy of RABVI through 
carefully designed simulation studies and on a diverse set of real-world model and data examples.

\end{abstract}

\end{frontmatter}

\section{Introduction}

A core strength of the Bayesian approach is that it is conceptually straightforward to carry out inference in \emph{arbitrary} models,
which enables the user to employ whatever model is most appropriate for the problem at hand. 
The flexibility and uncertainty quantification provided by Bayesian inference have led to its widespread use in 
statistics \citep{Robert:2007,Gelman:2013} and machine learning \citep{Bishop:2006,Murphy:2012}, 
including in deep learning \citep[e.g.,][]{Kingma:2014,Rezende:2014,Gal:2016dropout,Maddox:2019,Saatchi:2017,Johnson:2016,Burda:2016:IWAE}.
Using Bayesian methods in practice, however, typically requires using approximate inference algorithms to estimate
quantities of interest such as posterior functionals (e.g., means, covariances, predictive distributions, and tail probabilities)
and measures of model fit (e.g., marginal likelihoods and cross-validated predictive accuracy). 
Therefore, a core challenge in modern Bayesian statistics is the development of \emph{general-purpose} (or \emph{black-box})
algorithms that can accurately approximate these quantities for whatever model the user chooses.
In machine learning, rather than using Markov chain Monte Carlo (MCMC), 
\emph{black-box variational inference} (BBVI) has become the method of choice %
because of its scalability and wide-applicability \citep{Wainwright:2008,Blei:2017,Kingma:2014,Rezende:2014,Burda:2016:IWAE}. 
BBVI is implemented in many software packages for general-purpose inference such as Stan, Pyro, PyMC3, and TensorFlow Probability, which have seen widespread adoption 
by applied data analysts, statisticians, and data scientists.

Variational inference methods aim to minimize a measure of discrepancy between a parameterized family of distributions and
the posterior distribution, with the Kullback--Leibler divergence being the canonical choice of discrepancy.
Conventional approaches to variational inference leverage conditional conjugacy and other model-specific structure to derive
optimization algorithms.  
BBVI, on the other hand, uses stochastic optimization to avoid the need for model-specific derivations,
thereby significantly broadening the applicability of variational methods. 
Ensuring the efficiency and reliability of the BBVI optimization requires careful selection of optimization method
and the stochastic estimator of the discrepancy gradient. 
For example, using adaptive optimization procedures like Adam, RMSProp, and Adagrad can ensure stability \citep{Hinton:2012,Kingma:2015:adam,Duchi:2011:adagrad}.
Due to its relatively small variance, the most common gradient estimator is the reparameterization gradient \citep{salimans:2013,Kingma:2014,Rezende:2014,Ruiz:2016,Bamler:2017,Domke:2019}. 
However, sometimes alternatives such as the score function gradient are employed \citep{Paisley:2012,Ranganath:2014}. 
No matter the choice of gradient estimator, various variance reduction strategies like control variates have been introduced to stabilize and speed up optimization \citep{Miller:2017b,Roeder:2017,Geffner:2018,Geffner:2020,Boustati:2020,Wang:2023a, Wang:2023b}. 
Recent research has furthered our understanding of convergence behavior, tackling theoretical challenges in stochastic optimization and providing new convergence guarantees \citep{Domke:2023,Kim:2023}.

While some recent progress has been made in developing tools for assessing the accuracy of variational approximations \citep{Yao:2018:VI,Huggins:2020:VI,Wang2023},
stochastic optimization methods for BBVI remain unreliable and require substantial hand-tuning of the number of iterations and optimizer tuning parameters.
Moreover, there are few tools available for determining whether the variational parameters estimated by these frameworks are close to optimal
in any meaningful sense and, if not, how to address the problem -- More iterations? A different learning rate schedule? A smaller initial or final learning rate? 
\citet{Agrawal:2020:BBVI} demonstrate the absence of a reliable and coherent optimization methodology for BBVI. 
The authors synthesize and compare recent advances such as normalizing flows and gradient estimators using 30 benchmarked models. 
Despite the fact that these models vary greatly in the complexity and dimensionality of the posteriors, 
\citet{Agrawal:2020:BBVI} run each optimization for a fixed number of iterations (30,000) and for 5 different step sizes
because the existing literature does not provide any compelling guidance for how to automate the choice of step size and 
reliably determine when the optimization has converged.

\begin{figure}[tbp]
\begin{center}
\centering
\includegraphics[width=.80\textwidth]{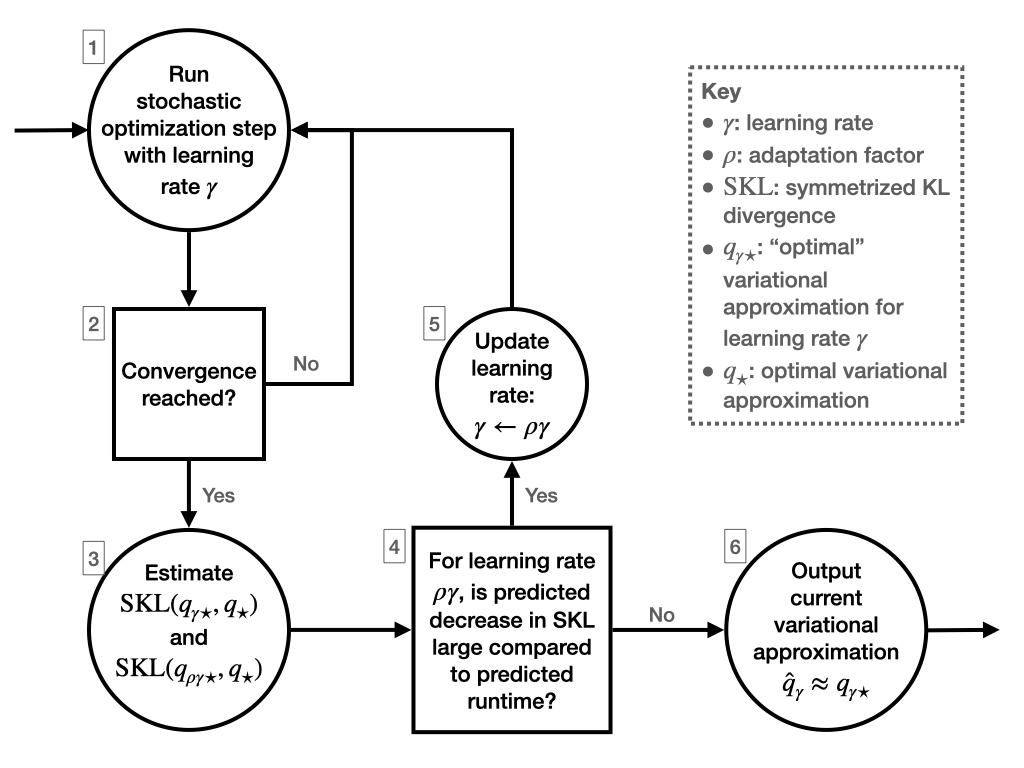}
\caption{
Schematic of the high-level logic of our proposed \emph{Robust and Automated BBVI} (RABVI) framework.} %
\label{fig:algorithm-overview}
\end{center}
\end{figure}

\begin{figure}[tbp]
\begin{center}
\begin{subfigure}[t]{.38\textwidth}
\includegraphics[width=\textwidth]{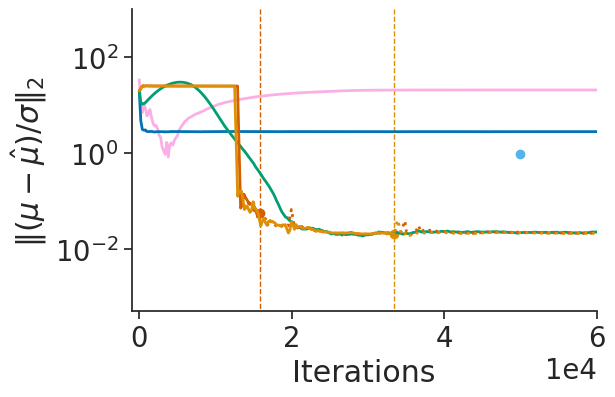}
\caption{\textit{nes2000}}
\end{subfigure}
\begin{subfigure}[t]{.60\textwidth}
\includegraphics[width=\textwidth]{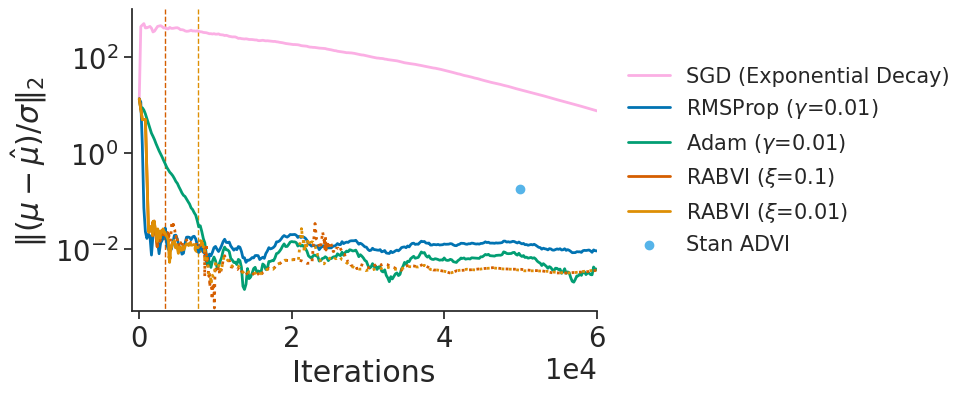}
\caption{\textit{dogs}}
\end{subfigure}
\caption{Accuracy comparison of variational inference algorithms using two datasets/models from the \texttt{posteriordb} package  (see \cref{subsec:posteriordb-datasets} for details).
Accuracy is measured in terms of relative mean error $\twonorm{(\mu - \hat\mu) / \sigma}$,
where $\mu$ and $\sigma$ are, respectively, the posterior mean and standard deviation vectors and $\hat\mu$ is the variational approximation to $\mu$.
The vertical lines indicate the termination points for RABVI, %
which uses averaged Adam (see \cref{sec:SKLD-termination-rule}).
The fixed-learning rate algorithms have a learning rate $\learningrate$ and RABVI has a user-specified accuracy threshold $\xi$.
For SGD exponential decay, we use an initial learning rate of $0.01$ for \textit{dogs} but a smaller initial learning rate of $0.001$ for 
\textit{nes2000} due to optimization instability.
}
\label{fig:example-result}
\end{center}
\end{figure}

\subsection{Contributions}

In view of the significant limitations of existing BBVI optimization methodology, in this paper we aim to provide a practical, cohesive, and theoretically well-grounded optimization framework for BBVI.
To ensure reliability and wide applicability, we develop a framework that is (1) automated, (2) intuitively adjustable by the user, and (3) robust to failure and tuning parameter selection. 
Our approach builds on a recent line of work inspired by \citet{Pflug:1990}, which uses a fixed learning rate $\learningrate$ that is adaptively decreased by 
a multiplicative factor $\rho$ once the optimization iterates, which form a homogenous Markov chain, have converged \citep{Chee:2018,Yaida:2019,Pesme:2020,Chee:2020,Zhang:2020:SALSA,Dhaka:2020}. 
A benefit of this approach is that, for a given learning rate, a dramatically more accurate estimate of the optimal variational parameter can be obtained by using \emph{iterate averaging} \citep{Ruppert:1988,Polyak:1992,Dieuleveut:2020:SGD-MC}. %
However, as we have shown in previous work, existing convergence checks can be unreliable and stop too early  \citep{Dhaka:2020}.
Since the learning rate is decreased by a constant multiplicative factor, decreasing it too early can slow down 
the optimization by an order of magnitude or more. 
Hence, it is crucial to develop methods that do not prematurely declare convergence. 
On the other hand, an optimization framework must also provide a termination criterion that triggers when it is no longer worthwhile to 
 decrease the learning rate further, either because the current variational approximation is sufficiently accurate or because further optimization would be
too time-consuming. 

The key idea that informs our solutions to these challenges is that we want $\approxdensity_{\learningrate*}$, the target variational approximation for learning rate $\learningrate$,
to be close to the optimal variational approximation $\approxdensity_{*}$. 
We measure closeness in terms of \emph{symmetrized Kullback--Leibler (KL) divergence} $\skl{\approxdensity_{*}}{\approxdensity_{\learningrate*}}$ and
show that closeness in symmetrized KL divergence can be translated into bounds on other widely used accuracy metrics like Wasserstein distance  \citep{Huggins:2020:VI,Bolley:2005}. 
\Cref{fig:algorithm-overview} summarizes our proposed framework, which we call \emph{Robust and Automated Black-box VI %
(RABVI)}.
The primary contributions of this paper are in steps 2, 3, and 4. 
In step 2, to determine convergence at a fixed step size, we build upon our approach in  \citet{Dhaka:2020}, where we establish that the scale-reduction factor $\widehat R$ \citep{Gelman:1992,Gelman:2013,Vehtari:2021:R-hat}, which is widely used to determine convergence
of Markov chain Monte Carlo algorithms, can be combined with a Monte Carlo standard error (MCSE) \citep{Geyer:1992,Vehtari:2021:R-hat} cutoff to construct a convergence criteria. %
We improve upon our previous proposal by:
\bealpha
\item adaptively finding the size of the convergence window, which may need to be large for challenging or high-dimensional distributions over the model parameters, and
\item developing a new rigorous MCSE cutoff condition that guarantees the symmetrized KL divergence between $\approxdensity_{\rho \learningrate*}$ 
and the estimate of $\approxdensity_{\rho \learningrate*}$ obtained via iterate averaging will be small. 
\eenum
In step 3, we leverage recent results that characterize the bias of the stationary distribution of SGD with a fixed learning rate \citep{Dieuleveut:2020:SGD-MC}
to estimate $\skl{\approxdensity_{*}}{\approxdensity_{\learningrate*}}$ and $\skl{\approxdensity_{*}}{\approxdensity_{\rho \learningrate*}}$ without 
access to $\approxdensity_{*}$. 
In step 4, these estimates enable the use of a termination criterion that compares (i) the predicted relative decrease in the KL divergence if the smaller learning $\rho \learningrate$ were used and
(ii) the predicted computation required to converge with the learning rate $\rho \learningrate$.
By trading off between (i) and (ii), the criterion enables the user to balance desired accuracy against computational cost. 
\Cref{fig:example-result} provides an example of the faster convergence, higher accuracy, and greater reliability achievable using RABVI compared to alternative optimization algorithms
and demonstrates how the user can trade off accuracy and computation by adjusting the accuracy threshold $\xi$.

In summary, by drawing on recent developments in theory and methods for fixed--learning-rate stochastic optimization, 
tools from MCMC methodology and results from functional analysis, RABVI delivers a number of benefits: 
\bitems
\item it relies on rigorously justified automation techniques, including automatic learning rate adaptation; 
\item it has an interpretable, user-adjustable accuracy parameter along with a small number of additional intuitive tuning parameters;
\item it detects inaccurate estimates of the optimal variational approximation; and
\item it can flexibly incorporate additional or future methodological improvements related to variational inference and stochastic optimization. 
\eitems
We demonstrate through synthetic comparisons and real-world model and data examples that %
RABVI provides high-quality black-box approximate inferences.
We make RABVI available as part of the open source Python package VIABEL.\footnote{\url{https://github.com/jhuggins/viabel}} %

\section{Preliminaries and Background} 
\label{sec:preliminaries}

\subsection{Bayesian inference}  \label{sec:bayes}
Let $\param \in \reals^{d}$ denote a parameter vector of interest, and let $x$
denote observed data. A Bayesian model consists of a prior density
$\priordensity(\dee\param)$ and a likelihood $\lik(x; \param)$.  Together, the
prior and likelihood define a joint distribution over the data and parameters.
The Bayesian posterior distribution $\postdensity$ is the conditional distribution of $\param$
given fixed data $x$, with $x$ suppressed in the notation since it is always fixed throughout
this work.
To write this conditional, we define the unnormalized posterior density
$\postdensity^{u}(\param) \defined \lik(x;\param)\priordensity(\dee\param)$ and the
marginal likelihood, or evidence, $\marginallik \defined \int \postdensity^{u}(\dee\param)$.
Then the posterior is
$
\postdensity \defined \postdensity^{u}/\marginallik. %
$
Typically, practitioners report \emph{posterior summaries}, such as point estimates and uncertainties, rather than the full posterior. 
For $\paramrv \dist \postdensity$, typical summaries include the mean $\mean{\postdensity} \defined \EE(\paramrv)$, 
the covariance $\Sigma_{\postdensity} \defined  \EE\{(\paramrv - \meanvec{\postdensity})(\paramrv - \meanvec{\postdensity})^{\top}\}$,
and $[a,b]$ interval probability $I_{\postdensity,i,a,b} \defined \Pr\left(\paramrv_i \in [a, b]\right) = \EE\{\ind(\paramrv_i \in [a, b])\}$, 
where $\ind(C)$ is equal to one when $C$ is true and zero otherwise. 

\subsection{Variational inference}
In most Bayesian models, it is infeasible to efficiently compute quantities of interest such
as posterior means, variances, and quantiles. 
Therefore, one must use an approximate inference method
that produces an approximation $\approxdensity$ to the posterior $\postdensity$.  
The summaries of $\approxdensity$ may in turn be used as
approximations to the summaries of $\postdensity$. 
One approach, \emph{variational inference}, is widely used in machine learning.
Variational inference aims to minimize some
measure of discrepancy $\discrepancy{\postdensity}{\cdot}$ over a tractable family
$\varfamily = \theset{ \approxdensity_{\varparam} : \varparam \in \reals^{\varparamdim}}$ of approximating
distributions~\citep{Wainwright:2008,Blei:2017}:
\[
\approxdensity_{\optvarparam} = \argmin_{\approxdensity_{\varparam} \in \varfamily} \discrepancy{\postdensity}{\approxdensity_{\varparam}}. \label{eq:varopt}
\]
The variational family $\varfamily$ is chosen to be tractable in the sense
that, for any $\approxdensity \in \varfamily$, we are able to efficiently calculate
relevant summaries either analytically or using independent and identically
distributed samples from $\approxdensity$.  

In variational inference, the standard choice for the discrepancy $\discrepancy{\postdensity}{\cdot}$ is the
\emph{Kullback--Leibler (KL) divergence}
$
\kl{\approxdensity}{\postdensity} \defined \int \log \left(\textder{\approxdensity}{\postdensity}\right) \dee\approxdensity
$
\citep{Bishop:2006}.
Note that the KL divergence is asymmetric in its arguments. The direction
$\discrepancy{\postdensity}{\approxdensity} = \kl{\approxdensity}{\postdensity}$ is most typical in variational inference, largely out of
convenience; the unknown marginal likelihood $\marginallik$ appears in an additive constant that does not influence
the optimization and computing gradients require estimating expectations
only with respect to $\approxdensity \in \varfamily$, which is chosen to be tractable. %
Minimizing $\kl{\approxdensity}{\postdensity}$ is equivalent to maximizing 
the \emph{evidence lower bound} \citep[ELBO;][]{Bishop:2006}: %
\[
\elbo{\approxdensity} &\defined  \int \log\left(\der{\postdensity^{u}}{\approxdensity} \right)\dee\approxdensity. %
\]
While numerous other divergences have been used in the literature \citep[e.g.,][]{HernandezLobato:2016,Li:2016,Bui:2017a,Dieng:2017,Wang:2018:f-divergence,Wan:2020vz}, we focus
on $\kl{\approxdensity}{\postdensity}$ since it is the most common choice; the default or only choice 
in widely used frameworks such as Stan, Pyro, and PyMC3; and easiest to estimate when using simple Monte Carlo sampling to approximate the gradient \citep{Dhaka:2021}.

\subsection{Black-box variational inference}
Black-box variational inference (BBVI) methods have greatly extended the applicability of variational inference %
by removing the need for model-specific derivations \citep{Cornebise:2008,Ranganath:2014,Kucukelbir:2015,Titsias:2014,Mohamed:2020}
and enabling the use of more flexible approximation families \citep{Kingma:2014,Salimans:2015,Papamakarios:2021}.
This flexibility is a result of using simple Monte Carlo (and automatic differentiation) to approximate the (gradients of the) 
expectations that define common choices of the discrepancy objective \citep{Papamakarios:2021,Mohamed:2020}. %
To estimate the optimal variational parameter $\optvarparam$, BBVI methods commonly use stochastic optimization schemes
which at iteration $\iternum$ are of the form 
\[
\sgditerate{\iternum+1} \gets \sgditerate{\iternum} - \learningrateiterate{\iternum} \descentdir{\iternum}, \label{eq:stochastic-opt}
\]
where $\descentdir{\iternum} \in \reals^{\varparamdim}$ is the \emph{descent direction} and $\learningrateiterate{\iternum} > 0$ is the \emph{learning rate} 
(also called the \emph{step size}).
Standard stochastic gradient descent corresponds to taking $\descentdir{\iternum} = \objgradest{\iternum}$, a (usually unbiased) 
stochastic estimate of the gradient $\objgrad(\sgditerate{\iternum}) \defined \grad_{\sgditerate{\iternum}}D_{\pi}(q_{\sgditerate{\iternum}})$.
We are particularly interested in adaptive stochastic optimization methods \citep[e.g.,][]{Duchi:2011:adagrad,Hinton:2012,Kingma:2015:adam} 
that use a smoothed and/or rescaled version of $\objgradest{\iternum}$ as the descent direction.
For example, RMSProp \citep{Hinton:2012} tracks an exponential moving average of the squared gradient,
$\gradsquareiterate{\iternum+1} = \decay \gradsquareiterate{\iternum} + (1-\decay) \objgradest{\iternum} \odot  \objgradest{\iternum}$,
which is used to rescale the current stochastic gradient: $\descentdir{\iternum} = \objgradest{\iternum+1}/\sqrt{\gradsquareiterate{\iternum}}$. 
Or, Adam \citep{Kingma:2015:adam} tracks an exponential moving average of the gradient  $\graditerate{\iternum+1} = \decayrate \graditerate{\iternum} + (1-\decayrate) \objgradest{\iternum}$ 
as well as the squared gradient $\gradsquareiterate{\iternum+1}$ %
and uses both to rescale the current stochastic gradient: $\descentdir{\iternum} = \graditerate{\iternum}/\sqrt{\gradsquareiterate{\iternum}}$.
The benefits of adaptive algorithms include that they tend to be more stable and are scale invariant, 
so the learning rate can be set in a problem-independent manner.

\subsection{Fixed--learning-rate stochastic optimization}
If the learning rate is fixed so that $\learningrateiterate{\iternum} = \learningrate$, then we can view 
the iterates $\sgditerate{1},  \sgditerate{2}, \dots$ produced by \cref{eq:stochastic-opt} as a homogenous Markov chain, which under certain conditions 
will have a stationary distribution $\stationarydist{\learningrate}$ \citep{Dieuleveut:2020:SGD-MC,Gitman:2019, Pflug:1990,Chee:2018,Yaida:2019}.
Stochastic optimization with a fixed learning rate exhibits two distinct phases: a transient (a.k.a.\ warm-up) phase  
during which iterates make rapid progress toward the optimum, 
followed by a stationary (a.k.a.\ mixing) phase during which iterates oscillate around the mean of the stationary distribution,
$\meansgditerate{\learningrate} \defined \int \varparam\,\stationarydist{\learningrate}(\dee\varparam)$ \citep{Gelman:1992}.%

The mean $\meansgditerate{\learningrate}$ is a natural target because even if the variance of the individual iterates $\sgditerate\iternum$ 
means they are far from $\optvarparam$, $\meansgditerate{\learningrate}$ can be a much more accurate approximation to $\optvarparam$.
For example, the following result quantifies the bias of standard fixed--learning-rate SGD
(\citet{Gitman:2019} provide similar results for momentum-based SGD algorithms):
\begin{theorem}[{\citet[Theorem 4]{Dieuleveut:2020:SGD-MC}}] \label{thm:SGD-bias}
Under regularity conditions on the objective function and the unbiased stochastic gradients,
there exist constant vectors $A, B \in \reals^{\varparamdim}$ such that\footnote{As stated in \citet{Dieuleveut:2020:SGD-MC}, the $B\learningrate^{2}$ term
is written as $O(\learningrate^{2})$. However, the fact that this term is of the form $B\learningrate^{2} + o(\learningrate^{2})$ can be 
extracted from the proof.}
\[
\meansgditerate{\learningrate} -  \optvarparam = A\learningrate + B\learningrate^{2} + o(\learningrate^{2}) \label{eq:SGD-bias}
\]
and a matrix $A' \in \reals^{\varparamdim \times \varparamdim}$ such that 
\[
 \int (\varparam - \optvarparam)(\varparam - \optvarparam)^{\top}\,\stationarydist{\learningrate}(\dee\varparam) = A'\learningrate  + O(\learningrate^{2}). 
\]
\end{theorem}
\begin{remark}
The regularity conditions required by \cref{thm:SGD-bias} are mostly mild.
For example, the stochastic gradients must be unbiased and
have finite variance that does not grow too quickly away from the optimum. 
However, it does require the stronger assumptions that the objective function is smooth and strongly convex.
While these conditions do not hold globally, we do not view it be a significant problem in practice because near the
optimum we expect the objective function to be locally smooth and strongly convex.
\end{remark}

\cref{thm:SGD-bias} shows that, at stationarity, a single iterate will satisfy $\sgditerate{k} -  \optvarparam = O(\gamma^{1/2})$ (with high probability) while 
its expectation will satisfy $\meansgditerate{\learningrate} -  \optvarparam = O(\gamma)$.
Therefore, when the learning rate is small, $\meansgditerate{\learningrate}$ is a substantially better estimator for $\optvarparam$ than $\sgditerate{k}$. 
In practice the \emph{iterate average} (i.e., sample mean)
\[
\iterateaverage{\learningrate} \defined \frac{1}{\iternumaverage}\sum_{\iternum=0}^{\iternumaverage-1}\sgditerate{\iternumconverged+\iternum} 
\label{eq:iterate-average}
\]
provides an estimate of $\meansgditerate{\learningrate}$, where $\iternumconverged$ is the iteration at which the optimization has reached the stationary phase
and $\iternumaverage$ is the number of iterations used to compute the average. 
Using $\iterateaverage{\learningrate}$ as an estimate of $\optvarparam$ is known as Polyak--Ruppert averaging \citep{Polyak:1992,Ruppert:1988,Bach:2013}. 

When using iterate averaging, it is crucial to ensure the iterate average accurately approximates $\meansgditerate{\learningrate}$. 
Considering a stationary Markov chain, we can compute the Monte Carlo estimate of the mean of the Markov Chain at stationarity.  
Then the notion of \emph{effective sample size} (ESS) aids in quantifying the accuracy of this Monte Carlo estimate. %
Further, the effective sample size can also be used to define the  \emph{Monte Carlo standard error} (MCSE) when the Markov chain satisfies
a central limit theorem. We can efficiently estimate the ESS and also approximate the MCSE (see \cref{sec:ESS-MCSE} for details). We denote the estimates of ESS and MCSE for the $i$th component using iterates $\sgditerate{\iternumconverged:\iternum}_{i}$  as $\widehat{\mathrm{ESS}}(\sgditerate{\iternumconverged:\iternum}_{i})$  and  $\widehat{\mathrm{MCSE}}(\sgditerate{\iternumconverged:\iternum}_{i})$ respectively. 
\citet{Dhaka:2020} use the conditions $\varparamdim^{-1}\sum_{i=1}^{\varparamdim}\widehat{\mathrm{MCSE}}(\sgditerate{\iternumconverged:\iternum}_{i}) < 0.02$
and $\widehat{\mathrm{ESS}}(\sgditerate{\iternumconverged:\iternum}_{i}) > 20$ to determine when to stop iterate averaging. 
However, no rigorous justification is given for the MCSE threshold. 

\subsection{Automatically scheduling learning rate decreases}
A benefit of using a fixed learning rate is that it can be adaptively and automatically decreased 
once the iterates reach stationarity.
For example, if the initial learning rate is  $\learningrate_{0}$, after reaching stationarity 
the learning rate can decrease to $\learningrate_{1} \defined \rho\learningrate_{0}$, where $\rho \in (0,1)$
is a user-specified adaptation factor.
The process can be repeated: when stationarity is reached at learning rate $\learningrate_{\epochnum}$,
the learning rate can decrease to $\learningrate_{\epochnum+1} \defined \rho\learningrate_{\epochnum}$.
In this way the learning rate is not decreased too early (when the iterates are still making fast progress toward the optimum) or too late
(when the accuracy of the iterates is no longer improving). 
Compare this adaptive approach to the canonical one of setting a schedule such as $\learningrateiterate{\iternum} = \triangle /(\bigcirc + \iternum)^{\square}$,
which requires the choice of three tuning parameters.
These tuning parameters can have a dramatic effect on the speed of convergence, particularly when $\sgditerate{0}$ is far from $\optvarparam$.

The question of how to determine when the stationary phase has been reached has a long history with recent renewed attention
\citep{Pesme:2020,Zhang:2020:SALSA,Chee:2020,Lang:2019,Yaida:2019,Pflug:1990,Chee:2018,Dhaka:2020}.
One line of work \citep{Yaida:2019,Lang:2019,Zhang:2020:SALSA} is based on finding an invariant function that has expectation 
zero under the stationary distribution of the iterates, then using a test for whether the empirical mean of the invariant function is sufficiently close to zero.
An alternative approach developed in \citet{Dhaka:2020} makes use of the potential scale reduction factor $\widehat{R}$,
perhaps the most widely used MCMC diagnostic for detecting stationarity \citep{Gelman:1992,Gelman:2013,Vehtari:2021:R-hat}. 
The standard approach to computing $\widehat{R}$ is to use multiple Markov chains.
If we have $J \ge 2$ chains and $K \gg 1$ iterates in each chain, %
then $\widehat{R} \defined  (\hat{\mathbb{V}}/\hat{\mathbb{W}})^{1/2}$,
where $\hat{\mathbb{V}}$ and  $\hat{\mathbb{W}}$ are estimates of, respectively, 
the between-chain and within-chain variances.
In the split-$\widehat{R}$ version, each chain is split into two before carrying out the computation above, which helps with detecting non-stationarity \citep{Gelman:2013,Vehtari:2021:R-hat} 
and allows for use even when $J = 1$. 
Let $\widehat{R}_{i}(W)$ denote the split-$\widehat{R}$ value computed from $\sgditerate{\iternum-W+1}_{i},\dots,\sgditerate{\iternum}_{i}$,
the $i$th dimension of the last $W$ iterates.
\citet{Dhaka:2020} uses the stationarity condition $\max_{i}\widehat{R}_{i}(100) < 1.1$,

\section{Methodological Criteria} \label{sec:problem-setup}
We now summarize our criteria when designing a robust and automatic optimization framework for BBVI.

\paragraph*{Robustness.}
 A robust method should not be too sensitive to the choice of tuning parameters. 
It should also work well on a wide range of ``typical'' problems. 
To achieve this we design an adaptive methods for setting parameters (such as the window size for detecting convergence) that 
are problem-dependent. 
\paragraph{Automation.}
 An automatic method should require minimal input from the user.
Any inputs that are required should be clearly necessary (e.g., the model
and the data) or be intuitive to an applied user who is not an expert 
in variational inference and optimization.
Therefore, we require the parameters of any adaptation scheme to either be intuitive 
or not require adjustment by the user. 
Examples of intuitive parameters include the maximum number of iterations,
maximum runtime, and, when defined appropriately, accuracy. \\

We ensure these criteria are satisfied when designing the two core components of a BBVI stochastic optimization framework
with automated learning rate scheduling:
\benum
\item A \textbf{termination rule} for stopping the optimization once the final approximation is close to
	the optimal approximation (\cref{sec:termination-rule}).
\item A \textbf{learning rate scheduler}, which must detect stationarity and determine how many iterates to average 
	before decreasing the learning rate (\cref{sec:fixed-learning-rate}). 
\eenum

\section{Termination Rule} \label{sec:termination-rule}

The development of our termination rule will proceed in three steps.
First, we will select an appropriate discrepancy measure between distributions.
Next, we will design an idealized termination rule based on this discrepancy measure.
Finally, we will develop an implementable version of the idealized termination rule
that satisfies the criteria from \cref{sec:problem-setup}.

\subsection{Choice of Accuracy Measure} \label{sec:discrepancy-measures}

To develop a termination rule, we must specify a measure of how close a variational approximation returned by the optimization algorithm, $\hat\approxdensity_{*}$,
is to the optimal variational approximation $\approxdensity_{*} \defined \approxdensity_{\optvarparam}$.
But the answer to this question depends upon choosing an appropriate measure of the discrepancy between
$\approxdensity_{*}$ and the posterior $\postdensity$. 
Based on the discussion in \cref{sec:bayes}, the goal should be for quantities such as $\mean{\approxdensity_{*}}$, $\Sigma_{\approxdensity_{*}}$,
and $I_{\approxdensity_{*},i,a,b}$ 
to be close to, respectively, $\mean{\postdensity}$, $\Sigma_{\postdensity}$, and $I_{\postdensity,i,a,b}$. 
The interval probabilities are already on an interpretable scale, so ensuring that 
$|I_{\approxdensity,i,a,b} - I_{\postdensity,i,a,b}|$ is much less than 1 is an intuitive notion of accuracy. 
Since $\Sigma_{\postdensity}^{1/2}$ establishes the relevant scale of the problem for means and standard deviations, so
it is appropriate to ensure that $\staticnorm{\Sigma_{\postdensity}^{-1/2}(\mean{\postdensity} - \mean{\approxdensity_{*}})}_{2}$
and $\staticnorm{\Sigma_{\postdensity}^{-1}(\Sigma_{\postdensity} - \Sigma_{\approxdensity_{*}})}_{2}
= \staticnorm{I - \Sigma_{\postdensity}^{-1}\Sigma_{\approxdensity_{*}}}_{2}$ 
are much less than 1.

While we want to choose a discrepancy measure that 
guarantees the accuracy of mean, covariance, and interval probabilities,
ideally it would also guarantee other plausible expectations of interest (e.g., predictive densities) are accurately approximated. 
The \emph{Wasserstein distance} provides one convenient metric for accomplishing this goal,
and is widely used in the analysis of MCMC algorithms and in large-scale data 
asymptotics~\citep[e.g.,][]{Joulin:2010,Madras:2010,Rudolf:2015,Durmus:2019,Durmus:2019b,Vollmer:2016,Eberle:2018}. 
For $p \ge 1$ and a positive-definite matrix $\Sigma \in \reals^{\paramdim \times \paramdim}$,
we define the \emph{$(p, \Sigma)$-Wasserstein distance} between distributions $\eta$ and $\zeta$ as
\[
\pwassSimple{p,\Sigma}{\eta}{\zeta} 
\defined \inf_{\coupling} \left\{ \int \staticnorm{\Sigma^{-1/2}(\param - \param')}_{2}^{p} \coupling(\dee \param, \dee \param') \right\}^{1/p},
\]
where the infimum is over the set of \emph{couplings} between $\eta$ and $\zeta$; that is,
Borel measures $\coupling$ on $\reals^{d} \times \reals^{d}$ such that $\eta = \coupling(\cdot, \reals^{d})$ and
$\zeta = \coupling(\reals^{d}, \cdot)$~\citep[Defs.~6.1 \& 1.1]{Villani:2009}. 
Small $(p, \Sigma)$-Wasserstein distance implies many functionals of the two distributions are close
relative to the scale determined by $\Sigma^{1/2}$.

Specifically, we have the following result, which is an immediate corollary of \citet[Theorem 3.4]{Huggins:2020:VI}. 
\begin{proposition} \label[proposition]{prop:Wasserstein-moment-bounds}
If $\pwassSimple{p,\Sigma}{\eta}{\zeta} \le \veps$ for any $p \ge 1$, then 
\[
\staticnorm{\Sigma^{-1/2}(\mean{\eta} - \mean{\zeta})}_{2} \le \veps 
\]
If $\pwassSimple{2,\Sigma}{\eta}{\zeta} \le \veps$, then, for $\varrho \defined \min\{\twonorm{\Sigma^{-1}\Sigma_{\eta}}^{1/2},\twonorm{\Sigma^{-1}\Sigma_{\zeta}}^{1/2}\}$,
\[
\staticnorm{\Sigma^{-1}(\Sigma_{\eta} - \Sigma_{\zeta})}_{2} < 2\veps(\varrho + \veps). %
\]
\end{proposition}
More generally, small $(p, \Sigma)$-Wasserstein distance for any $p \ge 1$ guarantees
the accuracy of expectations for any function $f$ with small Lipschitz constant
with respect the metric $d_{\Sigma}(\param, \param') \defined \staticnorm{\Sigma^{-1/2}(\param - \param')}_{2}$
-- that is, when $\sup_{\param \ne \param'}|f(\param) - f(\param')|/d_{\Sigma}(\param, \param')$ is small. 

While the Wasserstein distance controls the error in mean and covariance estimates,
it does not provide strong control on interval probability estimates.
The KL divergence, however, does, since for distributions $\eta$ and $\zeta$,
$|I_{\eta,i,a,b} - I_{\zeta,i,a,b}| \le \sqrt{\kl{\eta}{\zeta}/2}$ for all $a < b$ and $i$
(see \cref{sec:Pinskers-inequality}). 
As we show next, in many scenarios we can bound the Wasserstein distance by the KL divergence %
and therefore enjoy the benefits of both. 
Our result is based on the following definition, which makes the notion of the scale of a distribution precise:
\begin{definition} \label[definition]{def:exponentially-controlled}
For $p \ge 1$ and a positive-definite matrix $\Sigma \in \reals^{\paramdim \times \paramdim}$,
the distribution $\eta$ is said to be \emph{$(p,\Sigma)$-exponentially controlled} if 
\[
\inf_{\param'}  \log \int e^{\staticnorm{\Sigma^{-1/2}(\param - \param')}_{2}^p}\eta(\dee \param) \le \paramdim/2.  \label{eq:exp-controlled}
\]
\end{definition}
Specially, $\Sigma^{1/2}$ establishes the appropriate scale for uncertainty with respect to $\eta$.
For example, if $\eta = \distNorm(m, V)$, then
it is a straightforward exercise to confirm that $\eta$ is $(2, 1.78^{2}V)$-exponentially controlled. 

The following result establishes the relevant link between the KL divergence and the Wasserstein distance
via \cref{def:exponentially-controlled}.
\begin{proposition} \label[proposition]{prop:Wasserstein-control}
If $\eta$ is $(p, \Sigma)$-exponentially controlled,
then for all $\zeta$ absolutely continuous with respect to $\eta$, 
\[
\pwassSimple{p,\Sigma}{\zeta}{\eta} \le (3+d)\mcK_{p}(\zeta \mid \eta),
\]
where $\mcK_{p}(\zeta \mid \eta) \defined \kl{\zeta}{\eta}^{\frac{1}{p}} + \{\kl{\zeta}{\eta}/2\}^{\frac{1}{2p}}$. 
\end{proposition}
\begin{proof}
The result follows from \citet[Corollary 2.3]{Bolley:2005} after the change-of-variable $\param \mapsto \Sigma^{-1/2}\param$
and using the fact that the KL divergence is invariant under diffeomorphisms, then applying \cref{eq:exp-controlled}.
\end{proof}
If $\zeta$ and $\eta$ could operate over different scales, 
then we can use the \emph{symmetrized KL divergence} $\skl{\zeta}{\eta} \defined \kl{\zeta}{\eta} + \kl{\eta}{\zeta}$.
Indeed, it follows from \cref{prop:Wasserstein-control} that if $\skl{\zeta}{\eta}$ is small, then
the $(p,\Sigma)$-Wasserstein distance is small whenever \emph{either} $\eta$ or $\zeta$ is $(p, \Sigma)$-exponentially controlled.

\subsection{An Idealized Termination Rule} \label{sec:idealized-termination-rule}

Based on our developments in \cref{sec:discrepancy-measures}, we will define our termination rule in terms of the symmetrized KL divergence.
Recall that $\approxdensity_{*}$ denotes the optimal variational approximation to $\postdensity$ and $\hat\approxdensity_{*}$
denotes an estimate of  $\approxdensity_{*}$. 
Since the total variation and $(1, \Sigma)$-Wasserstein distances are controlled by the square root of the KL divergence,
we focused on the square root of the symmetrized KL divergence.
For the current learning rate $\learningrate > 0$, let $\approxdensity_{\learningrate*} \defined \approxdensity_{\meansgditerate{\learningrate}}$ denote the
target $\learningrate$--learning-rate variational approximation. 
The termination rule we propose is based on the trade-off between the improved accuracy of the approximation if the learning rate were reduced to $\rho\learningrate$
and the time required to reach that improved accuracy. 
To quantify the improved accuracy, we introduce a user-chosen target accuracy target $\xi$ for $\skl{\approxdensity_{*}}{\hat\approxdensity_{*}}^{1/2}$.
If the user expects $\kl{\postdensity}{\approxdensity_{*}}$ to be large, then setting $\xi$ to a moderate value such as 1 or 10 could give acceptable performance. 
If the user expects $\kl{\postdensity}{\approxdensity_{*}}$ to be small, then setting $\xi$ to a value such as 0.1 or 0.01 might be more appropriate.
Using $\xi$, define the \emph{relative SKL improvement}
\[
\mathrm{RSKL} \defined 
\frac{\skl{\approxdensity_{*}}{\approxdensity_{\rho \learningrate*}}^{1/2} + \xi}{\skl{\approxdensity_{*}}{\approxdensity_{\learningrate*}}^{1/2}},
\]
where the first term measures the relative improvement of the approximation if the learning rate were reduced and 
the second term measures the current accuracy relative to the desired accuracy. 
To quantify the time to obtain the relative accuracy improvement, we use the number of iterations to reach convergence for the fixed learning rate $\rho\gamma$.
Letting $\iterate_{\learningrate*}$ denote the number of iterations required to reach the target $\learningrate$--learning-rate variational approximation, 
we define the \emph{relative iteration increase}
\[
\mathrm{RI} \defined \frac{\iterate_{\rho \learningrate*}}{\iterate_{\learningrate*} + \iterate_0},
\]
where $\iterate_0$ denotes the number of iterations the user would consider ``small''. 
Combining $\mathrm{RSKL}$ and $\mathrm{RI}$, we obtain the \emph{inefficiency index} $\ineff = \mathrm{RSKL} \times \mathrm{RI}$,
the relative improvement in accuracy times the relative increase in runtime. 
Thus, we can interpret $\ineff$ as quantifying how much greater the increase in runtime cost (above a baseline of $\iterate_{0}$ iterations) 
will be compared to the reduction in error (down to a target error of $\xi$).
For example, $\ineff = 2$ means the increase in runtime cost is twice as large as the reduction in error.
Our \emph{idealized SKL inefficiency termination rule} triggers when $\ineff > \tau$, where
$\tau$ is a user-specified inefficiency threshold that allows the user to trade off accuracy with computation,
but only up to the point where $\skl{\approxdensity_{*}}{\approxdensity_{\learningrate*}}^{1/2} \approx \xi$.

\subsection{An Implementable Termination Rule} \label{sec:SKLD-termination-rule}

The idealized SKL inefficiency termination rule cannot be directly implemented since
$\approxdensity_{*}$ is unknown -- and if it were known, it would be unnecessary to run the optimization algorithm. 
However, we will show that it is possible to obtain a good estimate of the symmetrized KL divergence 
between the approximation obtained with a given learning rate $\learningrate'$ and the optimal approximation without access to $\approxdensity_{*}$. 
Recall that  $\approxdensity_{\learningrate*}$ denotes the
target $\learningrate$--learning-rate variational approximation.
With a slight abuse of notation, we let the optimal zero--learning-rate approximation refer to the optimal approximation: $\approxdensity_{0*} \defined \lim_{\learningrate \to 0} \approxdensity_{\learningrate *} =  \approxdensity_{*}$. 
Our approach is motivated by \cref{thm:SGD-bias} and in particular the form of the bias $\meansgditerate{\learningrate} - \optvarparam$
in \cref{eq:SGD-bias}.
We first consider the still-common setting when $\varfamily$ is the family of mean-field Gaussian distributions,
where the parameter $\varparam = (\tau, \psi) \in \reals^{2\paramdim}$ corresponds to the distribution $\approxdensity_{\varparam} = \distNorm(\tau, \diag e^{2\psi})$. 

\begin{proposition} \label[proposition]{prop:SKL-small-learning-rate}
Let $\varfamily$ be the family of mean-field Gaussian distributions.
If \cref{eq:SGD-bias} holds and $\learningrate' = O(\learningrate)$, then there is a constant $C > 0$ depending 
only on $A$ and $\optvarparam$ such that 
\[
\skl{\approxdensity_{\learningrate*}}{\approxdensity_{\learningrate'*}}
= C (\learningrate - \learningrate')^{2} + o(\learningrate^{2}). 
\]
\end{proposition}
See \cref{sec:SKL-small-learning-rate-proof} for the proof. 
Assuming that the current learning rate is $\learningrate$, then the previous learning rate was $\learningrate/\rho$. 
Let $\delta_{\learningrate} \defined \skl{\approxdensity_{\learningrate*}}{\approxdensity_{\learningrate/\rho*}}$
denote the symmetrized KL divergence between the optimal variational approximations obtained at each of these learning rates. 
In principle we can use \cref{prop:SKL-small-learning-rate} to estimate $C$ by 
\[
\hC = \delta_{\learningrate} \rho^{2}/\{\learningrate^{2}(1-\rho)^{2}\},
\]
and then estimate that 
\[
\skl{\approxdensity_{\learningrate*}}{\approxdensity_{*}} \approx \hC \learningrate^{2} = \delta_{\learningrate} \rho^{2}/(1-\rho)^{2}.
\]

There are, however, two problems with the tentative approach just outlined.
The first problem is that \cref{thm:SGD-bias} only holds for standard SGD; however, adaptive SGD algorithms 
are widely used in practice. 
Indeed, we observe empirically that $\skl{\approxdensity_{\learningrate*}}{\approxdensity_{\learningrate'*}} = \Theta(|\learningrate - \learningrate'|^{\biaspower/2})$ with $\biaspower \approx 1$ for RMSProp (\cref{fig:kappa_hat_comparisons-rmsprop}) and $\biaspower \in (1, 1.6)$ (with a point estimate at 1.2) for Adam \cref{fig:kappa_hat_comparisons-adam}).
Hence RMSProp and Adam both appear to have larger errors than SGD when the step size is small. 
To get the accuracy of SGD but also adaptivity, we modify the adaptive gradient methods to behave asymptotically (in the number of iterations) like SGD.
In the cases of RMSProp and Adam, we propose \emph{averaged RMSProp} (avgRMSProp) and \emph{averaged Adam} (avgAdam), which use the 
squared gradient update 
\[
\gradsquareiterate{\iternum+1} = \decay_{\iternum} \gradsquareiterate{\iternum} + (1-\decay_{\iternum}) \objgradest{\iternum} \odot  \objgradest{\iternum},
\]
for $\decay_{\iternum} = 1 - 1/\iternum$.
Hence, $\gradsquareiterate{\iternum+1} = (\iternum+1)^{-1}\sum_{\iternum'=0}^{\iternum}\objgradest{\iternum} \odot  \objgradest{\iternum}$
is the averaged squared gradient over all iterations \citep[\S4]{Mukkamala:2017}.
As long as the SGD Markov chain is ergodic and  $\EE[\norm{\objgradest{\iternum}}^2_2] < \infty$ at stationarity,
$\gradsquareiterate{\iternum}$ converges almost surely to a constant
and hence the SGD bias analysis also applies to avgRMSProp and avgAdam. 

The second problem is that \cref{prop:SKL-small-learning-rate} only holds for the mean-field Gaussian variational family.
However, other variational families such as normalizing flows are of substantial practical interest. 
Therefore, we consider the weaker assumption that either \cref{eq:SGD-bias} holds \emph{or} 
there exist constant vectors $\Lambda, A \in \reals^{\varparamdim}$ and a constant $\biaspower \in [1/2, 1)$ such that 
\[
\meansgditerate{\learningrate} = \optvarparam + \Lambda\learningrate^{\biaspower} +  A\learningrate  + o(\learningrate^{2\biaspower}).  \label{eq:SGD-bias-p}
\]
Adding the latter assumption, we have the following generalization of \cref{prop:SKL-small-learning-rate},
which holds for any sufficiently regular variational family:

\begin{proposition} \label[proposition]{prop:SKL-small-learning-rate-general}
Let $\varfamily$ be the variational approximation family.
If (i) \cref{eq:SGD-bias-p} holds for some $\biaspower \in [1/2, 1)$ or \cref{eq:SGD-bias} holds (in which case let $\kappa =  1$),
(ii) $\learningrate' = O(\learningrate)$, and
(iii) for all $\param \in \reals^{\paramdim}$, $\log \approxdensity_{\varparam}(\param)$ is three-times continuously differentiable with respect to $\varparam$,
then for some $C \ge 0$, 
\[
\skl{\approxdensity_{\learningrate*}}{\approxdensity_{\learningrate'*}}
= C \{\learningrate^{\biaspower} - (\learningrate')^{\biaspower}\}^{2} + o(\learningrate^{2\biaspower}).
\]
Moreover,  $C$ depends on only $\optvarparam$ and either $\Lambda$  (if $\biaspower \in [1/2, 1)$) or $A$ (if $\biaspower = 1$).
\end{proposition}
See \cref{sec:SKL-small-learning-rate-general-proof} for the proof. 
Using \cref{prop:SKL-small-learning-rate-general} and omitting $o(\learningrate^{2\biaspower})$ terms, we have
$
\delta_{\learningrate} 
\approx C \learningrate^{2\biaspower} (1/\rho^{\biaspower} - 1)^{2}.
$
To improve the reliability of the estimates based on \cref{prop:SKL-small-learning-rate-general}, 
we propose to use the symmetrized KL estimates between the variational approximations obtained at successive fixed learning rates. 
Let $\learningrate_{0}$ denote the initial learning rate, so that after $\epochnum$ learning rate decreases, the learning rate 
is $\learningrate_{\epochnum} \defined \learningrate_{0} \rho^{\epochnum}$. 
Let $\delta_{\epochnum} \defined \skl{\approxdensity_{\learningrate_{\epochnum}*}}{\approxdensity_{\learningrate_{\epochnum-1}*}}$
and assume the current learning rate is $\gamma_{\totalepochs}$. 

\begin{figure}[tbp]
\begin{center}
\begin{subfigure}[t]{.49\textwidth}
\includegraphics[width=\textwidth]{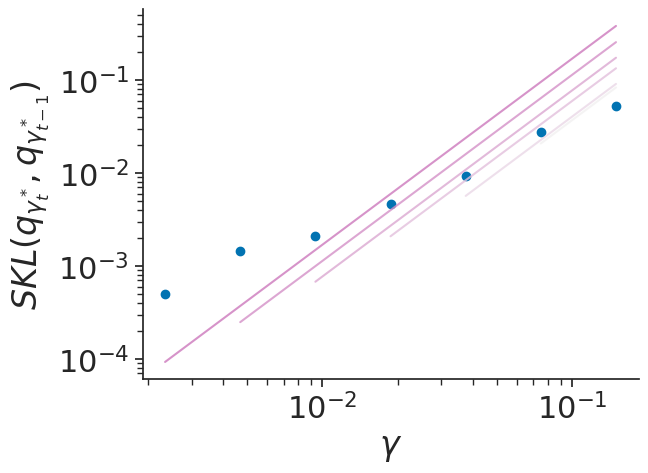}
\caption{} %
\end{subfigure}  
\begin{subfigure}[t]{.49\textwidth}
\includegraphics[width=\textwidth]{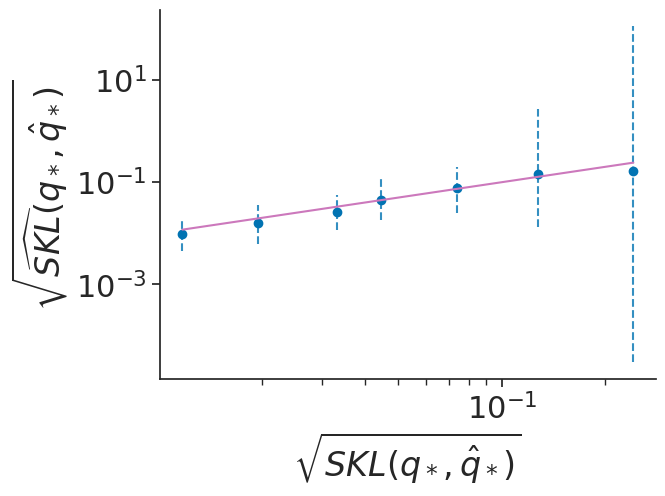}
\caption{} %
\end{subfigure}  
\caption{
Results for estimating the symmetrized KL divergence with avgAdam 
in the case of a Gaussian distribution $\distNorm(0, V)$ with $\paramdim = 100$ and $V_{ij} = j\ind[i=j]$ (diagonal non-identity covariance). 
\textbf{(a)} Learning rate versus symmetrized KL divergence of adjacent iterate averaged estimates of 
the optimal variational distribution. 
The lines indicate the linear regression fits, with setting $\biaspower = 1$. %
\textbf{(b)} Square root of true symmetrized KL divergence versus the estimated value with $95\%$ credible interval.
The uncertainty of the estimates decreases and remains well-calibrated as the learning rate decreases.}
\label{fig:SKL-accuracy}
\end{center}
\end{figure}

Depending on the optimization algorithm, we can estimate $\biaspower$ (or set $\biaspower=1$ if using modified adaptive SGD algorithm with a mean-field Gaussian variation family) and  $C$ using a regression model of the form
\[
 \log \delta_{\epochnum} &= \log C + 2 \log  (1/\rho^{\biaspower} - 1) + 2\biaspower \log \learningrate_{\epochnum} + \eta_{\epochnum}, & \epochnum &= 1,\dots,\totalepochs, \label{eq:SKL-regression-model}
\]
where $\eta_{\epochnum} \dist \distNorm(0, \sigma^{2})$. 
Given the estimate $\hC$ for $C$ and the estimate $\hat\biaspower$ for $\biaspower$ (or $\hat\biaspower = 1$) , %
we obtain the estimated relative SKL, 
\[
\widehat{\mathrm{RSKL}} = \rho^{\hat\biaspower}  +\frac{\xi} {\hC^{1/2} \learningrate_{t}^{\hat\biaspower}}.
\]
Because we use the regression model in \cref{eq:SKL-regression-model} in a low-data setting, we place (weak) priors on $\log C$ and $\sigma$: 
\[
\log C &\dist \distCauchy(0,10), &
\sigma &\dist \distCauchyPlus(0,10),
\]
where $\distCauchyPlus$ is the Cauchy distribution truncated to nonnegative values. 
If we use an adaptive stochastic optimization algorithm then we also place a prior on $\biaspower$:
\[
\biaspower \dist \distUnif(0, 1).
\]
Also, because we expect early SKL estimates to be less informative about $C$ (and $\biaspower$) 
due to the influence of $o(\gamma^{2\biaspower})$ terms, we use a weighted regression
with the likelihood term for $(\delta_{\epochnum}, \learningrate_{\epochnum})$ having 
weight 
\[
w_{\epochnum} = \{1 + (\totalepochs-\epochnum)^{2}/3^{2}\}^{-1/4}. \label{eq:regression-weights}
\] 
The weight formula enables the amplification of the significance of the most recent observations, with 
down-weighting becomes more significant after there are about 3 additional observations. 
On the other hand, the power of $1/4$ ensures a gradual reduction in weight, preventing a steep drop-off in importance.

We use the posterior mean(s) to estimate $\hC$ (and $\hat\biaspower$).
\Cref{fig:SKL-accuracy} validates that, in the case of avgAdam, the log of the learning rate and symmetrized KL divergence
have approximately a linear relationship and that our regression approach to estimating $C$ 
leads to reasonable estimates of $\skl{\approxdensity_{\learningrate*}}{\approxdensity_{*}}$. 
See \cref{fig:SKL-accuracy-more-avgrmsp} for similar results for other target 
distributions with avgAdam.

\begin{figure}[tbp]
\begin{center}
\centering
\includegraphics[width=.49\columnwidth]{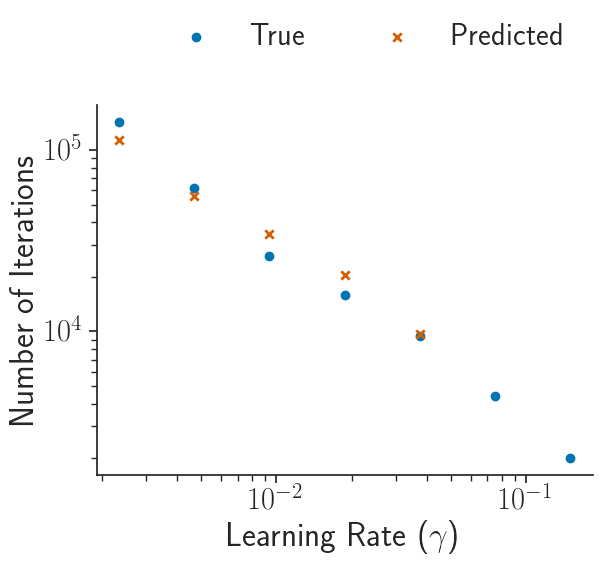}
\caption{
Results for predicting the number of iterations  needed to reach convergence at each learning rate decrease in the case of
Gaussian distribution $\distNorm(0, V)$ with $\paramdim = 100$ and $V_{ij} = j\ind[i=j]$ (diagonal non-identity covariance).
The blue points (orange crosses) represent the true (predicted) number of iterations needed to reach convergence.
}
\label{fig:SKL-convg_iterations}
\end{center}
\end{figure}

\begin{figure}[tbp]
\begin{center}
\begin{subfigure}[t]{.49\textwidth}
\includegraphics[width=\textwidth]{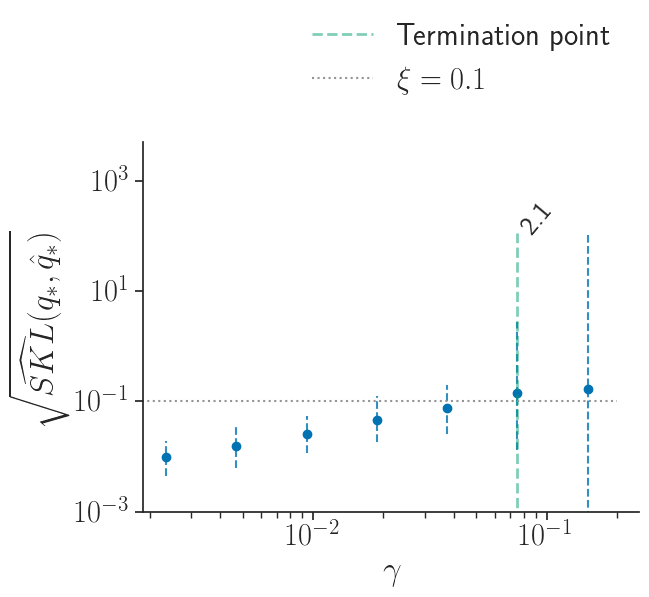}
\caption{}
\end{subfigure}  
\begin{subfigure}[t]{.49\textwidth}
\includegraphics[width=\textwidth]{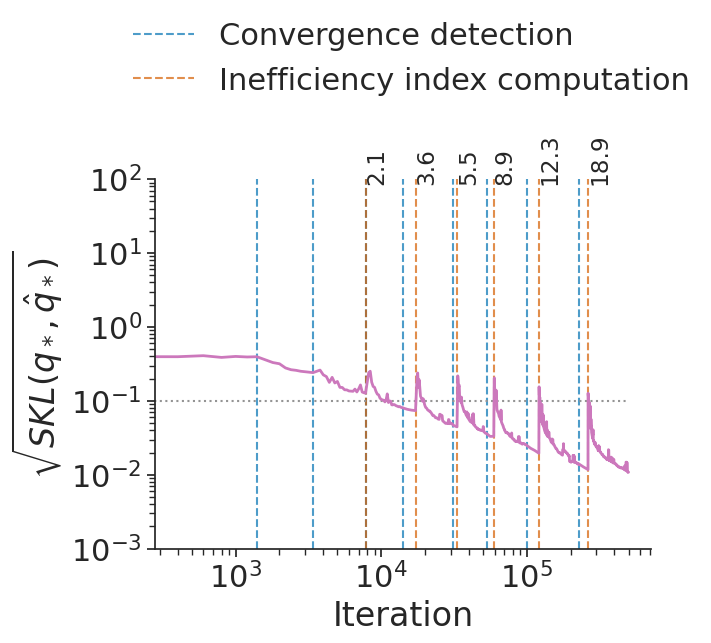}
\caption{}
\end{subfigure}
\caption{
Results for the termination rule trigger point in the case of a Gaussian distribution $\distNorm(0, V)$ with $\paramdim = 100$ and $V_{ij} = j\ind[i=j]$ (diagonal non-identity covariance). 
\textbf{(a)} Learning rate versus square root of estimated symmetrized KL divergence with 95\% credible interval (dashed blue line). 
The green vertical line indicates the termination rule trigger point with the corresponding
$\hat\ineff$ value.
\textbf{(b)} Iterations versus square root of symmetrized KL divergence between iterate average and optimal variational approximation.
The vertical lines indicate the convergence detection points using $\hat{R}$ (blue) and inefficiency index computation ($\hat\ineff$) points (orange) with corresponding values.}
\label{fig:SKL-termination_point}
\end{center}
\end{figure}

To estimate the relative iteration increase $\mathrm{RI}$, we need to estimate
the number of iterations to reach convergence at the next learning rate $\learningrate_{\epochnum+1}$.
It is reasonable to assume that there is exponential growth in the number of iterations to reach convergence 
as the learning rate decreases since stochastic gradient algorithms to converge at a polynomial rate \citep{Bubeck:2015:CO}. 
Recall that $\iterate_{\learningrate_\epochnum}$ is the number of iterations to reach convergence at the current learning rate. 
We fit a weighted least square regression model of the form
\[
 \log \iterate_{\learningrate_\epochnum}  &= \alpha \log \learningrate_\epochnum + \beta + \nu_\epochnum, & \epochnum &= 1,\dots,\totalepochs,
 \label{eq:iterates-vs-epoch-num}
\]
where $\nu_{\epochnum} \dist \distNorm(0, \sigma_{\epochnum}^{2})$.
We then use the coefficient estimates $\hat\alpha$ and $\hat\beta$ to predict the number of iterations required for convergence at the next learning rate
to be $\widehat{\iterate}_{\learningrate_{\epochnum+1}} \defined \learningrate_{\epochnum+1}^{\hat{\alpha}}  e^{\hat{\beta}}$. 
We use the same weights given in \cref{eq:regression-weights} for observations of the regression model due to the non-linear behavior 
of the earlier convergence iterate estimates. 
\Cref{fig:SKL-convg_iterations} demonstrates that linear relationship in \cref{eq:iterates-vs-epoch-num} does in fact hold
and that our weighted least square regression model predicts the number of convergence iterations $\iterate_{\learningrate_\epochnum}$ quite accurately.
The estimated relative iterations is then
$\widehat{\mathrm{RI}} = \widehat{\iterate}_{\learningrate_{\epochnum+1}}/(\iterate_{\learningrate_\epochnum}+\iterate_0)$.

Using the estimates $\widehat{\mathrm{RSKL}}$ and $\widehat{\mathrm{RI}}$ we obtain the termination rule
$\widehat{\ineff} = \widehat{\mathrm{RSKL}} \times \widehat{\mathrm{RI}} > \tau$.
\Cref{fig:SKL-termination_point} shows that when the user chosen target accuracy $\xi=0.1$, the termination rule triggers when the square root of the symmetrized KL divergence 
is approximately equal to $\xi$. 
\Cref{fig:SKL-termination-point-more,fig:SKL-termination-point-posteriordb,fig:SKL-termination-point-posteriordb-more} shows similar results of other Gaussian targets and \texttt{posteriordb} models and datasets (see \cref{subsec:posteriordb-datasets} for details).

\section{Learning Rate Scheduler} \label{sec:fixed-learning-rate}

For a fixed learning rate, computing the iterate average $\iterateaverage{\learningrate}$ defined in \cref{eq:iterate-average} requires
determining the iteration $\iternumconverged$ at which stationarity is reached and the number of iterations $\iternumaverage$ to use for computing the average.
We address each of these in turn. 

\subsection{Detecting convergence to stationarity}

We investigate two approaches to detecting stationarity: the SASA+ algorithm of \citet{Zhang:2020:SALSA} and the $\widehat{R}$-based
criterion from \citet{Dhaka:2020}. 
We make several adjustments to both approaches to reduce the number of tuning parameters and to make the remaining ones more intuitive. 
In our empirical findings, we have observed that the $\widehat{R}$ criterion outperforms the SASA+ criterion. Therefore, we describe the former here
and the latter in \cref{sec:SASA}. 

Let $\widehat{R}(\sgditerate{\iternum-W+1}_{i},\dots,\sgditerate{\iternum}_{i}) $ denotes the split-$\widehat{R}$ of the $i$th component of the last $W$ iterates
and define 
\[
\widehat{R}_{\max}(W) \defined \max_{1 \le i \le \varparamdim}\widehat{R}(\sgditerate{\iternum-W+1}_{i},\dots,\sgditerate{\iternum}_{i}).
\] 
An $\widehat{R}_{\max}(W)$ value close to 1 indicates the last $W$ iterates are close to stationarity.
In MCMC applications having $\widehat{R}_{\max}(W) \le 1.01$ is desirable \citep{Vats:2021:Rhat,Vehtari:2021:R-hat}.
\citet{Dhaka:2020} uses the weaker condition $\widehat{R}_{\max}(W) \le 1.1$ since iterate averaging does not require the same level of 
precision as MCMC. 
\citet{Dhaka:2020} take the window size $W = 100$, but in more challenging and high-dimensional problems a fixed smaller $W$ is insufficient.
Therefore, we instead search over window sizes between a minimum window size $W_{\min}$ and $0.95 \iternum$ to find the one that minimizes $\widehat{R}_{\max}(W)$.
The minimum window size is necessary to ensure the $\widehat{R}$ values are reliable. 
We use the upper bound $0.95 \iternum$ to always allow a small amount of ``warm-up'' without sacrificing more than 5\% efficiency. 
Therefore, we estimate $W_\mathrm{opt} = \argmin_{W_{\min} \le W \le 0.95\iternum}\widehat{R}_{\max}(W)$ using a grid search over 5 equally spaced values ranging from $W_{\min}$ to $0.95 \iternum$
and require $\widehat{R}_{\max}(W_\mathrm{opt}) \le 1.1$ as the stationarity condition. 

\Cref{fig:convergence-detection-comparison} compares our adaptive SASA+ and adaptive $\widehat{R}$ criteria to 
the criterion used in \citet{Dhaka:2020}  with a fixed window size of $W = 800$ and $\Delta$ELBO rule from \citet{Kucukelbir:2015}, which is used Stan's ADVI implementation (cf.~the results of \citet{Dhaka:2020}).
We do not use $W = 100$ as is done by \citet{Dhaka:2020} because it was too small to detect convergence.
Additionally, \cref{fig:convergence-detection-comparison} compares to another convergence detection approach proposed by \citet{Pesme:2020} (described in \cref{sec:distance-based}), 
where they use a distance-based statistic to detect convergence.
While adaptive SASA+, $\Delta$ELBO, fixed window size $\widehat{R}$, and the distance-based statistic approach sometimes trigger too early or too late or SASA+ use iterations before it reaches the convergence,
adaptive $\widehat{R}$ consistently triggers when the full window suggests convergence has been reached. 
See \cref{fig:convergence-detection-comparison-more} for additional Gaussian target examples.

\begin{figure}[tbp]
\begin{center}
\begin{subfigure}[t]{.49\textwidth}
\includegraphics[width=\textwidth]{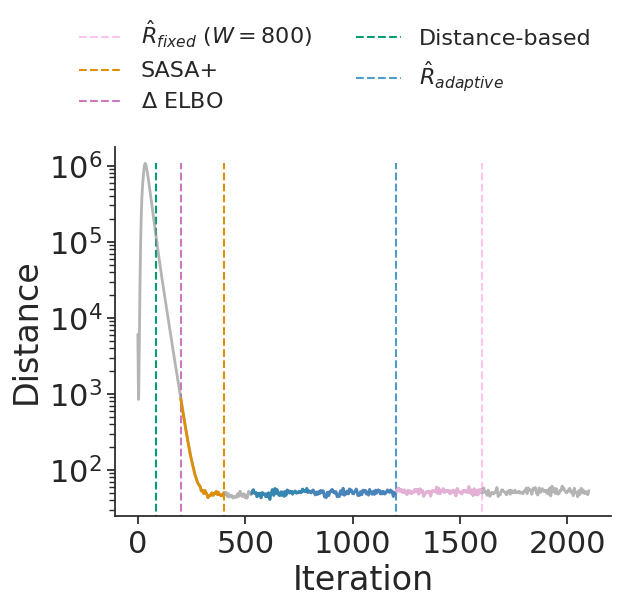}
\caption{} 
\end{subfigure}  
\begin{subfigure}[t]{.49\textwidth}
\includegraphics[width=\textwidth]{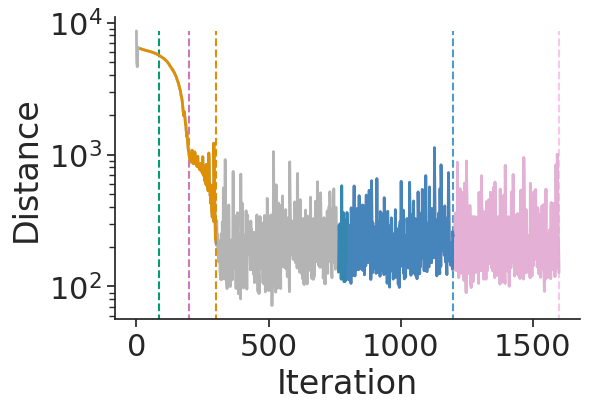}
\caption{} 
\end{subfigure}  
\caption{Iteration number versus distance between iterate average and current iterate. 
The vertical lines indicate convergence detection trigger points and (for SASA+ and $\widehat{R}$) the colored portion of 
the accuracy values indicate they are part of the window used for convergence detection.
\textbf{(a)} An uncorrelated Gaussian distribution $\distNorm(0, V)$ with $\paramdim = 500$ and $V = I$.
\textbf{(b)} A \texttt{posteriordb} dataset/model \textit{mcycle\_gp} with $\paramdim=66$.}
\label{fig:convergence-detection-comparison}
\end{center}
\end{figure}

\subsection{Determining the number of iterates for averaging}
After detecting convergence to stationarity, we need to find $\iternumaverage$ large enough to ensure the iterative average 
is sufficiently close to the mean $\meansgditerate{\learningrate_{\epochnum}}$.
But what is close enough? 
Building on our discussion in \cref{sec:problem-setup}, we aim to ensure the error in
the variational parameter estimates are small relative to the scale of uncertainty. 
For mean-field Gaussian distributions, the following result allows us to make such a guarantee precise.

\begin{proposition} \label[proposition]{prop:MF-Gaussian-error}
Let $\varfamily$ be the family of mean-field Gaussian distributions. %
Let $\hat\varparam = (\hat\tau, \hat\psi)$ denote an approximation to $\bar\varparam = (\bar\tau, \bar\psi)$. 
Define $\hat\sigma \defined \exp(\hat\psi)$ and $\bar\sigma \defined \exp(\bar\psi)$.
If there exists $\varepsilon \in (0,1/2)$ such that $|\hat\tau_{i} - \bar\tau_{i}| \le \varepsilon \hat\sigma_{i}$ and $|\hat\psi_{i} - \bar\psi_{i}| \le \varepsilon$,
then 
\[
\frac{|\hat\sigma_{i} - \bar\sigma_{i}| }{\bar\sigma_{i}} 
&\le 1.5\varepsilon & 
&\text{and} &
\frac{|\hat\tau_{i} - \bar\tau_{i}| }{\bar\sigma_{i}}
&\le 1.75\varepsilon. 
\]
\end{proposition}
See \cref{sec:MF-Gaussian-error-proof} for the proof.
Based on \cref{prop:MF-Gaussian-error}, for mean-field Gaussian variational families
we use the iterate average once the mean MCSEs $\paramdim^{-1}\sum_{i=1}^{\paramdim}\mathrm{MCSE}(\hat\tau_{\gamma,i})/\hat\sigma_{\gamma,i}$
and $\paramdim^{-1}\sum_{i=1}^{\paramdim}\mathrm{MCSE}(\hat\psi_{\gamma,i})$ are less than $\varepsilon$.%
For other variational families we rely on the less rigorous condition that $\varparamdim^{-1}\sum_{i=1}^{\varparamdim}\mathrm{MCSE}(\hat\varparam_{\gamma,i})$ is less than $\varepsilon$.
We also require the effective sample sizes of all parameters to be at least 50 to ensure the MCSE estimates 
are reliable.

Because the MCSE check requires computing $\paramdim$ ESS values, it can be computationally expensive, especially for high-dimensional models.
Therefore, it is important to optimize when conducting the checks.%
A well-known approach in such situations is the ``doubling trick.''
Let $W_\mathrm{conv}$ denote the window size when convergence is detected, and let $W_\mathrm{opt}$ denote the minimal window
size that satisfies the MCSE check. %
The doubling trick would suggest checking at iteration numbers $\iternumconverged + 2^{j}W_\mathrm{conv}$ for $j=0,1,\dots$,
in which case the total computational cost is within a factor of $4 \log W_\mathrm{opt}$ of the optimal scenario in 
in which the check is only done at $\iternumconverged + W_\mathrm{conv}$ and $\iternumconverged + W_\mathrm{opt}$). 
However, we can potentially do substantially better by accounting for the different computational
cost of the optimization versus the MCSE check. 

\begin{proposition} \label[proposition]{prop:optimal-MCSE-rechecking}
Assume that the cost of the MCSE check using $K$ iterates is $C_{E}K$
and the cost of $K$ iterations of optimization is $C_{O}K$. 
Let $r \defined C_{O}/C_{E}$, $\chi(r) \defined 1 + (1 + r)^{-1/2}$, and $g(r) \defined (2 + r + 2 (1+r)^{1/2})/(1 + r)$. 
If the MCSE check is done on iteration numbers $\iternumconverged + \chi(r)^{j}W_\mathrm{conv}$ for $j=0,1,\dots$, then
the total computational cost will be within a factor of $g(r)$ of optimal.
\end{proposition}
See \cref{sec:optimal-MCSE-rechecking-proof} for the proof. 
Since $g(0) = 4$ and $g(r)$ is monotonically decreasing in $r$,
when $r \approx 0$ -- that is, $C_{O}$ is negligible compared to $C_{E}$ -- we recover the doubling rule since $\chi(0) = 2.$ 
However, as long as $r$ is significantly greater than zero, the worst-case additional cost factor can be substantially less than 4.
Therefore we carry out the MCSE check on iteration numbers $\iternumconverged + \chi(r)^{j}W_\mathrm{conv}$
with $r$ estimated based on the actual runtimes of the optimization so far and the first MCSE check.

\begin{algorithm}[htb]%
\SetKwInput{Input}{Input}
\SetKw{Break}{break}
\Input{initial variational parameter $\sgditerate{0}$, \newline
learning rate $\learningrate$,   \newline
minimum window size $W_{\min}$,   \newline
initial iterate average relative error threshold $\varepsilon$,    \newline
maximum iterations $K_{\max}$} %
\IncMargin{1em}
$\iternumconverged \gets \mathbf{null}$ \tcp*[l]{iteration when stationarity reached} 
$\mathit{success} \gets \mathbf{false}$ \;
\For{$\iternum = 1,\dots,K_{\max}$}{
	compute stochastic gradient $\objgradest{\iternum}$ \;
	compute descent direction $d_{\iternum}$ \; %
	\tcp*[l]{step in descent direction}
	$\sgditerate{\iternum + 1} \gets \sgditerate{\iternum} - \learningrate d_{\iternum}$  \; %
	\If(check for convergence){$\iternumconverged = \mathbf{null}$ and $\iternum  \mod \iternum_\mathrm{check} = 0$ }{
		  \tcp*[l]{define window-based ESS} 
		 $\widehat{R}_{\max}(W) \defined \max_{1 \le i \le \varparamdim}\widehat{R}(\sgditerate{\iternum-W+1}_{i},\dots,\sgditerate{\iternum}_{i})$\;
		  \tcp*[l]{compute optimal window} 
		 $W_\mathrm{opt} \gets \argmin_{W_{\min} \le W \le 0.95\iternum}\widehat{R}_{\max}(W)$ \hspace{-.5em} \;
		\If{$\widehat{R}_{\max}(W_\mathrm{opt}) \le 1.1$}{
			$\iternumconverged \gets \iternum - W_\mathrm{opt}$ \;
			\tcp*[l]{window size at which to check MCSE}
			$W_\mathrm{check} \gets W_\mathrm{opt}$ \;
			$\chi_{*} \gets \chi(r)$ \tcp*[l]{see Prop.\ \ref{prop:optimal-MCSE-rechecking}}
		}
	}
	\If(check for accuracy of iterate average){$\iternumconverged \ne \mathbf{null}$ and $\iternum - \iternumconverged = W_\mathrm{check}$}{
		$W \gets W_\mathrm{check}$ \;
		$\hat\varparam \gets W^{-1}\sum_{i=\iternum-W+1}^{\iternum}\sgditerate{i}$ \;
		$\mcse \gets \mathrm{MCSE}( \sgditerate{\iternum-W}, \dots,  \sgditerate{\iternum})$\;
		$\mathrm{ESS}_{\min} \gets \min_{i}\mathrm{ESS}( \sgditerate{\iternum-W}_{i}, \dots,  \sgditerate{\iternum}_{i})$\;
		\lIf{mean-field Gaussian family}{
			$\mcse_{i} \gets \mcse_{i} / \exp(\hat\psi_{i})$ for $i=1,\dots,\paramdim$
		}
		\eIf{$\mathrm{mean} \   \mcse_{i} < \varepsilon$ and $\mathrm{ESS}_{\min} \ge 50$}{
			$\mathit{success} \gets \mathbf{true}$ \;
			\Break \;
		}{
			$W_\mathrm{check} \gets \chi_{*} W$    \;
		}
	}
}
$\meansgditerate{} \gets W^{-1}\sum_{i=\iternum-W+1}^{\iternum}\sgditerate{i}$ \;
\Return $(\iternum, \meansgditerate{}, \mathit{success})$
\caption{Fixed--learning-rate Automated Stochastic Optimization (FASO)}
\label{alg:FASO}
\end{algorithm}
\DecMargin{1em}

\begin{algorithm}[htb] %
\SetKwInput{Input}{Input}
\SetKw{Break}{break}
\SetKwFunction{FASO}{FASO} 
\Input{\small initial variational parameter $\sgditerate{0}$,  \newline
maximum number of iterations $K_{\max}$, \newline
initial learning rate $\learningrate_{0}$ (default: 0.3), \newline
minimum window size $W_{\min}$ (default: 200),   \newline
accuracy threshold $\xi$ (default: 0.1),   \newline
inefficiency threshold $\tau$ (default: 1.0),   \newline
initial iterate average error threshold $\varepsilon_0$ (default: 0.1),    \newline
adaptation factor  $\rho$ (default: 0.5)     \newline %
small iteration number $\iterate_0$ (default: 1000)
} %
\IncMargin{1em}
$\meansgditerate{\mathrm{curr}} \gets \sgditerate{0}$  \tcp*[l]{current iterate average} 
$\learningrate \gets \learningrate_{0}$ \tcp*[l]{learning rate} 
$\varepsilon \gets \varepsilon_{0}$ \tcp*[l]{iterate average error threshold}
$\iternum \gets 0$ \tcp*[l]{total iterations} 
$\epochnum \gets 0$ \tcp*[l]{total epochs} 
	\While{$\iternum < K_{\max}$}{
		$\meansgditerate{\mathrm{prev}} \gets \meansgditerate{\mathrm{curr}}$  \tcp*[l]{record previous iterate average}
		$\iternum_\mathrm{new}, \meansgditerate{\mathrm{curr}}, \mathit{success} \gets$ \FASO{$\meansgditerate{\mathrm{curr}}, \learningrate, W_{\min}, \varepsilon, K_{\max} - \iternum$} \;
		\If{$\operatorname{not} \mathit{success}$}{
			\textbf{print} ``Warning: failed to converge. Estimated error is $\mathit{error}$'' \;
			\Break\;
		}
		$\iternum \gets \iternum + \iternum_\mathrm{new}$ \tcp*[l]{update total iterations}
		\If{$\epochnum \ge 1$} {
			$\delta_{\epochnum} \gets \skl{q_{\meansgditerate{\mathrm{prev}}}}{q_{\meansgditerate{\mathrm{curr}}}}$\;
			compute estimates $\hat{C}$ and $\hat{\biaspower}$ using weighted linear regression  \;
			$\widehat{\mathrm{RSKL}} \gets \rho^{\hat{\biaspower}} + \xi / (\hat{C}^{1/2} \learningrate^{\hat{\biaspower}})$\;
			$\hat{\iterate}_{\learningrate} \gets \iternum_\mathrm{new}^{(t)} - \iternum_\mathrm{new}^{(t-1)}$\;
			\If{$\epochnum \ge 2$} {  \tcp*[l]{remove the converged iterations of initial variation parameter}
				compute estimates $\hat{\alpha}$ and $\hat{\beta}$ using weighted least squares \;
				\eIf{$\hat{\beta} < 0$}{
					$\hat{\iterate}_{\rho \learningrate} \gets (\rho \learningrate)^{\hat{\alpha}} e^{\hat{\beta}}$\;
				}{
					$\hat{\iterate}_{\rho \learningrate} \gets \hat{\iterate}_{\learningrate}$\;
				}
				$\widehat{\mathrm{RI}} \gets \hat{\iterate}_{\rho \learningrate} /(\hat{\iterate}_{\learningrate} + \iterate_0)$\;
				\If{$\widehat{\mathrm{RSKL}} \cdot \widehat{\mathrm{RI}} > \tau$}{
		 			\Break
				}
			}
		}
		$\learningrate \gets \rho\learningrate$  \tcp*[l]{decrease learning rate} 
		$\varepsilon \gets \rho\varepsilon$  \tcp*[l]{decrease iterate average error threshold}
		$\epochnum \gets \epochnum + 1$  \tcp*[l]{increment epoch counter}
	}
\Return $\meansgditerate{\mathrm{curr}}$ 
\caption{Robust and automated black-box variational inference (RABVI)}
\label{alg:RABVI}
\end{algorithm}
\DecMargin{1em}

\section{Complete Framework} \label{sec:summary}

Combining our innovations from \cref{sec:termination-rule,sec:fixed-learning-rate} leads to our complete framework.
When $\learningrate$ is fixed, our proposal from \cref{sec:fixed-learning-rate} is summarized in \cref{alg:FASO}, 
which we call \emph{Fixed--learning-rate Automated Stochastic Optimization} (FASO).
Combining the termination rule from \cref{sec:termination-rule} with FASO, we get our complete framework, 
\emph{Robust and Automated Black-box Variational Inference (RABVI)},
which we summarize in \cref{alg:RABVI}.
We will verify the robustness of RABVI through numerical experiments.
RABVI is automatic since the user is only required to provide a target distribution and the only 
tuning parameters we recommend changing from their defaults are defined on interpretable, intuitive scales:
\bitems
\item \textbf{accuracy threshold $\xi$:}
	The symmetrized KL divergence accuracy threshold can be set based on the expected accuracy of the variational approximation.
	If the user expects $\kl{\postdensity}{\approxdensity_{*}}$ to be large, then we recommend choosing $\xi \in [1,10]$.
	If the user expects $\kl{\postdensity}{\approxdensity_{*}}$ to be fairly small, then we recommend choosing $\xi \in [.01, 1]$.
	Our experiments suggest $\xi=0.1$ is a good default value.
\item \textbf{inefficiency threshold $\tau$:}
	We recommend setting the inefficiency threshold $\tau = 1$, as this weights accuracy and computation equally. 
	A larger value (e.g., 2) could be chosen if accuracy is more important while a smaller value (e.g., 1/2) would 
	be appropriate if computation is more of a concern. 
\item \textbf{maximum number of iterations $K_{\max}$:} 
	The maximum number of iterations can be set by the user based on their
	computational budget. 
	RABVI will warn the user
	\clearpage
	
	 if the maximum number of iterations is reached without convergence,
	so the user can either increase $K_{\max}$ or accept the estimated level of accuracy
	that has been reached. 
\eitems

We expect the remaining tuning parameters will typically not be adjusted by the user.
We summarize our recommendations:
\bitems
\item \textbf{initial learning rate $\learningrate_{0}$:}
	When using adaptive methods such as RMSProp or Adam, the initial learning rate can 
	essentially be set in a problem-independent manner. 
	We use $\learningrate_{0} = 0.3$ in all of our experiments.
	If using non-adaptive methods, a line search rule such as the one proposed in \citet{Zhang:2020:SALSA}
	could be used to find a good initial learning rate. 
\item \textbf{minimum window size $W_{\min}$:}  
	We recommend taking $W_{\min} = 200$ so that that each of the split-$\widehat{R}$ 
	values are based on at least 100 samples.
\item \textbf{small iteration number $\iterate_0$:}
	The value of $\iterate_0$ should represent a number of iterations the user considers to be fairly small (that is, not requiring too much
	computational effort). 
	We use $\iterate_0 = 5 W_\mathrm{min} = 1000$ for our experiments, but it could be adjusted by the user. 
\item \textbf{initial iterate average relative error threshold $\varepsilon_0$:}
	We recommend scaling $\veps_0$ with $\xi$
	since more accurate iterate averages are required for sufficiently accurate symmetrized
	KL estimates. Therefore, we take $\veps_0 = \xi$ by default.
\item \textbf{adaptation factor  $\rho$:} 
	We recommend taking $\rho = 0.5$ because using a smaller $\rho$ value could lead to too few
	$\delta_{\epochnum}$ values for the estimation of $C$ (and $\biaspower$) and 
	using a larger $\rho$ value would make the algorithm too slow.
\item \textbf{Monte Carlo samples $M$:}
	We find that $M=10$ provides a good balance between gradient accuracy and computational burden 
	but the performance is fairly robust to the choice of $M$ as long as it is not too small. 
\eitems

\section{Experiments} \label{sec:experiments}

Unless stated otherwise, all experiments use avgAdam to compute the descent direction, mean-field Gaussian distributions as the variational family,
and the tuning parameters values recommended in \cref{sec:summary}. 
We fit the regression model for $C$ (and $\biaspower$) in Stan, which result in extremely small computational overhead of less than {$0.5\%$}.
We compare RABVI 
to FASO, Stan's ADVI implementation,  stochastic gradient descent (SGD) using an exponential decay of the learning rate, and fixed--learning rate versions of RMSProp, Adam, and a windowed version of Adagrad (WAdagrad), which is the default optimizer in PyMC3.
Moreover, we compare RABVI with exponential decay and cosine learning rate schedules using Adam and RMSProp optimization methods.
We run all the algorithms that do not have a termination criterion for $K_{\max} = 100{,}000$ iterations and for the fixed--learning-rate algorithms we use learning rate $\learningrate = 0.01$ in an effort to balance speed with accuracy.
For exponential decay, we use a learning rate of $\learningrate = \learningrate_0  \delta^{\lfloor k/s \rfloor}$, where $\learningrate_0$ denotes the initial learning rate, $\delta$ denotes the decay rate, $k$ denotes the iteration, and $s$ denotes the decay step. We choose  $\learningrate_0=0.01$, $\delta = 0.96$, and $s=900$ so that the final learning rate is approximately $0.0001$ \citep{chen:2017}.
For cosine schedule, we use a learning rate of $\learningrate = \learningrate_{\min} + \frac{1}{2} (\learningrate_{\max} - \learningrate_{\min}) (1 + \cos(\frac{k}{K} \pi))$, where $\learningrate_{\min}$ and  $\learningrate_{\max}$ denote the minimum and maximum values of learning rate, $k$ denotes the current iteration, and $K$ denotes the maximum number of iterations \citep{Loshchilov2017}. We choose $\learningrate_{min} = 0.0001$, $\learningrate_{max} = 0.01$ to make it comparable with other methods. %

We use symmetrized KL divergence as the accuracy measure when we can compute the ground-truth optimal variational approximation. 
Otherwise, we use the following metrics (where $\mu$ and $\sigma$ are, respectively, the posterior mean and standard deviation vectors):
\bitems
\item \emph{Relative mean error} $\twonorm{(\mu - \hat\mu) / \sigma}$, where 
$\hat\mu$ is the variational approximation to $\mu$.
\item \emph{Relative standard deviation error} $\twonorm{\hat\sigma/\sigma - 1}$, where $\hat\sigma$ is the variational approximation to $\sigma$.
\item \emph{Under coverage error} of the variational approximation to the 95\% credible intervals $\min(0, \vert .95 - c_{i} \vert)$, 
where $c_{i} \defined \Pi(\{ \param \st \param_{i} \in (a_{i},b_{i})\})$ and $(a_{i},b_{i})$ is the variational estimate of the central 95\% credible interval for parameter $\param_{i}$. 
\item \emph{Maximum mean discrepancy (MMD)} $\mathrm{MMD}^2(P,Q) \defined \EE[k(x,x') - 2\,k(x,y) + k(y,y')]$, where $x, x' \sim P$ and $y, y' \sim Q$ are independent and $k(x, y) = \exp\{-\frac{1}{2}\norm{\frac{x - y}{\sigma}}_{2}^{2}\}$ is the squared exponential kernel \citep{gretton:2006}.%
\eitems

\subsection{Accuracy with Gaussian Targets}\label{subsec:gaussian-targets}

First, to explore optimization accuracy relative to the optimal variational approximation, we consider
Gaussian targets of the form $\postdensity = \distNorm(0, V)$.
In such cases, we can compute the ground-truth optimal variational approximation either analytically (because the distribution belongs
to the mean-field variational family and hence $\approxdensity_{*} = \pi$) 
or numerically using deterministic optimization (since the KL divergence between Gaussians is available in closed form). 
Specifically, we consider the following covariances that aid in assessing our framework across a range of condition numbers from 1 to around 9000:
\bitems
\item Identity covariance: $V = I$ 
\item  Diagonal non-identity covariance: $V_{ij} = j\ind[i=j]$
\item  Uniform covariance with correlation 0.8: $V_{ij} = \ind[i = j] + 0.8\ind[i \ne j]$
\item  Banded covariance with maximum correlation 0.8: $V_{ij} =  \ind[i = j] + 0.8^{\mid i-j\mid}\ind[i \ne j]$
\item Diagonal non-identity banded covariance with maximum correlation 0.8: $V_{ij} = j\ind[i=j] + 0.8^{\mid i-j\mid}\ind[i \ne j]$
\item Diagonal identity (except first entry) uniform covariance with maximum correlation 0.8: $V_{ij} = 1000\ind[i=j=1] + \ind[i=j \ne 1] + 0.8\ind[i \ne j]$
\item Diagonal identity (except first entry) banded covariance with maximum correlation 0.8: $V_{ij} = 1000\ind[i=j=1] + \ind[i=j \ne 1] + 0.8^{\mid i-j\mid}\ind[i \ne j]$
\eitems 
In our selection, we specifically included diagonal identity matrices (with the exception of the first entry) combined with either uniform or banded covariance structures, showcasing a maximum covariance of 0.8. This setting results in weaker correlation between the first component and the others. This choice was strategic to achieve higher condition numbers (around 5000 and 9000 respectively), given that correlations set at $0.8$ or  $0.8^{\mid i-j\mid}\ind[i \ne j]$ yield condition numbers around 400 and 80, respectively.

\Cref{fig:Accuracy-with-gaussian-targets,fig:Accuracy-with-gaussian-targets-high-condion-covariance-matrix,fig:learning-rate-schedules-with-gaussian-targets} show the comparison of RABVI  to FASO, Stan's ADVI implementation, SGD with exponential decay learning rate (SGD-ED), Adam with exponential decay and cosine learning rates (Adam-ED, Adam-C), RMSProp with exponential decay and cosine learning rates (RMSProp-ED, RMSProp-C), and fixed-learning rate versions of RMSProp, Adam,
Windowed Adagrad (WAdagrad). 
The findings demonstrate that RABVI consistently outperforms ADVI, SGD-ED, both adaptive learning rate versions of RMSProp, and the fixed-learning rate methods in a majority of the cases. While Adam and SGD-ED occasionally reach performance levels similar to RABVI, they tend to converge more slowly and with less reliability. Additionally, despite Adam-ED and Adam-C closely matching RABVI's performance, they lack a dependable mechanism for determining when to terminate the optimization.
On the other hand, by varying the accuracy threshold $\xi$, the quality of the final RABVI approximation $\hat\approxdensity_{*}$ also varies
such that $\skl{\approxdensity_{*}}{\hat\approxdensity_{*}}^{1/2} \approx \xi$. 

To demonstrate the flexibility of our framework, we used RABVI with a variety of optimization methods: 
RMSProp, avgRMSProp, avgAdam, natural gradient descent (NGD), and stochastic quasi-Newton (SQN).
See \cref{sec:ngd,sec:sqn} for details. 
\Cref{fig:RABVI-with-gaussian-targets} shows that avgAdam and avgRMSProp optimization methods have a similar improvement in symmetrized KL divergence between optimal and estimated variational approximation for all cases. 
NGD is not stable for the diagonal non-identity covariance structure and SQN does not perform well with uniform covariance structure. 
Even though RMSProp shows an improvement in accuracy for large step sizes, accuracy does not improve as the step size decreases. 

\begin{figure}[tbp]
\begin{center}
\begin{subfigure}[t]{.48\textwidth}
\centering
\includegraphics[width=\textwidth]{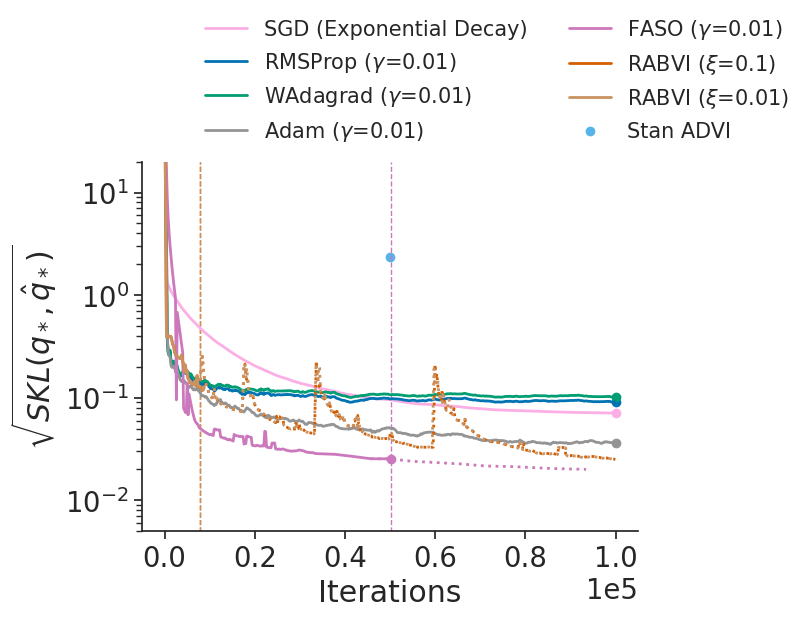}
\caption{uncorrelated $\paramdim = 100$} 
\end{subfigure}  
\begin{subfigure}[t]{.48\textwidth}
\centering
\includegraphics[width=\textwidth]{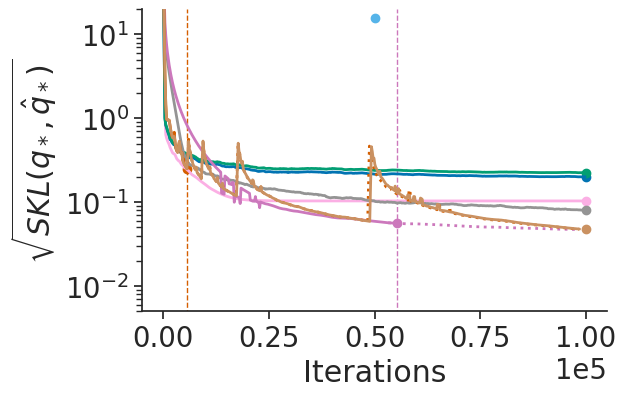}
\caption{uncorrelated $\paramdim = 500$} 
\end{subfigure} 
\begin{subfigure}[t]{.48\textwidth}
\centering
\includegraphics[width=\textwidth]{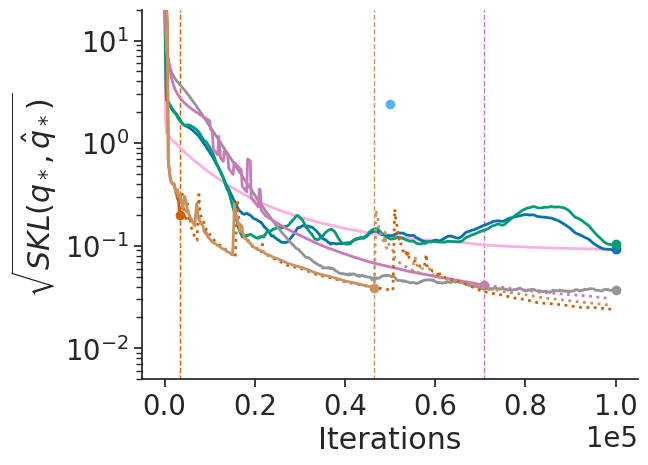}
\caption{uniform correlated $\paramdim = 100$} 
\end{subfigure}  
 \begin{subfigure}[t]{.48\textwidth}
\centering
\includegraphics[width=\textwidth]{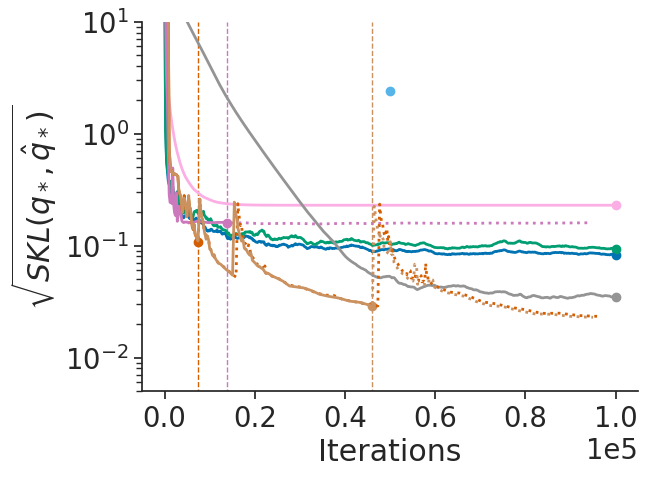}
\caption{banded correlated $\paramdim = 100$} 
\end{subfigure}  
\caption{
Accuracy comparison of variational inference algorithms using Gaussian targets,
where accuracy is measured in terms of the square root of symmetrized KL divergence between iterate average and optimal variational approximation.
The vertical lines indicate the termination rule trigger points of FASO and RABVI.
Iterate averages for Adam, RMSProp, and WAdagrad computed at every 200th iteration using a window 
size of 20\% of iterations. 
}
\label{fig:Accuracy-with-gaussian-targets}
\end{center}
\end{figure}

\subsection{Reliability Across Applications}\label{subsec:posteriordb-datasets} 

To validate the robustness and reliability of RABVI across realistic use cases, 
we consider 18 diverse dataset/model pairs found in the \texttt{posteriordb} package\footnote{\url{https://github.com/stan-dev/posteriordb}}  (see \cref{tab:posteriordb-datasets} for details). The \texttt{posteriordb} package contains a wide range of real-world data and models and
is specifically designed to provide realistic performance evaluations of approximate posterior inference algorithms.
The accuracy was computed based on ground-truth estimates obtained using the posterior draws included in \texttt{posteriordb} package
if available. Otherwise, we ran Stan's dynamic HMC algorithm \citep{stan20} to obtain the ground truth (4 chains for 50{,}000 iterations each).
To stabilize the optimization, we initialize the variational parameter estimates using RMSProp for the initial learning rate only. 
A comparison across optimization methods validates our choice of avgAdam over alternatives (\cref{fig:RABVI-with-posteriordb}). 

 \paragraph{Comparison to alternative optimization methods.}
To evaluate RABVI's effectiveness in real-world applications, we compared it against alternative optimization methods 
with both fixed and adaptive learning rate schedules. 
Based on the results described in \cref{subsec:gaussian-targets}, we opt to compare to Adam using either a fixed, 
exponential decay, or cosine learning rate schedule since they perform best overall in the Gaussian target experiments. 
Additionally, we include FASO, which used avgAdam, as another benchmark.
\Cref{fig:Accuracy-with-posteriordb-targets} shows RABVI is more consistent than all the alternative methods. 
While these methods sometimes matched RABVI's performance, RABVI's ability to identify an appropriate stopping 
point contributes to its overall efficiency, setting it apart from the competition.

 \begin{figure}[tbp]
\begin{center}
\begin{subfigure}[t]{.32\textwidth}
\centering
\includegraphics[width=\textwidth]{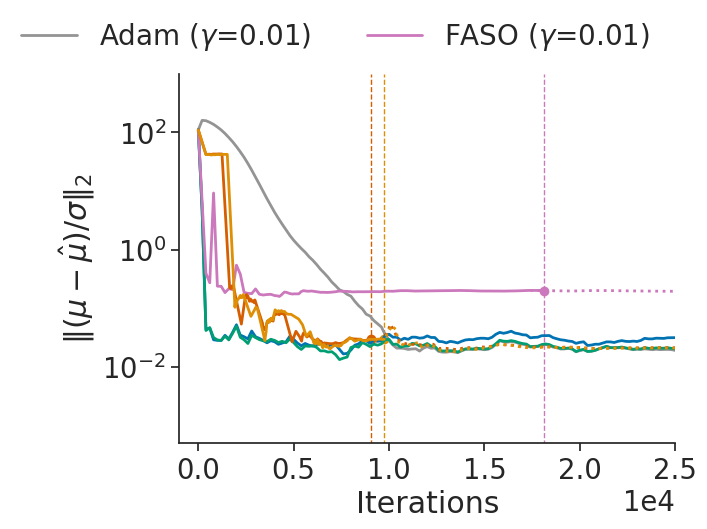}\\
\includegraphics[width=\textwidth]{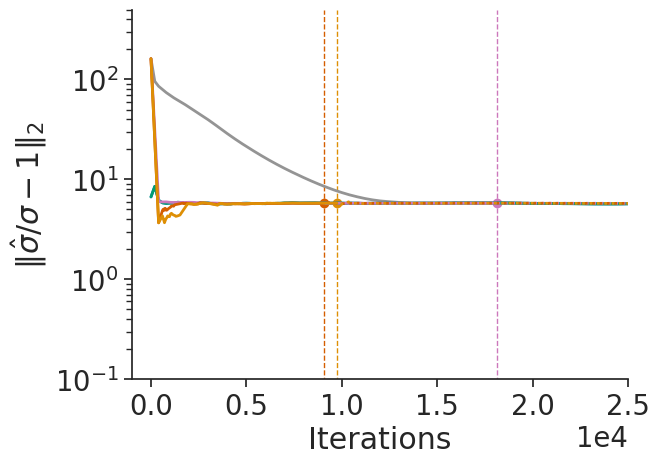}\\
\caption{arK} 
\end{subfigure}  
\begin{subfigure}[t]{.32\textwidth}
\centering
\includegraphics[width=\textwidth]{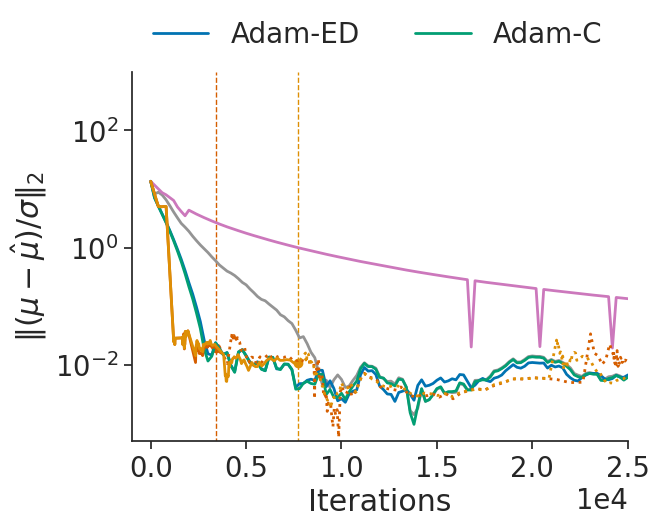}\\ %
\includegraphics[width=\textwidth]{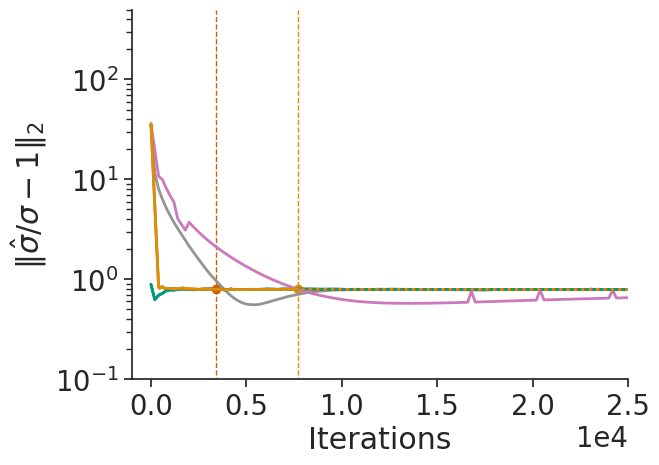}\\
\caption{dogs} 
\end{subfigure}  
\begin{subfigure}[t]{.32\textwidth}
\centering
\includegraphics[width=\textwidth]{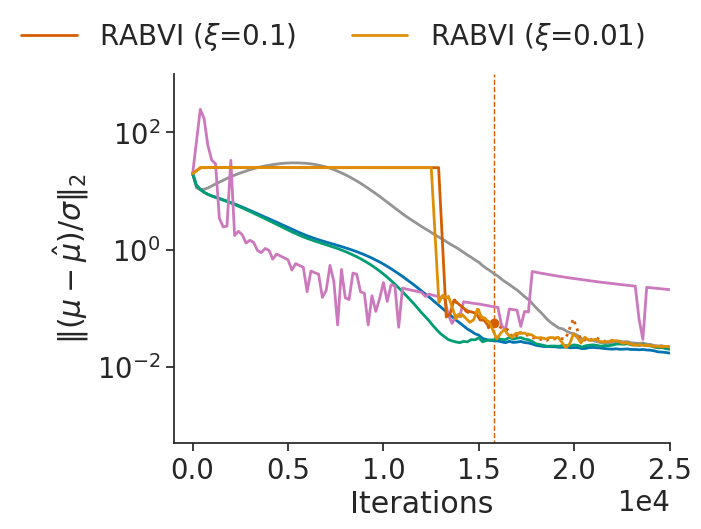}\\
\includegraphics[width=\textwidth]{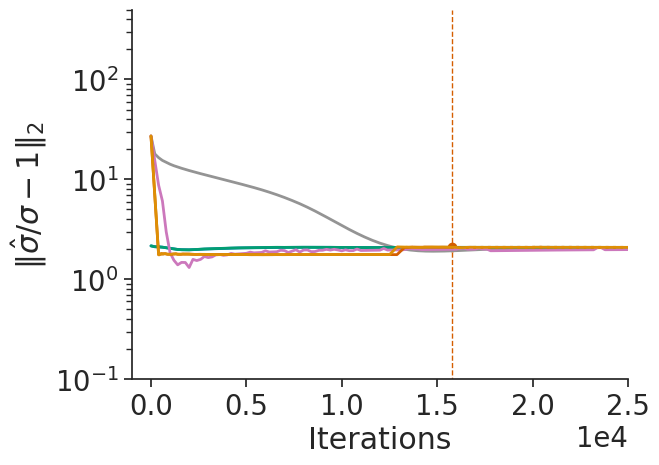}\\
\caption{nes2000} 
\end{subfigure}  
\begin{subfigure}[t]{.32\textwidth}
\centering
\includegraphics[width=\textwidth]{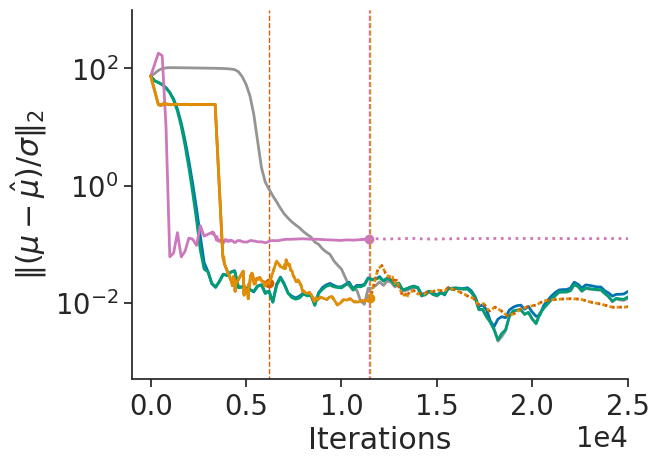}\\
\includegraphics[width=\textwidth]{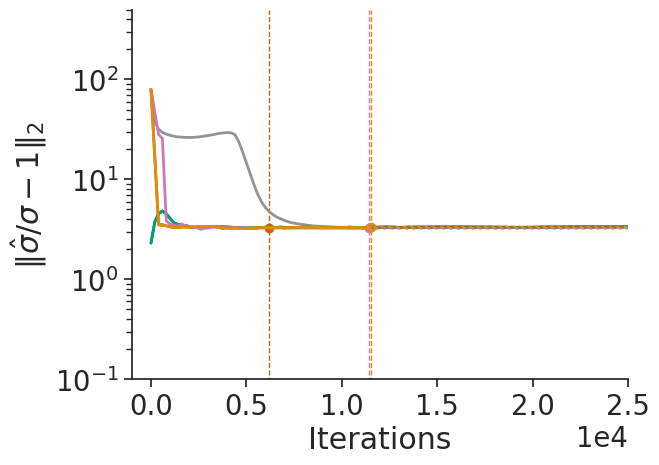}\\
\caption{low\_dim\_gauss} 
\end{subfigure}  
\begin{subfigure}[t]{.32\textwidth}
\centering
\includegraphics[width=\textwidth]{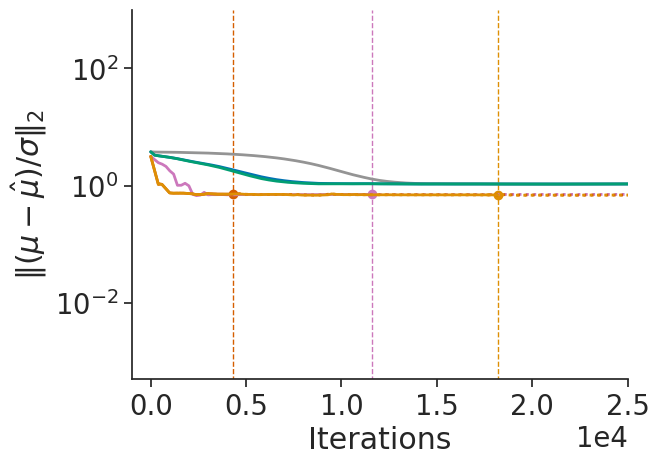}\\
\includegraphics[width=\textwidth]{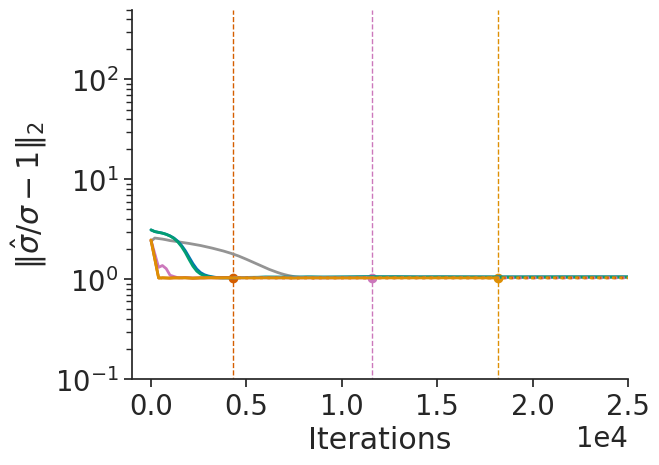}
\caption{8schools\_c} 
\end{subfigure} 
\begin{subfigure}[t]{.32\textwidth}
\centering
\includegraphics[width=\textwidth]{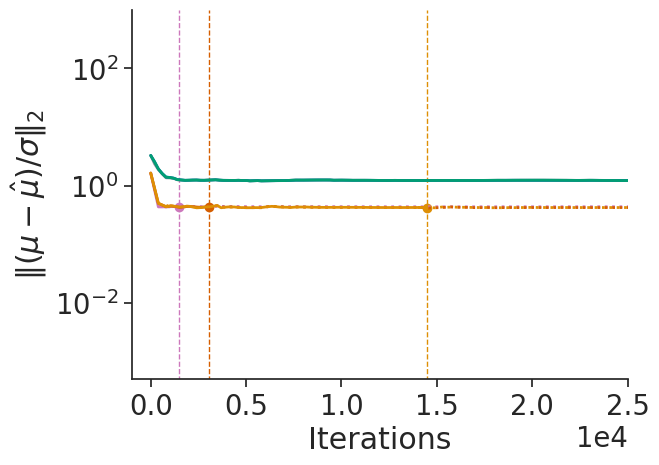}\\
\includegraphics[width=\textwidth]{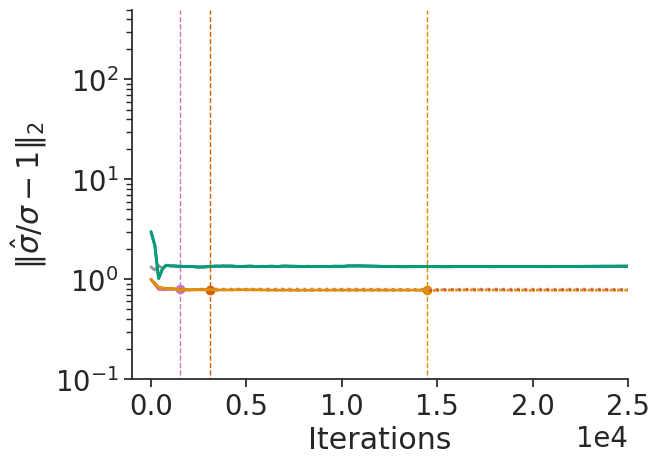}
\caption{8schools\_nc} 
\end{subfigure} 
\caption{
Accuracy comparison of variational inference algorithms using \texttt{posteriordb} models and datasets, where accuracy is measured in terms of relative mean error (top) and relative standard deviation error (bottom). The vertical lines indicate the termination rule trigger points of FASO and RABVI.
The iterate average for Adam is computed at every $200^{th}$ iteration using a window 
size of  20\% of iterations.}
\label{fig:Accuracy-with-posteriordb-targets}
\end{center}
\end{figure}

\paragraph{Accuracy and robustness.}
First, we investigate the accuracy and algorithmic robustness of RABVI. 
In terms of robustness, \Cref{fig:Posteriordb-accuracy-data-and-model-comparisons-mean,fig:Posteriordb-mf_vs_fr-data-and-model-comparisons-mean,fig:Posteriordb-mf_vs_fr-data-and-model-comparisons-std} validate our termination criteria since  
after reaching the termination point there is no considerable improvement of the accuracy for most of the models and datasets.
While in many cases the mean estimates are quite accurate, the standard deviation estimates tended to be poor, which is consistent with
typical behavior of mean-field approximations. 
To examine whether RABVI %
can achieve more accurate results with more flexible variational families, 
we conduct the same experiment using multivariate Gaussian approximation family and normalizing flow family using real-NVP flow with 2 hidden layers, 8 hidden units, and 3 coupling layers \citep{dinh:2017}. We employ FASO in our real-NVP experiments because the complexity of the approximation family prevents us from obtaining a closed form for the SKL divergence, which is necessary for computing the termination rule in RABVI. %
In some cases the accuracy of the mean and/or standard deviations estimates improve (\textit{bball\_0}, \textit{dogs\_log}, \textit{8schools\_c}, \textit{hudson\_lynx}, \textit{hmm\_example}, \textit{nes2000}, and \textit{sblrc}).
However, the results are inconsistent and sometimes worse due to the higher-dimensional, more challenging optimization problem. 
Our findings underscore the necessity of supplementing an improved optimization framework like RABVI with diagnostics for assessing the accuracy of the posterior approximation \citep{Yao:2018:VI,Huggins:2020:VI,Wang2023}.

\paragraph{Comparison to MCMC.}
We additionally benchmarked the runtime and accuracy of RABVI to \texttt{Stan}'s dynamic HMC algorithm,
for which we ran 1 chain for 25{,}000 iterations including 5{,}000 warmup iterations. 
We measure runtime in terms of the number of likelihood evaluations
and compared the relative error between the methods at the RABVI termination rule trigger point or final likelihood evaluation of HMC (whichever comes first).
\Cref{fig:Runtime-comparison-posteriordb,fig:Runtime-comparison-posteriordb-more,fig:Runtime-comparison-posteriordb-summary} show that RABVI tends to provide similar or better posterior mean estimates (the exceptions are \textit{gp\_pois\_regr}, \textit{hudson\_lynx}, and \textit{sblrc}). However, the RABVI standard deviation estimates tend to be less accurate even when using the full-rank Gaussian variational family. 
This could be because the optimization of full-rank Gaussian is more challenging having more variational parameters to estimate \citep{Bhatia:2022}. \Cref{fig:Runtime-comparison-posteriordb-mmd-summary} shows that, in terms of the MMD, the HMC approximation is closer to the target than BBVI as one would expect. Overall, the MMD values for RABVI are reasonably small.

We also compared the 95\% quantiles posterior under coverage error between RABVI and FASO methods using different approximation families and MCMC.  \Cref{fig:Runtime-comparison-posteriordb-cov-perc-summary} shows that HMC and real-NVP flows %
do not undercover the posterior. However, the mean-field and full-rank Gaussian families do.

\begin{figure}[tbp]
\begin{center}
\begin{subfigure}[t]{.32\textwidth}
\centering
\includegraphics[width=\textwidth]{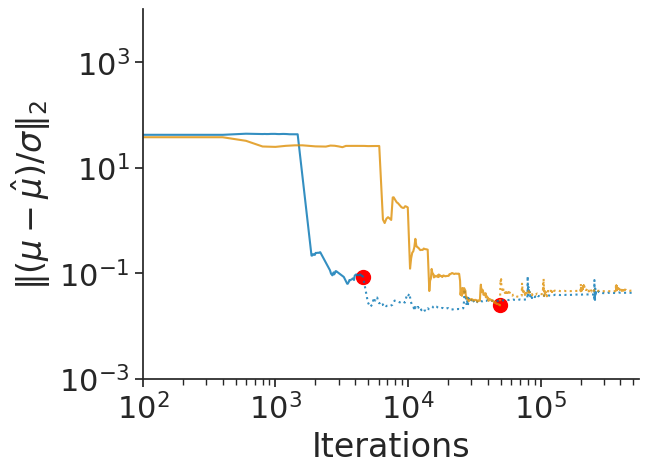}
\includegraphics[width=\textwidth]{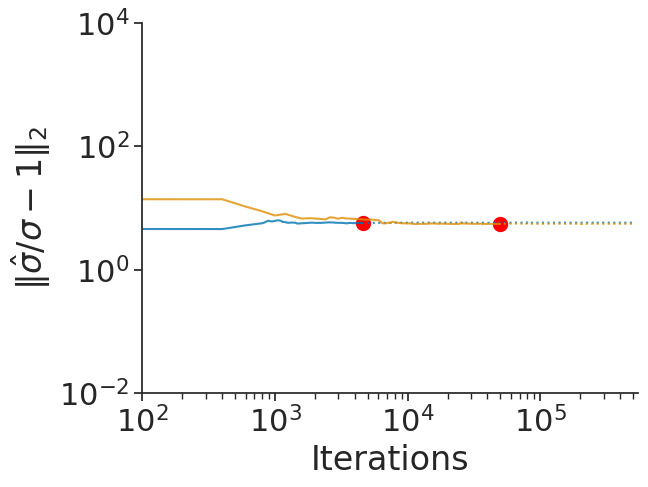}
\caption{arK} 
\end{subfigure} 
\begin{subfigure}[t]{.32\textwidth}
\centering
\includegraphics[width=\textwidth]{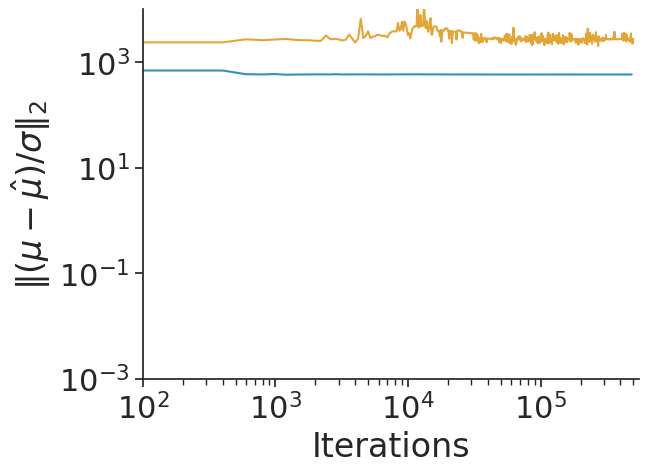}
\includegraphics[width=\textwidth]{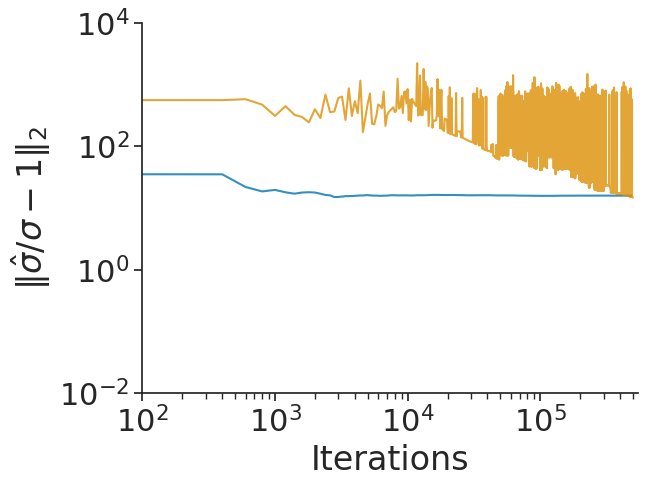}
\caption{diamonds} 
\end{subfigure}
\begin{subfigure}[t]{.32\textwidth}
\centering
\includegraphics[width=\textwidth]{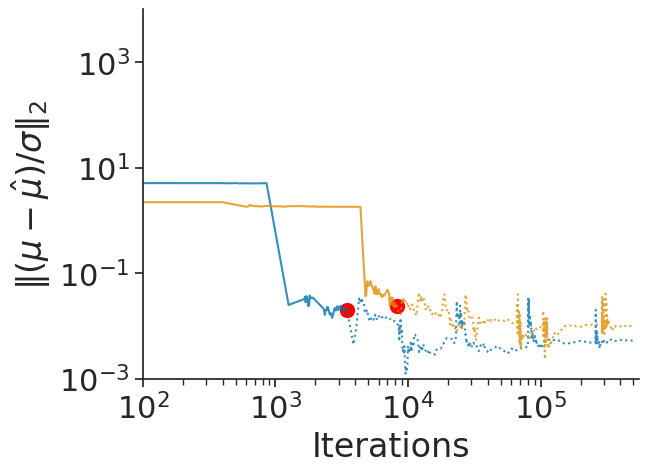}
\includegraphics[width=\textwidth]{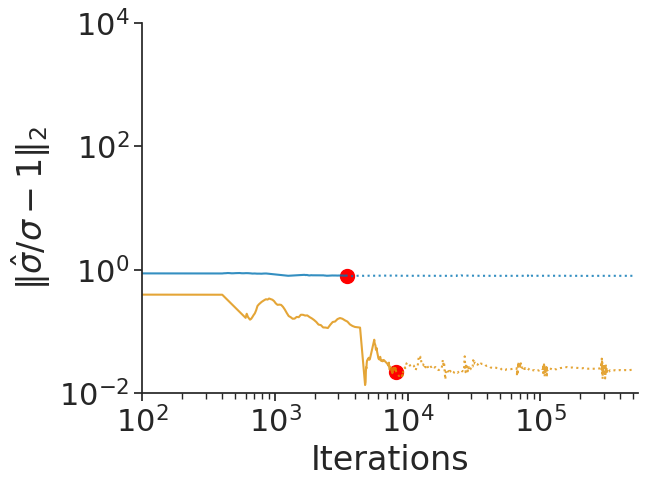}
\caption{dogs} 
\end{subfigure}
\begin{subfigure}[t]{.32\textwidth}
\centering
\includegraphics[width=\textwidth]{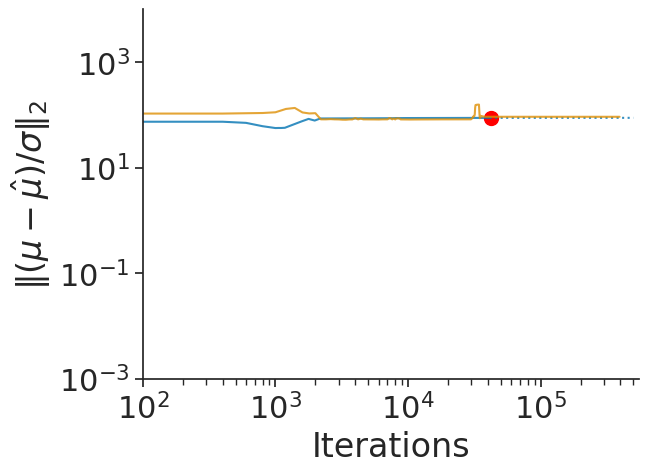}
\includegraphics[width=\textwidth]{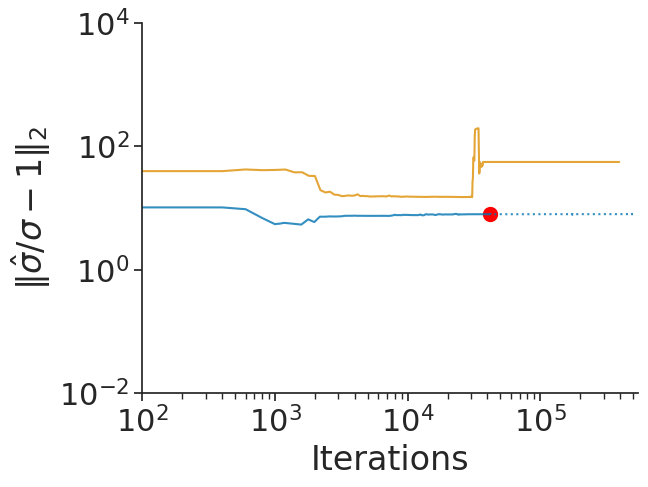}
\caption{gp\_pois\_regr} 
\end{subfigure}
\begin{subfigure}[t]{.32\textwidth}
\centering
\includegraphics[width=\textwidth]{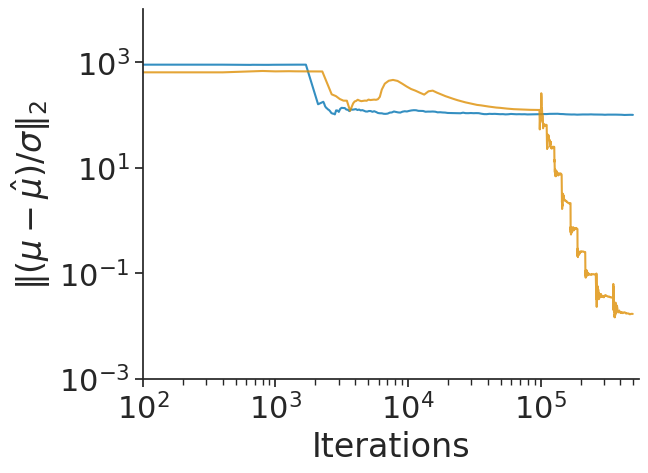}
\includegraphics[width=\textwidth]{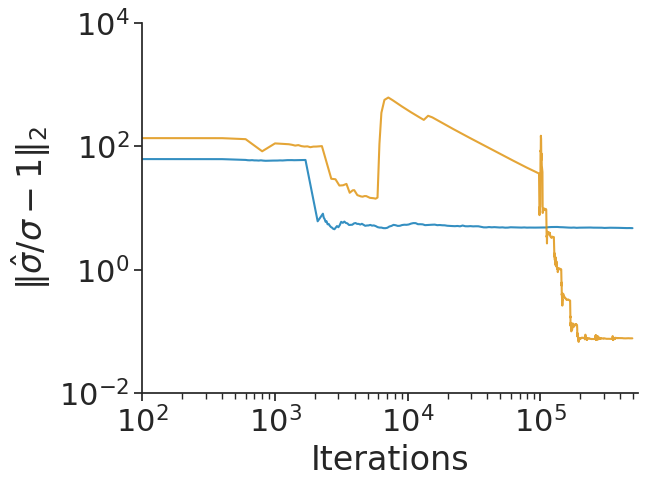}
\caption{sblrc} 
\end{subfigure}
\begin{subfigure}[t]{.32\textwidth}
\centering
\includegraphics[width=\textwidth]{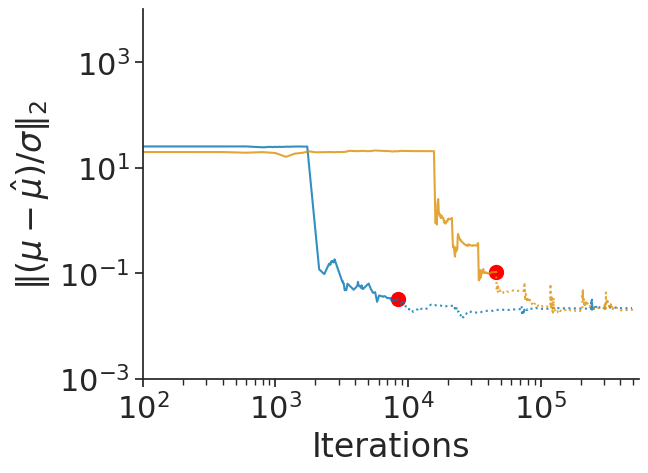}
\includegraphics[width=\textwidth]{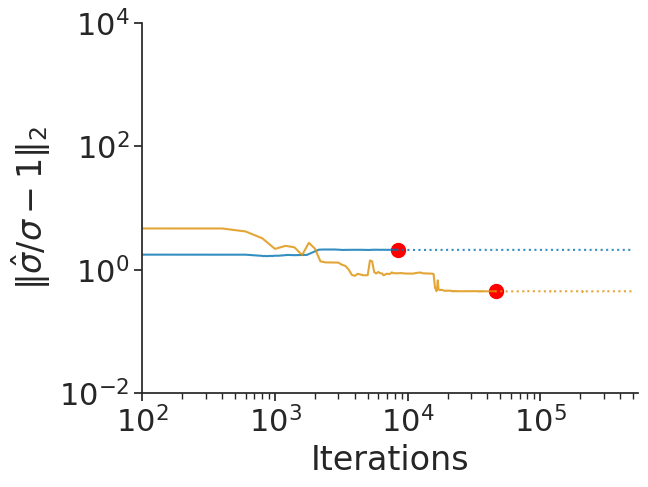}
\caption{nes2000} 
\end{subfigure}
\caption{
Accuracy of mean-field  (blue) and full-rank (orange) Gaussian family approximations for selected \texttt{posteriordb} data/models, where accuracy is measured in terms of relative mean error (top) and relative standard deviation error (bottom). 
The red dots indicate where the termination rule triggers.}
\label{fig:Posteriordb-accuracy-data-and-model-comparisons-mean}
\end{center}
\end{figure}

\begin{figure}[tbp]
\begin{center}
\centering
\includegraphics[width=.8\textwidth]{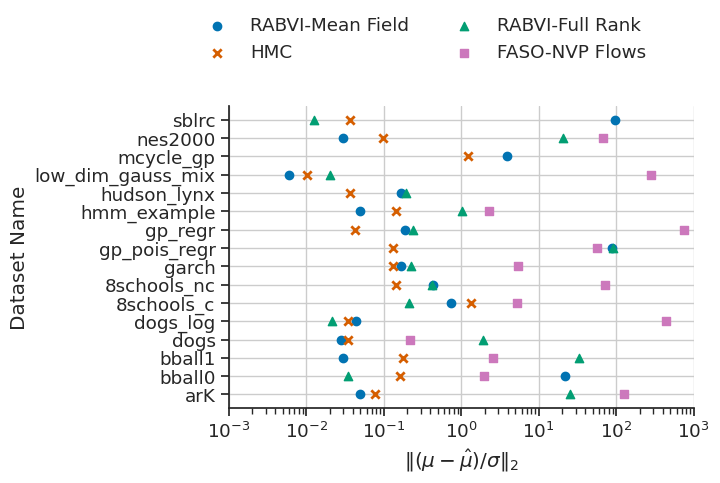} \\
\includegraphics[width=.8\textwidth]{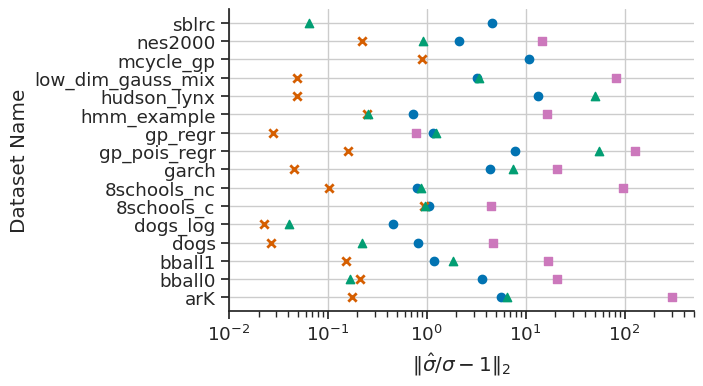}
\caption{
Results of RABVI with mean-field Gaussian and full-rank Gaussian family and FASO with real NVP flows comparison to dynamic HMC 
at the same computational cost (likelihood evaluations) in terms of relative mean error (top) and relative standard deviation error (bottom).}
\label{fig:Runtime-comparison-posteriordb-summary}
\end{center}
\end{figure}

\subsection{Robustness to Tuning Parameters: Ablation Study}
To validate the robustness of RABVI to different choices of algorithm tuning parameters, we consider the Gaussian targets and two \textit{posteriordb dataset/model pairs:  \textit{dogs} (logistic mixed-effects model) and \textit{arK} (AR(5) time-series model)}. 
We vary one tuning parameter while keeping the recommended default values for all others.
We consider the following values for each parameter (default in bold):
\begin{itemize}
\item initial learning rate $\learningrate_0$: $0.01, 0.1, \mathbf{0.3}, 0.5$
\item minimum window size $W_{\min}$: $100, \mathbf{200}, 300,500$
\item initial iterate average relative error threshold $\varepsilon_0$: $0.05, \mathbf{0.1}, 0.2, 0.5$
\item  inefficiency threshold $\tau$:  $0.1, 0.5, \mathbf{1.0}, 1.2$
\item Monte Carlo samples $M$: $1, 5, \mathbf{10}, 15, 25$.
\end{itemize}
We repeat each experimental condition 10 times to verify the robustness of different initializations of the variational parameters.
\Cref{fig:Robustness-tuning-parameters-gaussian,fig:Robustness-tuning-parameters-dogs-posteriordb,fig:Robustness-init-learning-rate,fig:Robustness-minimum-window-size-gaussian,fig:Robustness-mcse-threshold-gaussian,fig:Robustness-inefficiency-threshold-gaussian,fig:Robustness-mc-samples-gaussian,fig:Robustness-tuning-parameters-ark-posteriordb} 
suggest that overall the accuracy and runtime of RABVI is not too sensitive to the choice of the tuning parameters. 
However, extreme tuning parameter choices (e.g., $\learningrate = 0.01$ or $M \le 5$) can lead to longer runtimes.

\begin{figure}[tbp]
\begin{center}
\begin{subfigure}[t]{.24\textwidth}
\centering
\includegraphics[width=\textwidth]{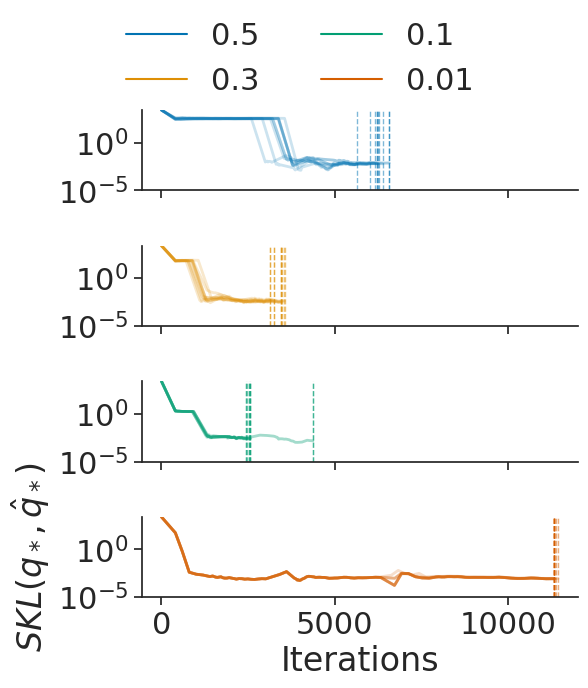}\\
\includegraphics[width=\textwidth]{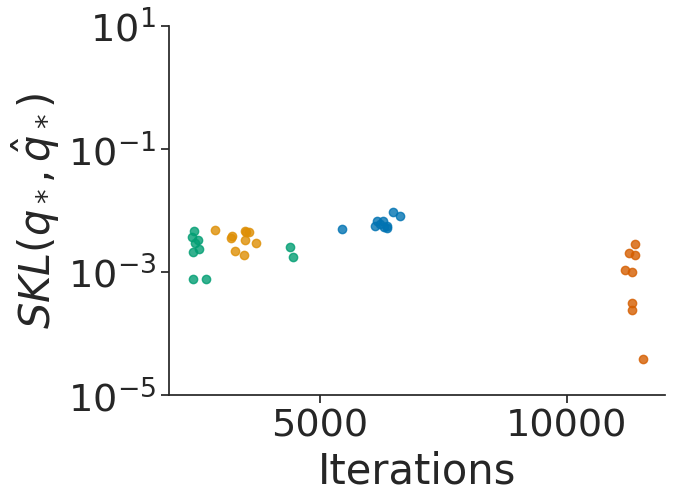}
\caption{$\learningrate_{0}$} 
\end{subfigure} 
\begin{subfigure}[t]{.24\textwidth}
\centering
\includegraphics[width=\textwidth]{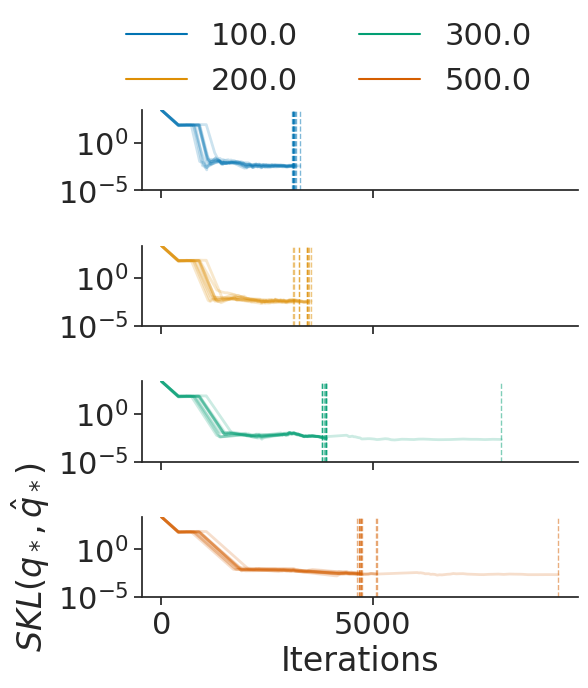} \\
\includegraphics[width=\textwidth]{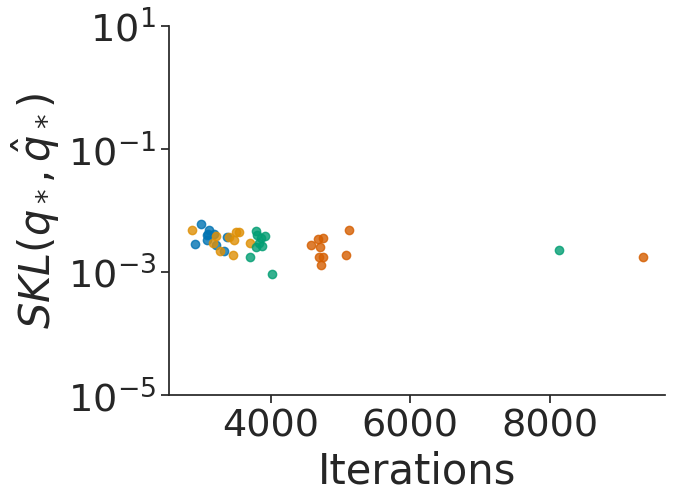}
\caption{$W_{min}$} 
\end{subfigure} 
\begin{subfigure}[t]{.24\textwidth}
\centering
\includegraphics[width=\textwidth]{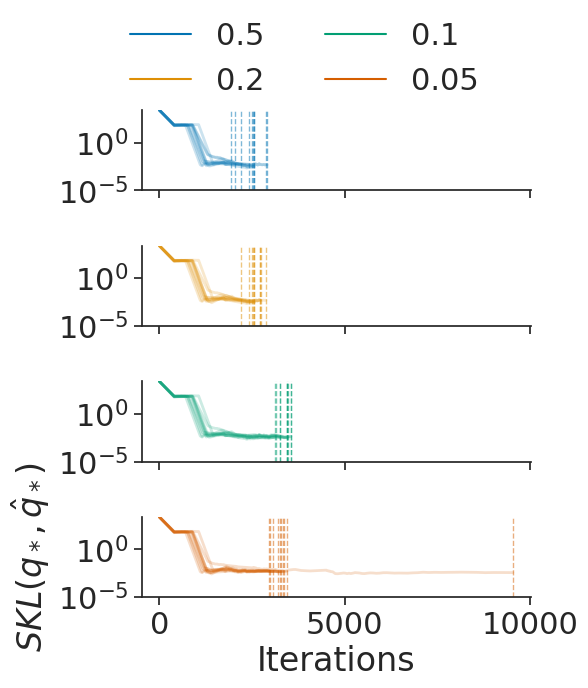} \\
\includegraphics[width=\textwidth]{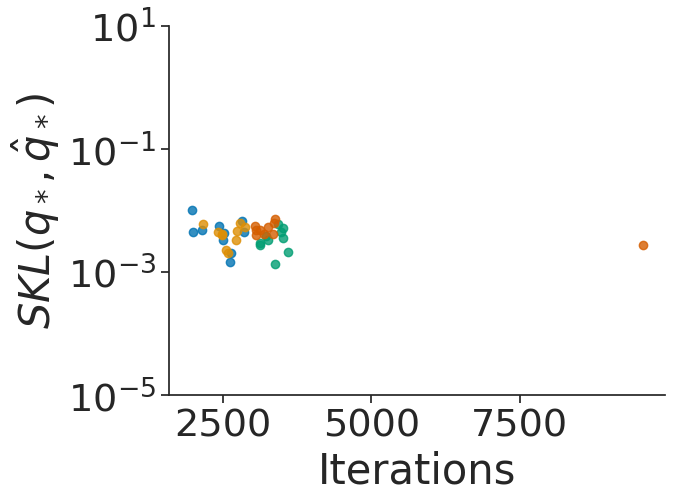}
\caption{$\varepsilon_{0}$} 
\end{subfigure} 
\begin{subfigure}[t]{.24\textwidth}
\centering
\includegraphics[width=\textwidth,height=0.198\textheight]{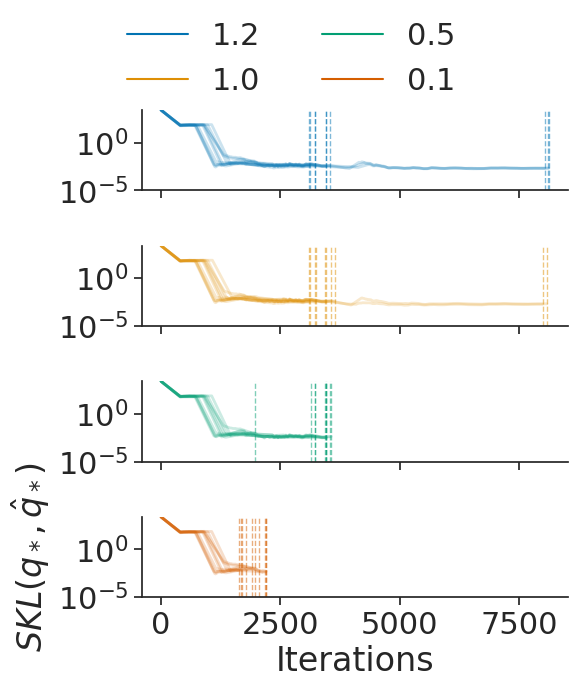} \\
\includegraphics[width=\textwidth,height=0.128\textheight]{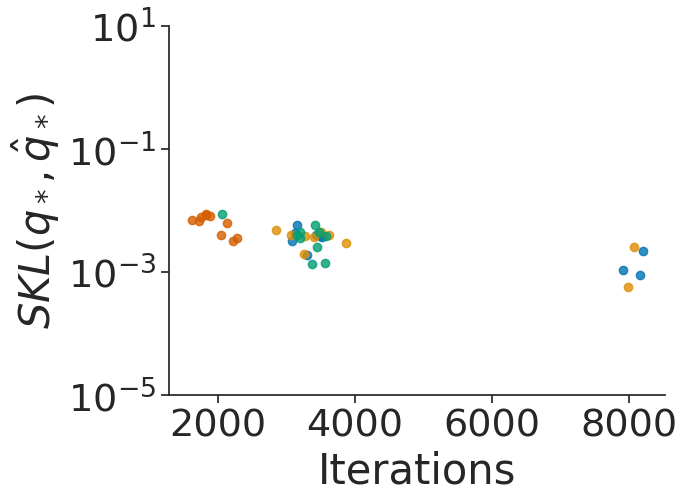}
\caption{$\tau$} 
\end{subfigure}
\caption{
Robustness to tuning parameters (a) initial learning rate $\learningrate_{0}$, (b) minimum window size $W_{\min}$, (c) initial iterate average relative error threshold $\varepsilon_{0}$, and (d)  inefficiency threshold $\tau$ using \textit{dogs}  dataset from \texttt{posteriordb} package.
\textbf{(top)} Iterations versus symmetrized KL divergence between iterate average and optimal variational approximation. 
The transparent lines represent repeated experiments and the vertical lines indicate the termination rule trigger points.
\textbf{(bottom)} Iterations versus symmetrized KL divergence between iterate average and optimal variational approximation
at the termination rule trigger point.}
\label{fig:Robustness-tuning-parameters-dogs-posteriordb}
\end{center}
\end{figure}

\section{Discussion}

As we have shown through both theory and experiments, RABVI, our stochastic optimization framework for black-box variational inference,
delivers a number of benefits compared to existing approaches:
\bitems
\item The user only needs to, at most, adjust a small number of tuning parameters which empowers the user to intuitively control and trade off
computational cost and accuracy. Moreover, RABVI is robust, both in terms of accuracy and computational cost, 
to small changes in all tuning parameters. 
\item Our framework can easily incorporate different stochastic optimization methods such as adaptive versions, natural gradient descent,
and stochastic quasi-Newton methods.  
In practice, we found that the averaged versions of RMSProp and Adam we propose perform particularly well.
However the performance of RABVI will benefit from future innovations in stochastic optimization methodology.
\item RABVI allows for any choice of tractable variational family and stochastic gradient estimator. 
For example, in many cases we find accuracy improves when using the full-rank Gaussian variational family rather than the mean-field one. 
\eitems
Our empirical results also highlight some of the limitations of BBVI, which sometimes is less accurate than dynamic HMC 
when given equal computational budgets.
However, BBVI can be further sped up using, for example, data subsampling when the dataset size is large (which was not the case for the 
\texttt{posteriodb} datasets from our experiments).

\section*{Acknowledgments}
M.~Welandawe and J.~H.~Huggins were supported by the National Institute of General Medical Sciences of the National Institutes of Health under grant number R01GM144963 as part of the Joint NSF/NIGMS Mathematical Biology Program. 
The content is solely the responsibility of the authors and does not necessarily represent the official views of the National Institutes of Health. A.~Vehtari was supported by Research Council of Finland Flagship programme: Finnish Center for Artificial Intelligence, FCAI, and M.~R.~Andersen was supported by Innovation Fund Denmark (grant number 8057-00036A).

\bibliographystyle{imsart-nameyear}
\bibliography{library,deep-models,robust-vi,extra}

\begin{thebibliography}{87}
% BibTex style file: imsart-nameyear.bst, 2017-11-03
% Default style options (sort=1,type=nameyear).
% Used options (sort=1,type=nameyear).

\bibitem[\protect\citeauthoryear{Agrawal, Sheldon and
  Domke}{2020}]{Agrawal:2020:BBVI}
\begin{binproceedings}[author]
\bauthor{\bsnm{Agrawal},~\bfnm{Abhinav}\binits{A.}},
  \bauthor{\bsnm{Sheldon},~\bfnm{Daniel}\binits{D.}} \AND
  \bauthor{\bsnm{Domke},~\bfnm{Justin}\binits{J.}}
(\byear{2020}).
\btitle{{Advances in Black-Box VI: Normalizing Flows, Importance Weighting, and
  Optimization.}}
In \bbooktitle{Advances in Neural Information Processing Systems}.
\end{binproceedings}
\endbibitem

\bibitem[\protect\citeauthoryear{Amari}{2016}]{amari:2016}
\begin{bbook}[author]
\bauthor{\bsnm{Amari},~\bfnm{Shun-ichi}\binits{S.-i.}}
(\byear{2016}).
\btitle{Information geometry and its applications}
\bvolume{194}.
\bpublisher{Springer}.
\end{bbook}
\endbibitem

\bibitem[\protect\citeauthoryear{Bach and Moulines}{2013}]{Bach:2013}
\begin{binproceedings}[author]
\bauthor{\bsnm{Bach},~\bfnm{F}\binits{F.}} \AND
  \bauthor{\bsnm{Moulines},~\bfnm{E}\binits{E.}}
(\byear{2013}).
\btitle{{Non-strongly-convex smooth stochastic approximation with convergence
  rate $O(1/n)$}}.
In \bbooktitle{Advances in Neural Information Processing Systems}
\bpages{1--9}.
\end{binproceedings}
\endbibitem

\bibitem[\protect\citeauthoryear{Bamler et~al.}{2017}]{Bamler:2017}
\begin{barticle}[author]
\bauthor{\bsnm{Bamler},~\bfnm{Robert}\binits{R.}},
  \bauthor{\bsnm{Zhang},~\bfnm{Cheng}\binits{C.}},
  \bauthor{\bsnm{Opper},~\bfnm{Manfred}\binits{M.}} \AND
  \bauthor{\bsnm{Mandt},~\bfnm{Stephan}\binits{S.}}
(\byear{2017}).
\btitle{Perturbative black box variational inference}.
\bjournal{Advances in Neural Information Processing Systems}
\bvolume{30}.
\end{barticle}
\endbibitem

\bibitem[\protect\citeauthoryear{Bhatia et~al.}{2022}]{Bhatia:2022}
\begin{barticle}[author]
\bauthor{\bsnm{Bhatia},~\bfnm{Kush}\binits{K.}},
  \bauthor{\bsnm{Kuang},~\bfnm{Nikki~Lijing}\binits{N.~L.}},
  \bauthor{\bsnm{Ma},~\bfnm{Yi-An}\binits{Y.-A.}} \AND
  \bauthor{\bsnm{Wang},~\bfnm{Yixin}\binits{Y.}}
(\byear{2022}).
\btitle{Statistical and Computational Trade-offs in Variational Inference: A
  Case Study in Inferential Model Selection}.
\bjournal{arXiv preprint arXiv:2207.11208}.
\end{barticle}
\endbibitem

\bibitem[\protect\citeauthoryear{Bishop}{2006}]{Bishop:2006}
\begin{bbook}[author]
\bauthor{\bsnm{Bishop},~\bfnm{C.~M.}\binits{C.~M.}}
(\byear{2006}).
\btitle{{Pattern Recognition and Machine Learning}}.
\bpublisher{Springer}.
\end{bbook}
\endbibitem

\bibitem[\protect\citeauthoryear{Blei, Kucukelbir and
  McAuliffe}{2017}]{Blei:2017}
\begin{barticle}[author]
\bauthor{\bsnm{Blei},~\bfnm{D.~M.}\binits{D.~M.}},
  \bauthor{\bsnm{Kucukelbir},~\bfnm{Alp}\binits{A.}} \AND
  \bauthor{\bsnm{McAuliffe},~\bfnm{Jon~D}\binits{J.~D.}}
(\byear{2017}).
\btitle{{Variational Inference: A Review for Statisticians}}.
\bjournal{Journal of the American Statistical Association}
\bvolume{112}
\bpages{859--877}.
\end{barticle}
\endbibitem

\bibitem[\protect\citeauthoryear{Bolley and Villani}{2005}]{Bolley:2005}
\begin{barticle}[author]
\bauthor{\bsnm{Bolley},~\bfnm{Fran{\c c}ois}\binits{F.}} \AND
  \bauthor{\bsnm{Villani},~\bfnm{C}\binits{C.}}
(\byear{2005}).
\btitle{{Weighted Csisz{\'a}r-Kullback-Pinsker inequalities and applications to
  transportation inequalities}}.
\bjournal{Annales de la faculte des sciences de Toulouse}
\bvolume{13}
\bpages{331--352}.
\end{barticle}
\endbibitem

\bibitem[\protect\citeauthoryear{Boustati et~al.}{2020}]{Boustati:2020}
\begin{binproceedings}[author]
\bauthor{\bsnm{Boustati},~\bfnm{Ayman}\binits{A.}},
  \bauthor{\bsnm{Vakili},~\bfnm{Sattar}\binits{S.}},
  \bauthor{\bsnm{Hensman},~\bfnm{James}\binits{J.}} \AND
  \bauthor{\bsnm{John},~\bfnm{ST}\binits{S.}}
(\byear{2020}).
\btitle{Amortized variance reduction for doubly stochastic objective}.
In \bbooktitle{Conference on Uncertainty in Artificial Intelligence}
\bpages{61--70}.
\bpublisher{PMLR}.
\end{binproceedings}
\endbibitem

\bibitem[\protect\citeauthoryear{Bubeck}{2015}]{Bubeck:2015:CO}
\begin{barticle}[author]
\bauthor{\bsnm{Bubeck},~\bfnm{Sébastien}\binits{S.}}
(\byear{2015}).
\btitle{{Convex Optimization: Algorithms and Complexity}}.
\bjournal{Foundations and Trends® in Machine Learning}
\bvolume{8}
\bpages{231 -- 357}.
\bdoi{10.1561/2200000050}
\end{barticle}
\endbibitem

\bibitem[\protect\citeauthoryear{Bui, Yan and Turner}{2017}]{Bui:2017a}
\begin{barticle}[author]
\bauthor{\bsnm{Bui},~\bfnm{Thang~D}\binits{T.~D.}},
  \bauthor{\bsnm{Yan},~\bfnm{Josiah}\binits{J.}} \AND
  \bauthor{\bsnm{Turner},~\bfnm{Richard~E}\binits{R.~E.}}
(\byear{2017}).
\btitle{{A Unifying Framework for Gaussian Process Pseudo-Point Approximations
  using Power Expectation Propagation}}.
\bjournal{Journal of Machine Learning Research}
\bvolume{18}
\bpages{1--72}.
\end{barticle}
\endbibitem

\bibitem[\protect\citeauthoryear{Burda, Grosse and
  Salakhutdinov}{2016}]{Burda:2016:IWAE}
\begin{binproceedings}[author]
\bauthor{\bsnm{Burda},~\bfnm{Y}\binits{Y.}},
  \bauthor{\bsnm{Grosse},~\bfnm{Roger~B}\binits{R.~B.}} \AND
  \bauthor{\bsnm{Salakhutdinov},~\bfnm{R}\binits{R.}}
(\byear{2016}).
\btitle{{Importance Weighted Autoencoders}}.
In \bbooktitle{International Conference on Learning Representations}.
\end{binproceedings}
\endbibitem

\bibitem[\protect\citeauthoryear{Chee and Li}{2020}]{Chee:2020}
\begin{binproceedings}[author]
\bauthor{\bsnm{Chee},~\bfnm{Jerry}\binits{J.}} \AND
  \bauthor{\bsnm{Li},~\bfnm{Ping}\binits{P.}}
(\byear{2020}).
\btitle{{Understanding and Detecting Convergence for Stochastic Gradient
  Descent with Momentum}}.
In \bbooktitle{2020 IEEE International Conference on Big Data (Big Data)}
\bpages{133--140}.
\bpublisher{IEEE}.
\end{binproceedings}
\endbibitem

\bibitem[\protect\citeauthoryear{Chee and Toulis}{2018}]{Chee:2018}
\begin{binproceedings}[author]
\bauthor{\bsnm{Chee},~\bfnm{Jerry}\binits{J.}} \AND
  \bauthor{\bsnm{Toulis},~\bfnm{Panos}\binits{P.}}
(\byear{2018}).
\btitle{{Convergence diagnostics for stochastic gradient descent with constant
  learning rate}}.
In \bbooktitle{International Conference on Artificial Intelligence and
  Statistics}.
\end{binproceedings}
\endbibitem

\bibitem[\protect\citeauthoryear{Chen et~al.}{2017}]{chen:2017}
\begin{bmisc}[author]
\bauthor{\bsnm{Chen},~\bfnm{Jianmin}\binits{J.}},
  \bauthor{\bsnm{Pan},~\bfnm{Xinghao}\binits{X.}},
  \bauthor{\bsnm{Monga},~\bfnm{Rajat}\binits{R.}},
  \bauthor{\bsnm{Bengio},~\bfnm{Samy}\binits{S.}} \AND
  \bauthor{\bsnm{Jozefowicz},~\bfnm{Rafal}\binits{R.}}
(\byear{2017}).
\btitle{Revisiting Distributed Synchronous {SGD}}.
\end{bmisc}
\endbibitem

\bibitem[\protect\citeauthoryear{Chen et~al.}{2019}]{chen2019stochastic}
\begin{barticle}[author]
\bauthor{\bsnm{Chen},~\bfnm{Huiming}\binits{H.}},
  \bauthor{\bsnm{Wu},~\bfnm{Ho-Chun}\binits{H.-C.}},
  \bauthor{\bsnm{Chan},~\bfnm{Shing-Chow}\binits{S.-C.}} \AND
  \bauthor{\bsnm{Lam},~\bfnm{Wong-Hing}\binits{W.-H.}}
(\byear{2019}).
\btitle{A stochastic quasi-Newton method for large-scale nonconvex optimization
  with applications}.
\bjournal{IEEE transactions on neural networks and learning systems}
\bvolume{31}
\bpages{4776--4790}.
\end{barticle}
\endbibitem

\bibitem[\protect\citeauthoryear{Cornebise, Moulines and
  Olsson}{2008}]{Cornebise:2008}
\begin{barticle}[author]
\bauthor{\bsnm{Cornebise},~\bfnm{Julien}\binits{J.}},
  \bauthor{\bsnm{Moulines},~\bfnm{E}\binits{E.}} \AND
  \bauthor{\bsnm{Olsson},~\bfnm{Jimmy}\binits{J.}}
(\byear{2008}).
\btitle{{Adaptive methods for sequential importance sampling with~application
  to state space models}}.
\bjournal{Statistics and Computing}
\bvolume{18}
\bpages{461--480}.
\end{barticle}
\endbibitem

\bibitem[\protect\citeauthoryear{Dhaka et~al.}{2020}]{Dhaka:2020}
\begin{binproceedings}[author]
\bauthor{\bsnm{Dhaka},~\bfnm{Akash~Kumar}\binits{A.~K.}},
  \bauthor{\bsnm{Catalina},~\bfnm{Alejandro}\binits{A.}},
  \bauthor{\bsnm{Andersen},~\bfnm{Michael~Riis}\binits{M.~R.}},
  \bauthor{\bsnm{Magnusson},~\bfnm{M{\aa}ns}\binits{M.}},
  \bauthor{\bsnm{Huggins},~\bfnm{Jonathan~H}\binits{J.~H.}} \AND
  \bauthor{\bsnm{Vehtari},~\bfnm{Aki}\binits{A.}}
(\byear{2020}).
\btitle{{Robust, Accurate Stochastic Optimization for Variational Inference}}.
In \bbooktitle{Advances in Neural Information Processing Systems}
\bvolume{33}
\bpages{10961–10973}.
\end{binproceedings}
\endbibitem

\bibitem[\protect\citeauthoryear{Dhaka et~al.}{2021}]{Dhaka:2021}
\begin{barticle}[author]
\bauthor{\bsnm{Dhaka},~\bfnm{Akash~Kumar}\binits{A.~K.}},
  \bauthor{\bsnm{Catalina},~\bfnm{Alejandro}\binits{A.}},
  \bauthor{\bsnm{Welandawe},~\bfnm{Manushi}\binits{M.}},
  \bauthor{\bsnm{Andersen},~\bfnm{Michael~R}\binits{M.~R.}},
  \bauthor{\bsnm{Huggins},~\bfnm{Jonathan}\binits{J.}} \AND
  \bauthor{\bsnm{Vehtari},~\bfnm{Aki}\binits{A.}}
(\byear{2021}).
\btitle{Challenges and Opportunities in High Dimensional Variational
  Inference}.
\bjournal{Advances in Neural Information Processing Systems}
\bvolume{34}.
\end{barticle}
\endbibitem

\bibitem[\protect\citeauthoryear{Dieng et~al.}{2017}]{Dieng:2017}
\begin{binproceedings}[author]
\bauthor{\bsnm{Dieng},~\bfnm{Adji~B}\binits{A.~B.}},
  \bauthor{\bsnm{Tran},~\bfnm{Dustin}\binits{D.}},
  \bauthor{\bsnm{Ranganath},~\bfnm{Rajesh}\binits{R.}},
  \bauthor{\bsnm{Paisley},~\bfnm{J.}\binits{J.}} \AND
  \bauthor{\bsnm{Blei},~\bfnm{D.~M.}\binits{D.~M.}}
(\byear{2017}).
\btitle{{Variational Inference via $\chi$ Upper Bound Minimization}}.
In \bbooktitle{Advances in Neural Information Processing Systems}.
\end{binproceedings}
\endbibitem

\bibitem[\protect\citeauthoryear{Dieuleveut, Durmus and
  Bach}{2020}]{Dieuleveut:2020:SGD-MC}
\begin{barticle}[author]
\bauthor{\bsnm{Dieuleveut},~\bfnm{Aymeric}\binits{A.}},
  \bauthor{\bsnm{Durmus},~\bfnm{Alain}\binits{A.}} \AND
  \bauthor{\bsnm{Bach},~\bfnm{F}\binits{F.}}
(\byear{2020}).
\btitle{{Bridging the Gap between Constant Step Size Stochastic Gradient
  Descent and Markov Chains}}.
\bjournal{The Annals of Statistics}
\bvolume{48}
\bpages{1348--1382}.
\end{barticle}
\endbibitem

\bibitem[\protect\citeauthoryear{Dinh, Sohl-Dickstein and
  Bengio}{2017}]{dinh:2017}
\begin{binproceedings}[author]
\bauthor{\bsnm{Dinh},~\bfnm{Laurent}\binits{L.}},
  \bauthor{\bsnm{Sohl-Dickstein},~\bfnm{Jascha}\binits{J.}} \AND
  \bauthor{\bsnm{Bengio},~\bfnm{Samy}\binits{S.}}
(\byear{2017}).
\btitle{{Density estimation using Real NVP}}.
In \bbooktitle{International Conference on Learning Representations}.
\end{binproceedings}
\endbibitem

\bibitem[\protect\citeauthoryear{Domke}{2019}]{Domke:2019}
\begin{barticle}[author]
\bauthor{\bsnm{Domke},~\bfnm{Justin}\binits{J.}}
(\byear{2019}).
\btitle{Provable gradient variance guarantees for black-box variational
  inference}.
\bjournal{Advances in Neural Information Processing Systems}
\bvolume{32}.
\end{barticle}
\endbibitem

\bibitem[\protect\citeauthoryear{Domke, Garrigos and Gower}{2023}]{Domke:2023}
\begin{barticle}[author]
\bauthor{\bsnm{Domke},~\bfnm{Justin}\binits{J.}},
  \bauthor{\bsnm{Garrigos},~\bfnm{Guillaume}\binits{G.}} \AND
  \bauthor{\bsnm{Gower},~\bfnm{Robert}\binits{R.}}
(\byear{2023}).
\btitle{Provable convergence guarantees for black-box variational inference}.
\bjournal{arXiv preprint arXiv:2306.03638}.
\end{barticle}
\endbibitem

\bibitem[\protect\citeauthoryear{Duchi, Hazan and
  Singer}{2011}]{Duchi:2011:adagrad}
\begin{barticle}[author]
\bauthor{\bsnm{Duchi},~\bfnm{John}\binits{J.}},
  \bauthor{\bsnm{Hazan},~\bfnm{Elad}\binits{E.}} \AND
  \bauthor{\bsnm{Singer},~\bfnm{Yoram}\binits{Y.}}
(\byear{2011}).
\btitle{{Adaptive Subgradient Methods for Online Learning and Stochastic
  Optimization}}.
\bjournal{Journal of Machine Learning Research}
\bvolume{12}
\bpages{2121--2159}.
\end{barticle}
\endbibitem

\bibitem[\protect\citeauthoryear{Durmus, Majewski and
  Miasojedow}{2019}]{Durmus:2019}
\begin{barticle}[author]
\bauthor{\bsnm{Durmus},~\bfnm{Alain}\binits{A.}},
  \bauthor{\bsnm{Majewski},~\bfnm{Szymon}\binits{S.}} \AND
  \bauthor{\bsnm{Miasojedow},~\bfnm{Blazej}\binits{B.}}
(\byear{2019}).
\btitle{{Analysis of Langevin Monte Carlo via convex optimization}}.
\bjournal{Journal of Machine Learning Research}
\bvolume{20}
\bpages{1--46}.
\end{barticle}
\endbibitem

\bibitem[\protect\citeauthoryear{Durmus and Moulines}{2019}]{Durmus:2019b}
\begin{barticle}[author]
\bauthor{\bsnm{Durmus},~\bfnm{Alain}\binits{A.}} \AND
  \bauthor{\bsnm{Moulines},~\bfnm{E}\binits{E.}}
(\byear{2019}).
\btitle{{High-dimensional Bayesian inference via the unadjusted Langevin
  algorithm}}.
\bjournal{Bernoulli}
\bvolume{25}
\bpages{2854--2882}.
\end{barticle}
\endbibitem

\bibitem[\protect\citeauthoryear{Eberle and Majka}{2019}]{Eberle:2018}
\begin{barticle}[author]
\bauthor{\bsnm{Eberle},~\bfnm{Andreas}\binits{A.}} \AND
  \bauthor{\bsnm{Majka},~\bfnm{Mateusz~B}\binits{M.~B.}}
(\byear{2019}).
\btitle{{Quantitative contraction rates for Markov chains on general state
  spaces}}.
\bjournal{Electronic Journal of Probability}
\bvolume{24}
\bpages{1--36}.
\end{barticle}
\endbibitem

\bibitem[\protect\citeauthoryear{Gal and Ghahramani}{2016}]{Gal:2016dropout}
\begin{binproceedings}[author]
\bauthor{\bsnm{Gal},~\bfnm{Yarin}\binits{Y.}} \AND
  \bauthor{\bsnm{Ghahramani},~\bfnm{Z.}\binits{Z.}}
(\byear{2016}).
\btitle{{Dropout as a Bayesian Approximation - Representing Model Uncertainty
  in Deep Learning.}}
In \bbooktitle{International Conference on Machine Learning}.
\end{binproceedings}
\endbibitem

\bibitem[\protect\citeauthoryear{Geffner and Domke}{2018}]{Geffner:2018}
\begin{barticle}[author]
\bauthor{\bsnm{Geffner},~\bfnm{Tomas}\binits{T.}} \AND
  \bauthor{\bsnm{Domke},~\bfnm{Justin}\binits{J.}}
(\byear{2018}).
\btitle{Using large ensembles of control variates for variational inference}.
\bjournal{Advances in Neural Information Processing Systems}
\bvolume{31}.
\end{barticle}
\endbibitem

\bibitem[\protect\citeauthoryear{Geffner and Domke}{2020}]{Geffner:2020}
\begin{barticle}[author]
\bauthor{\bsnm{Geffner},~\bfnm{Tomas}\binits{T.}} \AND
  \bauthor{\bsnm{Domke},~\bfnm{Justin}\binits{J.}}
(\byear{2020}).
\btitle{Approximation based variance reduction for reparameterization
  gradients}.
\bjournal{Advances in Neural Information Processing Systems}
\bvolume{33}
\bpages{2397--2407}.
\end{barticle}
\endbibitem

\bibitem[\protect\citeauthoryear{Gelman and Rubin}{1992}]{Gelman:1992}
\begin{barticle}[author]
\bauthor{\bsnm{Gelman},~\bfnm{Andrew}\binits{A.}} \AND
  \bauthor{\bsnm{Rubin},~\bfnm{Donald~B}\binits{D.~B.}}
(\byear{1992}).
\btitle{{Inference from iterative simulation using multiple sequences}}.
\bjournal{Statistical Science}
\bvolume{7}
\bpages{457--511}.
\end{barticle}
\endbibitem

\bibitem[\protect\citeauthoryear{Gelman et~al.}{2013}]{Gelman:2013}
\begin{bbook}[author]
\bauthor{\bsnm{Gelman},~\bfnm{Andrew}\binits{A.}},
  \bauthor{\bsnm{Carlin},~\bfnm{John}\binits{J.}},
  \bauthor{\bsnm{Stern},~\bfnm{Hal}\binits{H.}},
  \bauthor{\bsnm{Dunson},~\bfnm{David~B}\binits{D.~B.}},
  \bauthor{\bsnm{Vehtari},~\bfnm{Aki}\binits{A.}} \AND
  \bauthor{\bsnm{Rubin},~\bfnm{Donald~B}\binits{D.~B.}}
(\byear{2013}).
\btitle{{Bayesian Data Analysis}},
\bedition{Third} ed.
\bpublisher{Chapman and Hall/CRC}.
\end{bbook}
\endbibitem

\bibitem[\protect\citeauthoryear{Geyer}{1992}]{Geyer:1992}
\begin{barticle}[author]
\bauthor{\bsnm{Geyer},~\bfnm{C~J}\binits{C.~J.}}
(\byear{1992}).
\btitle{{Practical Markov Chain Monte Carlo}}.
\bjournal{Statistical Science}
\bvolume{7}
\bpages{473--483}.
\end{barticle}
\endbibitem

\bibitem[\protect\citeauthoryear{Gitman et~al.}{2019}]{Gitman:2019}
\begin{binproceedings}[author]
\bauthor{\bsnm{Gitman},~\bfnm{Igor}\binits{I.}},
  \bauthor{\bsnm{Lang},~\bfnm{Hunter}\binits{H.}},
  \bauthor{\bsnm{Zhang},~\bfnm{Pengchuan}\binits{P.}} \AND
  \bauthor{\bsnm{Xiao},~\bfnm{Lin}\binits{L.}}
(\byear{2019}).
\btitle{{Understanding the Role of Momentum in Stochastic Gradient Methods}}.
In \bbooktitle{Advances in Neural Information Processing Systems}
\bpages{9633--9643}.
\end{binproceedings}
\endbibitem

\bibitem[\protect\citeauthoryear{Gretton et~al.}{2006}]{gretton:2006}
\begin{barticle}[author]
\bauthor{\bsnm{Gretton},~\bfnm{Arthur}\binits{A.}},
  \bauthor{\bsnm{Borgwardt},~\bfnm{Karsten}\binits{K.}},
  \bauthor{\bsnm{Rasch},~\bfnm{Malte}\binits{M.}},
  \bauthor{\bsnm{Sch{\"o}lkopf},~\bfnm{Bernhard}\binits{B.}} \AND
  \bauthor{\bsnm{Smola},~\bfnm{Alex}\binits{A.}}
(\byear{2006}).
\btitle{A kernel method for the two-sample-problem}.
\bjournal{Advances in neural information processing systems}
\bvolume{19}.
\end{barticle}
\endbibitem

\bibitem[\protect\citeauthoryear{Hern{\'a}ndez-Lobato
  et~al.}{2016}]{HernandezLobato:2016}
\begin{binproceedings}[author]
\bauthor{\bsnm{Hern{\'a}ndez-Lobato},~\bfnm{Jos{\'e}~Miguel}\binits{J.~M.}},
  \bauthor{\bsnm{Li},~\bfnm{Yingzhen}\binits{Y.}},
  \bauthor{\bsnm{Rowland},~\bfnm{Mark}\binits{M.}},
  \bauthor{\bsnm{Bui},~\bfnm{Thang~D}\binits{T.~D.}},
  \bauthor{\bsnm{Hern{\'a}ndez-Lobato},~\bfnm{Daniel}\binits{D.}} \AND
  \bauthor{\bsnm{Turner},~\bfnm{Richard~E}\binits{R.~E.}}
(\byear{2016}).
\btitle{{Black-Box Alpha Divergence Minimization}}.
In \bbooktitle{International Conference on Machine Learning}.
\end{binproceedings}
\endbibitem

\bibitem[\protect\citeauthoryear{Hinton and Tieleman}{2012}]{Hinton:2012}
\begin{binproceedings}[author]
\bauthor{\bsnm{Hinton},~\bfnm{G.~E.}\binits{G.~E.}} \AND
  \bauthor{\bsnm{Tieleman},~\bfnm{Tijmen}\binits{T.}}
(\byear{2012}).
\btitle{{Lecture 6.5 -- Rmsprop: Divide the gradient by a running average of
  its recent magnitude}}.
In \bbooktitle{Coursera: Neural networks for machine learning}.
\end{binproceedings}
\endbibitem

\bibitem[\protect\citeauthoryear{Huggins et~al.}{2020}]{Huggins:2020:VI}
\begin{binproceedings}[author]
\bauthor{\bsnm{Huggins},~\bfnm{Jonathan~H}\binits{J.~H.}},
  \bauthor{\bsnm{Kasprzak},~\bfnm{Mikolaj}\binits{M.}},
  \bauthor{\bsnm{Campbell},~\bfnm{Trevor}\binits{T.}} \AND
  \bauthor{\bsnm{Broderick},~\bfnm{T.}\binits{T.}}
(\byear{2020}).
\btitle{{Validated Variational Inference via Practical Posterior Error
  Bounds}}.
In \bbooktitle{International Conference on Artificial Intelligence and
  Statistics}.
\end{binproceedings}
\endbibitem

\bibitem[\protect\citeauthoryear{Johnson et~al.}{2016}]{Johnson:2016}
\begin{binproceedings}[author]
\bauthor{\bsnm{Johnson},~\bfnm{Matthew~J}\binits{M.~J.}},
  \bauthor{\bsnm{Duvenaud},~\bfnm{D}\binits{D.}},
  \bauthor{\bsnm{Wiltschko},~\bfnm{Alexander~B}\binits{A.~B.}},
  \bauthor{\bsnm{Datta},~\bfnm{Sandeep~R}\binits{S.~R.}} \AND
  \bauthor{\bsnm{Adams},~\bfnm{R~P}\binits{R.~P.}}
(\byear{2016}).
\btitle{{Composing graphical models with neural networks for structured
  representations and fast inference}}.
In \bbooktitle{Advances in Neural Information Processing Systems}.
\end{binproceedings}
\endbibitem

\bibitem[\protect\citeauthoryear{Joulin and Ollivier}{2010}]{Joulin:2010}
\begin{barticle}[author]
\bauthor{\bsnm{Joulin},~\bfnm{Ald{\'e}ric}\binits{A.}} \AND
  \bauthor{\bsnm{Ollivier},~\bfnm{Yann}\binits{Y.}}
(\byear{2010}).
\btitle{{Curvature, concentration and error estimates for Markov chain Monte
  Carlo}}.
\bjournal{The Annals of Probability}
\bvolume{38}
\bpages{2418--2442}.
\end{barticle}
\endbibitem

\bibitem[\protect\citeauthoryear{Khan and Lin}{2017}]{Khan:2017:CVI}
\begin{binproceedings}[author]
\bauthor{\bsnm{Khan},~\bfnm{Mohammad~Emtiyaz}\binits{M.~E.}} \AND
  \bauthor{\bsnm{Lin},~\bfnm{Wu}\binits{W.}}
(\byear{2017}).
\btitle{{Conjugate-Computation Variational Inference : Converting Variational
  Inference in Non-Conjugate Models to Inferences in Conjugate Models}}.
In \bbooktitle{International Conference on Artificial Intelligence and
  Statistics}.
\end{binproceedings}
\endbibitem

\bibitem[\protect\citeauthoryear{Kim et~al.}{2023}]{Kim:2023}
\begin{binproceedings}[author]
\bauthor{\bsnm{Kim},~\bfnm{Kyurae}\binits{K.}},
  \bauthor{\bsnm{Oh},~\bfnm{Jisu}\binits{J.}},
  \bauthor{\bsnm{Wu},~\bfnm{Kaiwen}\binits{K.}},
  \bauthor{\bsnm{Ma},~\bfnm{Yian}\binits{Y.}} \AND
  \bauthor{\bsnm{Gardner},~\bfnm{Jacob~R.}\binits{J.~R.}}
(\byear{2023}).
\btitle{On the Convergence of Black-Box Variational Inference}.
In \bbooktitle{Thirty-seventh Conference on Neural Information Processing
  Systems}.
\end{binproceedings}
\endbibitem

\bibitem[\protect\citeauthoryear{Kingma and Ba}{2015}]{Kingma:2015:adam}
\begin{binproceedings}[author]
\bauthor{\bsnm{Kingma},~\bfnm{Diederik~P}\binits{D.~P.}} \AND
  \bauthor{\bsnm{Ba},~\bfnm{Jimmy}\binits{J.}}
(\byear{2015}).
\btitle{{Adam: A Method for Stochastic Optimization}}.
In \bbooktitle{International Conference on Learning Representations}.
\end{binproceedings}
\endbibitem

\bibitem[\protect\citeauthoryear{Kingma and Welling}{2014}]{Kingma:2014}
\begin{binproceedings}[author]
\bauthor{\bsnm{Kingma},~\bfnm{Diederik~P}\binits{D.~P.}} \AND
  \bauthor{\bsnm{Welling},~\bfnm{Max}\binits{M.}}
(\byear{2014}).
\btitle{{Auto-Encoding Variational Bayes}}.
In \bbooktitle{International Conference on Learning Representations}.
\end{binproceedings}
\endbibitem

\bibitem[\protect\citeauthoryear{Kucukelbir et~al.}{2015}]{Kucukelbir:2015}
\begin{binproceedings}[author]
\bauthor{\bsnm{Kucukelbir},~\bfnm{Alp}\binits{A.}},
  \bauthor{\bsnm{Ranganath},~\bfnm{Rajesh}\binits{R.}},
  \bauthor{\bsnm{Gelman},~\bfnm{Andrew}\binits{A.}} \AND
  \bauthor{\bsnm{Blei},~\bfnm{D.~M.}\binits{D.~M.}}
(\byear{2015}).
\btitle{{Automatic Variational Inference in Stan}}.
In \bbooktitle{Advances in Neural Information Processing Systems}.
\end{binproceedings}
\endbibitem

\bibitem[\protect\citeauthoryear{Lang, Zhang and Xiao}{2019}]{Lang:2019}
\begin{barticle}[author]
\bauthor{\bsnm{Lang},~\bfnm{Hunter}\binits{H.}},
  \bauthor{\bsnm{Zhang},~\bfnm{Pengchuan}\binits{P.}} \AND
  \bauthor{\bsnm{Xiao},~\bfnm{Lin}\binits{L.}}
(\byear{2019}).
\btitle{{Using Statistics to Automate Stochastic Optimization}}.
\bjournal{Advances in neural information processing systems}
\bvolume{32}.
\end{barticle}
\endbibitem

\bibitem[\protect\citeauthoryear{Li and Turner}{2016}]{Li:2016}
\begin{binproceedings}[author]
\bauthor{\bsnm{Li},~\bfnm{Yingzhen}\binits{Y.}} \AND
  \bauthor{\bsnm{Turner},~\bfnm{Richard~E}\binits{R.~E.}}
(\byear{2016}).
\btitle{{R{\'e}nyi Divergence Variational Inference}}.
In \bbooktitle{Advances in Neural Information Processing Systems}
\bpages{1073--1081}.
\end{binproceedings}
\endbibitem

\bibitem[\protect\citeauthoryear{Liu and Owen}{2021}]{liu2021quasi}
\begin{barticle}[author]
\bauthor{\bsnm{Liu},~\bfnm{Sifan}\binits{S.}} \AND
  \bauthor{\bsnm{Owen},~\bfnm{Art~B}\binits{A.~B.}}
(\byear{2021}).
\btitle{Quasi-Monte Carlo Quasi-Newton in Variational Bayes}.
\bjournal{Journal of Machine Learning Research}
\bvolume{22}
\bpages{1--23}.
\end{barticle}
\endbibitem

\bibitem[\protect\citeauthoryear{Loshchilov and Hutter}{2017}]{Loshchilov2017}
\begin{binproceedings}[author]
\bauthor{\bsnm{Loshchilov},~\bfnm{Ilya}\binits{I.}} \AND
  \bauthor{\bsnm{Hutter},~\bfnm{Frank}\binits{F.}}
(\byear{2017}).
\btitle{{SGDR}: Stochastic Gradient Descent with Warm Restarts}.
In \bbooktitle{International Conference on Learning Representations}.
\end{binproceedings}
\endbibitem

\bibitem[\protect\citeauthoryear{Maddox et~al.}{2019}]{Maddox:2019}
\begin{binproceedings}[author]
\bauthor{\bsnm{Maddox},~\bfnm{Wesley}\binits{W.}},
  \bauthor{\bsnm{Garipov},~\bfnm{Timur}\binits{T.}},
  \bauthor{\bsnm{Izmailov},~\bfnm{Pavel}\binits{P.}},
  \bauthor{\bsnm{Vetrov},~\bfnm{Dmitry}\binits{D.}} \AND
  \bauthor{\bsnm{Wilson},~\bfnm{Andrew~Gordon}\binits{A.~G.}}
(\byear{2019}).
\btitle{{A Simple Baseline for Bayesian Uncertainty in Deep Learning}}.
In \bbooktitle{Advances in Neural Information Processing Systems}.
\end{binproceedings}
\endbibitem

\bibitem[\protect\citeauthoryear{Madras and Sezer}{2010}]{Madras:2010}
\begin{barticle}[author]
\bauthor{\bsnm{Madras},~\bfnm{Neal}\binits{N.}} \AND
  \bauthor{\bsnm{Sezer},~\bfnm{Deniz}\binits{D.}}
(\byear{2010}).
\btitle{{Quantitative bounds for Markov chain convergence: Wasserstein and
  total variation distances}}.
\bjournal{Bernoulli}
\bvolume{16}
\bpages{882--908}.
\end{barticle}
\endbibitem

\bibitem[\protect\citeauthoryear{Miller et~al.}{2017}]{Miller:2017b}
\begin{barticle}[author]
\bauthor{\bsnm{Miller},~\bfnm{Andrew}\binits{A.}},
  \bauthor{\bsnm{Foti},~\bfnm{Nick}\binits{N.}},
  \bauthor{\bsnm{D'Amour},~\bfnm{Alexander}\binits{A.}} \AND
  \bauthor{\bsnm{Adams},~\bfnm{Ryan~P}\binits{R.~P.}}
(\byear{2017}).
\btitle{Reducing reparameterization gradient variance}.
\bjournal{Advances in Neural Information Processing Systems}
\bvolume{30}.
\end{barticle}
\endbibitem

\bibitem[\protect\citeauthoryear{Mohamed et~al.}{2020}]{Mohamed:2020}
\begin{barticle}[author]
\bauthor{\bsnm{Mohamed},~\bfnm{Shakir}\binits{S.}},
  \bauthor{\bsnm{Rosca},~\bfnm{Mihaela}\binits{M.}},
  \bauthor{\bsnm{Figurnov},~\bfnm{Michael}\binits{M.}} \AND
  \bauthor{\bsnm{Mnih},~\bfnm{Andriy}\binits{A.}}
(\byear{2020}).
\btitle{{Monte Carlo Gradient Estimation in Machine Learning}}.
\bjournal{Journal of Machine Learning Research}
\bvolume{21}
\bpages{1--62}.
\end{barticle}
\endbibitem

\bibitem[\protect\citeauthoryear{Mukkamala and Hein}{2017}]{Mukkamala:2017}
\begin{binproceedings}[author]
\bauthor{\bsnm{Mukkamala},~\bfnm{Mahesh~Chandra}\binits{M.~C.}} \AND
  \bauthor{\bsnm{Hein},~\bfnm{Matthias}\binits{M.}}
(\byear{2017}).
\btitle{{Variants of RMSProp and Adagrad with Logarithmic Regret Bounds}}.
In \bbooktitle{International Conference on Machine Learning}.
\end{binproceedings}
\endbibitem

\bibitem[\protect\citeauthoryear{Murphy}{2012}]{Murphy:2012}
\begin{bbook}[author]
\bauthor{\bsnm{Murphy},~\bfnm{K}\binits{K.}}
(\byear{2012}).
\btitle{{Machine Learning: A Probabilistic Perspective}}.
\bpublisher{MIT Press}, \baddress{Cambridge, MA}.
\end{bbook}
\endbibitem

\bibitem[\protect\citeauthoryear{Nocedal and
  Wright}{2006}]{nocedal2006numerical}
\begin{bbook}[author]
\bauthor{\bsnm{Nocedal},~\bfnm{Jorge}\binits{J.}} \AND
  \bauthor{\bsnm{Wright},~\bfnm{Stephen}\binits{S.}}
(\byear{2006}).
\btitle{Numerical optimization}.
\bpublisher{Springer Science \& Business Media}.
\end{bbook}
\endbibitem

\bibitem[\protect\citeauthoryear{Paisley, Blei and Jordan}{2012}]{Paisley:2012}
\begin{binproceedings}[author]
\bauthor{\bsnm{Paisley},~\bfnm{John~William}\binits{J.~W.}},
  \bauthor{\bsnm{Blei},~\bfnm{David~M.}\binits{D.~M.}} \AND
  \bauthor{\bsnm{Jordan},~\bfnm{Michael~I.}\binits{M.~I.}}
(\byear{2012}).
\btitle{Variational Bayesian Inference with Stochastic Search}.
In \bbooktitle{International Conference on Machine Learning}.
\end{binproceedings}
\endbibitem

\bibitem[\protect\citeauthoryear{Papamakarios et~al.}{2021}]{Papamakarios:2021}
\begin{barticle}[author]
\bauthor{\bsnm{Papamakarios},~\bfnm{George}\binits{G.}},
  \bauthor{\bsnm{Nalisnick},~\bfnm{Eric}\binits{E.}},
  \bauthor{\bsnm{Rezende},~\bfnm{Danilo~Jimenez}\binits{D.~J.}},
  \bauthor{\bsnm{Mohamed},~\bfnm{Shakir}\binits{S.}} \AND
  \bauthor{\bsnm{Lakshminarayanan},~\bfnm{Balaji}\binits{B.}}
(\byear{2021}).
\btitle{{Normalizing Flows for Probabilistic Modeling and Inference}}.
\bjournal{Journal of Machine Learning Research}
\bvolume{22}
\bpages{1--64}.
\end{barticle}
\endbibitem

\bibitem[\protect\citeauthoryear{Pesme, Dieuleveut and
  Flammarion}{2020}]{Pesme:2020}
\begin{binproceedings}[author]
\bauthor{\bsnm{Pesme},~\bfnm{Scott}\binits{S.}},
  \bauthor{\bsnm{Dieuleveut},~\bfnm{Aymeric}\binits{A.}} \AND
  \bauthor{\bsnm{Flammarion},~\bfnm{Nicolas}\binits{N.}}
(\byear{2020}).
\btitle{{On Convergence-Diagnostic based Step Sizes for Stochastic Gradient
  Descent}}.
In \bbooktitle{International Conference on Machine Learning}.
\end{binproceedings}
\endbibitem

\bibitem[\protect\citeauthoryear{Pflug}{1990}]{Pflug:1990}
\begin{barticle}[author]
\bauthor{\bsnm{Pflug},~\bfnm{Georg~Ch}\binits{G.~C.}}
(\byear{1990}).
\btitle{{Non-asymptotic confidence bounds for stochastic approximation
  algorithms with constant step size}}.
\bjournal{Monatshefte f{\"u}r Mathematik}
\bvolume{110}
\bpages{297--314}.
\end{barticle}
\endbibitem

\bibitem[\protect\citeauthoryear{Polyak and Juditsky}{1992}]{Polyak:1992}
\begin{barticle}[author]
\bauthor{\bsnm{Polyak},~\bfnm{B~T}\binits{B.~T.}} \AND
  \bauthor{\bsnm{Juditsky},~\bfnm{A~B}\binits{A.~B.}}
(\byear{1992}).
\btitle{{Acceleration of stochastic approximation by averaging}}.
\bjournal{SIAM Journal of Control and Optimization}
\bvolume{30}
\bpages{838--855}.
\end{barticle}
\endbibitem

\bibitem[\protect\citeauthoryear{Ranganath, Gerrish and
  Blei}{2014}]{Ranganath:2014}
\begin{binproceedings}[author]
\bauthor{\bsnm{Ranganath},~\bfnm{Rajesh}\binits{R.}},
  \bauthor{\bsnm{Gerrish},~\bfnm{Sean}\binits{S.}} \AND
  \bauthor{\bsnm{Blei},~\bfnm{D.~M.}\binits{D.~M.}}
(\byear{2014}).
\btitle{{Black Box Variational Inference}}.
In \bbooktitle{International Conference on Artificial Intelligence and
  Statistics}
\bpages{814--822}.
\end{binproceedings}
\endbibitem

\bibitem[\protect\citeauthoryear{Rezende, Mohamed and
  Wierstra}{2014}]{Rezende:2014}
\begin{binproceedings}[author]
\bauthor{\bsnm{Rezende},~\bfnm{Danilo~Jimenez}\binits{D.~J.}},
  \bauthor{\bsnm{Mohamed},~\bfnm{Shakir}\binits{S.}} \AND
  \bauthor{\bsnm{Wierstra},~\bfnm{Daan}\binits{D.}}
(\byear{2014}).
\btitle{{Stochastic Backpropagation and Approximate Inference in Deep
  Generative Models}}.
In \bbooktitle{International Conference on Machine Learning}.
\end{binproceedings}
\endbibitem

\bibitem[\protect\citeauthoryear{Robert}{2007}]{Robert:2007}
\begin{bbook}[author]
\bauthor{\bsnm{Robert},~\bfnm{C~P}\binits{C.~P.}}
(\byear{2007}).
\btitle{{The Bayesian Choice}},
\bedition{2nd} ed.
\bpublisher{Springer}, \baddress{New York, NY}.
\end{bbook}
\endbibitem

\bibitem[\protect\citeauthoryear{Roeder, Wu and Duvenaud}{2017}]{Roeder:2017}
\begin{barticle}[author]
\bauthor{\bsnm{Roeder},~\bfnm{Geoffrey}\binits{G.}},
  \bauthor{\bsnm{Wu},~\bfnm{Yuhuai}\binits{Y.}} \AND
  \bauthor{\bsnm{Duvenaud},~\bfnm{David~K}\binits{D.~K.}}
(\byear{2017}).
\btitle{Sticking the landing: Simple, lower-variance gradient estimators for
  variational inference}.
\bjournal{Advances in Neural Information Processing Systems}
\bvolume{30}.
\end{barticle}
\endbibitem

\bibitem[\protect\citeauthoryear{Rudolf and Schweizer}{2018}]{Rudolf:2015}
\begin{barticle}[author]
\bauthor{\bsnm{Rudolf},~\bfnm{Daniel}\binits{D.}} \AND
  \bauthor{\bsnm{Schweizer},~\bfnm{Nikolaus}\binits{N.}}
(\byear{2018}).
\btitle{{Perturbation theory for Markov chains via Wasserstein distance}}.
\bjournal{Bernoulli}
\bvolume{4A}
\bpages{2610--2639}.
\end{barticle}
\endbibitem

\bibitem[\protect\citeauthoryear{Ruiz et~al.}{2016}]{Ruiz:2016}
\begin{barticle}[author]
\bauthor{\bsnm{Ruiz},~\bfnm{Francisco~R}\binits{F.~R.}},
  \bauthor{\bsnm{AUEB},~\bfnm{Titsias~RC}\binits{T.~R.}},
  \bauthor{\bsnm{Blei},~\bfnm{David}\binits{D.}} \betal{et~al.}
(\byear{2016}).
\btitle{The generalized reparameterization gradient}.
\bjournal{Advances in neural information processing systems}
\bvolume{29}.
\end{barticle}
\endbibitem

\bibitem[\protect\citeauthoryear{Ruppert}{1988}]{Ruppert:1988}
\begin{btechreport}[author]
\bauthor{\bsnm{Ruppert},~\bfnm{D}\binits{D.}}
(\byear{1988}).
\btitle{{Efficient estimations from a slowly convergent Robbins--Monro
  process}}
\btype{Technical Report},
\bpublisher{Cornell University Operations Research and Industrial Engineering}.
\end{btechreport}
\endbibitem

\bibitem[\protect\citeauthoryear{Saatchi and Wilson}{2017}]{Saatchi:2017}
\begin{binproceedings}[author]
\bauthor{\bsnm{Saatchi},~\bfnm{Yunus}\binits{Y.}} \AND
  \bauthor{\bsnm{Wilson},~\bfnm{Andrew~Gordon}\binits{A.~G.}}
(\byear{2017}).
\btitle{{Bayesian GAN}}.
In \bbooktitle{Advances in Neural Information Processing Systems}.
\end{binproceedings}
\endbibitem

\bibitem[\protect\citeauthoryear{Salimans, Kingma and
  Welling}{2015}]{Salimans:2015}
\begin{binproceedings}[author]
\bauthor{\bsnm{Salimans},~\bfnm{Tim}\binits{T.}},
  \bauthor{\bsnm{Kingma},~\bfnm{Diederik~P}\binits{D.~P.}} \AND
  \bauthor{\bsnm{Welling},~\bfnm{Max}\binits{M.}}
(\byear{2015}).
\btitle{{Markov Chain Monte Carlo and Variational Inference: Bridging the
  Gap}}.
In \bbooktitle{International Conference on Machine Learning}.
\end{binproceedings}
\endbibitem

\bibitem[\protect\citeauthoryear{Salimans and Knowles}{2013}]{salimans:2013}
\begin{barticle}[author]
\bauthor{\bsnm{Salimans},~\bfnm{Tim}\binits{T.}} \AND
  \bauthor{\bsnm{Knowles},~\bfnm{David~A}\binits{D.~A.}}
(\byear{2013}).
\btitle{Fixed-form variational posterior approximation through stochastic
  linear regression}.
\bjournal{ArXiv}.
\end{barticle}
\endbibitem

\bibitem[\protect\citeauthoryear{{Stan Development Team}}{2020}]{stan20}
\begin{bmisc}[author]
\bauthor{\bsnm{{Stan Development Team}}}
(\byear{2020}).
\btitle{Stan Modeling Language Users Guide and Reference Manual}.
\end{bmisc}
\endbibitem

\bibitem[\protect\citeauthoryear{Titsias and
  L{\'a}zaro-Gredilla}{2014}]{Titsias:2014}
\begin{binproceedings}[author]
\bauthor{\bsnm{Titsias},~\bfnm{Michalis~K}\binits{M.~K.}} \AND
  \bauthor{\bsnm{L{\'a}zaro-Gredilla},~\bfnm{Miguel}\binits{M.}}
(\byear{2014}).
\btitle{{Doubly Stochastic Variational Bayes for non-Conjugate Inference}}.
In \bbooktitle{International Conference on Machine Learning}.
\end{binproceedings}
\endbibitem

\bibitem[\protect\citeauthoryear{Vats and Knudson}{2021}]{Vats:2021:Rhat}
\begin{barticle}[author]
\bauthor{\bsnm{Vats},~\bfnm{Dootika}\binits{D.}} \AND
  \bauthor{\bsnm{Knudson},~\bfnm{Christina}\binits{C.}}
(\byear{2021}).
\btitle{Revisiting the Gelman-Rubin Diagnostic}.
\bjournal{Statistical Science}
\bvolume{36}
\bpages{518--529}.
\end{barticle}
\endbibitem

\bibitem[\protect\citeauthoryear{Vehtari et~al.}{2021}]{Vehtari:2021:R-hat}
\begin{barticle}[author]
\bauthor{\bsnm{Vehtari},~\bfnm{Aki}\binits{A.}},
  \bauthor{\bsnm{Gelman},~\bfnm{Andrew}\binits{A.}},
  \bauthor{\bsnm{Simpson},~\bfnm{Daniel}\binits{D.}},
  \bauthor{\bsnm{Carpenter},~\bfnm{Bob}\binits{B.}} \AND
  \bauthor{\bsnm{B{\"u}rkner},~\bfnm{Paul-Christian}\binits{P.-C.}}
(\byear{2021}).
\btitle{{Rank-Normalization, Folding, and Localization: An Improved
  $\widehat{R}$ for Assessing Convergence of MCMC}}.
\bjournal{Bayesian Analysis}
\bvolume{16}
\bpages{667--718}.
\end{barticle}
\endbibitem

\bibitem[\protect\citeauthoryear{Villani}{2009}]{Villani:2009}
\begin{bbook}[author]
\bauthor{\bsnm{Villani},~\bfnm{C}\binits{C.}}
(\byear{2009}).
\btitle{{Optimal transport: old and new}}.
\bseries{Grundlehren der mathematischen Wissenschaften}
\bvolume{338}.
\bpublisher{Springer}.
\end{bbook}
\endbibitem

\bibitem[\protect\citeauthoryear{Vollmer, Zygalakis and
  Teh}{2016}]{Vollmer:2016}
\begin{barticle}[author]
\bauthor{\bsnm{Vollmer},~\bfnm{Sebastian~J}\binits{S.~J.}},
  \bauthor{\bsnm{Zygalakis},~\bfnm{Konstantinos~C}\binits{K.~C.}} \AND
  \bauthor{\bsnm{Teh},~\bfnm{Y~W}\binits{Y.~W.}}
(\byear{2016}).
\btitle{{(Non-) asymptotic properties of Stochastic Gradient Langevin
  Dynamics}}.
\bjournal{Journal of Machine Learning Research}
\bvolume{17}
\bpages{1--48}.
\end{barticle}
\endbibitem

\bibitem[\protect\citeauthoryear{Wainwright and Jordan}{2008}]{Wainwright:2008}
\begin{barticle}[author]
\bauthor{\bsnm{Wainwright},~\bfnm{M.~J.}\binits{M.~J.}} \AND
  \bauthor{\bsnm{Jordan},~\bfnm{M~I}\binits{M.~I.}}
(\byear{2008}).
\btitle{{Graphical Models, Exponential Families, and Variational Inference}}.
\bjournal{Foundations and Trends{\textregistered} in Machine Learning}
\bvolume{1}
\bpages{1--305}.
\end{barticle}
\endbibitem

\bibitem[\protect\citeauthoryear{Wan, Li and Hovakimyan}{2020}]{Wan:2020vz}
\begin{barticle}[author]
\bauthor{\bsnm{Wan},~\bfnm{Neng}\binits{N.}},
  \bauthor{\bsnm{Li},~\bfnm{Dapeng}\binits{D.}} \AND
  \bauthor{\bsnm{Hovakimyan},~\bfnm{Naira}\binits{N.}}
(\byear{2020}).
\btitle{{f-Divergence Variational Inference.}}
\bjournal{Advances in neural information processing systems}
\bvolume{33}
\bpages{17370--17379}.
\end{barticle}
\endbibitem

\bibitem[\protect\citeauthoryear{Wang, Geffner and Domke}{2023a}]{Wang:2023a}
\begin{binproceedings}[author]
\bauthor{\bsnm{Wang},~\bfnm{Xi}\binits{X.}},
  \bauthor{\bsnm{Geffner},~\bfnm{Tomas}\binits{T.}} \AND
  \bauthor{\bsnm{Domke},~\bfnm{Justin}\binits{J.}}
(\byear{2023}a).
\btitle{A Dual Control Variate for accelerated black-box variational
  inference}.
In \bbooktitle{Fifth Symposium on Advances in Approximate Bayesian Inference}.
\end{binproceedings}
\endbibitem

\bibitem[\protect\citeauthoryear{Wang, Geffner and Domke}{2023b}]{Wang:2023b}
\begin{bmisc}[author]
\bauthor{\bsnm{Wang},~\bfnm{Xi}\binits{X.}},
  \bauthor{\bsnm{Geffner},~\bfnm{Tomas}\binits{T.}} \AND
  \bauthor{\bsnm{Domke},~\bfnm{Justin}\binits{J.}}
(\byear{2023}b).
\btitle{Joint control variate for faster black-box variational inference}.
\end{bmisc}
\endbibitem

\bibitem[\protect\citeauthoryear{Wang, Kasprzak and Huggins}{2023}]{Wang2023}
\begin{binproceedings}[author]
\bauthor{\bsnm{Wang},~\bfnm{Yu}\binits{Y.}},
  \bauthor{\bsnm{Kasprzak},~\bfnm{Mikolaj}\binits{M.}} \AND
  \bauthor{\bsnm{Huggins},~\bfnm{Jonathan~H}\binits{J.~H.}}
(\byear{2023}).
\btitle{A Targeted Accuracy Diagnostic for Variational Approximations}.
In \bbooktitle{International Conference on Artificial Intelligence and
  Statistics}
\bpages{8351--8372}.
\bpublisher{PMLR}.
\end{binproceedings}
\endbibitem

\bibitem[\protect\citeauthoryear{Wang, Liu and
  Liu}{2018}]{Wang:2018:f-divergence}
\begin{binproceedings}[author]
\bauthor{\bsnm{Wang},~\bfnm{Dilin}\binits{D.}},
  \bauthor{\bsnm{Liu},~\bfnm{Hao}\binits{H.}} \AND
  \bauthor{\bsnm{Liu},~\bfnm{Qiang}\binits{Q.}}
(\byear{2018}).
\btitle{{Variational Inference with Tail-adaptive f-Divergence}}.
In \bbooktitle{Advances in Neural Information Processing Systems}.
\end{binproceedings}
\endbibitem

\bibitem[\protect\citeauthoryear{Yaida}{2019}]{Yaida:2019}
\begin{binproceedings}[author]
\bauthor{\bsnm{Yaida},~\bfnm{Sho}\binits{S.}}
(\byear{2019}).
\btitle{{Fluctuation-dissipation relations for stochastic gradient descent}}.
In \bbooktitle{International Conference on Learning Representations}.
\end{binproceedings}
\endbibitem

\bibitem[\protect\citeauthoryear{Yao et~al.}{2018}]{Yao:2018:VI}
\begin{binproceedings}[author]
\bauthor{\bsnm{Yao},~\bfnm{Yuling}\binits{Y.}},
  \bauthor{\bsnm{Vehtari},~\bfnm{Aki}\binits{A.}},
  \bauthor{\bsnm{Simpson},~\bfnm{Daniel}\binits{D.}} \AND
  \bauthor{\bsnm{Gelman},~\bfnm{Andrew}\binits{A.}}
(\byear{2018}).
\btitle{{Yes, but Did It Work?: Evaluating Variational Inference}}.
In \bbooktitle{International Conference on Machine Learning}.
\bseries{PMLR}
\bvolume{80}
\bpages{5577-5586}.
\end{binproceedings}
\endbibitem

\bibitem[\protect\citeauthoryear{Zhang et~al.}{2020}]{Zhang:2020:SALSA}
\begin{barticle}[author]
\bauthor{\bsnm{Zhang},~\bfnm{Pengchuan}\binits{P.}},
  \bauthor{\bsnm{Lang},~\bfnm{Hunter}\binits{H.}},
  \bauthor{\bsnm{Liu},~\bfnm{Qiang}\binits{Q.}} \AND
  \bauthor{\bsnm{Xiao},~\bfnm{Lin}\binits{L.}}
(\byear{2020}).
\btitle{{Statistical Adaptive Stochastic Gradient Methods}}.
\bjournal{arXiv.org}
\bvolume{arXiv:2002.10597 [stat.ML]}.
\end{barticle}
\endbibitem

\end{thebibliography}

\clearpage

\appendix

\numberwithin{equation}{section}
\numberwithin{figure}{section}
\numberwithin{table}{section}

\section{Proofs}

\subsection{Proof of \cref{prop:SKL-small-learning-rate}} \label{sec:SKL-small-learning-rate-proof}

Since the KL divergence of mean-field Gaussians factors across dimensions, without loss of generality we consider the case of $\dim = 1$.
The symmetrized KL divergence between two Gaussians is 
\[
\skl{\varparam_{1}}{\varparam_{2}} 
&= \frac{1}{2}\left\{ e^{2\psi_{1} - 2\psi_{2}} + e^{2\psi_{2} - 2\psi_{1}} + (\tau_{1} - \tau_{2})^{2}(e^{-2\psi_{1}} + e^{-2\psi_{2}}) - 2\right\} \\
&= \frac{1}{2}\left\{ e^{2\psi_{1} - 2\psi_{2}} + e^{2\psi_{2} - 2\psi_{1}} + (\tau_{1} - \tau_{2})^{2}e^{-2\psi_{1}}(1 + e^{2\psi_{1} - 2\psi_{2}}) - 2\right\}.
\]
Recall our assumption that for some constants $A = (A_{\tau}, A_{\psi})$ and $B = (B_{\tau}, B_{\psi})$, 
\[
\meansgditerate{\learningrate} = \optvarparam + A\learningrate  + B\learningrate^{2} + o(\learningrate^{2}).
\]
Hence
\[
e^{2\bar\psi_{\learningrate} - 2\psi_{*}} 
&= 1 +  2A_{\psi} \learningrate +  2B_{\psi} \learningrate^{2} + \frac{1}{2}(2A_{\psi}\learningrate + 2B_{\psi} \learningrate^{2})^{2} + o(\learningrate^{2}) \\
&= 1 + 2A_{\psi}\learningrate  + 2B_{\psi}\learningrate^{2}  + 2A_{\psi}^{2}\learningrate^{2}  + o(\learningrate^{2})
\]
and similarly $e^{2\psi_{*} - 2\bar\psi_{\learningrate}} = 1 - 2A_{\psi}\learningrate  - 2B_{\psi}\learningrate^{2}  + 2A_{\psi}^{2}\learningrate^{2}  + o(\learningrate^{2})$.
Therefore $e^{2\bar\psi_{\learningrate} - 2\psi_{*}} + e^{2\psi_{*} - 2\bar\psi_{\learningrate}} - 2 = 4A_{\psi}^{2}\learningrate^{2} + o(\learningrate^{2})$
and $(\bar\tau_{\learningrate} - \tau_{*})^{2} = A_{\tau}^{2}\learningrate^{2} + o(\learningrate^{2})$. 
Putting these results together, we have 
\[
\skl{\meansgditerate{\learningrate}}{\optvarparam} 
&=  \frac{1}{2}\Big[4A_{\psi}^{2}\learningrate^{2}  + o(\learningrate^{2}) +  A_{\tau}^{2}\learningrate^{2} e^{-\psi_{*}}\{2  -  2A_{\psi}\learningrate - 2B_{\psi}\learningrate^{2}  + 2\learningrate^{2} A_{\psi}^{2} + o(\learningrate^{2})\} \Big] \\
&= (2A_{\psi}^{2} + A_{\tau}^{2}e^{-2\psi_{*}}) \learningrate^{2} + o(\learningrate^{2}).
\]
Similarly, $\skl{\meansgditerate{\learningrate}}{\meansgditerate{\learningrate'}}  =  (\learningrate-\learningrate')^{2}(2A_{\psi}^{2} + A_{\tau}^{2}e^{-2\psi_{*}}) + o(\learningrate^{2})$ as long as $\learningrate' = O(\learningrate)$.
In other words, letting $\meansgditerate{0} \defined \optvarparam$, we have the relation 
\[
\skl{\meansgditerate{\learningrate}}{\meansgditerate{\learningrate'}} = C(\learningrate-\learningrate')^{2} + o(\learningrate^{2})
\]
for some unknown constant $C$.

\subsection{Proof of \cref{prop:SKL-small-learning-rate-general}} \label{sec:SKL-small-learning-rate-general-proof}

Define the Fisher information matrix $\fishinf{\varparam} = \EE_{\param \dist \approxdensity_{\varparam}} \left[- \hessian_{\varparam} \log\{\approxdensity_{\varparam}(\param)\} \right]$. 
Since $\varparam \mapsto \approxdensity_{\varparam}(\param)$ is three-times differentiable, locally the KL divergence behaves in a quadratic form \citep{amari:2016}:
\[
\skl{\varparam_{1}}{\varparam_{2}} 
&= \frac{1}{2}\left\{ (\varparam_{1} - \varparam_{2})^\top \fishinf{\varparam_{2}} (\varparam_{1} - \varparam_{2}) + (\varparam_{2} - \varparam_{1})^\top\fishinf{\varparam_{1}}  (\varparam_{1} - \varparam_{2}) \right\} + o(\norm{\varparam_{1} - \varparam_{2}}^2).
\]
Moreover, by taking first-order Taylor expansion, we have 
\[
 \fishinf{\varparam_{i}} = \fishinf{\optvarparam} + O(\statictwonorm{\optvarparam - \varparam_{i}}). 
\]

If $\biaspower = 1$, recall that for some constant vectors $A$ and $B$,  
\[
\meansgditerate{\learningrate} = \optvarparam + A\learningrate  + B\learningrate^{2} + o(\learningrate^{2}).
\]
Assume that $\learningrate_{1} = \learningrate$ and $\learningrate_{2} = O(\learningrate)$, so 
$\statictwonorm{\optvarparam - \meansgditerate{\learningrate_{i}}} = O(\learningrate)$. 
For $\biaspower' > 0$, let $\eps_{\biaspower'} \defined \gamma_{1}^{\biaspower'} - \gamma_{2}^{\biaspower'}$.
Then we have 
\[
\begin{split}
\skl{\meansgditerate{\learningrate_{1}}}{\meansgditerate{\learningrate_{2}}}
&= \frac{1}{2} \Bigg[ (A \eps_{1} + B \eps_{2} + o(\learningrate^{2}))^\top (\fishinf{\optvarparam} + O(\gamma))(A \eps_{1} + B \eps_{2} + o(\learningrate^{2})) \\
&\phantom{=  \frac{1}{2} \Bigg[} +  (-A \eps_{1} - B \eps_{2} - o(\learningrate^{2}))^\top (\fishinf{\optvarparam} + O(\gamma)) (-A \eps_{1} - B \eps_{2} - o(\learningrate^{2}))  \Bigg]  \\
&\phantom{=  \frac{1}{2} \Bigg[} + o(\statictwonorm{A \eps_{1} + B \eps_{2} + o(\learningrate^{2})}^2)
\end{split} \\
&= A^\top  \fishinf{\optvarparam} A \eps_{1}^2  +   o(\learningrate^{2})\\
&= C  (\gamma_1 - \gamma_2)^2 +  o(\learningrate^{2}),
\]
where $C = A^\top  \fishinf{\optvarparam}  A$.

If $\biaspower \in [1/2,1)$, recall that for some constant vectors $\Lambda$ and $A$,  
\[
\meansgditerate{\learningrate} = \optvarparam + \Lambda\learningrate^{\biaspower} +  A\learningrate  + o(\learningrate^{2\biaspower}).
\]
Assume that $\learningrate_{1} = \learningrate$ and $\learningrate_{2} = O(\learningrate)$, so 
$\statictwonorm{\optvarparam - \meansgditerate{\learningrate_{i}}} = O(\learningrate^{\biaspower})$. 
Then we have
\[
\begin{split}
\skl{\meansgditerate{\learningrate_{1}}}{\meansgditerate{\learningrate_{2}}}
&= \frac{1}{2} \Bigg[ (\Lambda \eps_{\biaspower} + A \eps_{1} + o(\learningrate^{2\biaspower}))^\top (\fishinf{\optvarparam} + O(\learningrate^{\biaspower})) (\Lambda \eps_{\biaspower} + A \eps_{1} + o(\learningrate^{2\biaspower})) \\
&\phantom{=  \frac{1}{2} \Bigg[} +  (-\Lambda \eps_{\biaspower} - A \eps_{1} - o(\learningrate^{2\biaspower}))^\top (\fishinf{\optvarparam} + O(\learningrate^{\biaspower}))(-\Lambda \eps_{\biaspower} - A \eps_{1} - o(\learningrate^{2\biaspower})) \Bigg] \\
&\phantom{=  \frac{1}{2} } + o(\statictwonorm{\Lambda \eps_{\biaspower} + A \eps_{1} + o(\learningrate^{2\biaspower})}^2)
\end{split} \\
\begin{split}
&= \Lambda^\top \fishinf{\optvarparam} \Lambda \eps_{\biaspower}^2 +  
2 \Lambda^\top  \fishinf{\optvarparam} A   \eps_{\biaspower}  \eps_{1} +  A^\top  \fishinf{\optvarparam} A \eps_{1}^2 
+ o(\learningrate^{2\biaspower})
\end{split} \\
&= \Lambda^\top   \fishinf{\optvarparam}  \Lambda \eps_{\biaspower}^2 + O(\gamma^{1 + \biaspower}) + o(\learningrate^{2\biaspower})   \\
&= C  (\gamma_1^\biaspower - \gamma_2^\biaspower)^2 +  o(\learningrate^{2\biaspower}),
\]
where $C = \Lambda^\top   \fishinf{\optvarparam}  \Lambda$.

\subsection{Symmetric KL divergence termination rule}

We can use \cref{prop:SKL-small-learning-rate-general} to derive a termination rule. 
If $\learningrate$ is the current learning rate, the previous learning rate was $\learningrate/\rho$.
Ignoring $o(\learningrate^{2\biaspower})$ terms, we have
\[
\delta_{\learningrate}
\defined  \skl{\meansgditerate{\learningrate/\rho}}{\meansgditerate{\learningrate}} 
= C \learningrate^{2\biaspower} (1/\rho^{\biaspower} - 1)^{2}.
\]
Therefore, we can estimate $p$ and $C$ using a regression model of the form
\[
\log \delta_{\learningrate} = \log C + 2 \log  (1/\rho^{\biaspower} - 1) + 2p \log \learningrate.
\]
Given estimates $\hat{\biaspower}$ and $\hC$, we can estimate $\skl{\meansgditerate{\learningrate}}{\optvarparam} \approx \hC \gamma^{2\hat{\biaspower}}$.

\subsection{Proof of \cref{prop:MF-Gaussian-error}} \label{sec:MF-Gaussian-error-proof}

Since for $|x| < 1/2$ it holds that $|\exp(x) - 1| \le 1.5|x|$, we have 
\[
\frac{|\hat\sigma_{i} - \bar\sigma_{i}| }{\bar\sigma_{i}}
&= |\exp(\hat\psi_{i} - \bar\psi_{i}) - 1| 
\le 1.5\varepsilon
\]
and
\[
\frac{|\hat\tau_{i} - \bar\tau_{i}| }{\bar\sigma_{i}}
&\le \varepsilon' \hat\sigma_{i}/\bar\sigma_{i}  
\le \varepsilon'(1 + 1.5\varepsilon) 
\le 1.75\varepsilon. 
\]

\subsection{Proof of \cref{prop:optimal-MCSE-rechecking}} \label{sec:optimal-MCSE-rechecking-proof}

In the optimal case, the total cost for the iterations used for iterate averaging is 
\[
\mathrm{OPT}
&= C_{O}W_\mathrm{opt} + C_{E}W_\mathrm{conv} + C_{E}W_\mathrm{opt} \\
&= C_{E}(r W_\mathrm{opt} + W_\mathrm{conv} + W_\mathrm{opt}).
\]
On the other hand, if checking at window sizes $W_{j} \defined \chi^{j}W_\mathrm{conv}\; (j=1,2,\dots)$, 
then the window size at which convergence will be detected is $j_{*} \defined \lceil \log(W_\mathrm{opt}/W_\mathrm{conv})/\log(\chi) \rceil$.
In particular , 
\[
W_{j_{*}} \le \chi W_\mathrm{opt}. 
\]
Therefore we can bound the actual computational cost as 
\[
\lefteqn{C_{O}W_{j_{*}} + C_{E}W_\mathrm{conv}\log W_\mathrm{conv} + C_{E}\sum_{j=1}^{j_{*}}W_{j}} \\
&\le C_{O}\chi W_\mathrm{opt} + C_{E}W_\mathrm{conv}\log W_\mathrm{conv} + C_{E}\sum_{j=1}^{j_{*}}\chi^{j}W_\mathrm{conv} \\
&\le C_{O}\chi W_\mathrm{opt} + C_{E}W_\mathrm{conv}\log W_\mathrm{conv} + C_{E}\frac{\chi (\chi^{j_{*}} - 1) W_\mathrm{conv}}{\chi - 1} \\
&\le C_{O}\chi W_\mathrm{opt} + C_{E}W_\mathrm{conv}\log W_\mathrm{conv} + C_{E}\frac{\chi^{2}W_\mathrm{opt}}{\chi - 1} \\
&=  C_{E}\left[\chi r W_\mathrm{opt} + W_\mathrm{conv} + \frac{\chi^{2}}{\chi - 1} W_\mathrm{opt}\right].
\]
If $\chi = 2$, then the actual computational cost is bounded by
\[
{C_{E} \left[ 2 r W_\mathrm{opt} + W_\mathrm{conv} + 4  W_\mathrm{opt}\right]}
& \le  4 C_{E} \left[  r W_\mathrm{opt} + W_\mathrm{conv} + W_\mathrm{opt}\right] \\
& =  4\,\mathrm{OPT}.
\]
On the other hand, we can minimize the total cost by solving $\chi_{*} \defined \argmin_{\chi > 1} \{r\chi + \chi^{2}/(\chi - 1)\} = \chi(r)$. 
Plugging this back in, we get the bound 
\[
\lefteqn{C_{E}\frac{2 + r + 2 \sqrt{1 + r}}{1 + r}\left(r W_\mathrm{opt} + W_\mathrm{conv}\log W_\mathrm{conv} + W_\mathrm{opt}\right)} \\
&= \frac{2 + r + 2 \sqrt{1 + r}}{1 + r}  \mathrm{OPT} \\
&< 4\, \mathrm{OPT}.
\]

\section{Further Details}

\subsection{Effective sample size (ESS) and Monte Carlo standard error (MCSE)} \label{sec:ESS-MCSE}

Let $\chain{1}, \chain{2},\dots$ denote a stationary Markov chain,
let $\bar v \defined \EE[\chain{1}]$ denote the mean at stationarity, 
and let $\hat v \defined K^{-1}\sum_{k=1}^{K}\chain{k}$ denote the Monte Carlo estimate for $\bar v$. 
The (ideal) effective sample size is defined as
\[
\mathrm{ESS}(K) \defined K/(1 + \sum_{\iternum=1}^{\infty}\rho_{\iternum}), 
\]
where $\rho_{\iternum}$ is the
autocorrelation of the Markov chain at lag $\iternum$. 
The ESS can be efficiently estimated using a variety of methods \citep{Geyer:1992,Vehtari:2021:R-hat}. 
We write $\widehat{\mathrm{ESS}}(\chain{1:K})$ to denote an estimator for $\mathrm{ESS}(K)$ based 
on the sequence $\chain{1:K} \defined (\chain{1},\dots,\chain{K})$. 
The Monte Carlo standard error of $\hat v$ is given by 
\[
MCSE(\hat v) \defined \sigma(\chain{1})/\mathrm{ESS}(K),
\]
where  $\sigma(\chain{1})$ denote the standard deviation 
of the random variable $\chain{1}$. 
Given the empirical standard deviation of $\chain{1},\dots,\chain{K}$, which we denote $\hat\sigma(\chain{1:K})$, the MCSE
can be approximated by 
\[
\widehat{\mathrm{MCSE}}(\chain{1:K}) \defined \hat\sigma(\chain{1:K})/\widehat{\mathrm{ESS}}(\chain{1:K}). 
\]

\subsection{Total variation distance and KL divergence} \label{sec:Pinskers-inequality}

For distributions $\eta$ and $\zeta$,
$|I_{\eta,i,a,b} - I_{\zeta,i,a,b}| \le \sqrt{\kl{\eta}{\zeta}/2}$ for all $a < b$ and $i$. 
This guarantee follows from Pinsker's inequality, 
which relates the KL divergence to the total variation distance 
\[
d_\mathrm{TV}(\eta, \zeta) \defined  \sup_{f : \normarg{\infty}{f} \le 1} \left|\int f(\param)\eta(\dee\param) - \int f(\param)\zeta(\dee\param)\right|,
\]
where $\normarg{\infty}{f} \defined \sup_{\param}f(\param) - \inf_{\param}f(\param)$. 
Specifically, Pinsker's inequality states that $d_\mathrm{TV}(\eta, \zeta) \le \sqrt{\kl{\eta}{\zeta}/2}$. 
Thus, small KL divergence implies small total variance distance, which implies
the difference between expectations for any function $f$ such that $\normarg{\infty}{f}$ is small.
In the case of interval probabilities, since $\normarg{\infty}{ \ind(\cdot \in [a, b])} = 1$, it follows that 
\[
|I_{\eta,i,a,b} - I_{\zeta,i,a,b}|
&= \left|\int \ind(\param \in [a, b])\eta(\dee\param) - \int \ind(\param \in [a, b])\zeta(\dee\param)\right| \\
&\le \normarg{\infty}{ \ind(\cdot \in [a, b])}d_\mathrm{TV}(\eta, \zeta)  \\
&\le \sqrt{\kl{\eta}{\zeta}/2}.
\]

\subsection{Adaptive SASA+} \label{sec:SASA}

The SASA+ algorithm of \citet{Zhang:2020:SALSA} generalizes the approach of \citet{Yaida:2019}.
The main idea is to find an appropriate \emph{invariant function} $\invariant(d, \varparam)$ that satisfies 
$\int \invariant(d, \varparam)\stationarydist{\learningrate}(\dee d, \dee\varparam) = 0$.
\citet{Yaida:2019} derived valid forms of $\invariant$ for specific choices of the descent direction, while \citet{Zhang:2020:SALSA}
showed that for any optimizer of the form \cref{eq:stochastic-opt} that is time-homogenous with $\learningrateiterate{\iternum} = \learningrate$, 
the map $(d, \varparam) \mapsto 2\inner{d}{\varparam} - \learningrate\norm{d}^{2}$ is a valid invariant function. 
The SASA+ algorithm proposed by \citet{Zhang:2020:SALSA} uses a hypothesis test to determining when the iterates are sufficiently close to stationarity.
Let $\invariantiterate{\iternum} \defined 2\inner{\descentdir{\iternum}}{\sgditerate{\iternum}} - \learningrate\norm{\descentdir{\iternum}}^{2}$
and let $W = \lceil\windowprop \iternum\rceil$ denote the window size to use for checking stationarity. 
Once $W$ is at least equal to a minimum window size $W_{\min}$, SASA+ uses $\invariantiterate{\iternum-W+1},\dots,\invariantiterate{\iternum}$
to carry out a hypothesis test, where the null hypothesis is that $\EE[\invariantiterate{k}] = 0$.

We make several adjustments to reduce the number of tuning parameters and to make the remaining ones more intuitive. 
Note that the SASA+ convergence criterion requires the choice of three parameters: $\windowprop$, $W_{\min}$, and the size of the hypothesis test $\alpha$. 
\citet{Zhang:2020:SALSA} showed empirically and our numerical experiments confirmed that the choice of $\alpha$ has little effect and 
therefore does not need to be adjusted by the user. 
The choices for $\windowprop$ and $W_{\min}$, however, have a substantial effect on efficiency. 
If $\windowprop$ is too big, then early iterations that are not at stationarity will be included, preventing the detection of convergence.
On the other hand, if $\windowprop$ is too small, then the total number of iterations must be large (specifically, greater than $W_{\min}/\windowprop$)
before the window size is large enough to trigger the first check for stationarity.
Moreover, the correct choice of $W_{\min}$ will vary depending on the problem.
If the iterates have large autocorrelation then $W_{\min}$ should be large, while if the autocorrelation is small or negative,
then $W_{\min}$ can be small.

Our approach to determining the optimal window size instead relies on the effective sample size (ESS).
$W_\mathrm{opt}$ that maximizes 
$\widehat{\mathrm{ESS}}(W) \defined \widehat{\mathrm{ESS}}(\invariantiterate{\iternum-W+1},\dots,\invariantiterate{\iternum})$,
where $\iternum$ is the current iteration. 
To ensure reliability, we impose additional conditions on $W_\mathrm{opt}$.
First, we require that $\widehat{\mathrm{ESS}}(W_\mathrm{opt}) \ge N_{\min}$, a user-specified minimum effective sample size.
Unlike $W_{\min}$, $N_{\min}$ has an intuitive and direct interpretation. 
The second condition is, when finding $W_\mathrm{opt}$, the search over values of $W$ is constrained to the lower bound of $N_{\min}$ 
(to ensure the estimator is sufficiently reliable) and the upper bound of $0.95\iternum$ (to always allow for some ``burn-in'').
In practice we do not check all $W \in \theset{N_{\min},\dots, 0.95\iternum}$, but rather perform a grid search 
over 5 equally spaced values ranging from $N_{\min}$ to $0.95\iternum$. 

A slightly different version which may improve power is to instead define the multivariate invariant function
$\vec\invariant(d, \varparam) = 2d \odot \varparam -  d \odot d$, where we recall that $\odot$ denotes component-wise 
multiplication. 
In this case a multivariate hypothesis test such as Hotelling's $T^{2}$ test or the multivariate sign test,
where a single effective sample size (e.g., the median component-wise ESS) could be used. 
Alternatively, a separate hypothesis test for each of the $\varparamdim$ components could be used,
with stationarity declared once all tests confirm stationarity (using, e.g., a test size of $\alpha/\varparamdim$).
While this approach might be more computationally efficient when $\varparamdim$ is large, it could come at the cost of test power.

\subsection{Distance Based Convergence Detection} \label{sec:distance-based}

\citet{Pesme:2020} proposed a distance based diagnostic algorithm to detect the stationarity of the SGD optimization
algorithm. The main idea behind this algorithm is to find the distance between the current iterate $\varparam_\iternum$
and the optimal variational parameter $\optvarparam$, $\norm{\eta_\iternum} \defined \norm{\varparam_\iternum - \optvarparam}$.
Since optimal variational parameter $\optvarparam$ unknown, this distance cannot be directly observed. Therefore,
they suggested using the distance between the current iterate $\varparam_\iternum$ and the initial iterate of the current 
learning rate, $\norm{\Omega_\iternum} \defined \norm{\varparam_\iternum - \varparam_0}$ and showed that $\norm{\eta_\iternum}$ and 
$\norm{\Omega_\iternum} $ have a similar behavior. 

Under the setting of quadratic objective function with additive noise, they computed the behavior of the expectation of  
$\norm{\Omega_\iternum}^2$ in closed-form to detect the convergence of $\norm{\varparam_\iternum - \varparam_0}^2$.
With the result, they have shown the asymptotic behavior of the $\EE[\norm{\Omega_\iternum}^2]$ in the transient and stationary 
phases, where $\EE[\norm{\Omega_\iternum}^2]$ has a slope greater than $1$ and slope of $0$ in a log-log plot respectively. 
Hence, the slope 
$S \defined \frac{ \log \norm{\varparam_\iternum - \varparam_0}^2 - \log  \norm{\varparam_\iternum/q - \varparam_0}^2}{\log \iternum - \log \iternum/q}$ 
computed between iterations $q^n$ and $q^{n+1}$ for $q>1$ and $n \ge n_0$ where $n_0 \in \mathcal{N}^*$ and if $S < \mathrm{thresh}$ ($\mathrm{thresh} \in (0,2]$) 
then the convergence will be detected.

\subsection{Stochastic Quasi-Newton Optimization}\label{sec:sqn}

The algorithm of \citet{liu2021quasi} provide a randomized approach of classical quasi-Newton (QN) method that is known as stochastic quasi-Newton (SQN) method.
Even though, \citet{liu2021quasi} use randomized quasi-Monte Carlo (RMCQ) samples we use Monte Carlo (MC) samples to compute the gradient .
In classical quasi-Newton optimization the Newton update is
\[
\sgditerate{\iternum+1} \gets \sgditerate{\iternum} -  (\grad^2_{\sgditerate{\iternum}}D_{\pi}(\approxdensity_{\sgditerate{\iternum}}))^{-1}  \grad_{\sgditerate{\iternum}}D_{\pi}(\approxdensity_{\sgditerate{\iternum}}),
\]
where $\grad^2_{\sgditerate{\iternum}}D_{\pi}(\approxdensity_{\sgditerate{\iternum}})$ is the Hessian matrix of $D_{\pi}(\approxdensity_{\sgditerate{\iternum}})$.
However, computation cost of Hessian matrix and its inverse is high and it also takes a large space. Therefore, we can use BFGS (discovered by Broyden, Fletcher, Goldfarb, and Shanno) method where it approximate the inverse of $\grad^2_{\sgditerate{\iternum}}D_{\pi}(\approxdensity_{\sgditerate{\iternum}})$ using $H_{\iternum}$ at the $\iternum$th iteration by initializing it with an identity matrix.  
Then the update is modified by
\[
\sgditerate{\iternum+1} \gets \sgditerate{\iternum} -  \learningrateiterate{\iternum} \approxhessian{\iternum} \grad_{\sgditerate{\iternum}}D_{\pi}(\approxdensity_{\sgditerate{\iternum}}),
\]
where 
\[
\approxhessian{\iternum+1} \gets \Big(I - \frac{s_{\iternum}y^\top_{\iternum}}{s^\top_{\iternum}y_{\iternum}}\Big) \approxhessian{\iternum} 
\Big(I - \frac{y_{\iternum}s^\top_{\iternum}}{s^\top_{\iternum}y_{\iternum}}\Big) + \frac{s_{\iternum}s^\top_{\iternum}}{s^\top_{\iternum}y_{\iternum}},
\]
$s_{\iternum} = \sgditerate{\iternum+1} - \sgditerate{\iternum}$, and $y_{\iternum} = \grad_{\sgditerate{\iternum+1}}D_{\pi}(\approxdensity_{\sgditerate{\iternum+1}}) - \grad_{\sgditerate{\iternum}}D_{\pi}(\approxdensity_{\sgditerate{\iternum}})$.
Even using the above Hessian approximation $\approxhessian{\iternum}$ will require large space for storage. To over come that problem we can use Limited-memory BFGS (L-BFGS) \citep{nocedal2006numerical} that computes 
$\approxhessian{\iternum}  \grad_{\sgditerate{\iternum}}D_{\pi}(\approxdensity_{\sgditerate{\iternum}})$ using $m$ most recent  $(s_{\iternum},y_{\iternum})$ correction pairs. 
\citet{liu2021quasi} use the  stochastic quasi-Newton approach proposed by \citet{chen2019stochastic} and they compute the correction pairs after every $B$ iterations by computing the iterate average of parameters using the most recent $B$ iterations.
To compute the $y_{\iternum}$, the gradients of the objective function is estimated by using MC samples that is independent from the MC samples obtained to compute the gradient in the update step.   

\subsection{Natural Gradient Descent Optimization}\label{sec:ngd}

\citet{Khan:2017:CVI} proposed natural gradient descent (NGD) approach for variational inference that utilize the information geometry of the variational distribution. Given the variational distribution is a exponential family
distribution
\[
\approxdensity_\varparam = h \exp(\phi^\top \varparam - A(\varparam)) 
\]
the NGD update is
\[
\sgditerate{\iternum+1} \gets \sgditerate{\iternum} -  \learningrateiterate{\iternum} (F(\sgditerate{\iternum}))^{-1} \grad_{\sgditerate{\iternum}}D_{\pi}(\approxdensity_{\sgditerate{\iternum}}),
\]
where $F(\sgditerate{\iternum})$ denote the Fisher Information Matrix (FIM) of the distribution and $\varparam$ denote its natural-parameter. Without directly computing the
inverse of FIM \citet{Khan:2017:CVI}  simplified the above update by using the relationship between natural parameter $\varparam$ and expectation parameter $\expectationparam = \mathrm{E}_{\approxdensity}[\phi]$ of the exponential family
\[
F(\varparam)^{-1} \grad_{\varparam} D_{\pi}(\approxdensity_{\varparam}) =  \grad_{\expectationparam} D_{\pi}(\approxdensity_{\expectationparam}).
\]
Therefore, the natural-gradient update can be simplified as
\[
\sgditerate{\iternum+1} \gets \sgditerate{\iternum} -  \learningrateiterate{\iternum}  \grad_{\expectationparam^{(\iternum)}}D_{\pi}(\approxdensity_{\expectationparam^{(\iternum)}}).
\]
This can be applied for the mean-field Gaussian variational family $\approxdensity = \prod_{i=1}^{d} \mathcal{N}(\mu_i, \sigma^2_i)$ where we can define the natural parameters
$\varparam_i = (\mu_i/\sigma^2_i, -0.5/\sigma^2_i)$ and expectation parameters $\expectationparam =  (\mu_i, \mu_i^2 + \sigma^2_i)$.

\section{Additional Experimental Details and Results}

\begin{table}[tbp]
\begin{center}
\caption{Datasets of PosteriorDB Package}
\label{tab:posteriordb-datasets}
\begin{tabular}{@{}llrl@{}}
\toprule
Dataset                   & Short Name& \multicolumn{1}{l}{Dimensions} &    Model Description                  \\ \midrule
arK-arK								& arK							&	7	&	AR(5) time series	\\
bball\_drive\_event\_0-hmm\_drive\_0		& bball\_0							& 	8 	&	Hidden Markov model				 \\
bball\_drive\_event\_1-hmm\_drive\_1		& bball\_1							& 	8 	&	Hidden Markov model				\\
diamonds-diamonds						& diamonds						&	26 	&  Log-Log	\\
dogs-dogs 							& dogs							&	3      & Logistic mixed-effects        \\
dogs-dogs\_log							& dogs\_log						&	2	& Logarithmic mixed-effects	\\
earnings-logearn\_interaction				& earnings						&	5	& Log-linear \\
eight\_schools-eight\_schools\_noncentered 	& 8schools\_nc						&	10	& Non-centered hierarchical 				\\
eight\_schools-eight\_schools\_centered		& 8schools\_c						&	10	& Centered hierarchical				\\
garch-garch11							& garch							&	4	& GARCH(1,1) time series  \\
gp\_pois\_regr-gp\_pois\_regr				& gp\_pois\_regr					&	13	&  \begin{tabular}[c]{@{}l@{}}Gaussian process Poisson\\ regression\end{tabular}	\\
gp\_pois\_regr-gp\_regr					& gp\_regr							&	3	& Gaussian process regression	\\
hmm\_example-hmm\_example				& hmm\_example					&	6	& Hidden Markov model \\
hudson\_lynx-hare\_lotka\_volterra			& hudson\_lynx						&	8	& Lotka-Volterra error \\
low\_dim\_gauss\_mix-low\_dim\_gauss\_mix	& low\_dim\_gauss\_mix				&	5	&  \begin{tabular}[c]{@{}l@{}}Two-dimensional\\ Gaussian mixture\end{tabular}	\\
mcycle\_gp-accel\_gp					& mcycle\_gp						&	66	&	Gaussian process\\
nes2000-nes							& nes2000						&	10	&	Multiple predictor linear \\
sblrc-blr								&sblrc							&	6	&	Linear  \\
\bottomrule
\end{tabular}
\end{center}
\end{table}

\begin{figure}[tbp]
\begin{center}
\centering
\begin{subfigure}[t]{.48\textwidth}
\centering
\includegraphics[width=\textwidth]{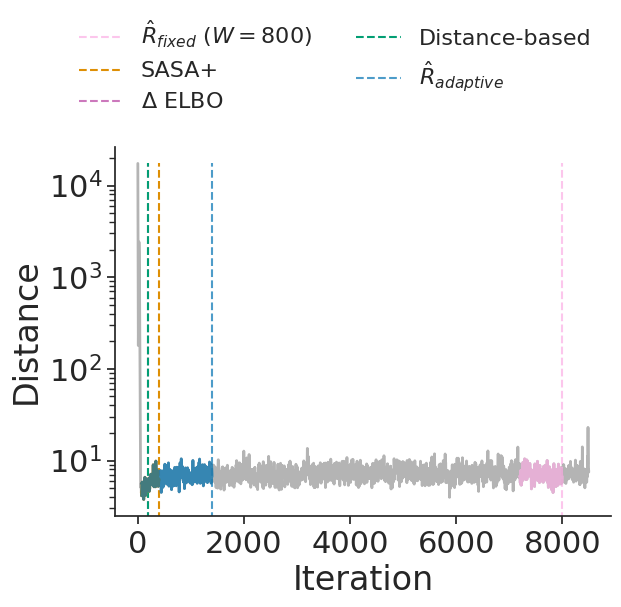}
\caption{uncorrelated $\paramdim = 100$} 
\end{subfigure} 
\begin{subfigure}[t]{.48\textwidth}
\centering
\includegraphics[width=\textwidth]{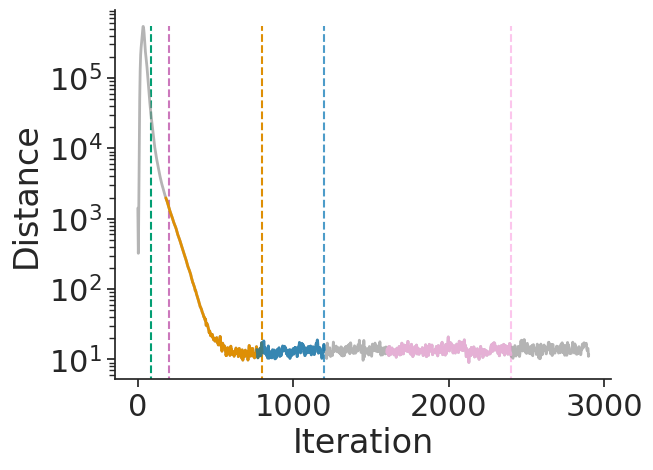}
\caption{uniform correlated $\paramdim = 100$} 
\end{subfigure}
\begin{subfigure}[t]{.48\textwidth}
\centering
\includegraphics[width=\textwidth]{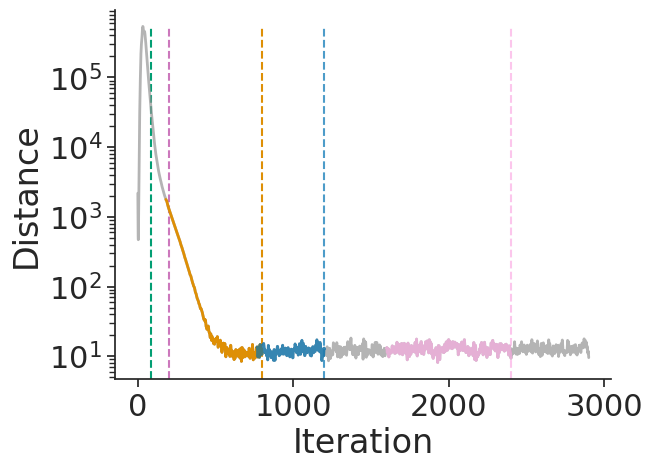}
\caption{banded correlated $\paramdim = 100$} 
\end{subfigure}
\caption{Iteration number versus distance between iterate average and current iterate for Gaussian targets. 
The vertical lines indicate convergence detection trigger points and (for SASA+ and $\widehat{R}$) the colored portion of 
the accuracy values indicate they are part of the window used for  convergence detection.}
\label{fig:convergence-detection-comparison-more}
\end{center}
\end{figure}

\begin{figure}[tbp]
\begin{center}
\begin{subfigure}[t]{\textwidth}
\centering
\includegraphics[width=0.4\textwidth]{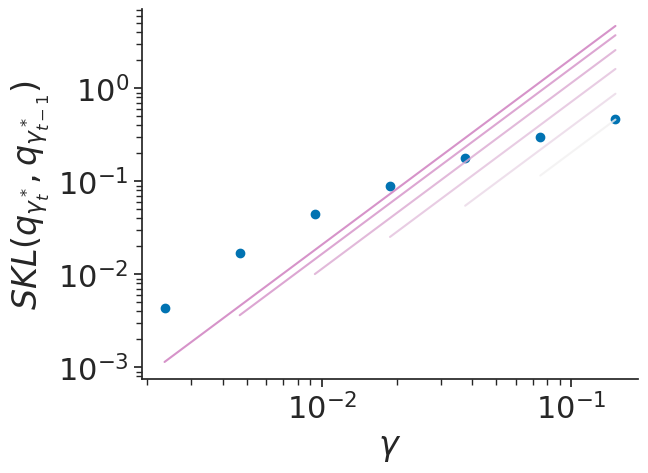} 
\includegraphics[width=0.4\textwidth]{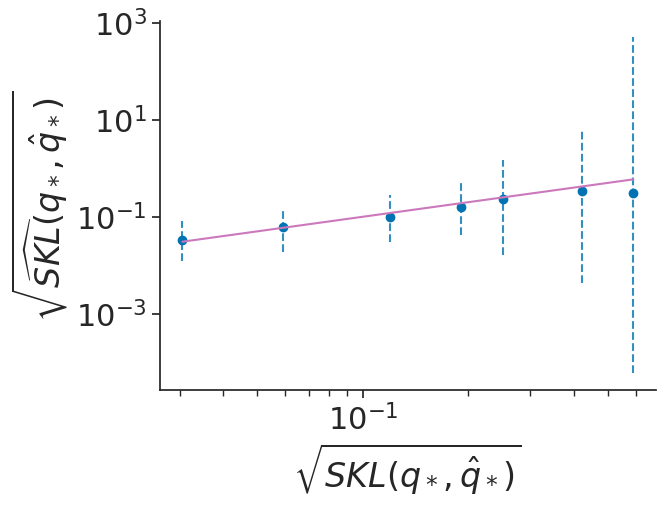}
\caption{uncorrelated $\paramdim = 500$} 
\end{subfigure}  
\begin{subfigure}[t]{\textwidth}
\centering
\includegraphics[width=0.4\textwidth]{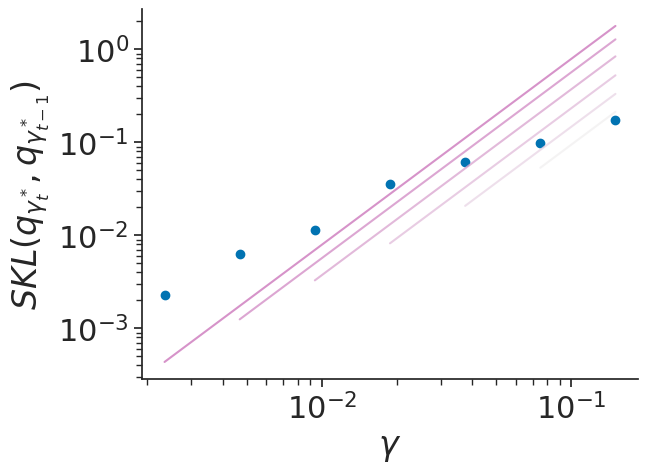} 
\includegraphics[width=0.4\textwidth]{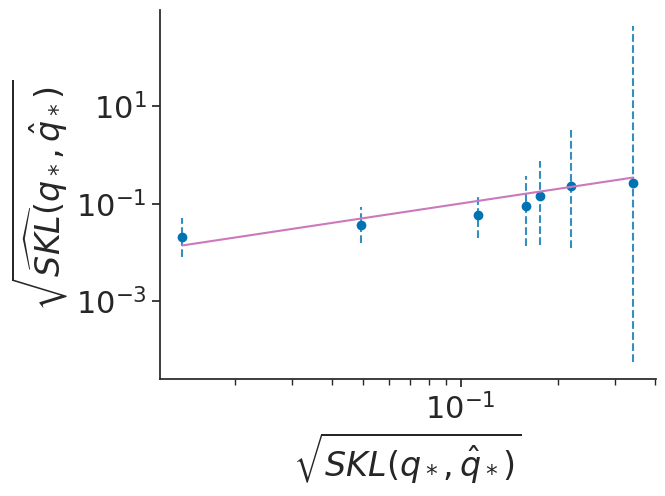}
\caption{uniform correlated $\paramdim = 100$} 
\end{subfigure}  
\begin{subfigure}[t]{\textwidth}
\centering
\includegraphics[width=0.4\textwidth]{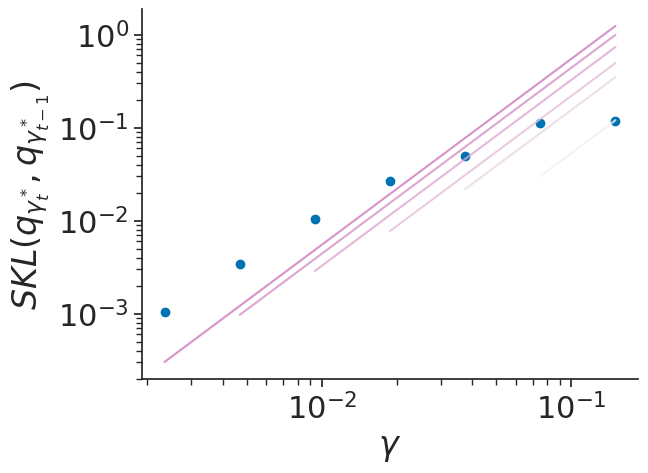} 
\includegraphics[width=0.4\textwidth]{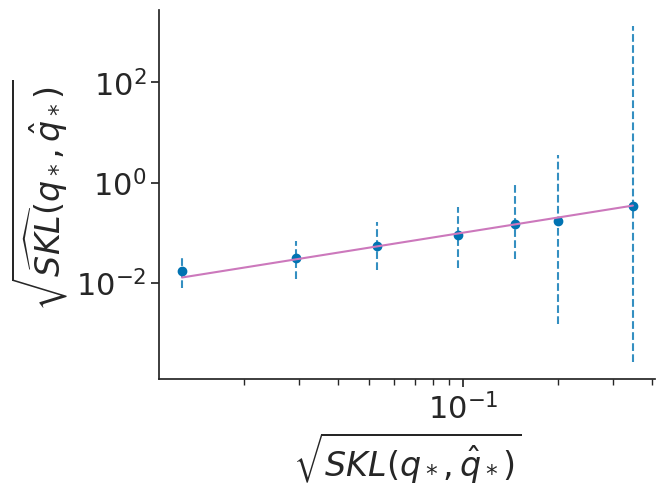}
\caption{banded correlated $\paramdim = 100$} 
\end{subfigure}  
\caption{
Results for estimating the symmetrized KL divergence with avgAdam. 
\textbf{(left)} Learning rate versus symmetrized KL divergence of adjacent iterate averaged estimates of 
optimal variational distribution. 
The lines indicate the linear regression fits, with setting $\biaspower = 1$. 
\textbf{(right)} Square root of true symmetrized KL divergence versus the estimated value with $95\%$ credible interval.
The uncertainty of the estimates decreases and remains well-calibrated as the learning rate decreases.}
\label{fig:SKL-accuracy-more-avgrmsp}
\end{center}
\end{figure}

\begin{figure}[tbp]
\begin{center}
\begin{subfigure}[t]{\textwidth}
\centering
\includegraphics[width=0.4\textwidth]{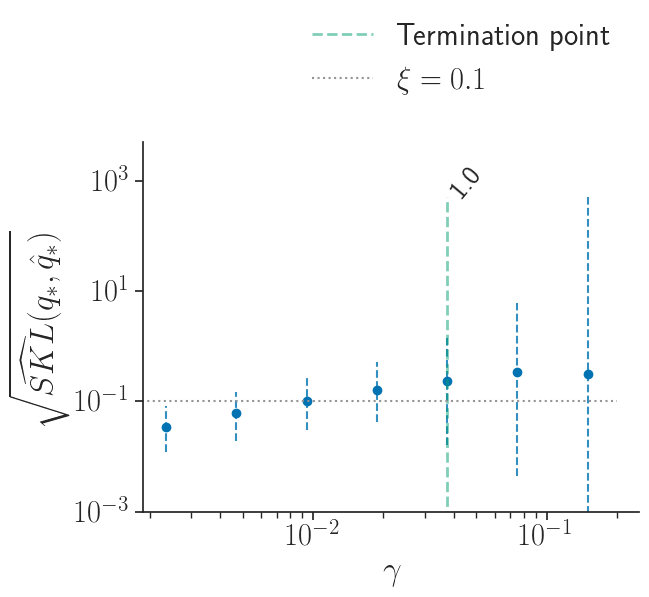} 
\includegraphics[width=0.4\textwidth]{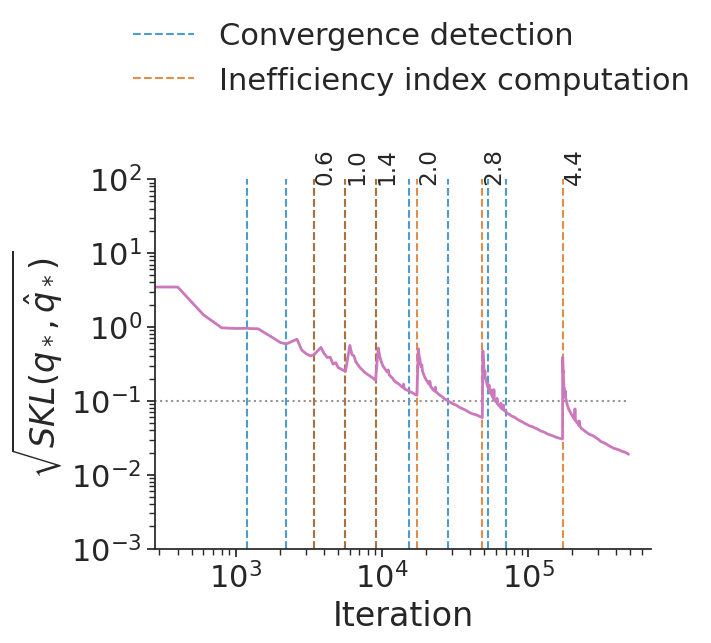}
\caption{uncorrelated $\paramdim = 500$} 
\end{subfigure} 
\begin{subfigure}[t]{\textwidth}
\centering
\includegraphics[width=0.4\textwidth]{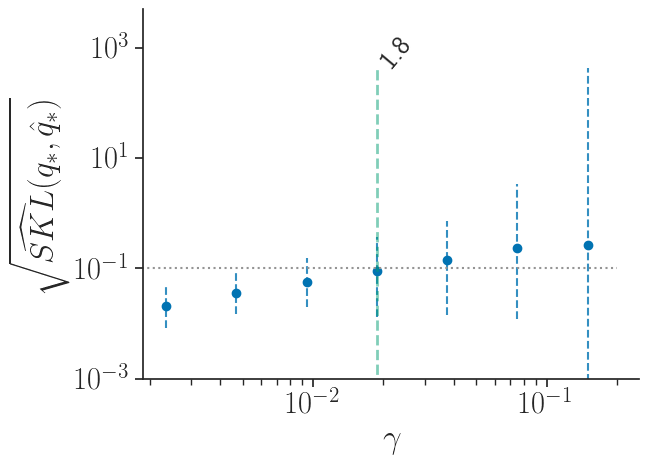}
\includegraphics[width=0.4\textwidth]{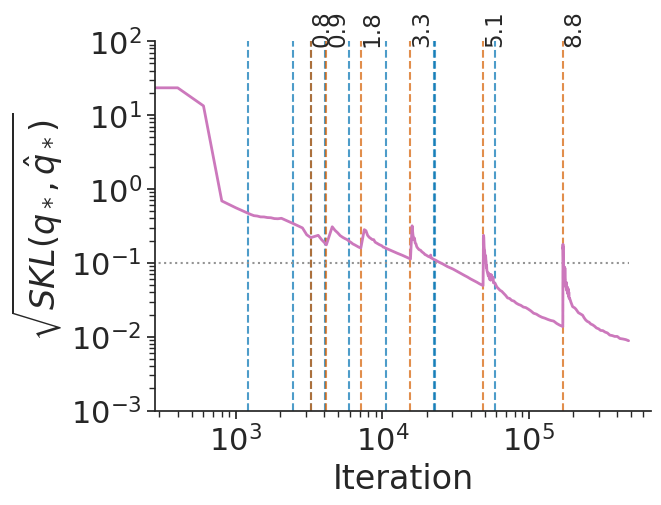}
\caption{uniform correlated $\paramdim = 100$} 
\end{subfigure}  
\begin{subfigure}[t]{\textwidth}
\centering
\includegraphics[width=0.4\textwidth]{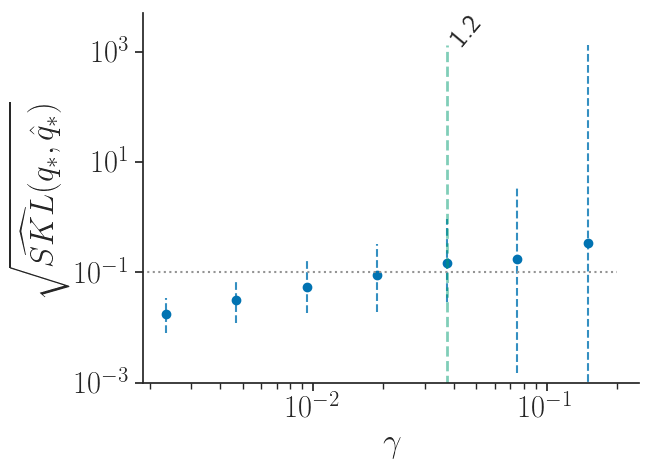}
\includegraphics[width=0.4\textwidth]{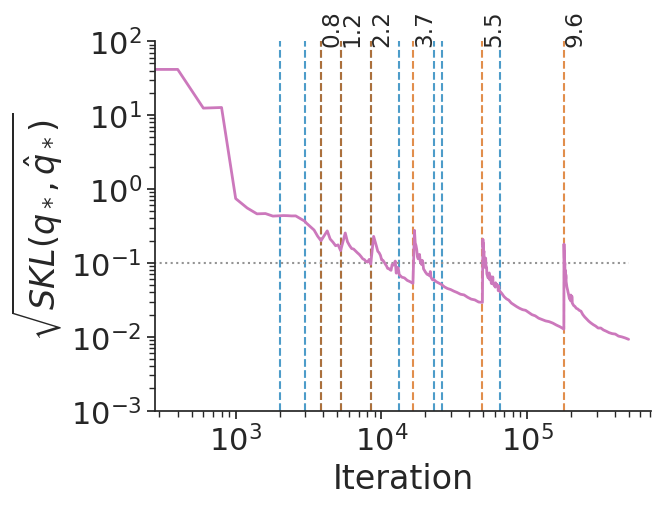}
\caption{banded correlated $\paramdim = 100$} 
\end{subfigure}   
\caption{
Results of termination rule trigger point of Gaussian targets.
\textbf{(left)} Learning rate versus square root of estimated symmetrized KL divergence with 95\% credible interval (dashed blue line).  The green vertical line indicates
the termination rule trigger point with corresponding $\hat\ineff$ value.
\textbf{(right)} Iterations versus square root of symmetrized KL divergence between iterate average and optimal variational approximation.
The vertical lines indicate the convergence detection points using $\hat{R}$ (blue) and inefficiency index computation ($\hat\ineff$) points (orange) with corresponding values.}
\label{fig:SKL-termination-point-more}
\end{center}
\end{figure}
\begin{figure}[tbp]
\begin{center}
\begin{subfigure}[t]{\textwidth}
\centering
\includegraphics[width=0.4\textwidth]{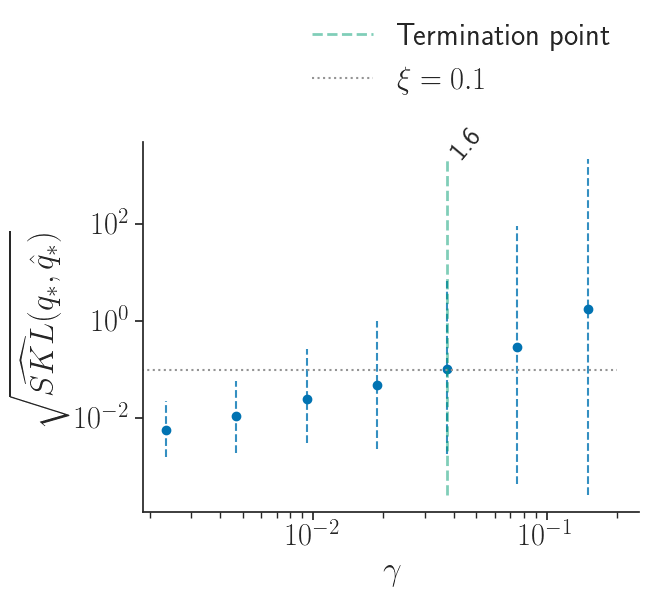} 
\includegraphics[width=0.4\textwidth]{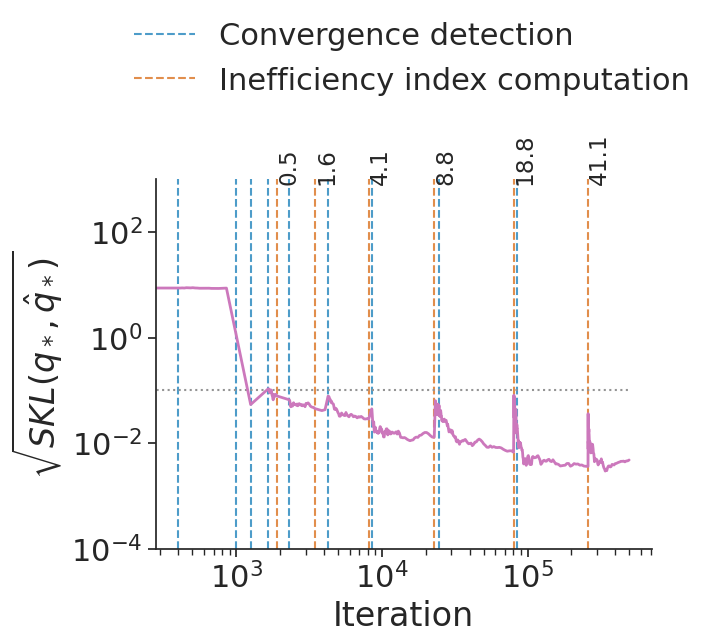}
\caption{dogs} 
\end{subfigure} 
\begin{subfigure}[t]{\textwidth}
\centering
\includegraphics[width=0.4\textwidth]{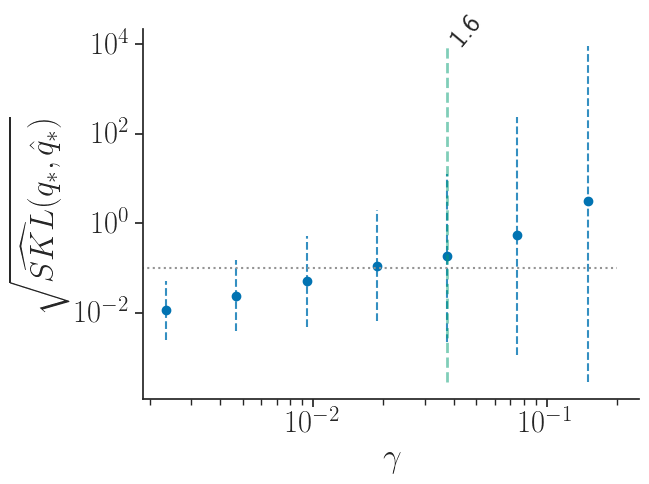} 
\includegraphics[width=0.4\textwidth]{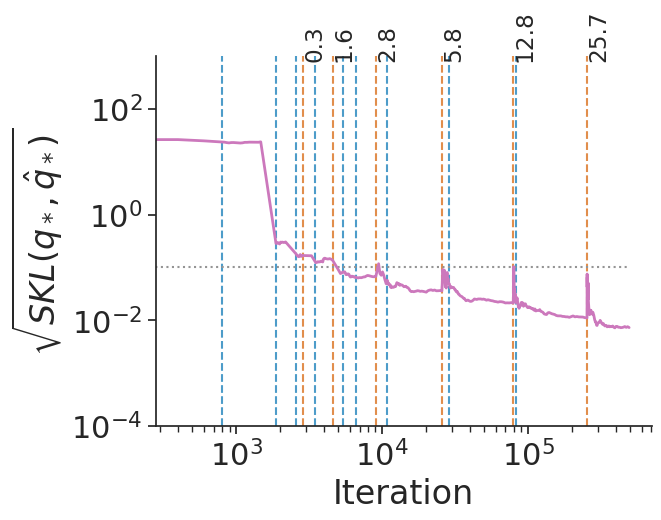}
\caption{arK} 
\end{subfigure}  
\begin{subfigure}[t]{\textwidth}
\centering
\includegraphics[width=0.4\textwidth]{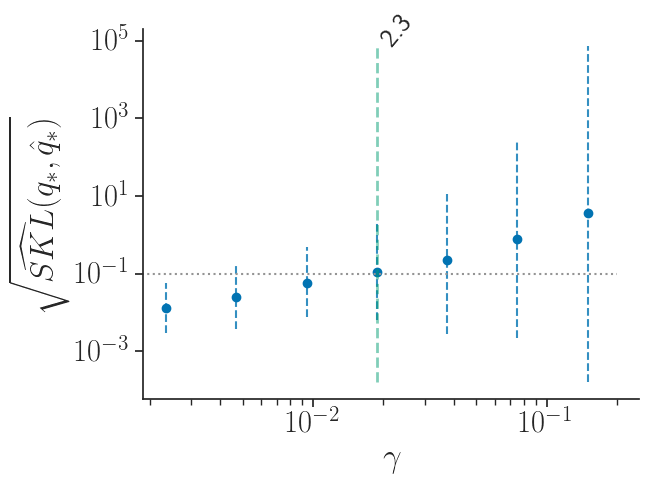} 
\includegraphics[width=0.4\textwidth]{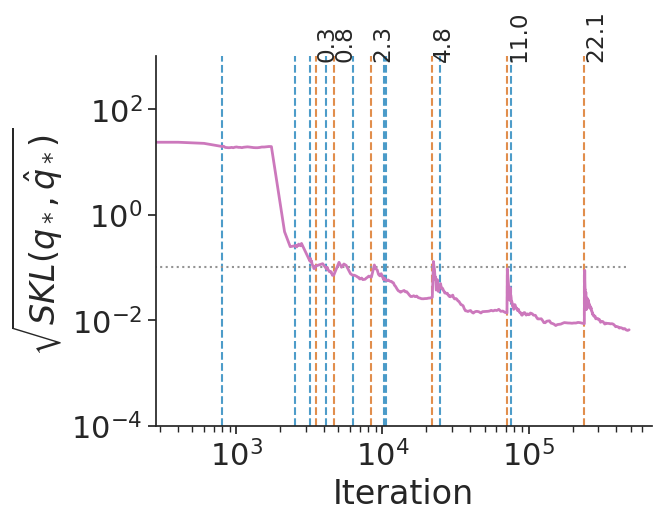}
\caption{nes2000} 
\end{subfigure}   
\caption{
Results of termination rule trigger point of  \texttt{posteriordb} package datasets/models.
\textbf{(left)} Learning rate versus square root of estimated symmetrized KL divergence with 95\% credible interval (dashed blue line).  The green vertical line indicates
the termination rule trigger point with corresponding $\hat\ineff$ value.
\textbf{(right)} Iterations versus square root of symmetrized KL divergence between iterate average and optimal variational approximation.
The vertical lines indicate the convergence detection points using $\hat{R}$ (blue) and inefficiency index computation ($\hat\ineff$) points (orange) with corresponding values.}
\label{fig:SKL-termination-point-posteriordb}
\end{center}
\end{figure}

\begin{figure}[tbp]
\begin{center}
\begin{subfigure}[t]{\textwidth}
\centering
\includegraphics[width=0.4\textwidth]{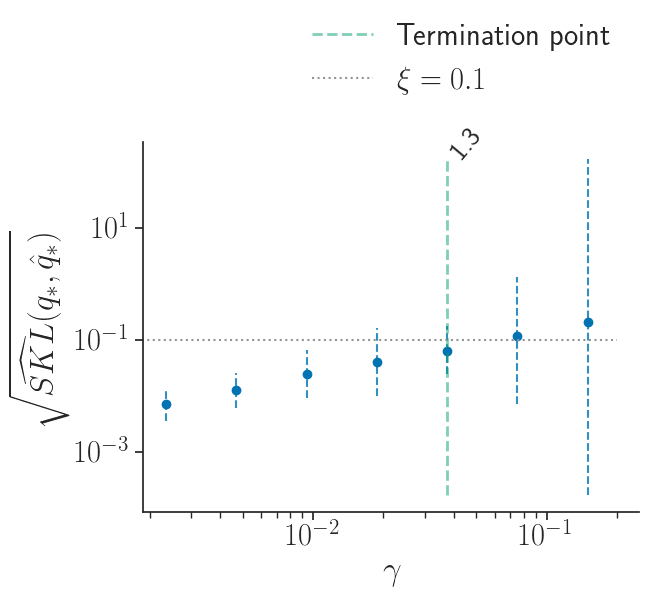} 
\includegraphics[width=0.4\textwidth]{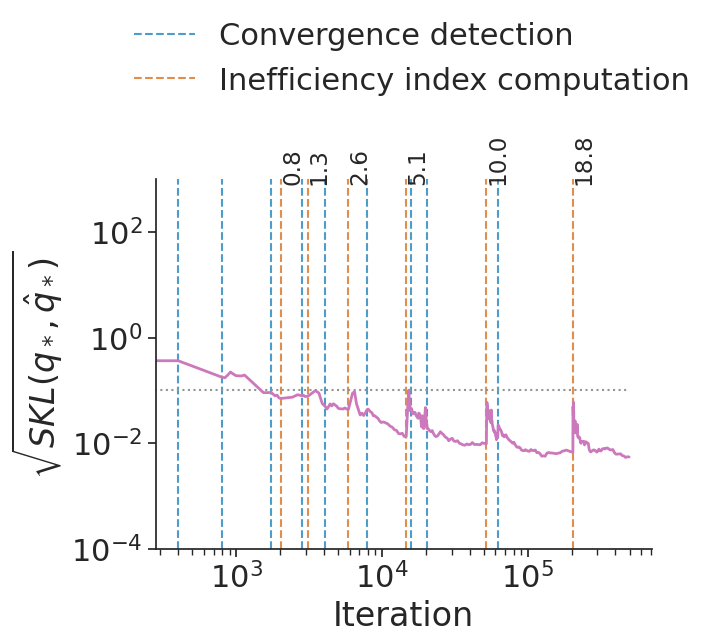}
\caption{8schools\_nc} 
\end{subfigure} 
\begin{subfigure}[t]{\textwidth}
\centering
\includegraphics[width=0.4\textwidth]{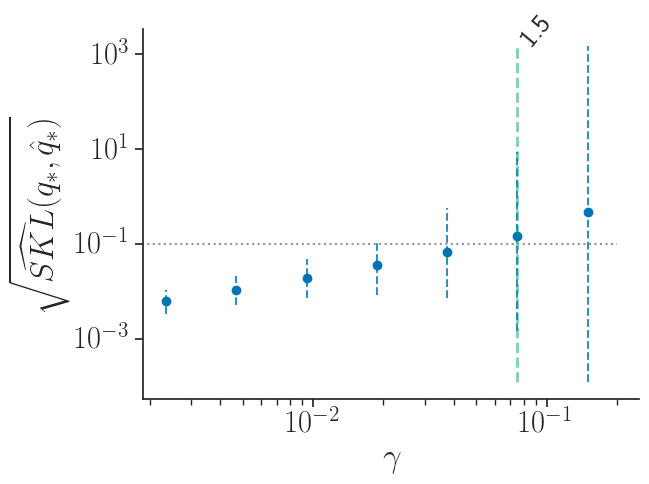} 
\includegraphics[width=0.4\textwidth]{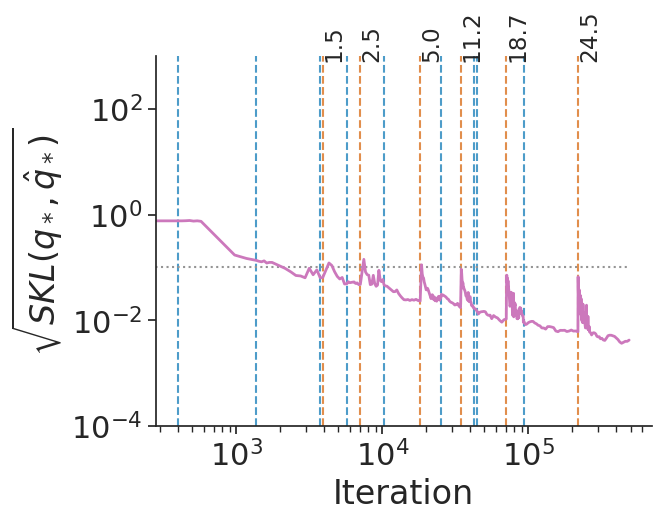}
\caption{8schools\_c} 
\end{subfigure} 
\caption{
Results of termination rule trigger point of  \texttt{posteriordb} package datasets/models.
\textbf{(left)} Learning rate versus square root of estimated symmetrized KL divergence with 95\% credible interval (dashed blue line).  The green vertical line indicates
the termination rule trigger point with corresponding $\hat\ineff$ value.
\textbf{(right)} Iterations versus square root of symmetrized KL divergence between iterate average and optimal variational approximation.
The vertical lines indicate the convergence detection points using $\hat{R}$ (blue) and inefficiency index computation ($\hat\ineff$) points (orange) with corresponding values.}
\label{fig:SKL-termination-point-posteriordb-more}
\end{center}
\end{figure}

\begin{figure}[tbp]
\begin{center}
\begin{subfigure}[t]{.4\textwidth}
\centering
\includegraphics[width=\textwidth]{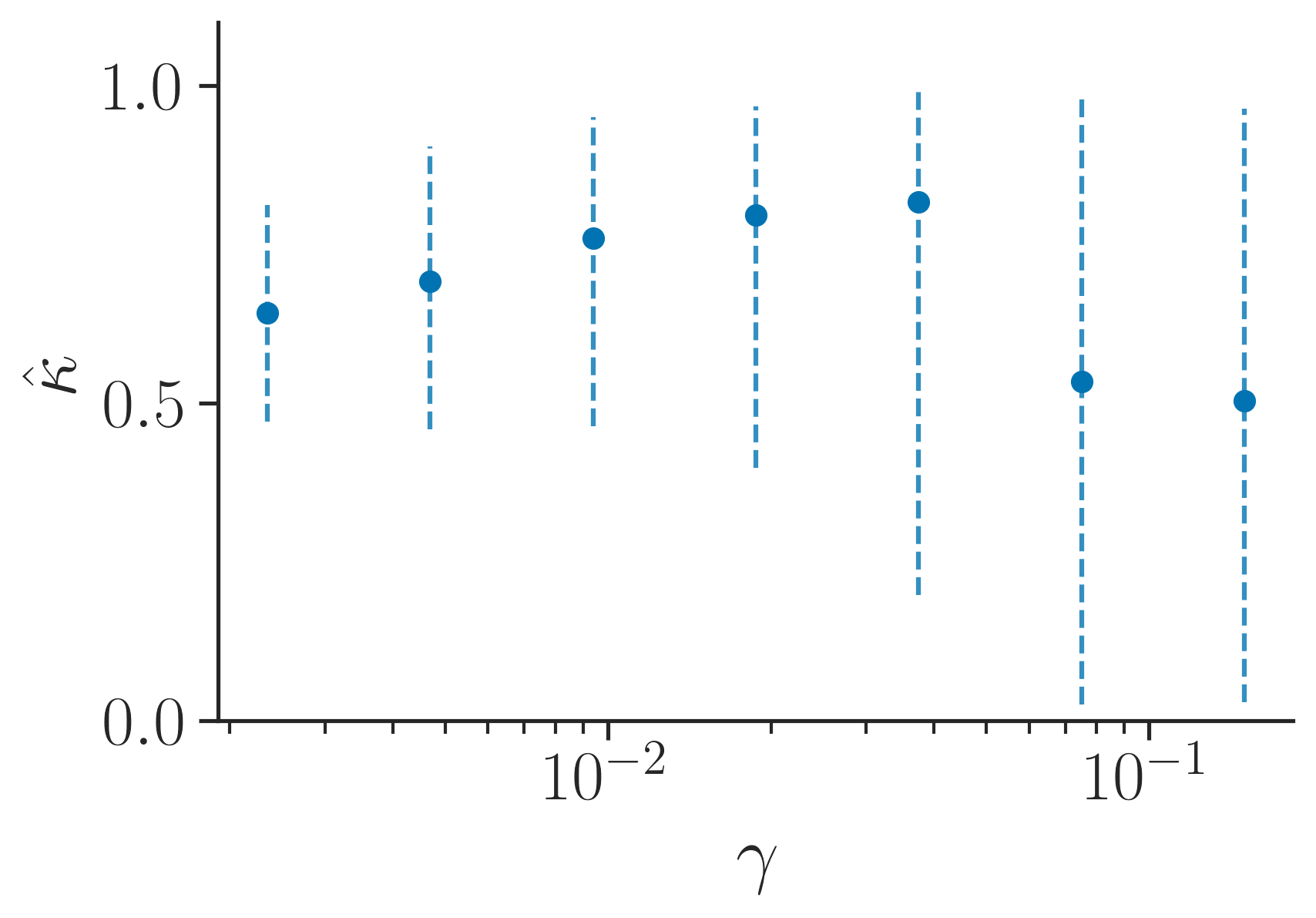}
\caption{uncorrelated $\paramdim = 100$} 
\end{subfigure} 
\begin{subfigure}[t]{.4\textwidth}
\centering
\includegraphics[width=\textwidth]{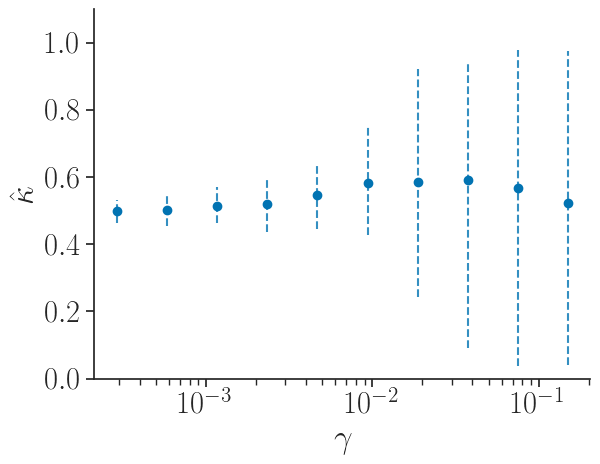}
\caption{uncorrelated $\paramdim = 500$} 
\end{subfigure} 
\begin{subfigure}[t]{.4\textwidth}
\centering
\includegraphics[width=\textwidth]{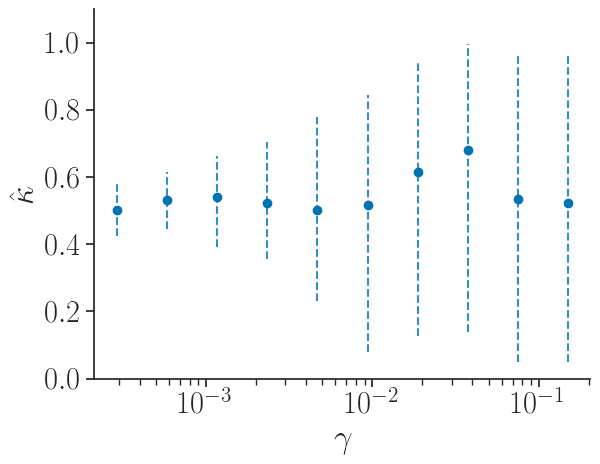}
\caption{uniform correlated $\paramdim = 100$} 
\end{subfigure}  
\begin{subfigure}[t]{.4\textwidth}
\centering
\includegraphics[width=\textwidth]{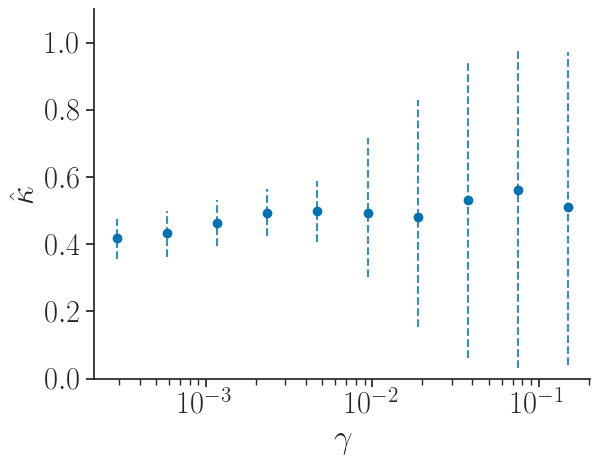}
\caption{banded correlated $\paramdim = 100$} 
\end{subfigure}   
\caption{Learning rate versus $\hat{\biaspower}$ for Gaussian targets using RMSProp with 95\% credible interval. 
The estimates suggestion $\biaspower$ is approximately $0.5$.}
\label{fig:kappa_hat_comparisons-rmsprop}
\end{center}
\end{figure}

\begin{figure}[tbp]
\begin{center}
\begin{subfigure}[t]{.4\textwidth}
\centering
\includegraphics[width=\textwidth]{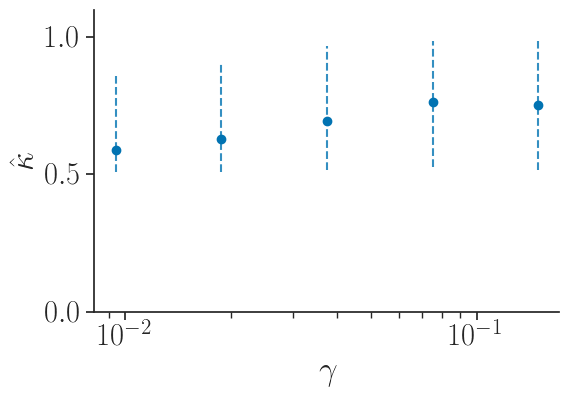}
\caption{uncorrelated $\paramdim = 100$} 
\end{subfigure} 
\begin{subfigure}[t]{.4\textwidth}
\centering
\includegraphics[width=\textwidth]{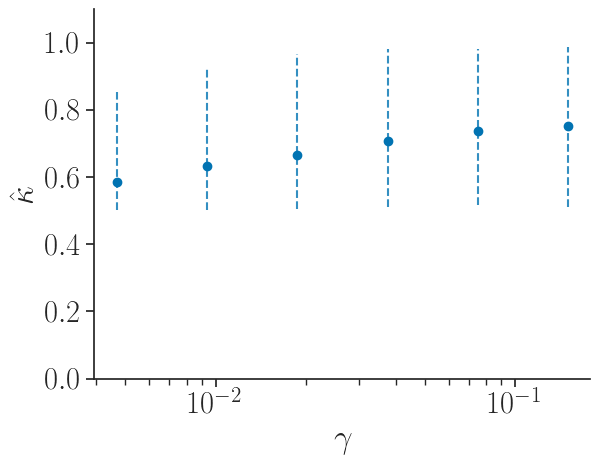}
\caption{uncorrelated $\paramdim = 500$} 
\end{subfigure} 
\begin{subfigure}[t]{.4\textwidth}
\centering
\includegraphics[width=\textwidth]{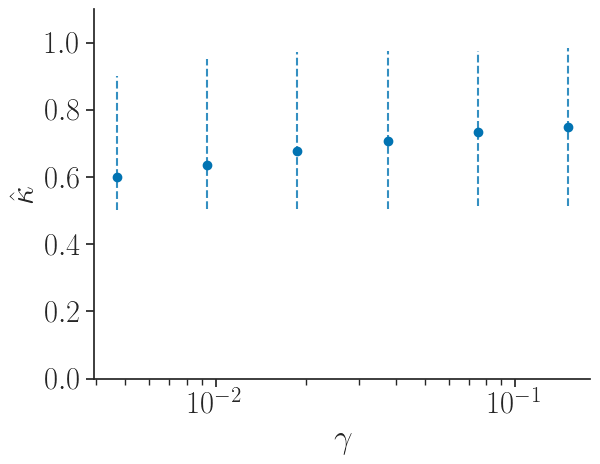}
\caption{uniform correlated $\paramdim = 100$} 
\end{subfigure}  
\begin{subfigure}[t]{.4\textwidth}
\centering
\includegraphics[width=\textwidth]{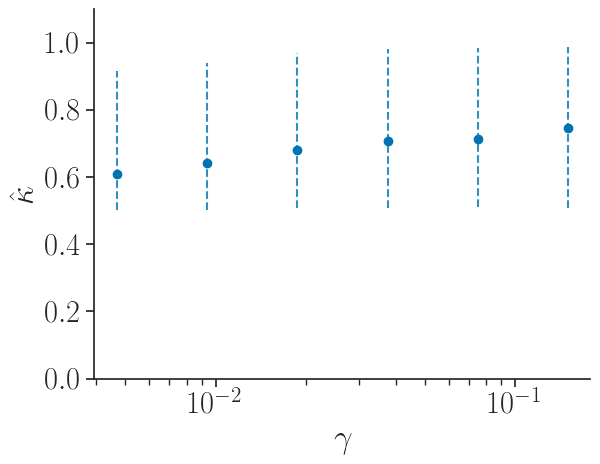}
\caption{banded correlated $\paramdim = 100$} 
\end{subfigure}   
\caption{
Learning rate versus $\hat{\biaspower}$ for Gaussian targets using Adam with 95\% credible interval. 
The estimates suggest $\biaspower$ is less than 0.8, with all point estimates close to 0.6.}
\label{fig:kappa_hat_comparisons-adam}
\end{center}
\end{figure}

\begin{figure}[tbp]
\begin{center}
\begin{subfigure}[t]{.48\textwidth}
\centering
\includegraphics[width=\textwidth]{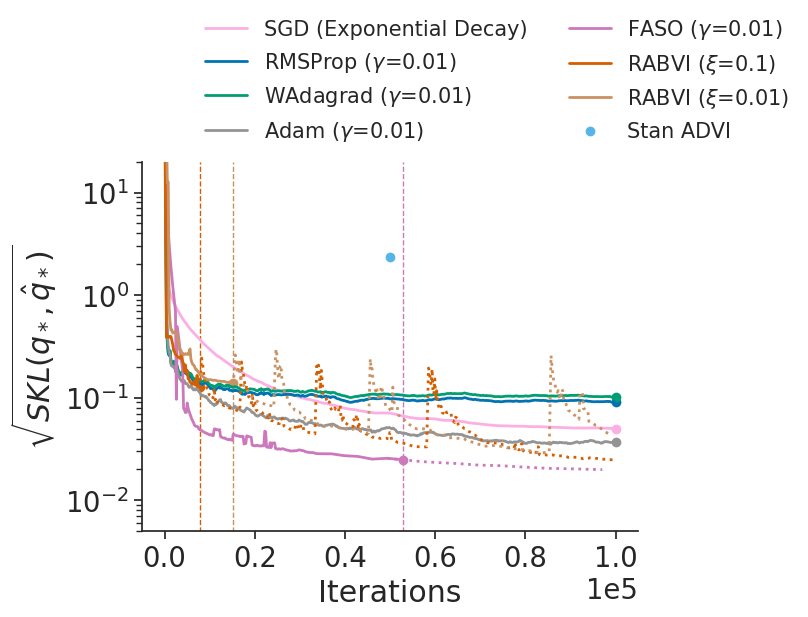}
\caption{diagonal non-identity banded correlated $\paramdim = 100$} 
\end{subfigure}  
\begin{subfigure}[t]{.48\textwidth}
\centering
\includegraphics[width=\textwidth]{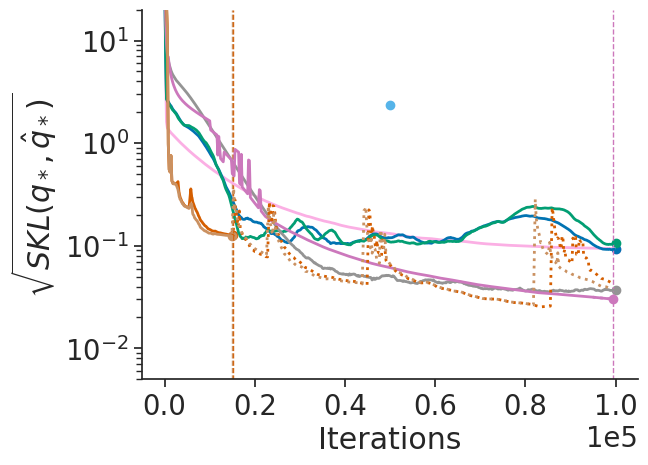}
\caption{diagonal identity (except first entry) uniform correlated $\paramdim = 100$} 
\end{subfigure} 
\begin{subfigure}[t]{.48\textwidth}
\centering
\includegraphics[width=\textwidth]{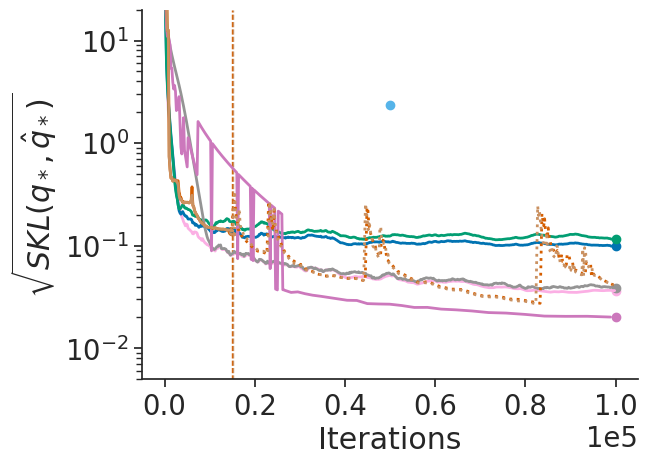}
\caption{diagonal identity (except first entry) banded correlated $\paramdim = 100$} 
\end{subfigure}  
\caption{
Accuracy comparison of variational inference algorithms using Gaussian targets,
where accuracy is measured in terms of the square root of symmetrized KL divergence between iterate average and optimal variational approximation.
The vertical lines indicate the termination rule trigger points of FASO and RABVI.
Iterate averages for Adam, RMSProp, and WAdagrad computed at every 200th iteration using a window 
size of 20\% of iterations.
}
\label{fig:Accuracy-with-gaussian-targets-high-condion-covariance-matrix}
\end{center}
\end{figure}

\begin{figure}[tbp]
\begin{center}
\begin{subfigure}[t]{.48\textwidth}
\centering
\includegraphics[width=\textwidth]{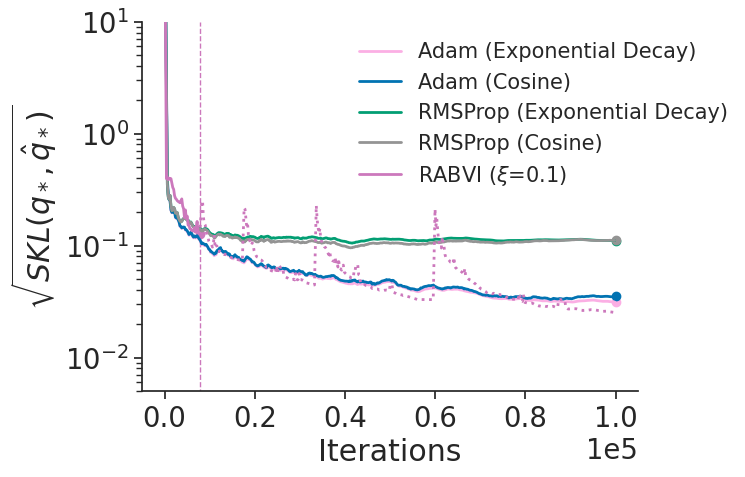}\\
\caption{uncorrelated $\paramdim = 100$} 
\end{subfigure}  
\begin{subfigure}[t]{.48\textwidth}
\centering
\includegraphics[width=\textwidth]{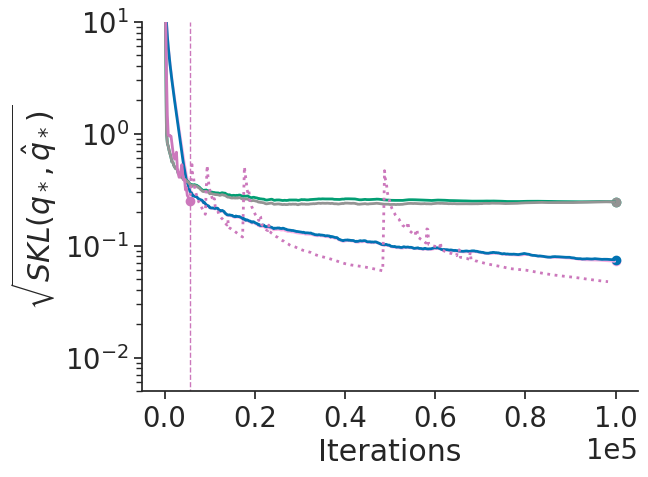}\\
\caption{uncorrelated $\paramdim = 500$} 
\end{subfigure} 
\begin{subfigure}[t]{.48\textwidth}
\centering
\includegraphics[width=\textwidth]{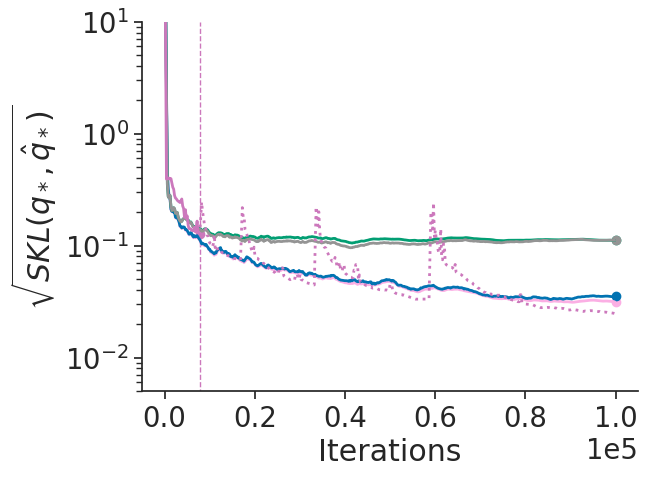}\\
\caption{uniform correlated $\paramdim = 100$} 
\end{subfigure}  
\begin{subfigure}[t]{.48\textwidth}
\centering
\includegraphics[width=\textwidth]{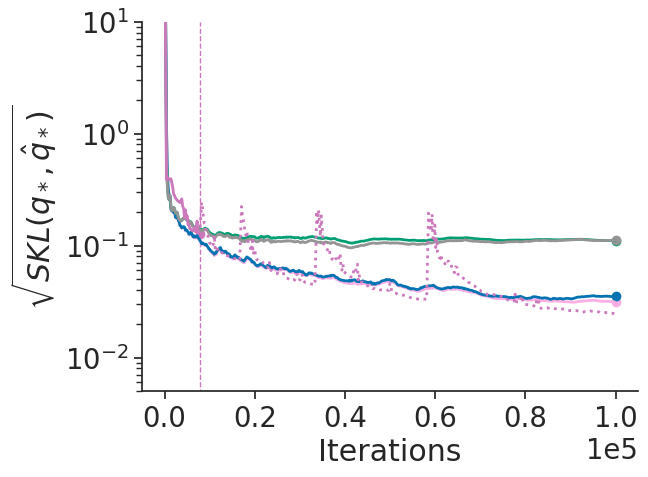}\\
\caption{banded correlated $\paramdim = 100$} 
\end{subfigure} 
\caption{
Accuracy comparison across learning rate schedules with Gaussian targets where accuracy is measured in terms of the square root of symmetrized KL divergence between iterate average and optimal variational approximation. 
The vertical lines indicate the termination rule trigger points of RABVI.}
\label{fig:learning-rate-schedules-with-gaussian-targets}
\end{center}
\end{figure}

\begin{figure}[tbp]
\begin{center}
\begin{subfigure}[t]{.48\textwidth}
\centering
\includegraphics[width=\textwidth]{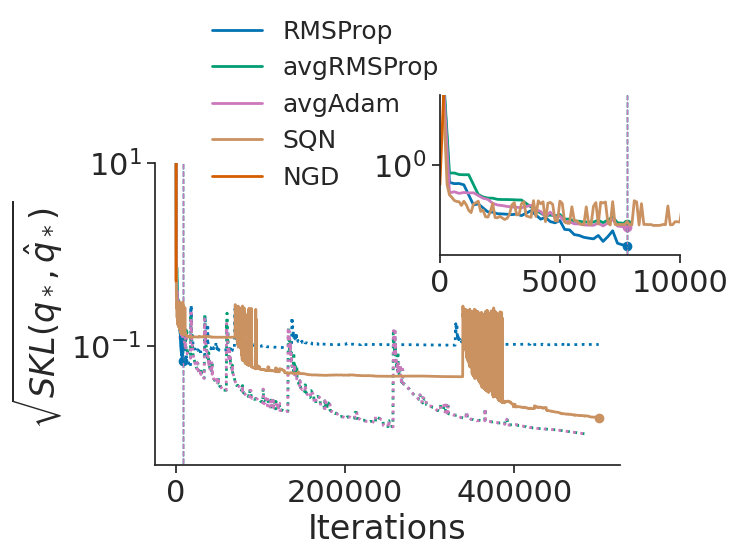}\\
\caption{uncorrelated $\paramdim = 100$} 
\end{subfigure}  
\begin{subfigure}[t]{.48\textwidth}
\centering
\includegraphics[width=\textwidth]{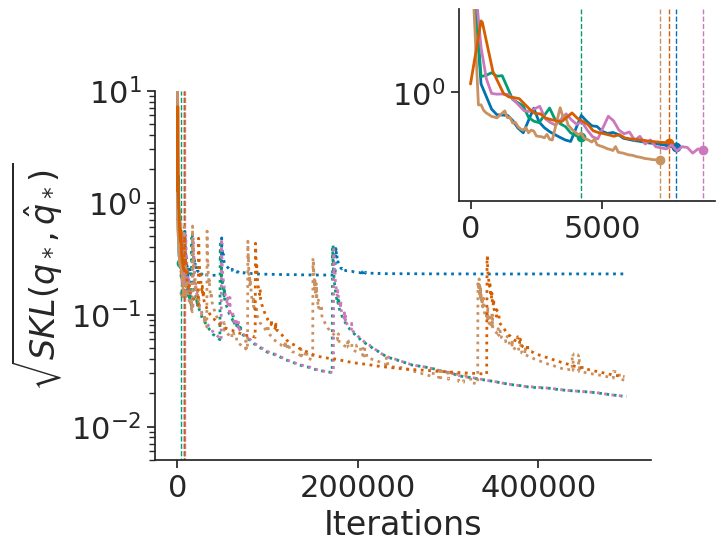}\\
\caption{uncorrelated $\paramdim = 500$} 
\end{subfigure} 
\begin{subfigure}[t]{.48\textwidth}
\centering
\includegraphics[width=\textwidth]{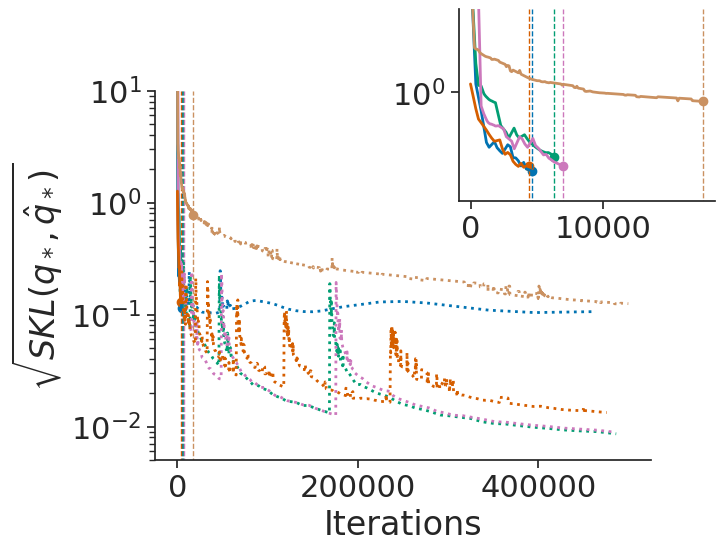}\\
\caption{uniform correlated $\paramdim = 100$} 
\end{subfigure}  
\begin{subfigure}[t]{.48\textwidth}
\centering
\includegraphics[width=\textwidth]{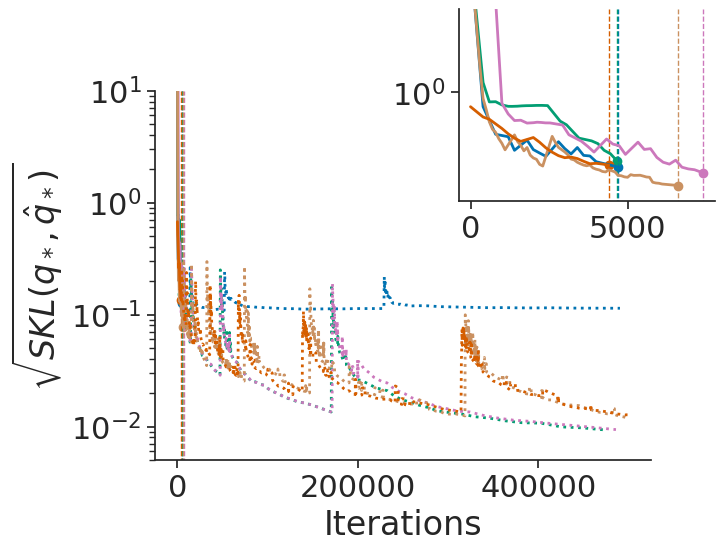}\\
\caption{banded correlated $\paramdim = 100$} 
\end{subfigure} 
\caption{
Accuracy comparison of \textit{RMSProp, avgRMSProp, avgAdam, SQN,} and \textit{NGD} optimization methods in RABVI using Gaussian targets where accuracy is measured in terms of square root of symmetrized KL divergence between iterate average and optimal variational approximation. 
The vertical lines indicate the termination rule trigger points and the behavior of optimization methods at the trigger points showed in inset plots.}
\label{fig:RABVI-with-gaussian-targets}
\end{center}
\end{figure}

\begin{figure}[tbp]
\begin{center}
\begin{subfigure}[t]{.32\textwidth}
\centering
\includegraphics[width=\textwidth,height=0.17\textheight]{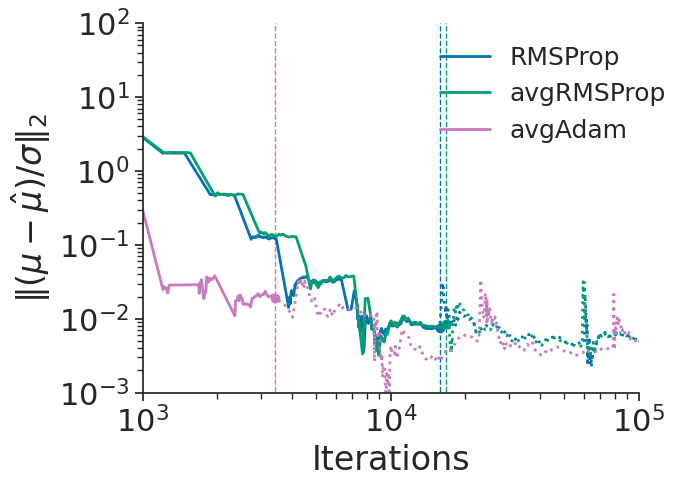}\\
\includegraphics[width=\textwidth]{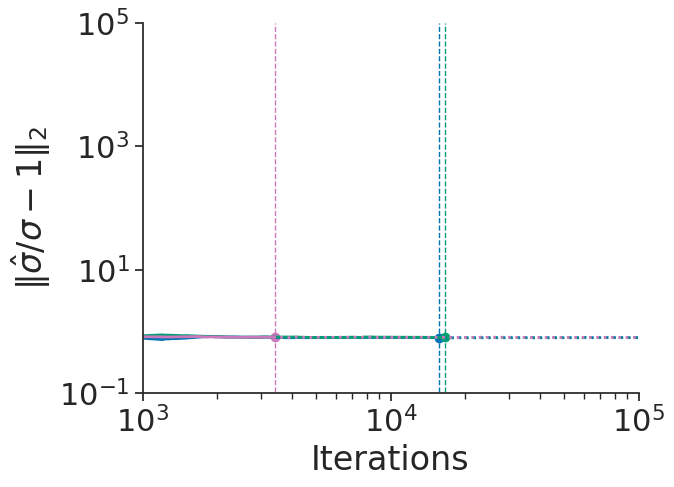}
\caption{dogs} 
\end{subfigure}  
\begin{subfigure}[t]{.32\textwidth}
\centering
\includegraphics[width=\textwidth]{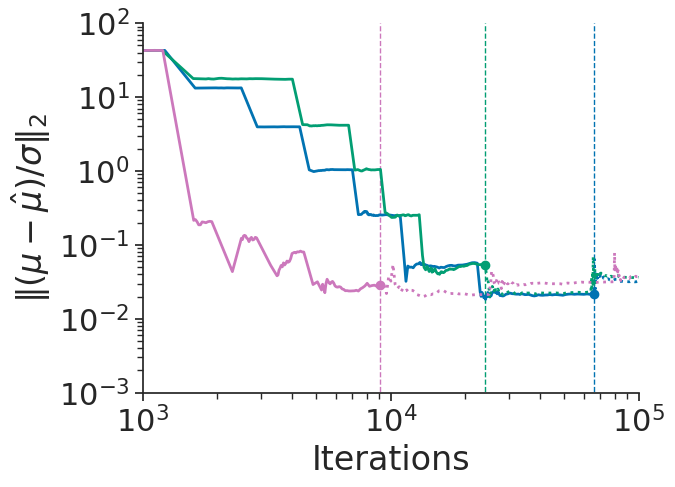}\\
\includegraphics[width=\textwidth]{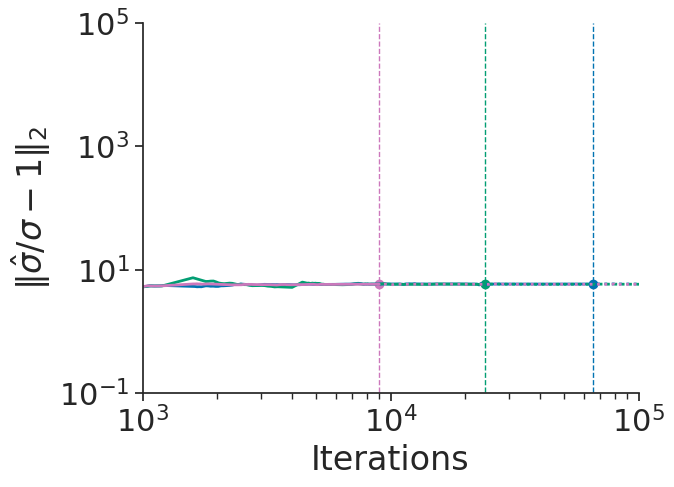}
\caption{arK} 
\end{subfigure}  
\begin{subfigure}[t]{.32\textwidth}
\centering
\includegraphics[width=\textwidth]{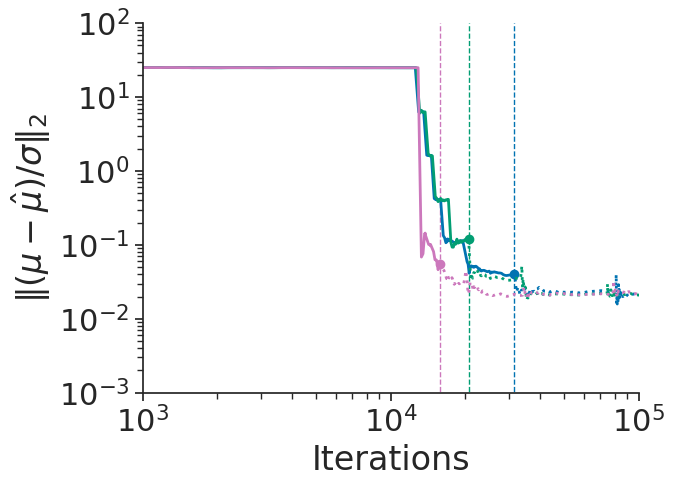}\\
\includegraphics[width=\textwidth]{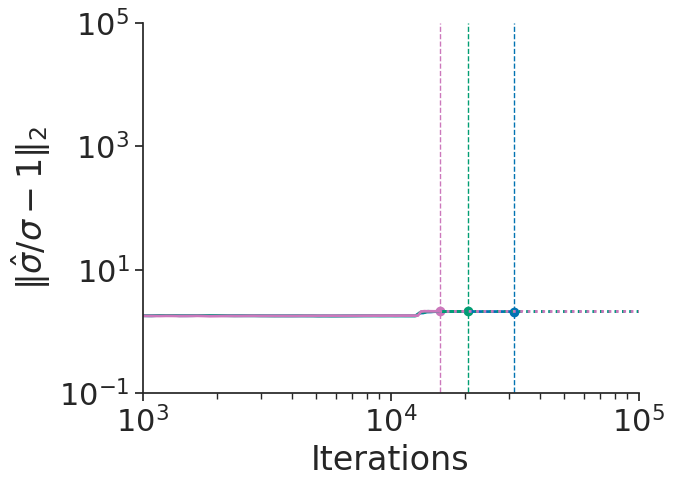}
\caption{nes2000} 
\end{subfigure}  
\begin{subfigure}[t]{.32\textwidth}
\centering
\includegraphics[width=\textwidth]{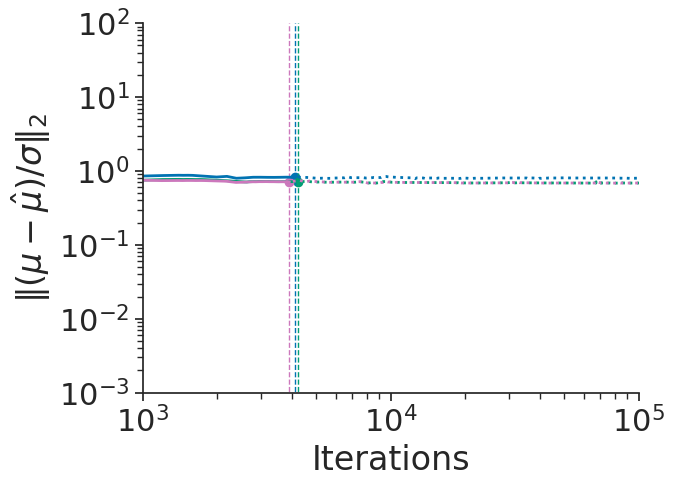}\\
\includegraphics[width=\textwidth]{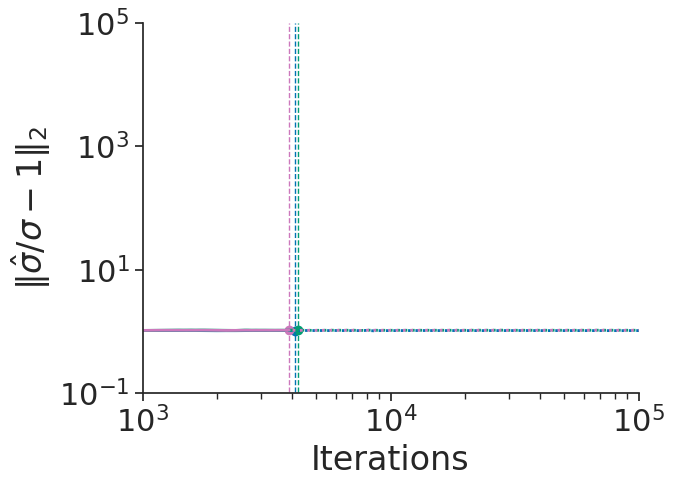}
\caption{8schools\_c} 
\end{subfigure}  
\begin{subfigure}[t]{.32\textwidth}
\centering
\includegraphics[width=\textwidth]{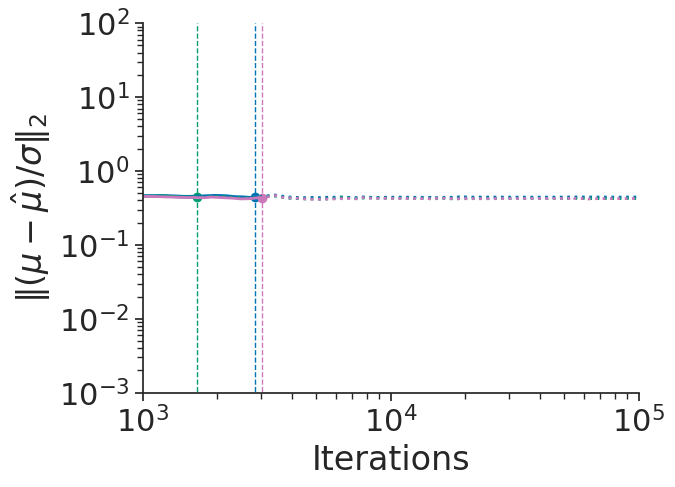}\\
\includegraphics[width=\textwidth]{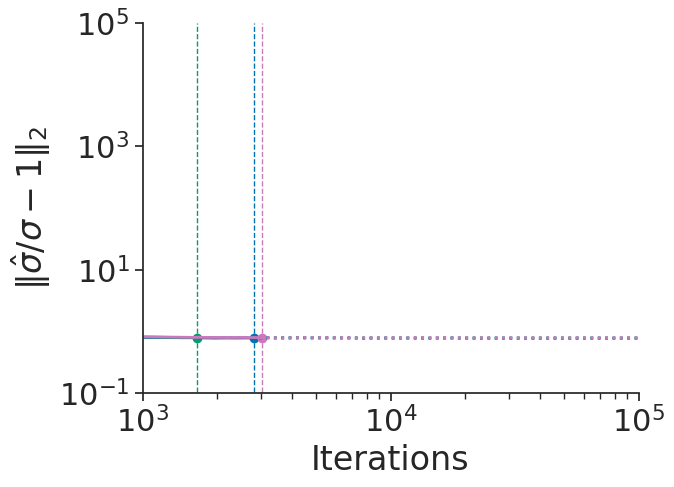}
\caption{8schools\_nc} 
\end{subfigure}  
\caption{
Accuracy comparison of  \textit{RMSProp, avgRMSProp,} and \textit{avgAdam} optimization methods in RABVI using \textit{posteriordb} datasets, where accuracy measured in terms of relative mean error (top) and relative standard deviation error (bottom). 
The vertical lines indicate the termination rule trigger points.}
\label{fig:RABVI-with-posteriordb}
\end{center}
\end{figure}

\begin{figure}[tbp]
\begin{center}
\begin{subfigure}[t]{.24\textwidth}
\centering
\includegraphics[width=\textwidth]{Posteriordb/mf_vs_fr/mf_vs_fr_arK_arK_mean_avgadam.png}
\caption{arK} 
\end{subfigure} 
\begin{subfigure}[t]{.24\textwidth}
\centering
\includegraphics[width=\textwidth]{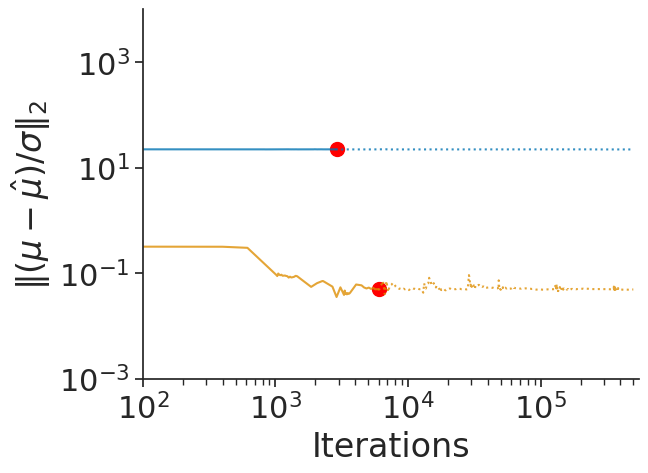}
\caption{bball\_0} 
\end{subfigure}
\begin{subfigure}[t]{.24\textwidth}
\centering
\includegraphics[width=\textwidth]{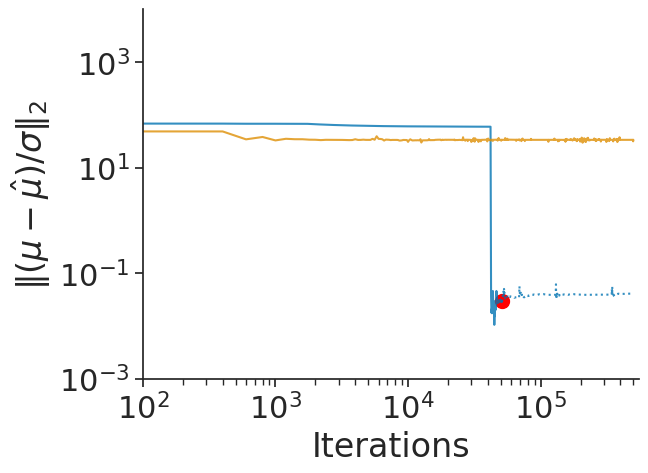}
\caption{bball\_1}
\end{subfigure}
\begin{subfigure}[t]{.24\textwidth}
\centering
\includegraphics[width=\textwidth]{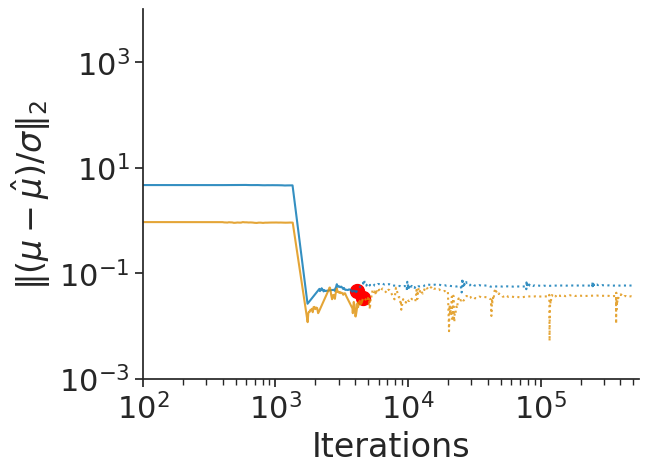}
\caption{dogs\_log} 
\end{subfigure}
\begin{subfigure}[t]{.24\textwidth}
\centering
\includegraphics[width=\textwidth]{Posteriordb/mf_vs_fr/mf_vs_fr_dogs_dogs_mean_avgadam.png}
\caption{dogs} 
\end{subfigure}
\begin{subfigure}[t]{.24\textwidth}
\centering
\includegraphics[width=\textwidth]{Posteriordb/mf_vs_fr/mf_vs_fr_diamonds_diamonds_mean_avgadam.png}
\caption{diamonds} 
\end{subfigure}
\begin{subfigure}[t]{.24\textwidth}
\centering
\includegraphics[width=\textwidth]{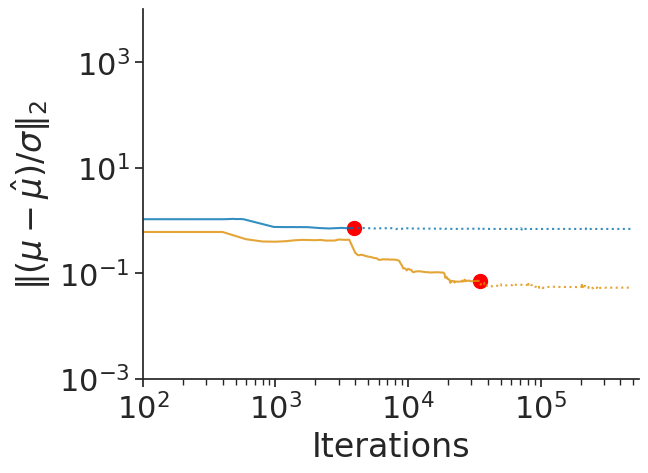}
\caption{8schools\_c} 
\end{subfigure}
\begin{subfigure}[t]{.24\textwidth}
\centering
\includegraphics[width=\textwidth]{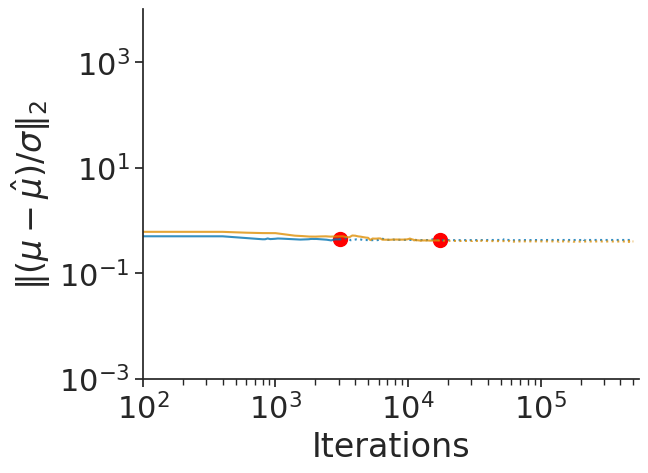}
\caption{8schools\_nc} 
\end{subfigure}
\begin{subfigure}[t]{.24\textwidth}
\centering
\includegraphics[width=\textwidth]{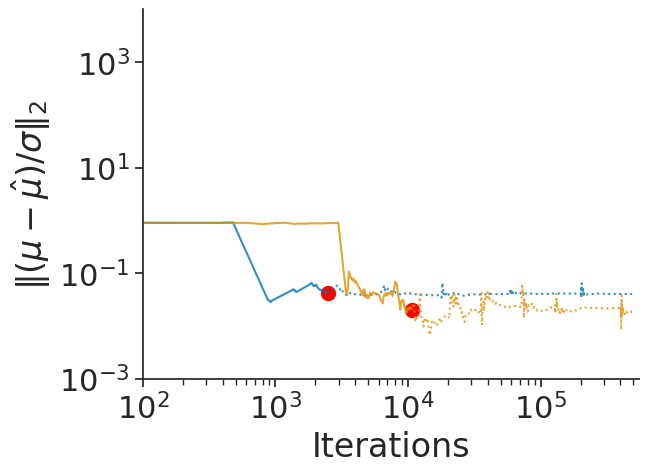}
\caption{hmm\_example}
\end{subfigure}
\begin{subfigure}[t]{.24\textwidth}
\centering
\includegraphics[width=\textwidth]{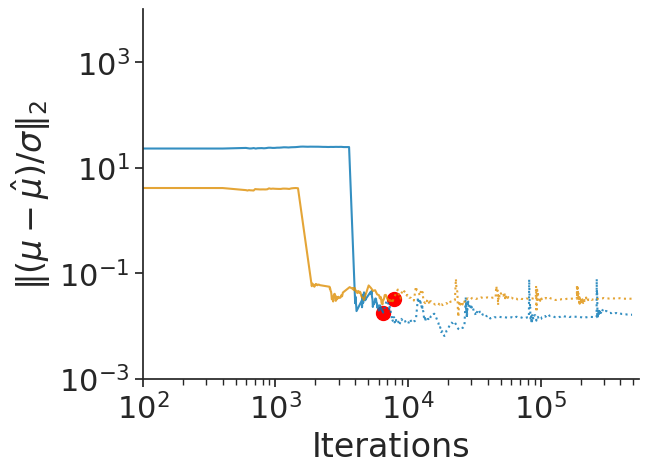}
\caption{low\_dim\_gauss} 
\end{subfigure}
\begin{subfigure}[t]{.24\textwidth}
\centering
\includegraphics[width=\textwidth]{Posteriordb/mf_vs_fr/mf_vs_fr_nes2000_nes_mean_avgadam.png}
\caption{nes2000} 
\end{subfigure}
\begin{subfigure}[t]{.24\textwidth}
\centering
\includegraphics[width=\textwidth]{Posteriordb/mf_vs_fr/mf_vs_fr_sblrc_blr_mean_avgadam.png}
\caption{sblrc} 
\end{subfigure}
\begin{subfigure}[t]{.24\textwidth}
\centering
\includegraphics[width=\textwidth]{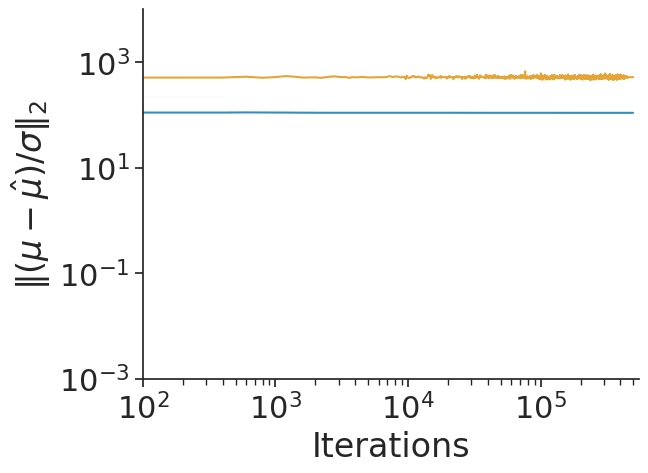}
\caption{earnings} 
\end{subfigure}
\begin{subfigure}[t]{.24\textwidth}
\centering
\includegraphics[width=\textwidth]{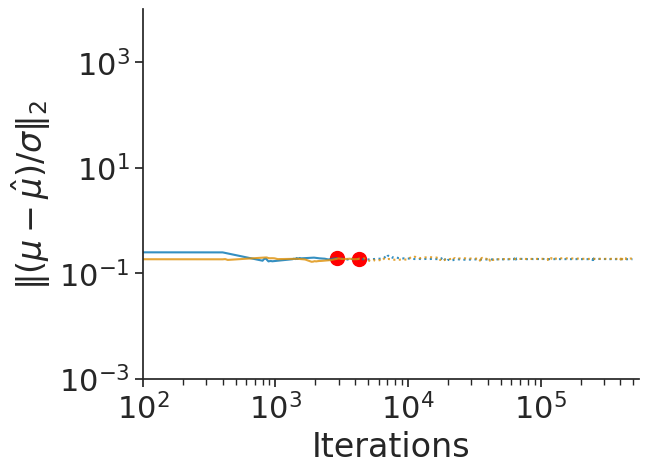}
\caption{garch} 
\end{subfigure}
\begin{subfigure}[t]{.24\textwidth}
\centering
\includegraphics[width=\textwidth]{Posteriordb/mf_vs_fr/mf_vs_fr_gp_pois_regr_gp_pois_regr_mean_avgadam.png}
\caption{gp\_pois\_regr} 
\end{subfigure} 
\begin{subfigure}[t]{.24\textwidth}
\centering
\includegraphics[width=\textwidth]{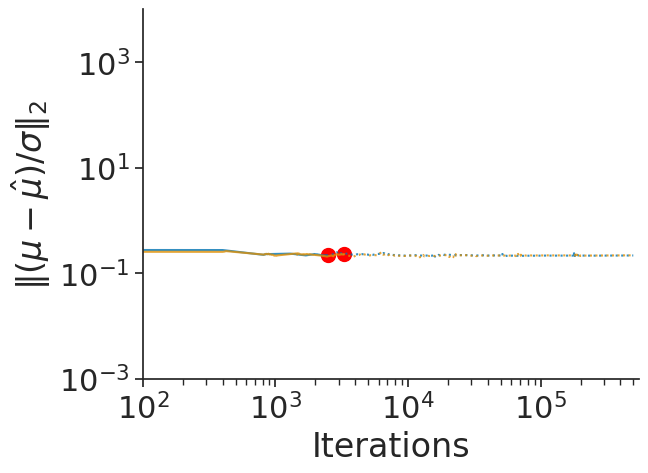}
\caption{gp\_regr} 
\end{subfigure}
\begin{subfigure}[t]{.24\textwidth}
\centering
\includegraphics[width=\textwidth]{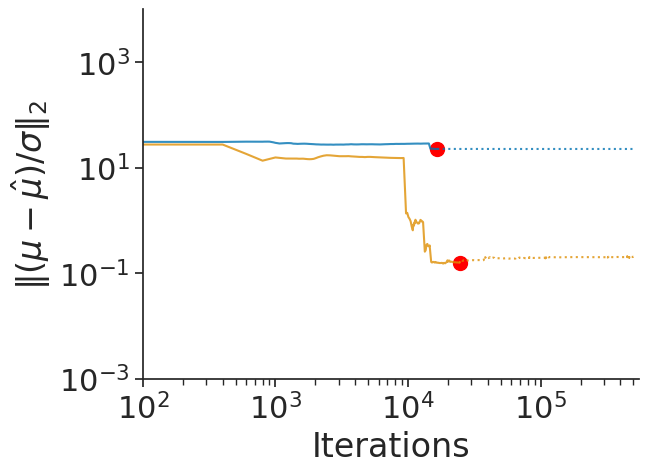}
\caption{hudson\_lynx} 
\end{subfigure}
\begin{subfigure}[t]{.24\textwidth}
\centering
\includegraphics[width=\textwidth]{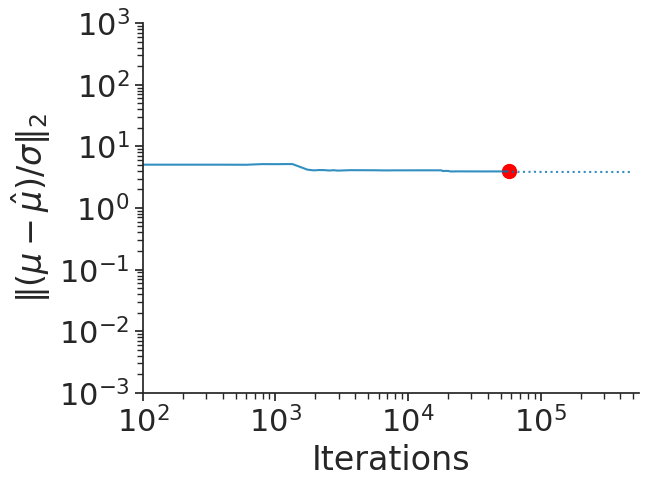}
\caption{mcycle\_gp} 
\end{subfigure}
\caption{
Accuracy of mean-field  (blue) and full-rank (orange) Gaussian family approximations for selected \texttt{posteriordb} data/models, where accuracy is measured in terms of relative mean error. 
The red dots indicate where the termination rule triggers}
\label{fig:Posteriordb-mf_vs_fr-data-and-model-comparisons-mean}
\end{center}
\end{figure}

\begin{figure}[tbp]
\begin{center}
\begin{subfigure}[t]{.24\textwidth}
\centering
\includegraphics[width=\textwidth]{Posteriordb/mf_vs_fr/mf_vs_fr_arK_arK_std_avgadam.png}
\caption{arK} 
\end{subfigure} 
\begin{subfigure}[t]{.24\textwidth}
\centering
\includegraphics[width=\textwidth]{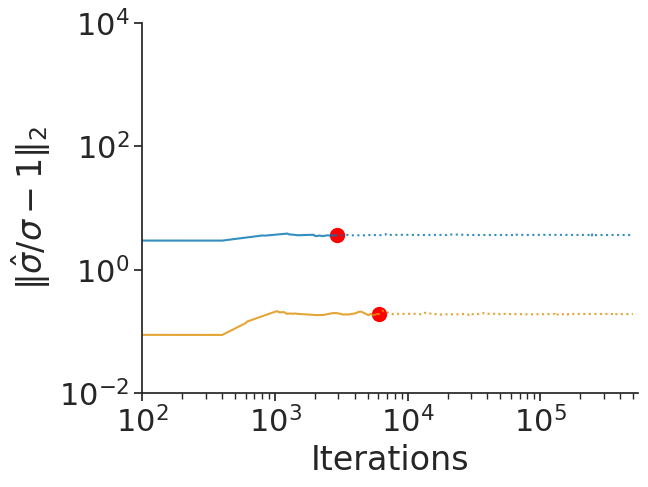}
\caption{bball\_0} 
\end{subfigure}
\begin{subfigure}[t]{.24\textwidth}
\centering
\includegraphics[width=\textwidth]{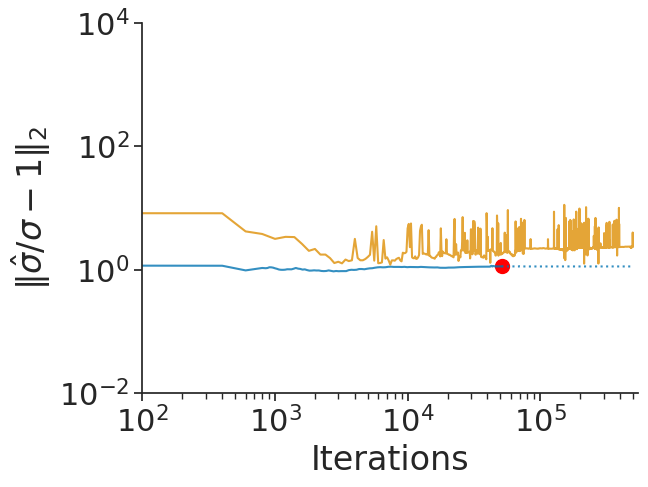}
\caption{bball\_1}
\end{subfigure}
\begin{subfigure}[t]{.24\textwidth}
\centering
\includegraphics[width=\textwidth]{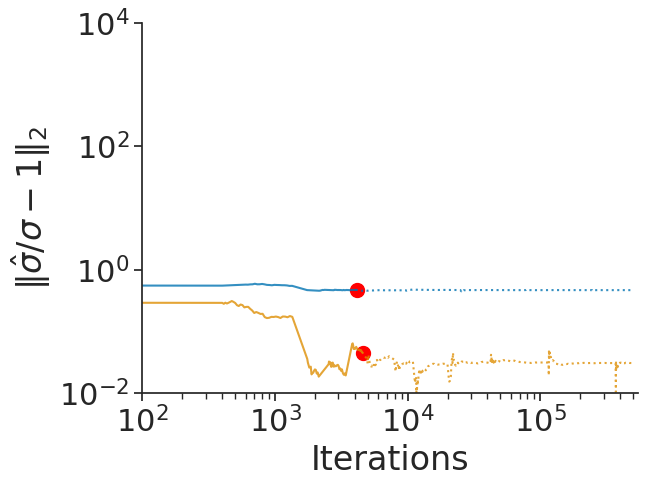}
\caption{dogs\_log} 
\end{subfigure}
\begin{subfigure}[t]{.24\textwidth}
\centering
\includegraphics[width=\textwidth]{Posteriordb/mf_vs_fr/mf_vs_fr_dogs_dogs_std_avgadam.png}
\caption{dogs} 
\end{subfigure}
\begin{subfigure}[t]{.24\textwidth}
\centering
\includegraphics[width=\textwidth]{Posteriordb/mf_vs_fr/mf_vs_fr_diamonds_diamonds_std_avgadam.png}
\caption{diamonds} 
\end{subfigure}
\begin{subfigure}[t]{.24\textwidth}
\centering
\includegraphics[width=\textwidth]{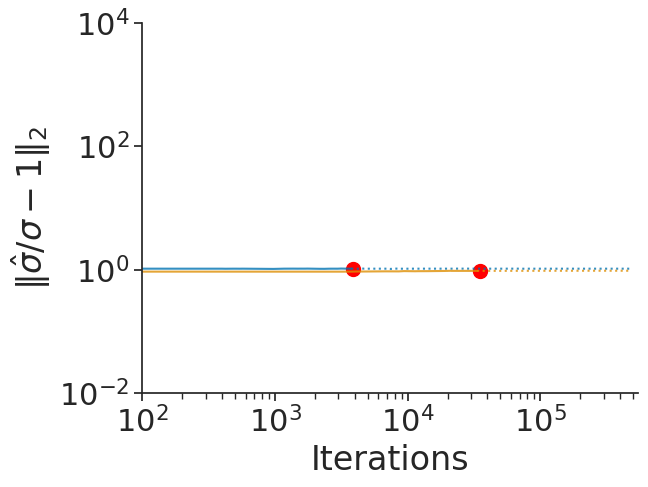}
\caption{8schools\_c} 
\end{subfigure}
\begin{subfigure}[t]{.24\textwidth}
\centering
\includegraphics[width=\textwidth]{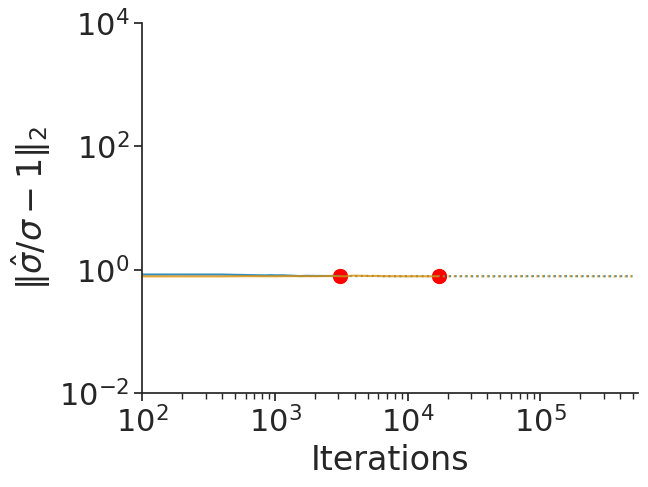}
\caption{8schools\_nc} 
\end{subfigure}
\begin{subfigure}[t]{.24\textwidth}
\centering
\includegraphics[width=\textwidth]{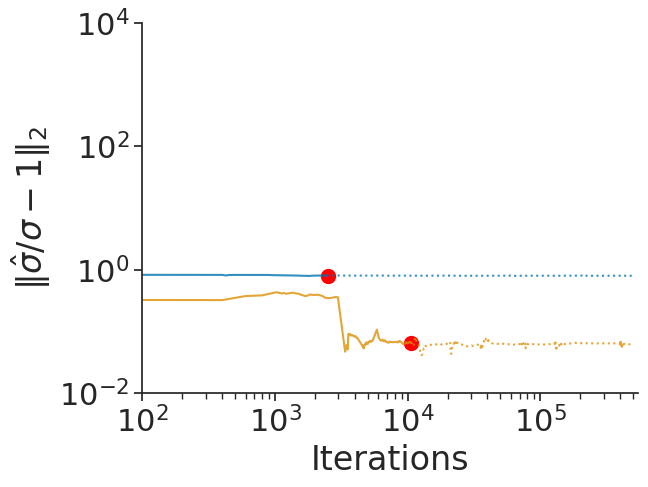}
\caption{hmm\_example}
\end{subfigure}
\begin{subfigure}[t]{.24\textwidth}
\centering
\includegraphics[width=\textwidth]{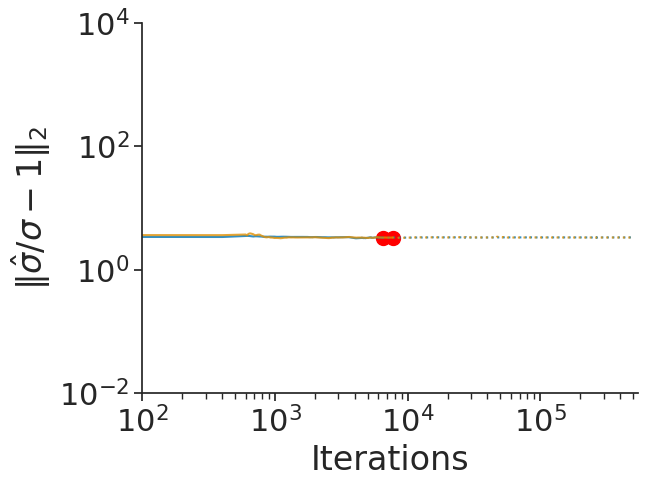}
\caption{low\_dim\_gauss} 
\end{subfigure}
\begin{subfigure}[t]{.24\textwidth}
\centering
\includegraphics[width=\textwidth]{Posteriordb/mf_vs_fr/mf_vs_fr_nes2000_nes_std_avgadam.png}
\caption{nes2000} 
\end{subfigure}
\begin{subfigure}[t]{.24\textwidth}
\centering
\includegraphics[width=\textwidth]{Posteriordb/mf_vs_fr/mf_vs_fr_sblrc_blr_std_avgadam.png}
\caption{sblrc} 
\end{subfigure}
\begin{subfigure}[t]{.24\textwidth}
\centering
\includegraphics[width=\textwidth]{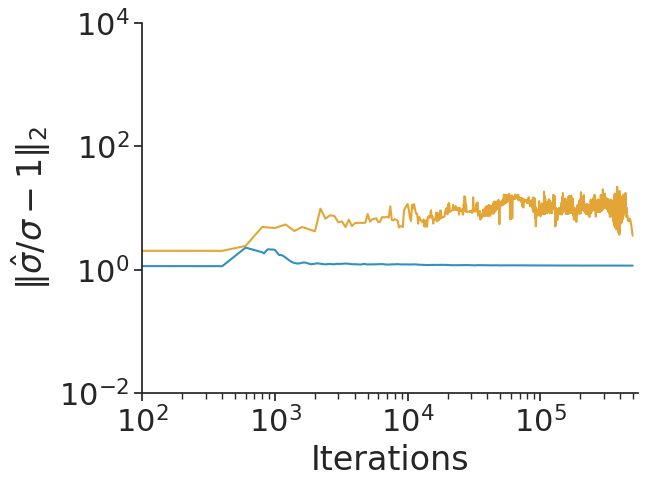}
\caption{earnings} 
\end{subfigure}
\begin{subfigure}[t]{.24\textwidth}
\centering
\includegraphics[width=\textwidth]{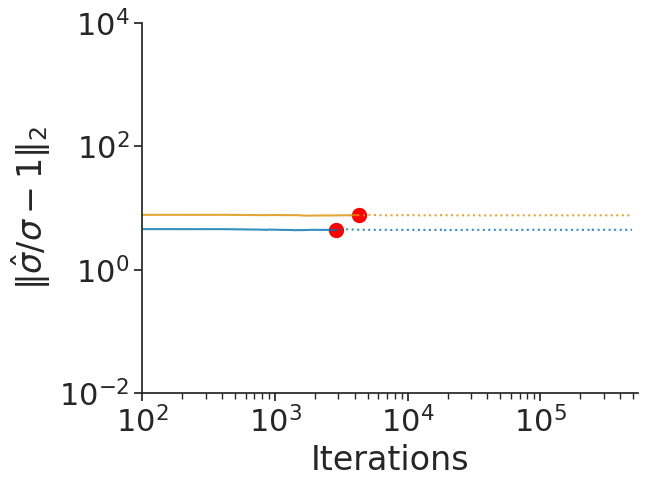}
\caption{garch} 
\end{subfigure}
\begin{subfigure}[t]{.24\textwidth}
\centering
\includegraphics[width=\textwidth]{Posteriordb/mf_vs_fr/mf_vs_fr_gp_pois_regr_gp_pois_regr_std_avgadam.png}
\caption{gp\_pois\_regr} 
\end{subfigure} 
\begin{subfigure}[t]{.24\textwidth}
\centering
\includegraphics[width=\textwidth]{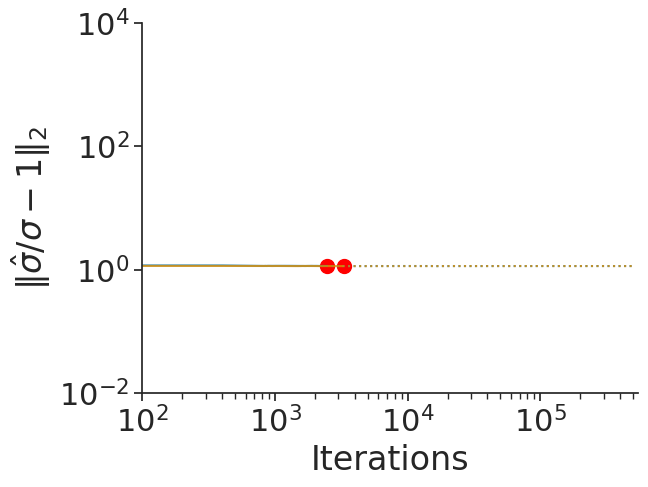}
\caption{gp\_regr} 
\end{subfigure}
\begin{subfigure}[t]{.24\textwidth}
\centering
\includegraphics[width=\textwidth]{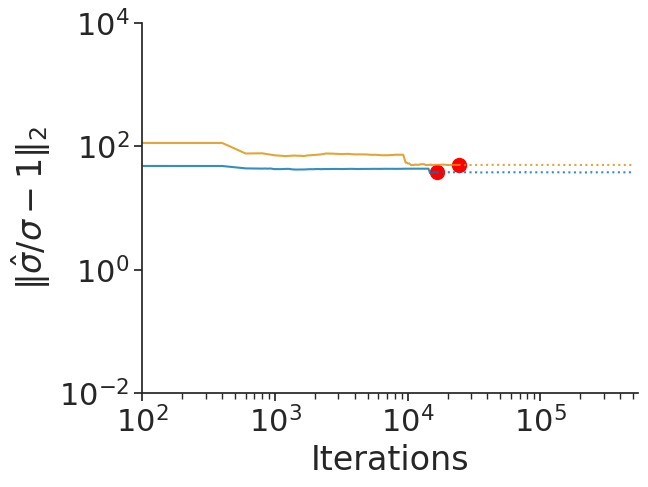}
\caption{hudson\_lynx} 
\end{subfigure}
\begin{subfigure}[t]{.24\textwidth}
\centering
\includegraphics[width=\textwidth]{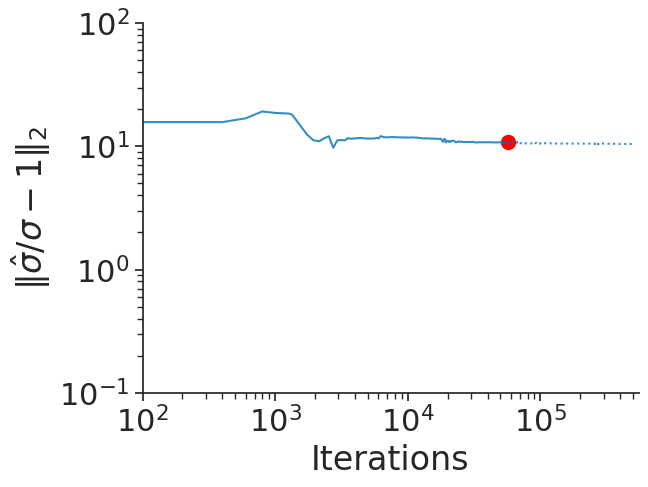}
\caption{mcycle\_gp} 
\end{subfigure}
\caption{
Accuracy of mean-field  (blue) and full-rank (orange) Gaussian family approximations for selected \texttt{posteriordb} data/models, where accuracy is measured in terms of relative standard deviation error. 
The red dots indicate where the termination rule triggers}
\label{fig:Posteriordb-mf_vs_fr-data-and-model-comparisons-std}
\end{center}
\end{figure}

\begin{figure}[tbp]
\begin{center}
\begin{subfigure}[t]{.24\textwidth}
\centering
\includegraphics[width=\textwidth]{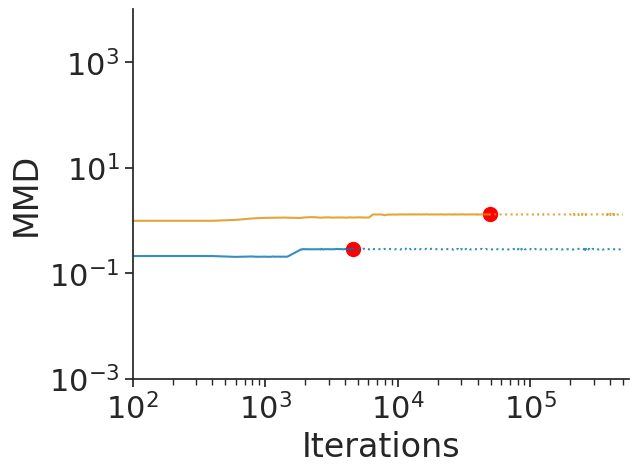}
\caption{arK} 
\end{subfigure} 
\begin{subfigure}[t]{.24\textwidth}
\centering
\includegraphics[width=\textwidth]{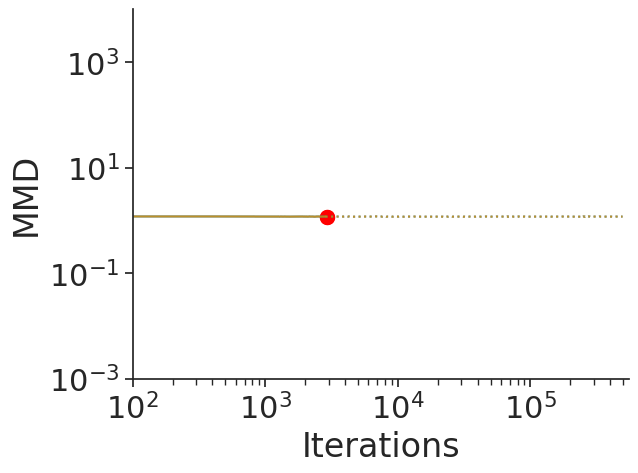}
\caption{bball\_0} 
\end{subfigure}
\begin{subfigure}[t]{.24\textwidth}
\centering
\includegraphics[width=\textwidth]{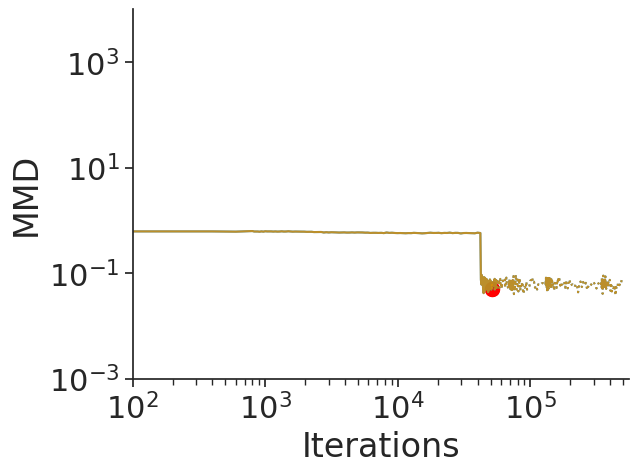}
\caption{bball\_1}
\end{subfigure}
\begin{subfigure}[t]{.24\textwidth}
\centering
\includegraphics[width=\textwidth]{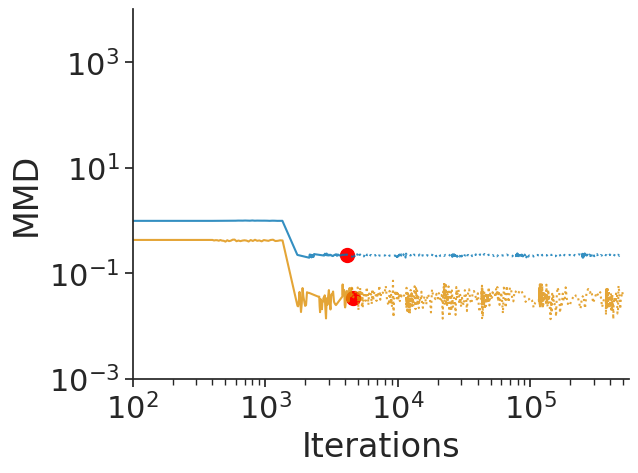}
\caption{dogs\_log} 
\end{subfigure}
\begin{subfigure}[t]{.24\textwidth}
\centering
\includegraphics[width=\textwidth]{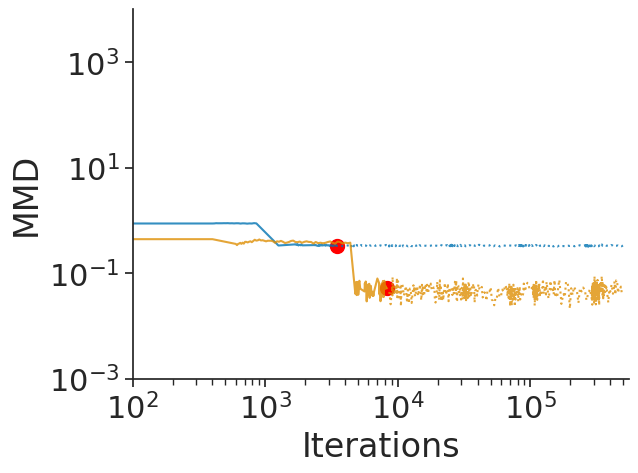}
\caption{dogs} 
\end{subfigure}
\begin{subfigure}[t]{.24\textwidth}
\centering
\includegraphics[width=\textwidth]{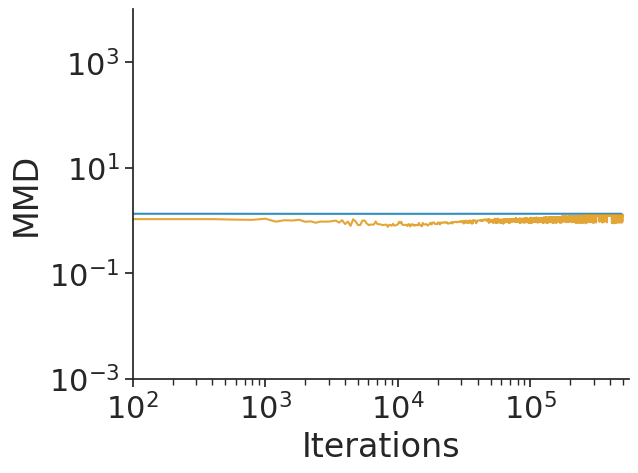}
\caption{diamonds} 
\end{subfigure}
\begin{subfigure}[t]{.24\textwidth}
\centering
\includegraphics[width=\textwidth]{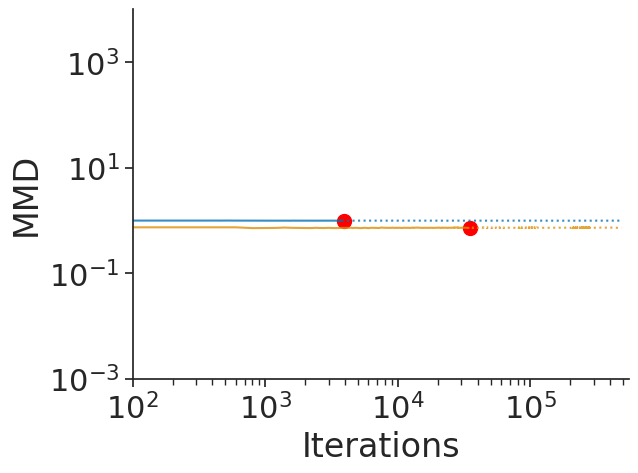}
\caption{8schools\_c} 
\end{subfigure}
\begin{subfigure}[t]{.24\textwidth}
\centering
\includegraphics[width=\textwidth]{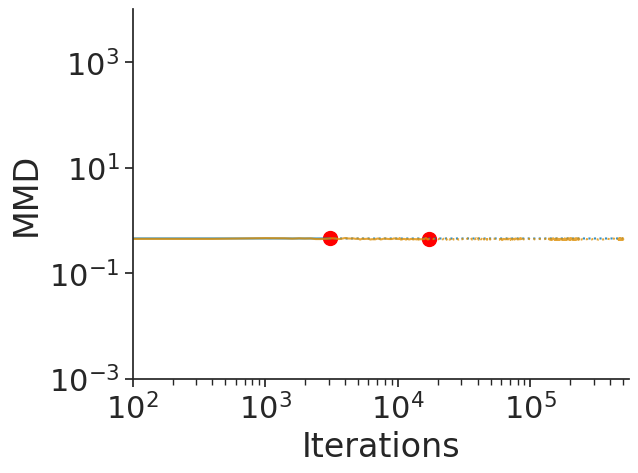}
\caption{8schools\_nc} 
\end{subfigure}
\begin{subfigure}[t]{.24\textwidth}
\centering
\includegraphics[width=\textwidth]{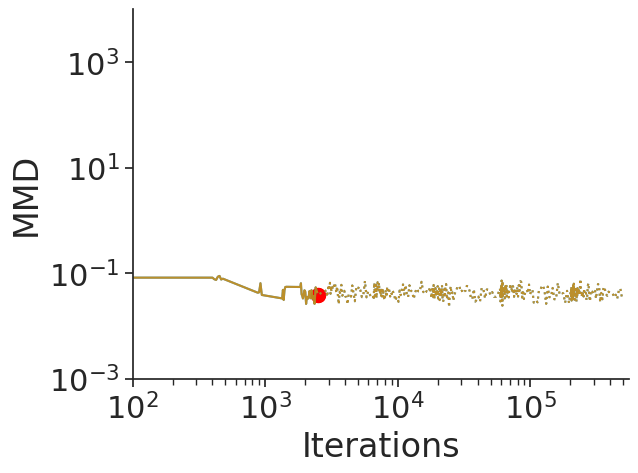}
\caption{hmm\_example}
\end{subfigure}
\begin{subfigure}[t]{.24\textwidth}
\centering
\includegraphics[width=\textwidth]{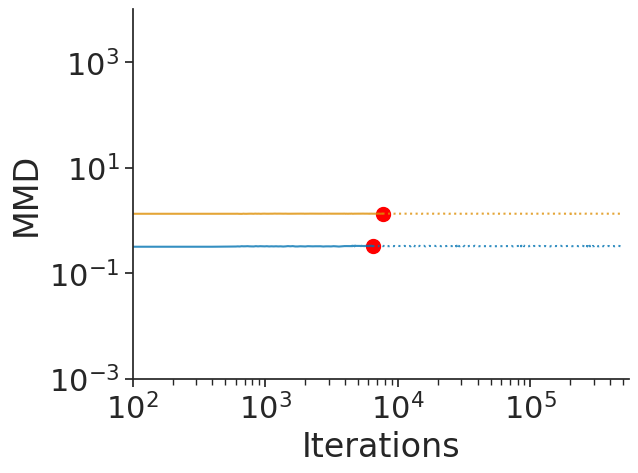}
\caption{low\_dim\_gauss} 
\end{subfigure}
\begin{subfigure}[t]{.24\textwidth}
\centering
\includegraphics[width=\textwidth]{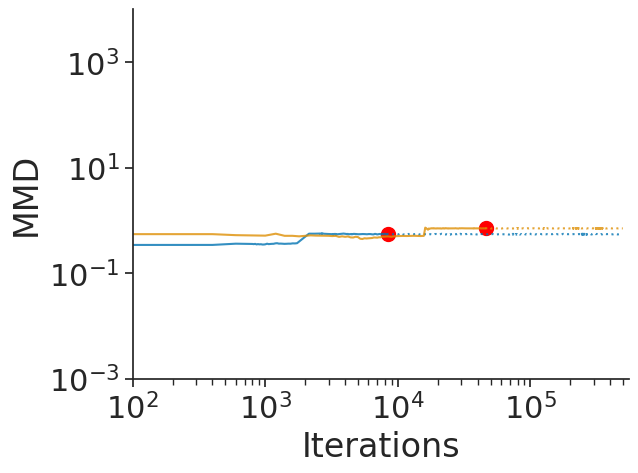}
\caption{nes2000} 
\end{subfigure}
\begin{subfigure}[t]{.24\textwidth}
\centering
\includegraphics[width=\textwidth]{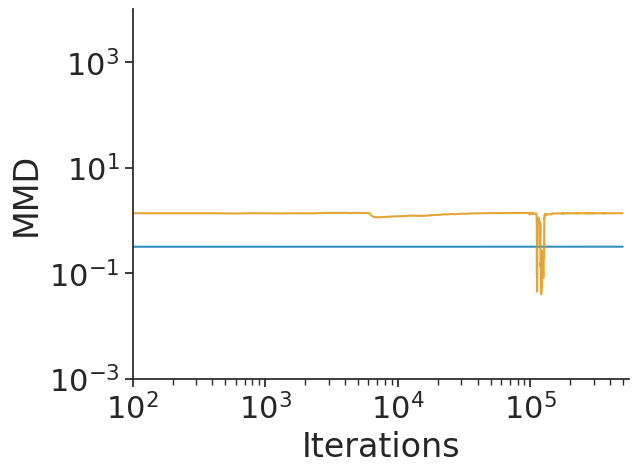}
\caption{sblrc} 
\end{subfigure}
\begin{subfigure}[t]{.24\textwidth}
\centering
\includegraphics[width=\textwidth]{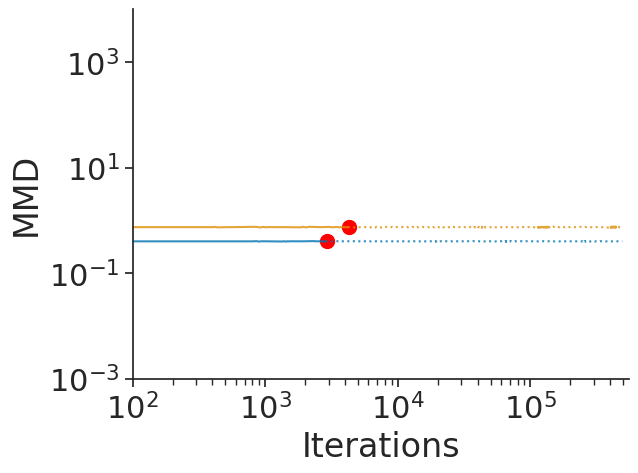}
\caption{garch} 
\end{subfigure}
\begin{subfigure}[t]{.24\textwidth}
\centering
\includegraphics[width=\textwidth]{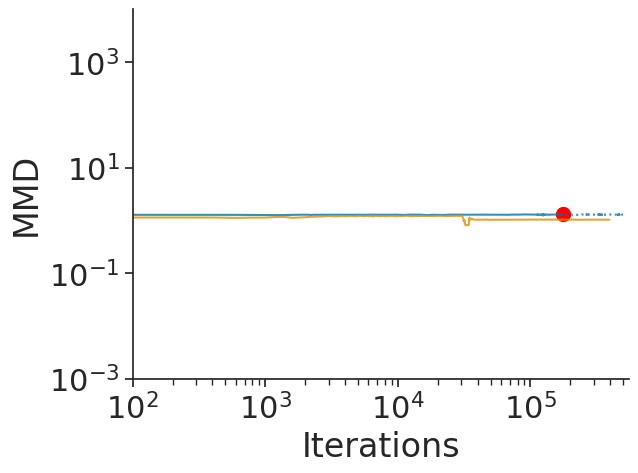}
\caption{gp\_pois\_regr} 
\end{subfigure} 
\begin{subfigure}[t]{.24\textwidth}
\centering
\includegraphics[width=\textwidth]{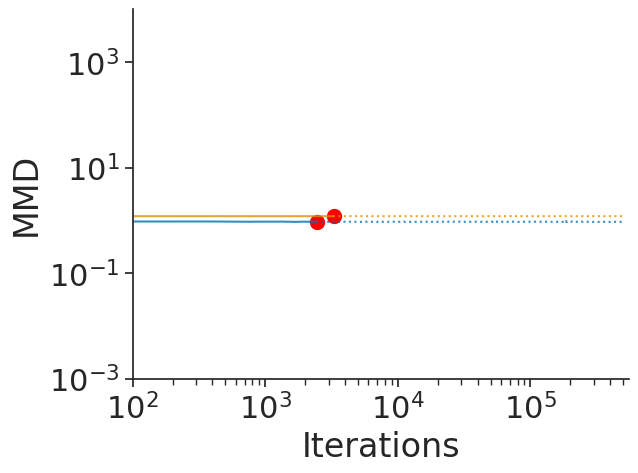}
\caption{gp\_regr} 
\end{subfigure}
\begin{subfigure}[t]{.24\textwidth}
\centering
\includegraphics[width=\textwidth]{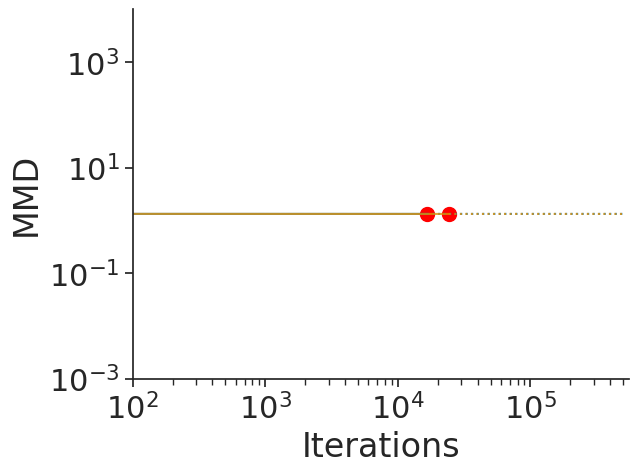}
\caption{hudson\_lynx} 
\end{subfigure}
\begin{subfigure}[t]{.24\textwidth}
\centering
\includegraphics[width=\textwidth]{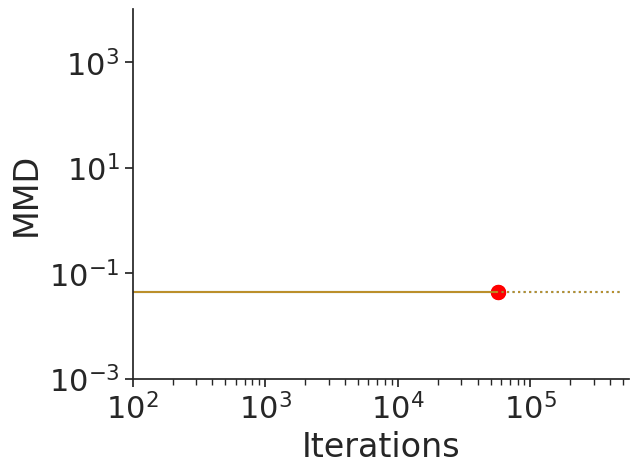}
\caption{mcycle\_gp} 
\end{subfigure}
\caption{
Accuracy of mean-field  (blue) and full-rank (orange) Gaussian family approximations for selected \texttt{posteriordb} data/models, where accuracy is measured in MMD. 
The red dots indicate where the termination rule triggers}
\label{fig:Posteriordb-mf_vs_fr-data-and-model-comparisons-mean}
\end{center}
\end{figure}

\begin{figure}[tbp]
\begin{center}
\begin{subfigure}[t]{.24\textwidth}
\centering
\includegraphics[width=\textwidth]{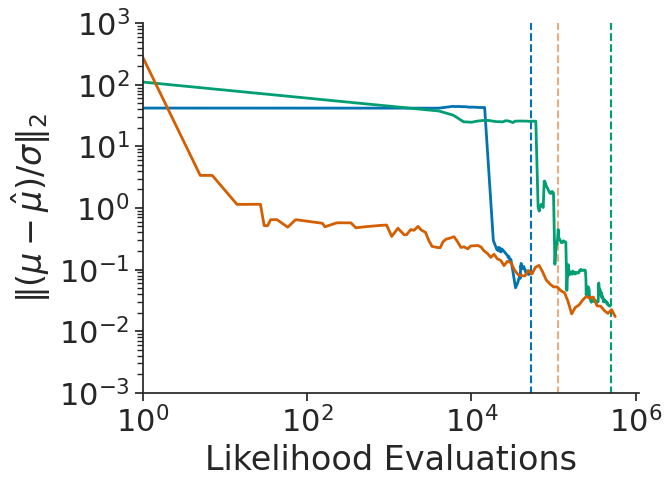} \\
\includegraphics[width=\textwidth]{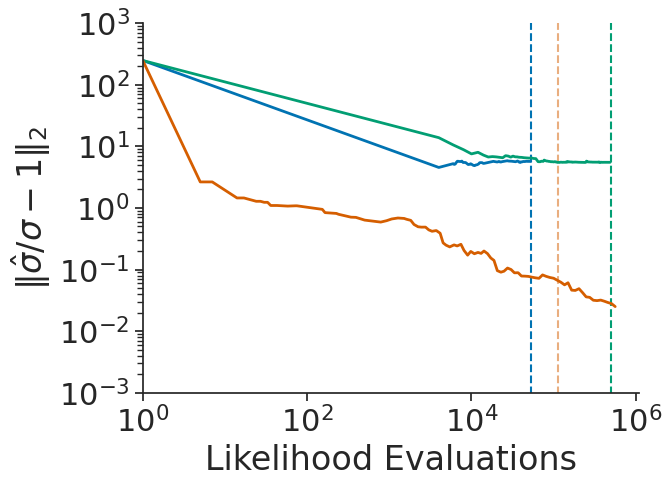}
\caption{arK} 
\end{subfigure} 
\begin{subfigure}[t]{.24\textwidth}
\centering
\includegraphics[width=\textwidth]{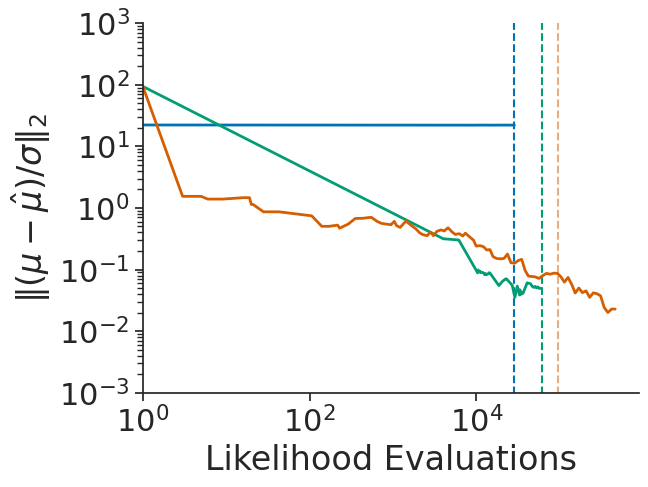} \\
\includegraphics[width=\textwidth]{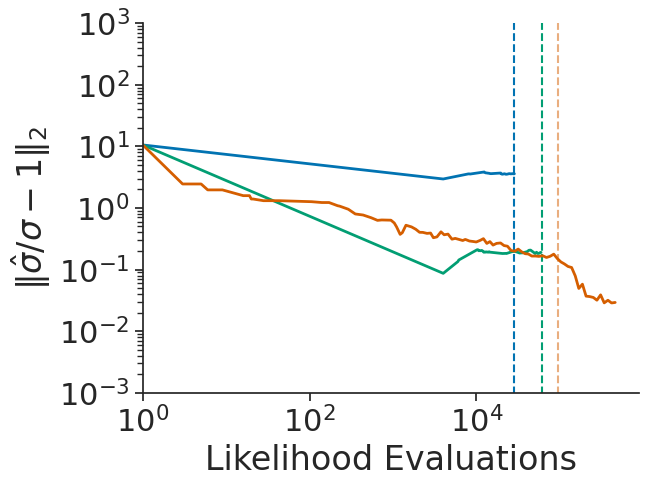}
\caption{bball\_0} 
\end{subfigure} 
\begin{subfigure}[t]{.24\textwidth}
\centering
\includegraphics[width=\textwidth]{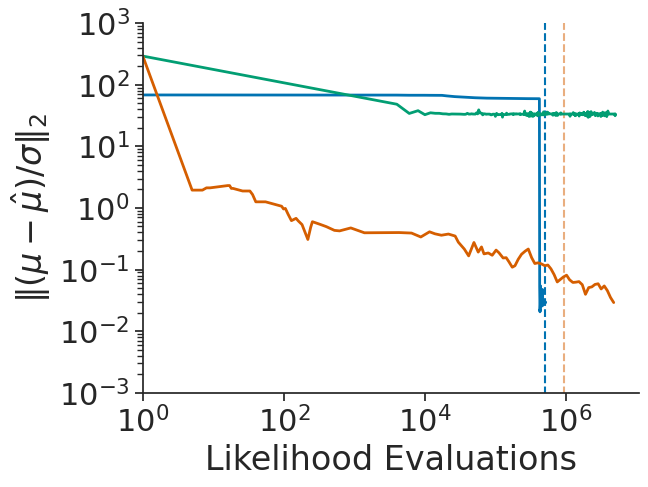} \\
\includegraphics[width=\textwidth]{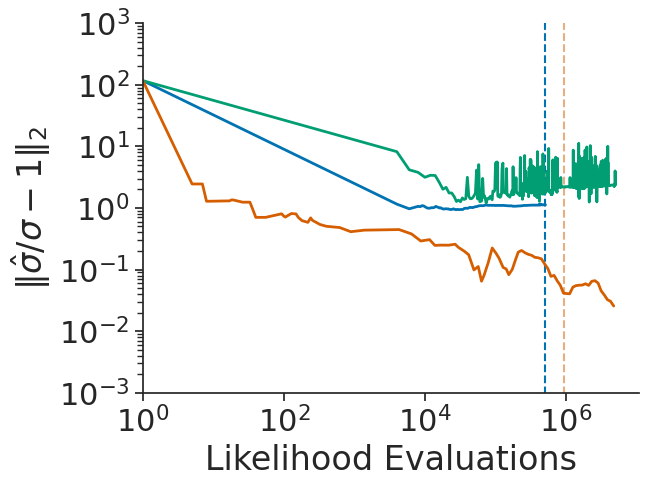}
\caption{bball\_1} 
\end{subfigure} 
\begin{subfigure}[t]{.24\textwidth}
\centering
\includegraphics[width=\textwidth]{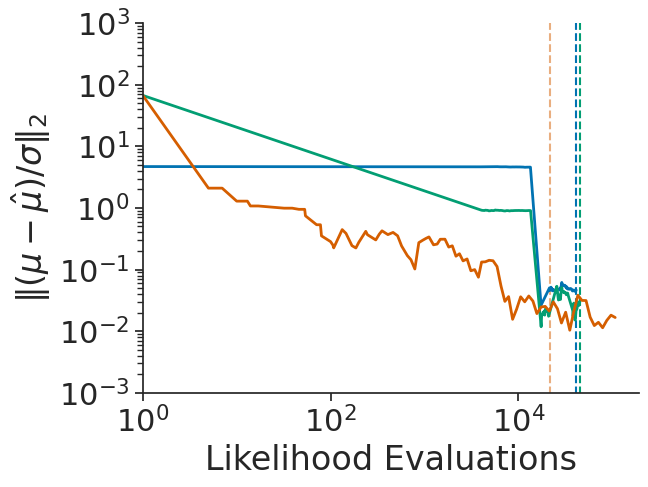} \\
\includegraphics[width=\textwidth]{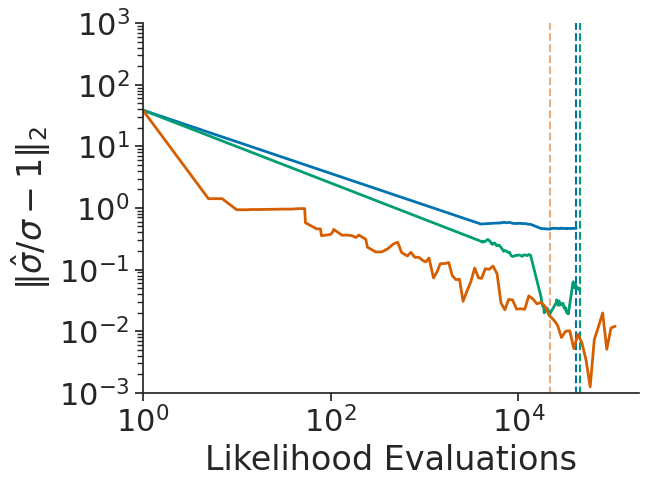}
\caption{dogs\_logs}  
\end{subfigure} 
\begin{subfigure}[t]{.24\textwidth}
\centering
\includegraphics[width=\textwidth]{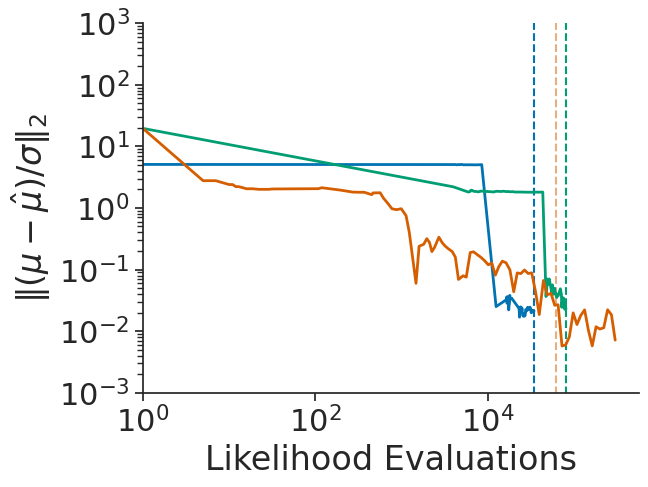} \\
\includegraphics[width=\textwidth]{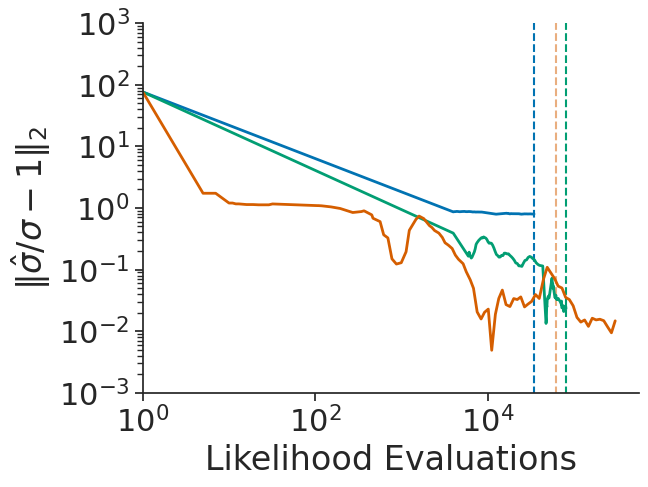}
\caption{dogs}  
\end{subfigure} 
\begin{subfigure}[t]{.24\textwidth}
\centering
\includegraphics[width=\textwidth]{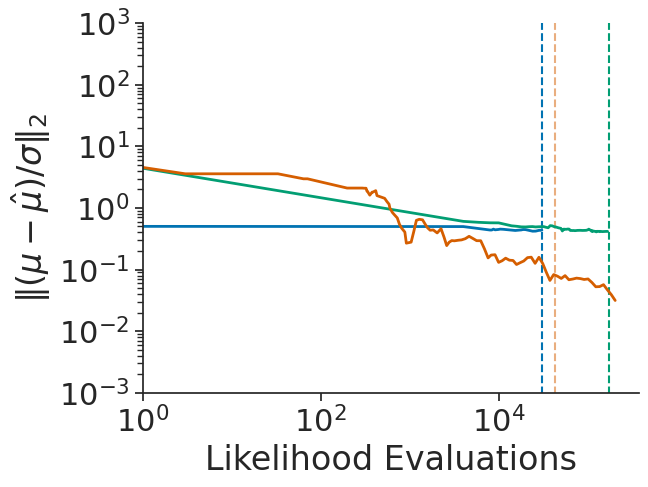} \\
\includegraphics[width=\textwidth]{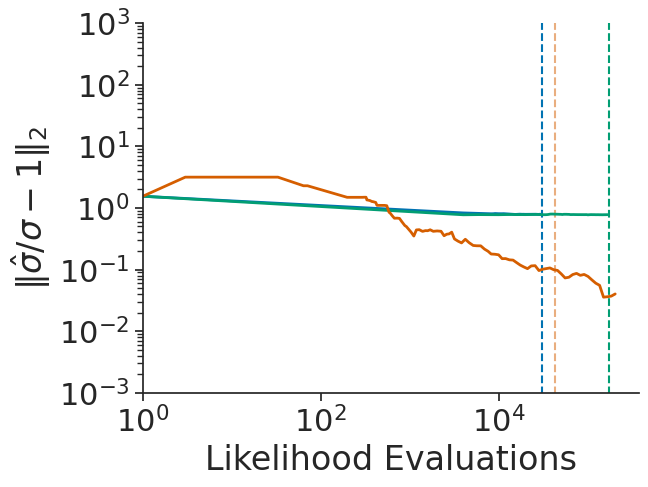}
\caption{8schools\_nc} 
\end{subfigure} 
\begin{subfigure}[t]{.24\textwidth}
\centering
\includegraphics[width=\textwidth]{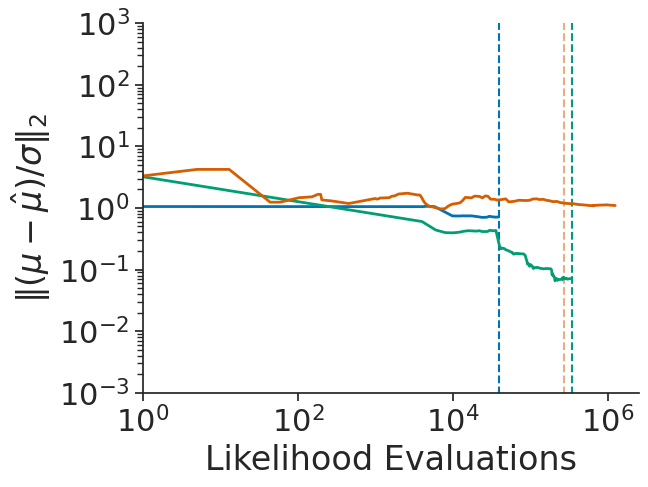} \\
\includegraphics[width=\textwidth]{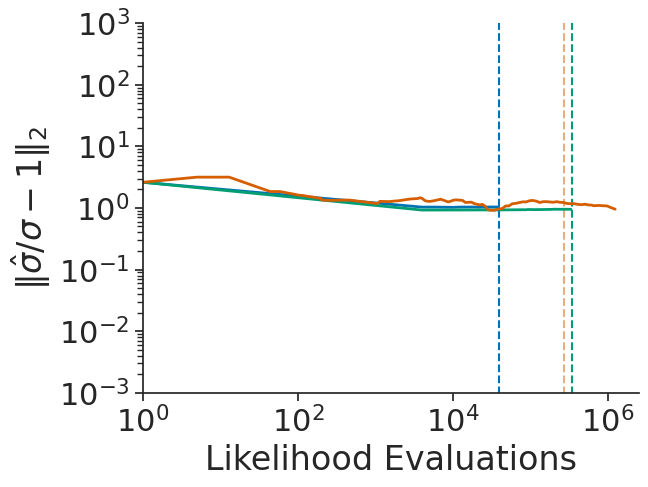}
\caption{8schools\_c} 
\end{subfigure} 
\begin{subfigure}[t]{.24\textwidth}
\centering
\includegraphics[width=\textwidth]{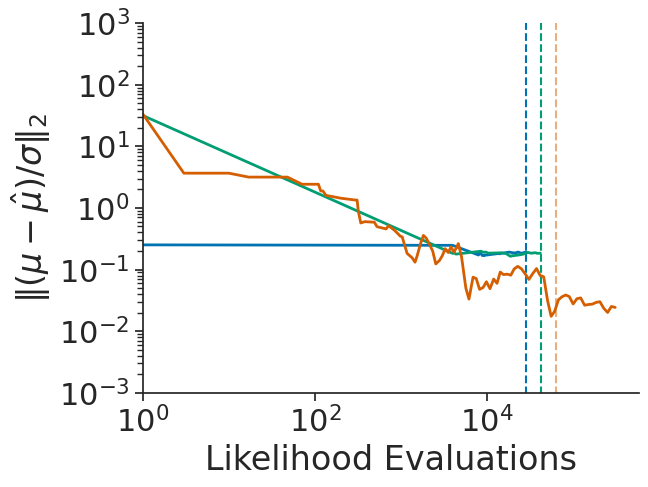} \\
\includegraphics[width=\textwidth]{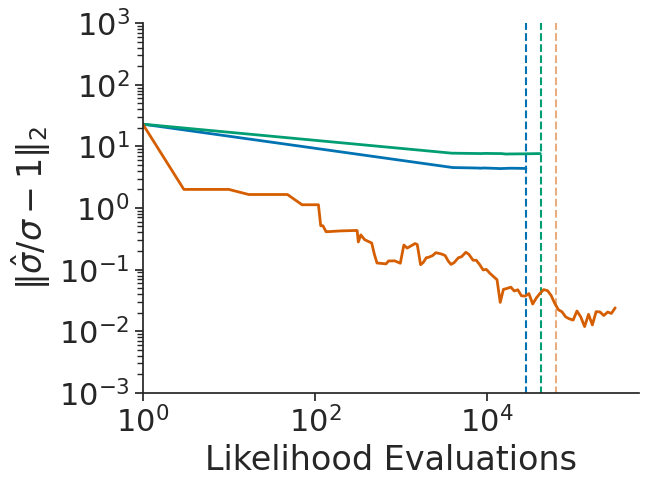}
\caption{garch} 
\end{subfigure} 
\caption{
Results of RABVI with mean-field Gaussian (blue) and full-rank Gaussian (green) family comparison to dynamic  HMC (orange) in terms of relative mean error (top) and relative standard error (bottom).
Blue and green vertical lines show the termination rule triggering points in RABVI and orange vertical line shows the end of warm-up period in dynamic HMC.}
\label{fig:Runtime-comparison-posteriordb}
\end{center}
\end{figure}

\begin{figure}[tbp]
\begin{center}
\begin{subfigure}[t]{.24\textwidth}
\centering
\includegraphics[width=\textwidth]{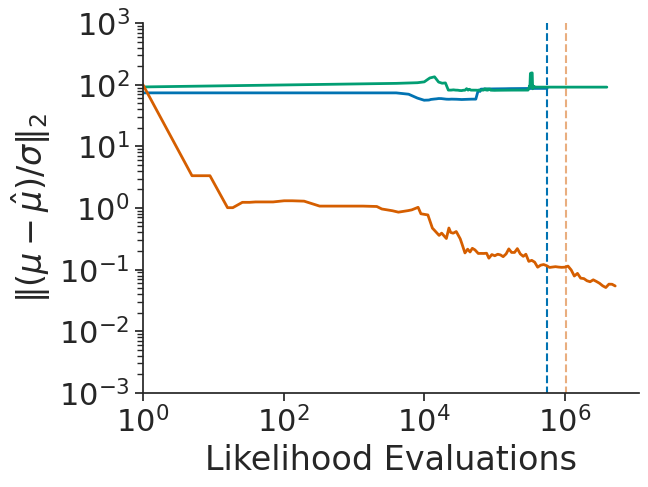} \\
\includegraphics[width=\textwidth]{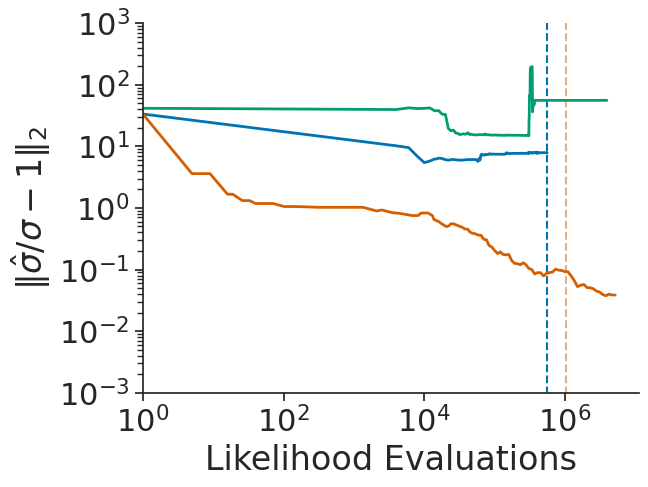}
\caption{gp\_pois\_regr} 
\end{subfigure} 
\begin{subfigure}[t]{.24\textwidth}
\centering
\includegraphics[width=\textwidth]{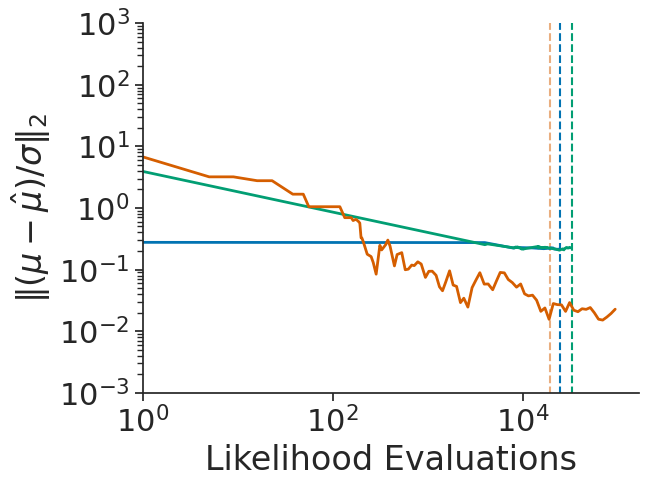} \\
\includegraphics[width=\textwidth]{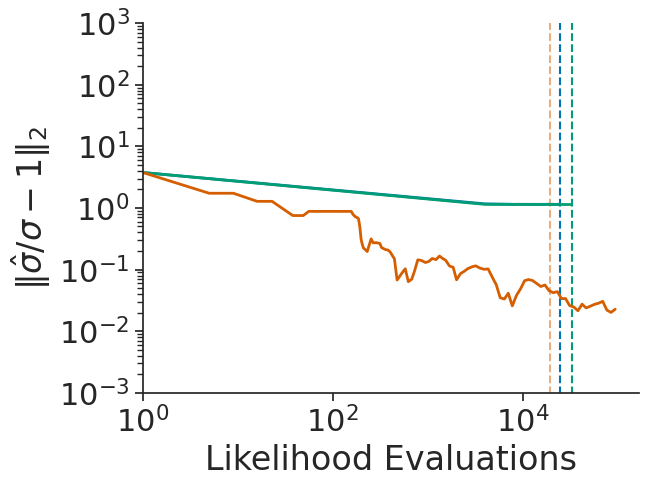}
\caption{gp\_regr} 
\end{subfigure} 
\begin{subfigure}[t]{.24\textwidth}
\centering
\includegraphics[width=\textwidth]{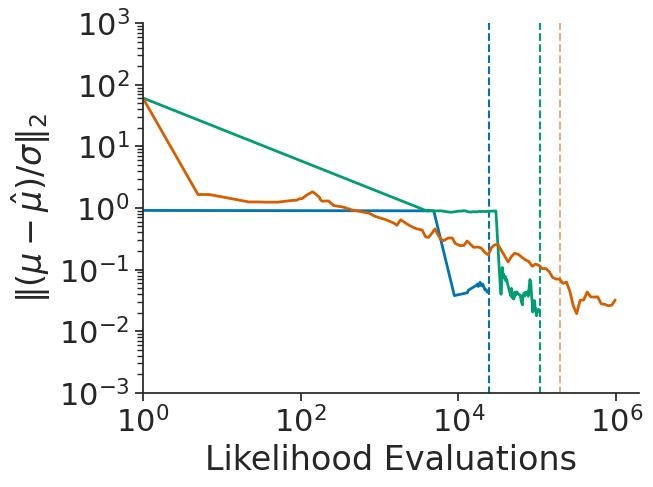} \\
\includegraphics[width=\textwidth]{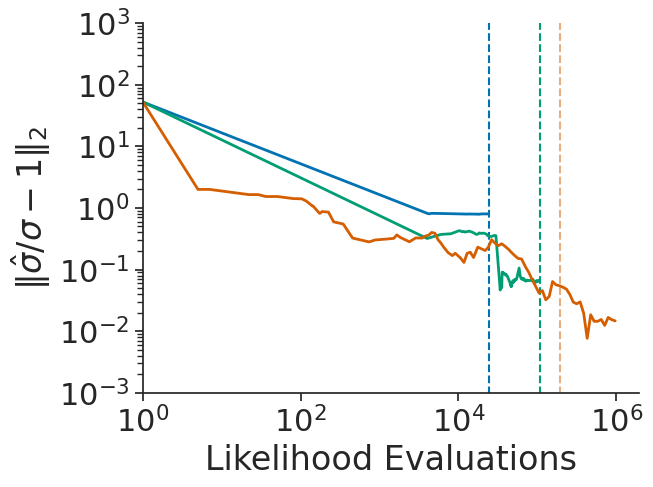}
\caption{hmm\_example} 
\end{subfigure} 
\begin{subfigure}[t]{.24\textwidth}
\centering
\includegraphics[width=\textwidth]{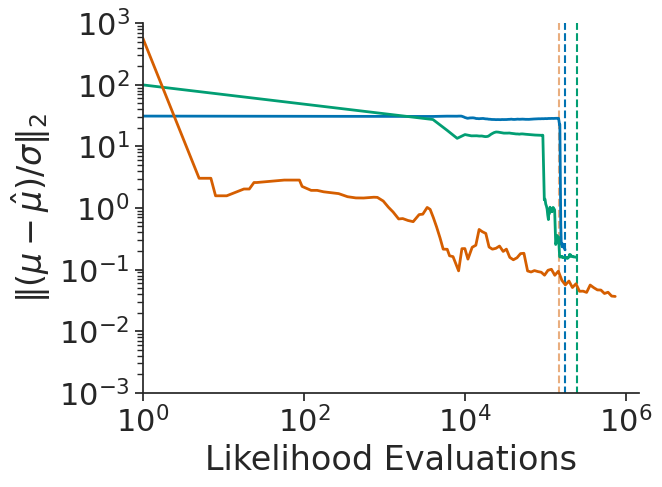} \\
\includegraphics[width=\textwidth]{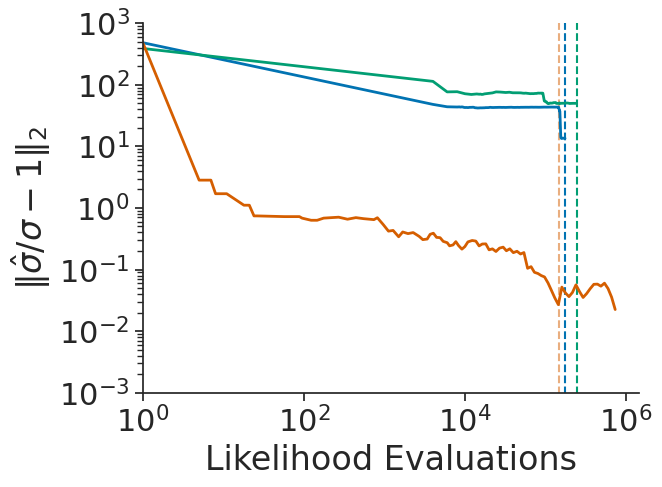}
\caption{hudson\_lynx} 
\end{subfigure} 
\begin{subfigure}[t]{.24\textwidth}
\centering
\includegraphics[width=\textwidth]{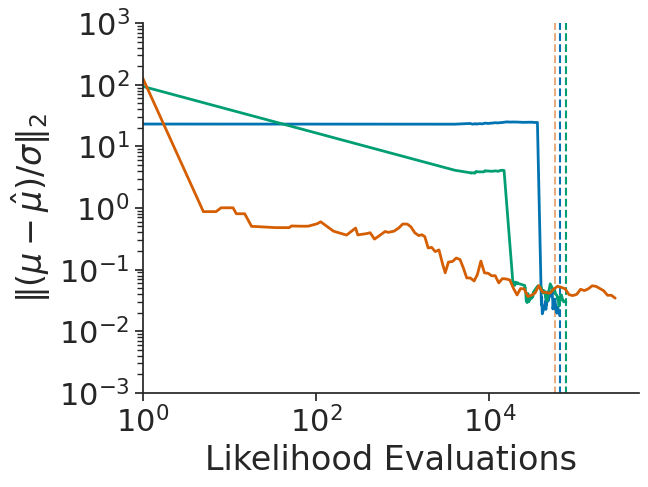} \\
\includegraphics[width=\textwidth]{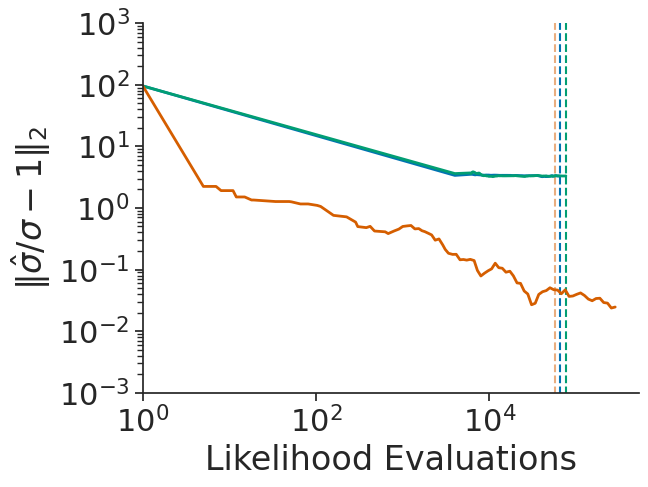}
\caption{low\_dim\_gauss} 
\end{subfigure} 
\begin{subfigure}[t]{.24\textwidth}
\centering
\includegraphics[width=\textwidth]{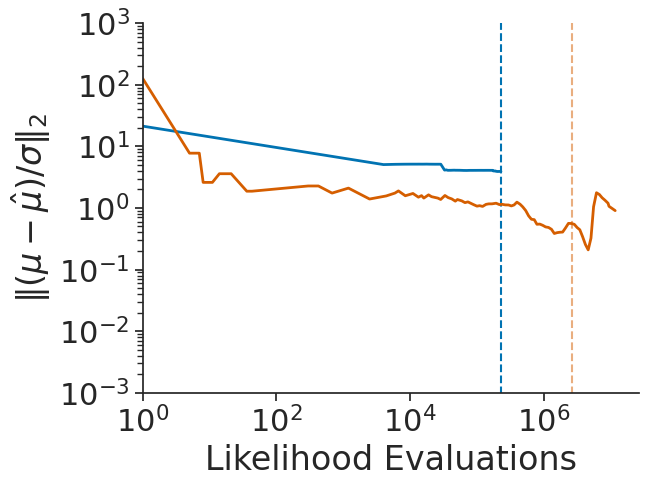} \\
\includegraphics[width=\textwidth]{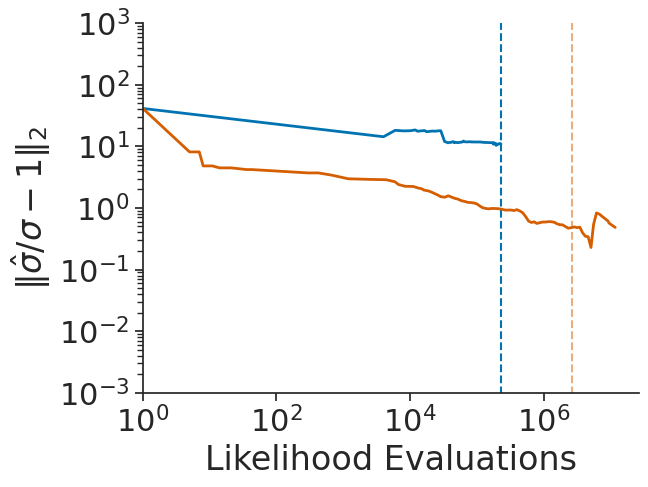}
\caption{mcycle\_gp} 
\end{subfigure} 
\begin{subfigure}[t]{.24\textwidth}
\centering
\includegraphics[width=\textwidth]{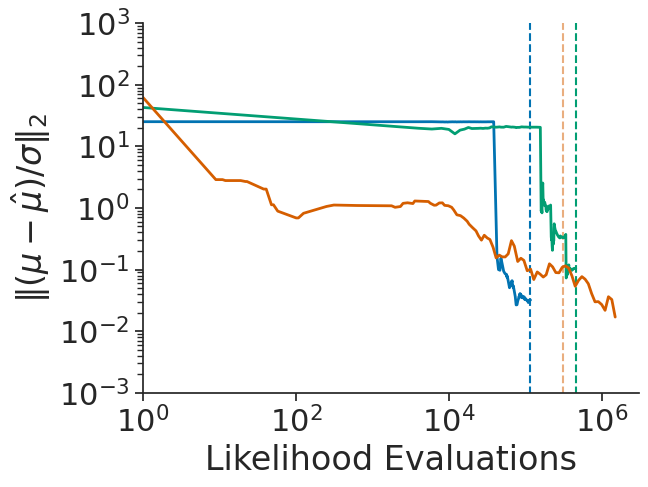} \\
\includegraphics[width=\textwidth]{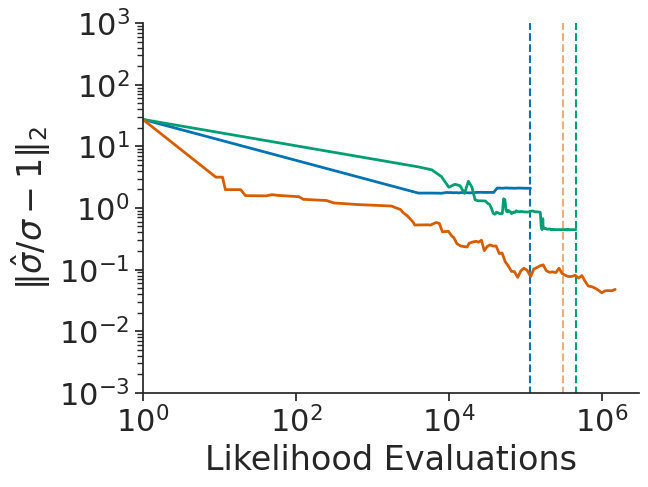}
\caption{nes2000} 
\end{subfigure} 
\begin{subfigure}[t]{.24\textwidth}
\centering
\includegraphics[width=\textwidth]{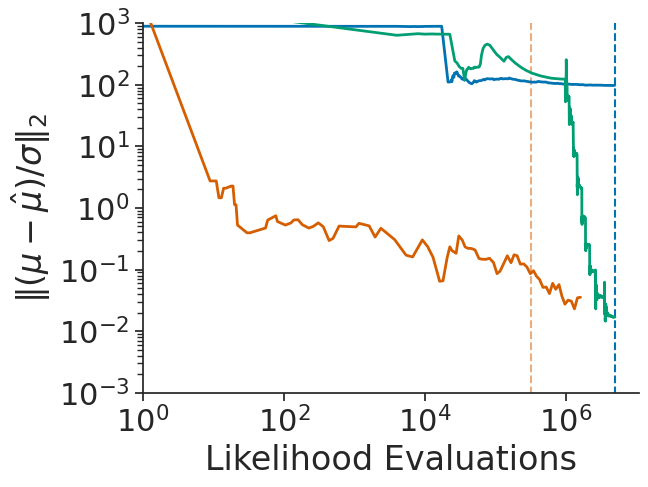} \\
\includegraphics[width=\textwidth]{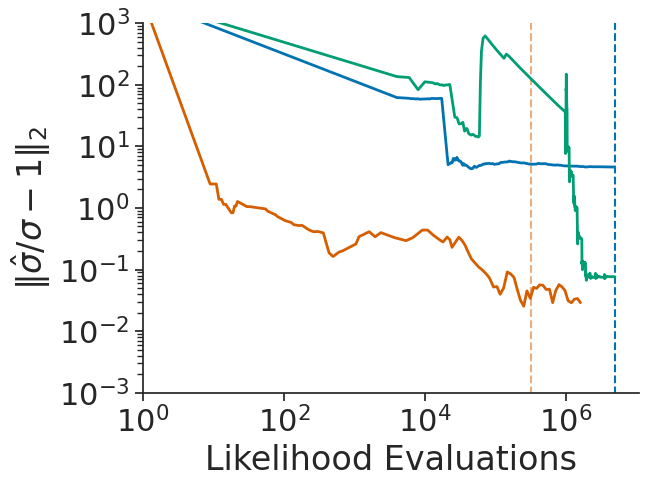}
\caption{sblrc} 
\end{subfigure} 
\caption{
Results of RABVI with mean-field Gaussian (blue) and full-rank Gaussian (green) family comparison to dynamic  HMC (orange) in terms of relative mean error (top) and relative standard error (bottom).
Blue and green vertical lines show the termination rule triggering points in RABVI and orange vertical line shows the end of warm-up period in dynamic HMC.}
\label{fig:Runtime-comparison-posteriordb-more}
\end{center}
\end{figure}

\begin{figure}[tbp]
\begin{center}
\centering
\includegraphics[width=.8\textwidth]{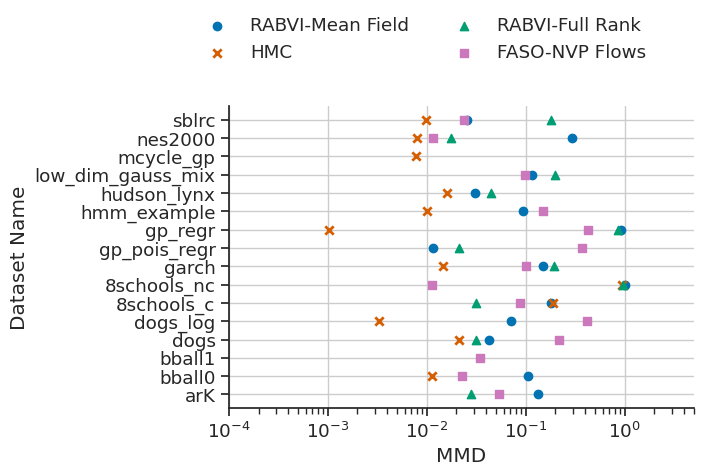}
\caption{
Results of RABVI with mean-field Gaussian and full-rank Gaussian family and FASO with NVP flows comparison to dynamic HMC 
at the same computational cost (likelihood evaluations) in terms of MMD.}
\label{fig:Runtime-comparison-posteriordb-mmd-summary}
\end{center}
\end{figure}

\begin{figure}[tbp]
\begin{center}
\centering
\includegraphics[width=.8\textwidth]{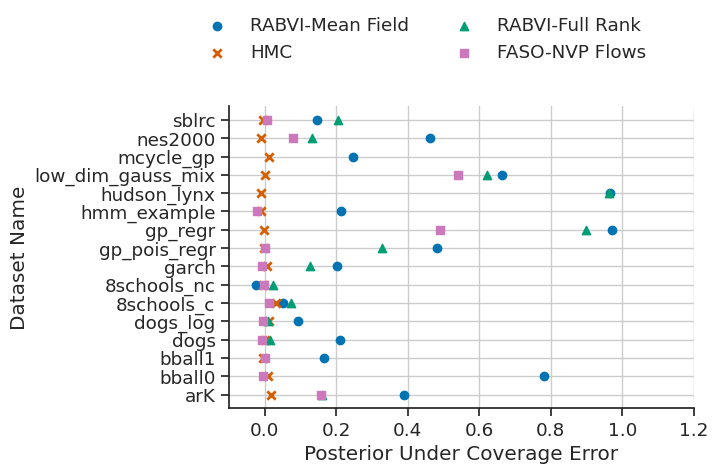}
\caption{
Results of RABVI with mean-field Gaussian and full-rank Gaussian family and FASO with NVP flows comparison to dynamic HMC 
at the same computational cost (likelihood evaluations) in terms of posterior under coverage error of 95\% quantiles.}
\label{fig:Runtime-comparison-posteriordb-cov-perc-summary}
\end{center}
\end{figure}

\begin{figure}[tbp]
\begin{center}
\begin{subfigure}[t]{.24\textwidth}
\centering
\includegraphics[width=\textwidth]{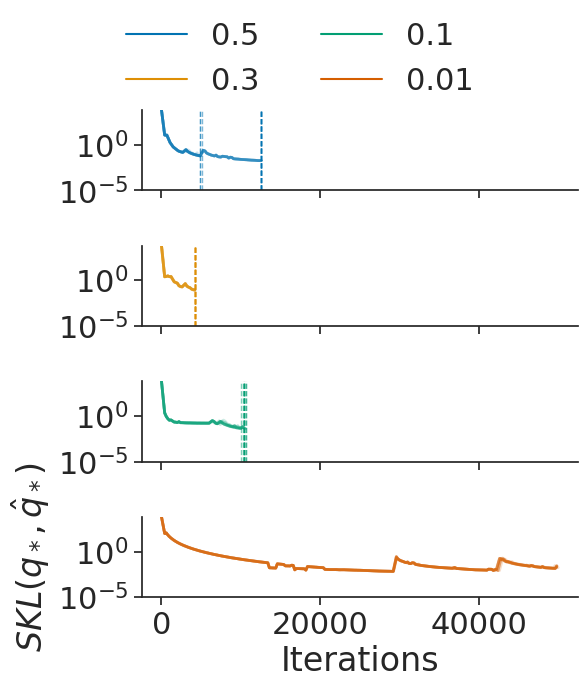}\\
\includegraphics[width=\textwidth]{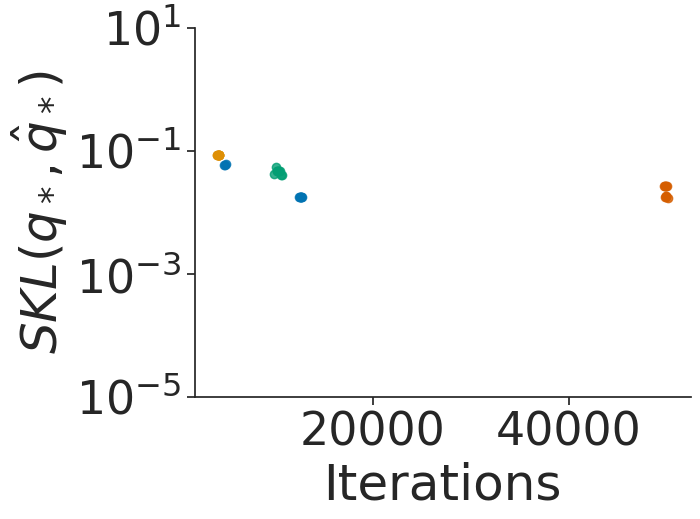}
\caption{$\learningrate_{0}$} 
\end{subfigure} 
\begin{subfigure}[t]{.24\textwidth}
\centering
\includegraphics[width=\textwidth]{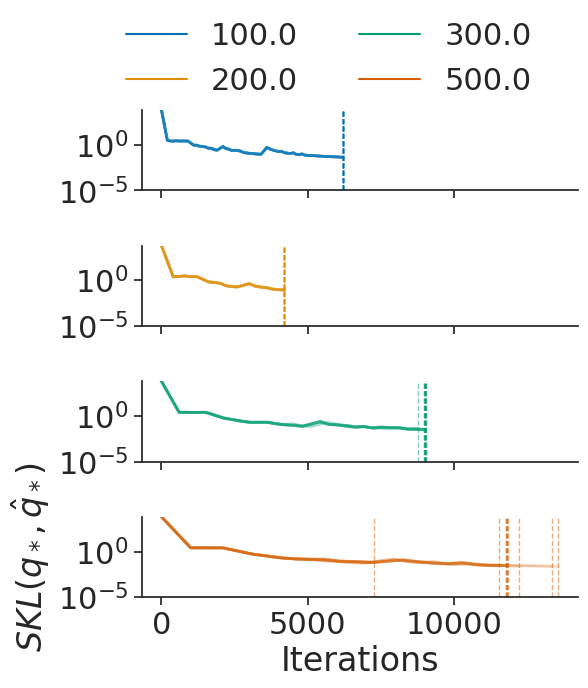} \\
\includegraphics[width=\textwidth]{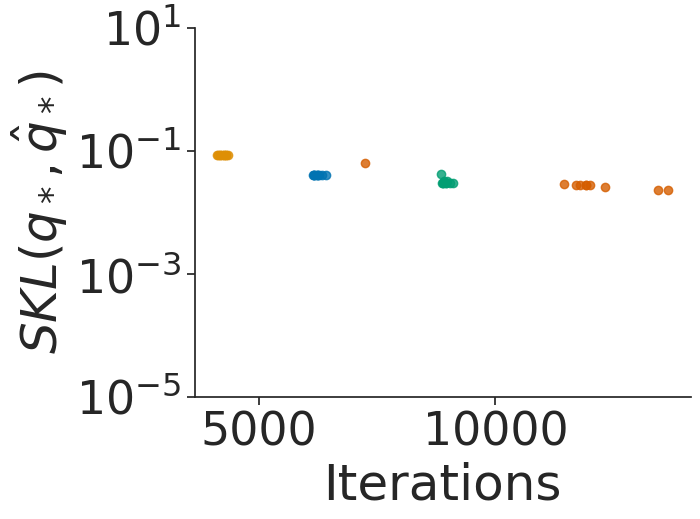}
\caption{$W_{min}$} 
\end{subfigure} 
\begin{subfigure}[t]{.24\textwidth}
\centering
\includegraphics[width=\textwidth]{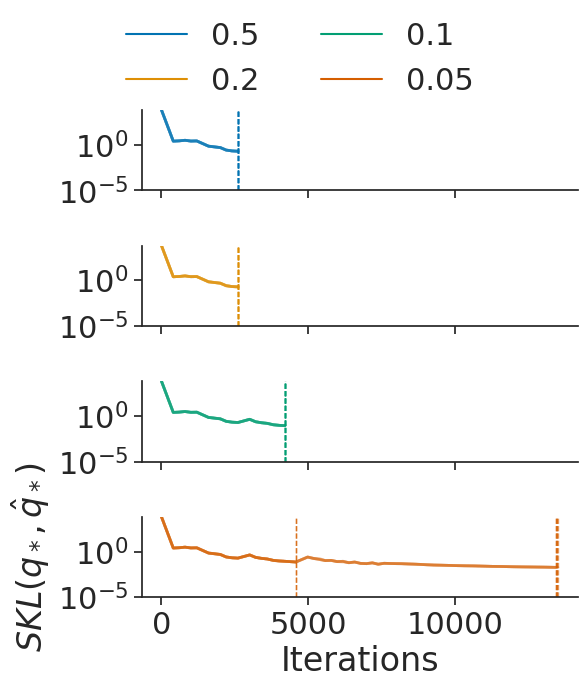} \\
\includegraphics[width=\textwidth]{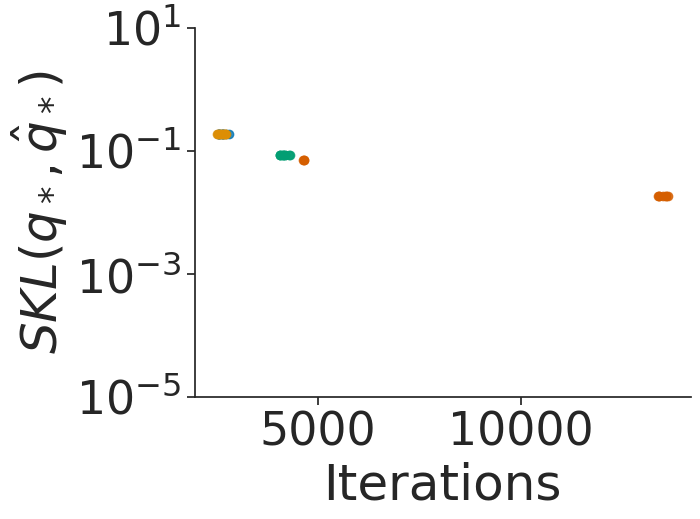}
\caption{$\varepsilon_{0}$} 
\end{subfigure} 
\begin{subfigure}[t]{.24\textwidth}
\centering
\includegraphics[width=\textwidth,height=0.198\textheight]{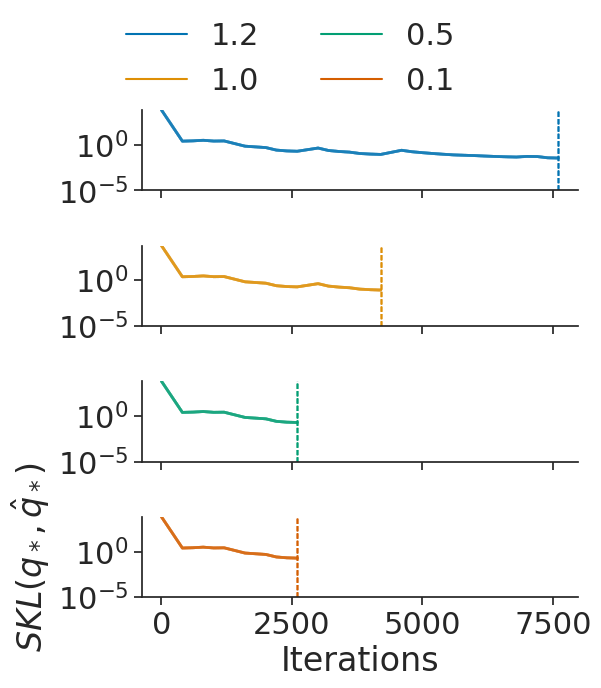} \\
\includegraphics[width=\textwidth,height=0.128\textheight]{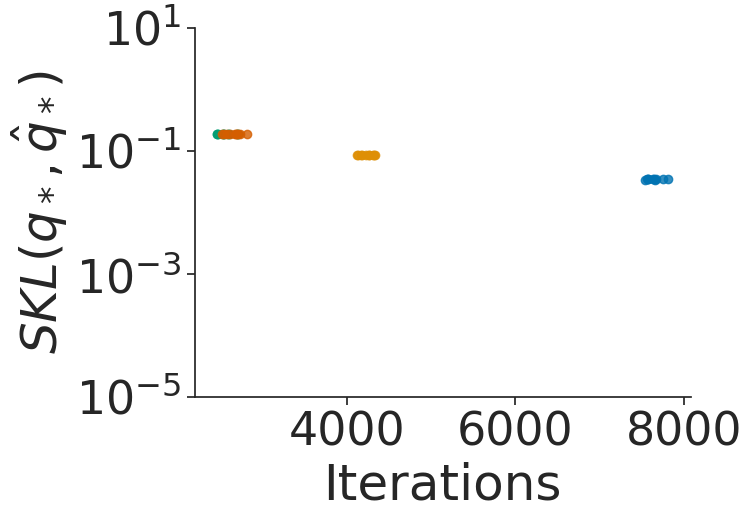}
\caption{$\tau$} 
\end{subfigure}
\caption{
Robustness to tuning parameters \textbf{(a)} initial learning rate $\learningrate_{0}$, 
\textbf{(b)} minimum window size $W_{\min}$, 
\textbf{(c)} initial iterate average relative error threshold $\varepsilon_{0}$, and 
\textbf{(d)} inefficiency threshold $\tau$.
Results use Gaussian target $\distNorm(0, V)$ with $\paramdim = 100$ and $V= I$ (identity covariance).
\textbf{(top)} Iterations versus symmetrized KL divergence between iterate average and optimal variational approximation. 
The distinct lines represent repeated experiments and the vertical lines indicate the termination rule trigger points.
\textbf{(bottom)} Iterations versus symmetrized KL divergence between iterate average and optimal variational approximation
at the termination rule trigger point.}
\label{fig:Robustness-tuning-parameters-gaussian}
\end{center}
\end{figure}

\begin{figure}[tbp]
\begin{center}
\begin{subfigure}[t]{.24\textwidth}
\centering
\includegraphics[width=\columnwidth]{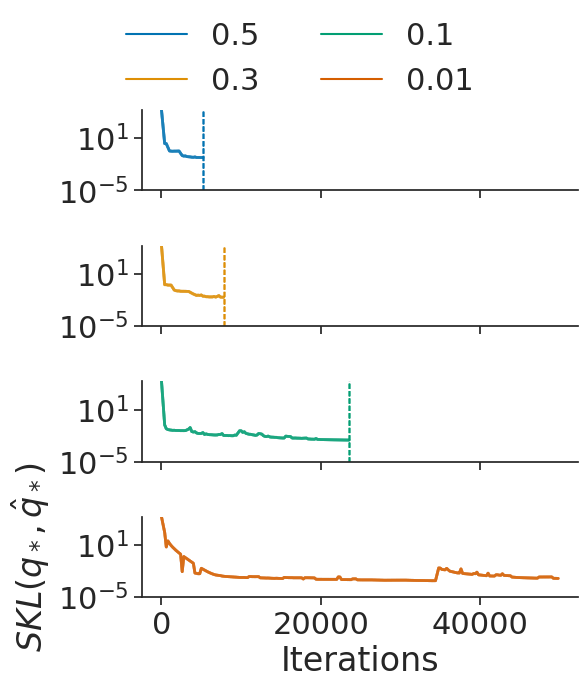}\\
\includegraphics[width=\columnwidth]{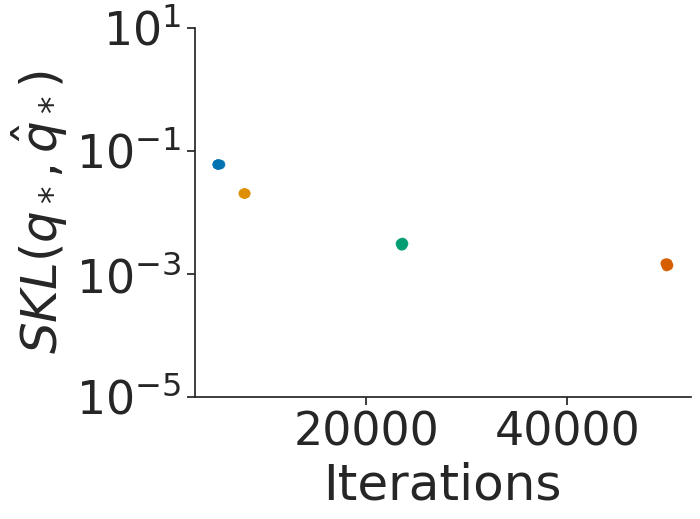}
\caption{\footnotesize uncorrelated $\paramdim = 100$} 
\end{subfigure}  
\begin{subfigure}[t]{.24\textwidth}
\centering
\includegraphics[width=\columnwidth]{Hyper_parameter_robust/avgadam/gaussian_500D_uncorrelated_var_constant_check_of_gamma0_lines_avgadam.png}\\
\includegraphics[width=\columnwidth]{Hyper_parameter_robust/avgadam/gaussian_500D_uncorrelated_var_constant_check_of_gamma0_scatter_avgadam.png}
\caption{\footnotesize uncorrelated $\paramdim = 500$} 
\end{subfigure} 
\begin{subfigure}[t]{.24\textwidth}
\centering
\includegraphics[width=\columnwidth]{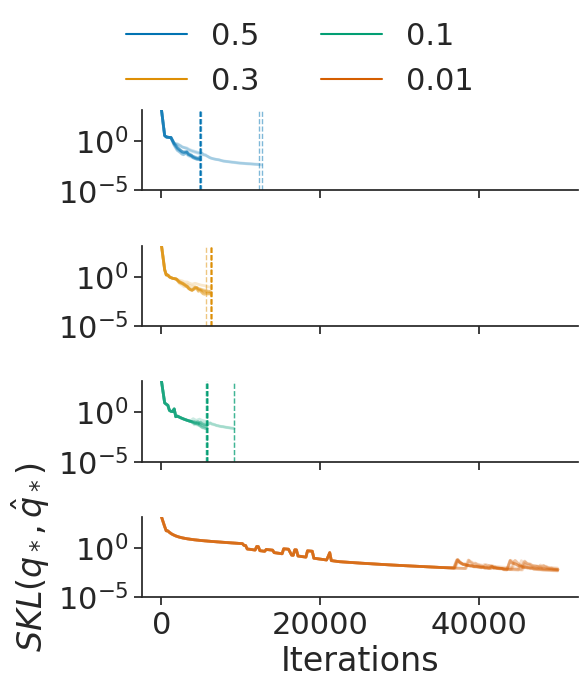}\\
\includegraphics[width=\columnwidth]{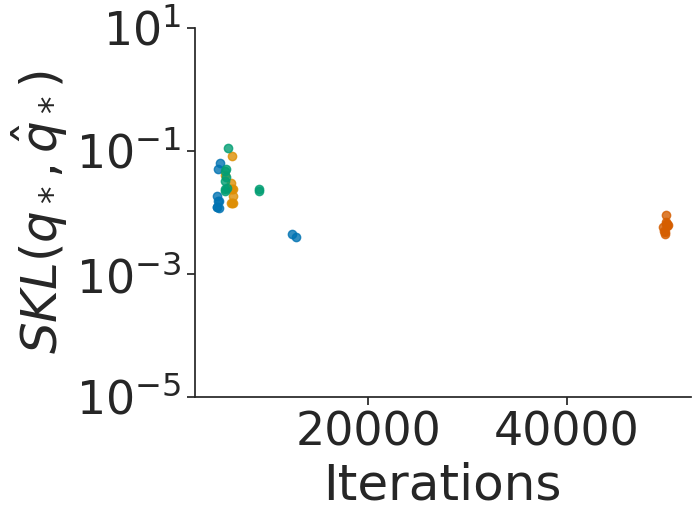}
\caption{\footnotesize uniform correlated $\paramdim = 100$} 
\end{subfigure}  
 \begin{subfigure}[t]{.24\textwidth}
\centering
\includegraphics[width=\columnwidth]{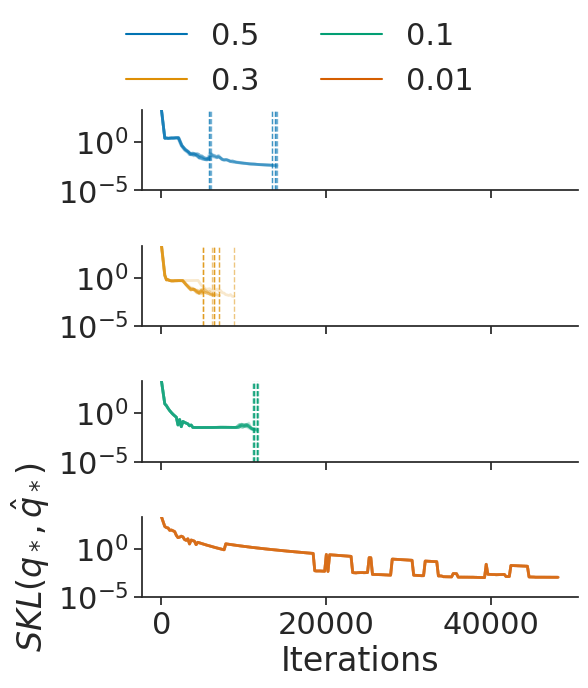}\\
\includegraphics[width=\columnwidth]{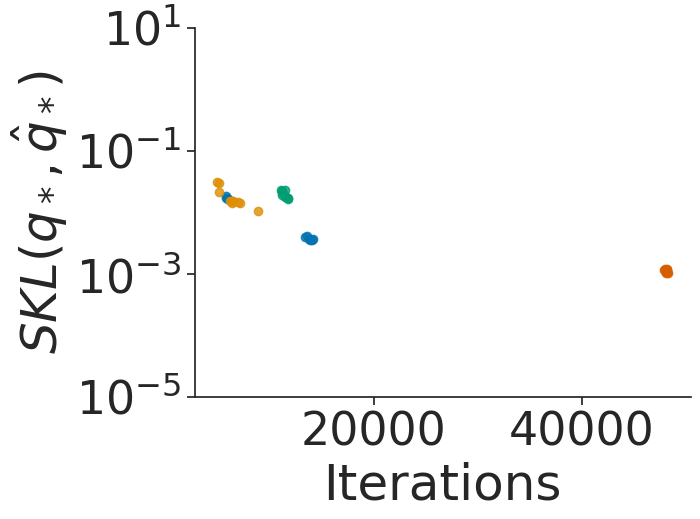}
\caption{\footnotesize banded correlated $\paramdim = 100$} 
\end{subfigure}  
\caption{
Robustness to initial learning rate $\learningrate_{0}$ using Gaussian targets.
\textbf{(top)} Iterations versus symmetrized KL divergence between iterate average and optimal variational approximation. 
The distinct lines represent repeated experiments and the vertical lines indicate the termination rule trigger points.
\textbf{(bottom)} Iterations versus symmetrized KL divergence between iterate average and optimal variational approximation
at the termination rule trigger point.}
\label{fig:Robustness-init-learning-rate}
\end{center}
\end{figure}

\begin{figure}[tbp]
\begin{center}
\begin{subfigure}[t]{.24\textwidth}
\centering
\includegraphics[width=\textwidth]{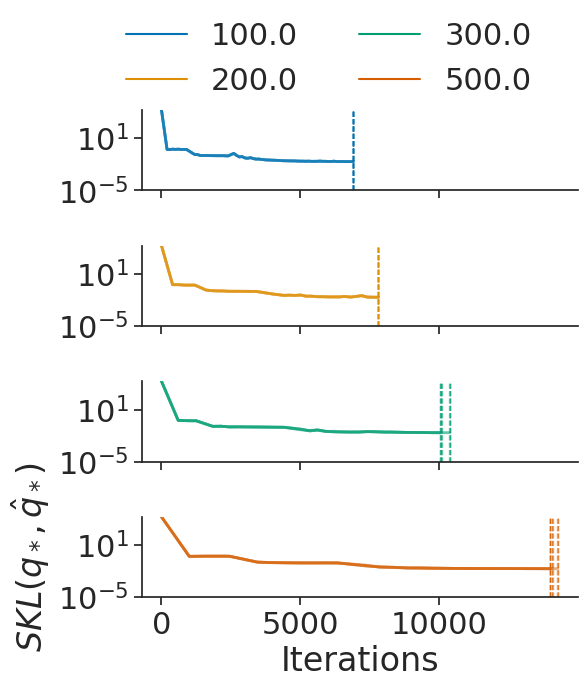} \\
\includegraphics[width=\textwidth]{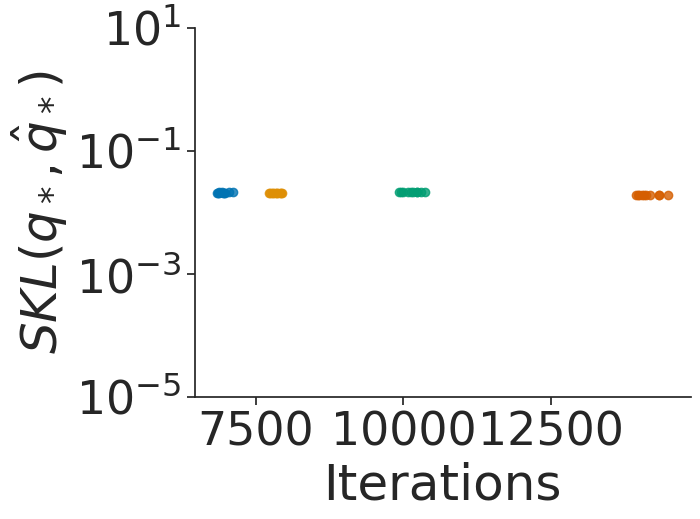}
\caption{\footnotesize uncorrelated $\paramdim = 100$} 
\end{subfigure}  
\begin{subfigure}[t]{.24\textwidth}
\centering
\includegraphics[width=\textwidth]{Hyper_parameter_robust/avgadam/gaussian_500D_uncorrelated_var_constant_check_of_W_min_lines_avgadam.png} \\
\includegraphics[width=\textwidth]{Hyper_parameter_robust/avgadam/gaussian_500D_uncorrelated_var_constant_check_of_W_min_scatter_avgadam.png}
\caption{\footnotesize uncorrelated $\paramdim = 500$} 
\end{subfigure} 
\begin{subfigure}[t]{.24\textwidth}
\centering
\includegraphics[width=\textwidth,height=0.198\textheight]{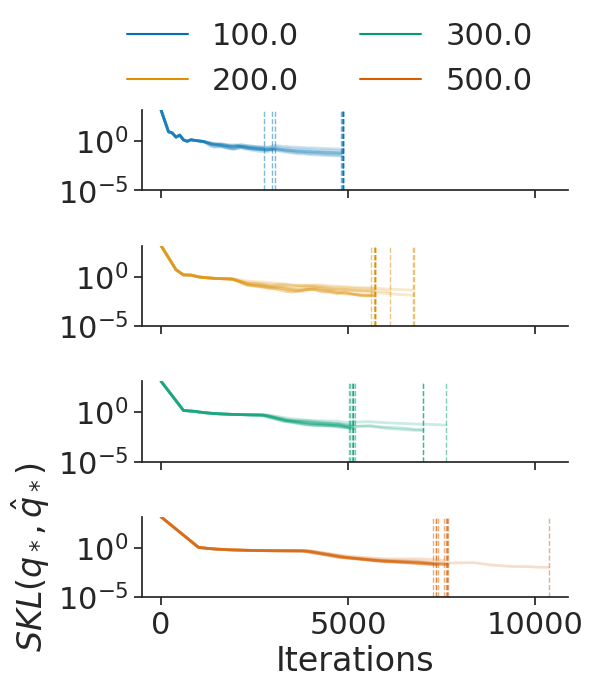} \\
\includegraphics[width=\textwidth,height=0.128\textheight]{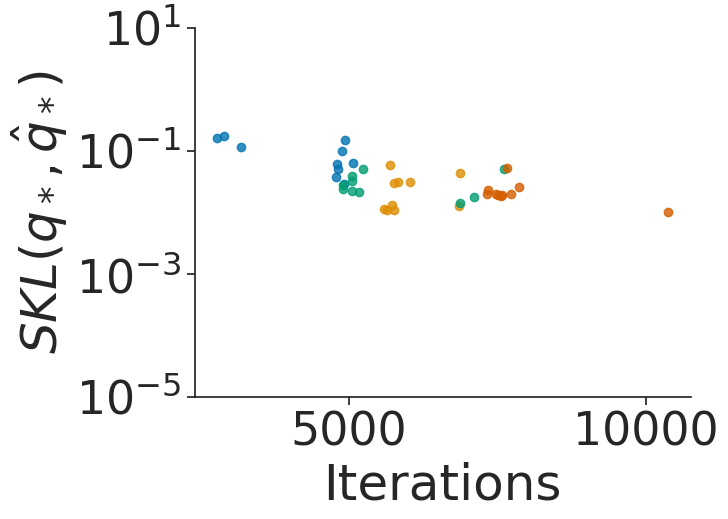}
\caption{\footnotesize uniform correlated $\paramdim = 100$} 
\end{subfigure}  
 \begin{subfigure}[t]{.24\textwidth}
\centering
\includegraphics[width=\textwidth,height=0.198\textheight]{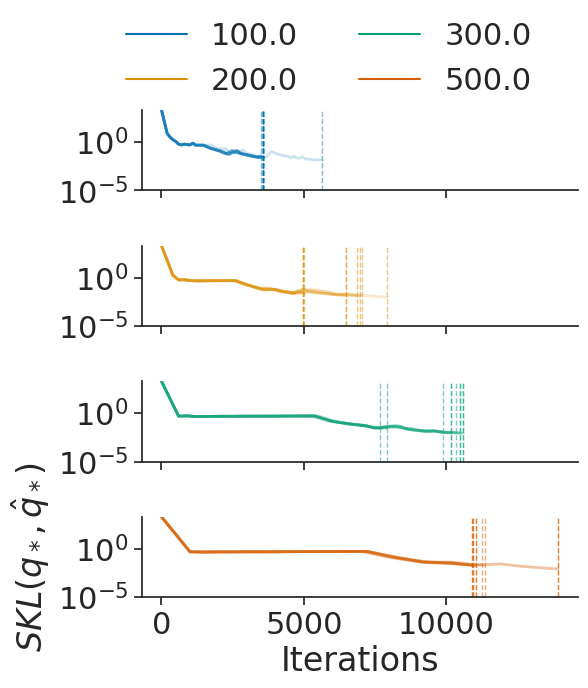} \\
\includegraphics[width=\textwidth,height=0.128\textheight]{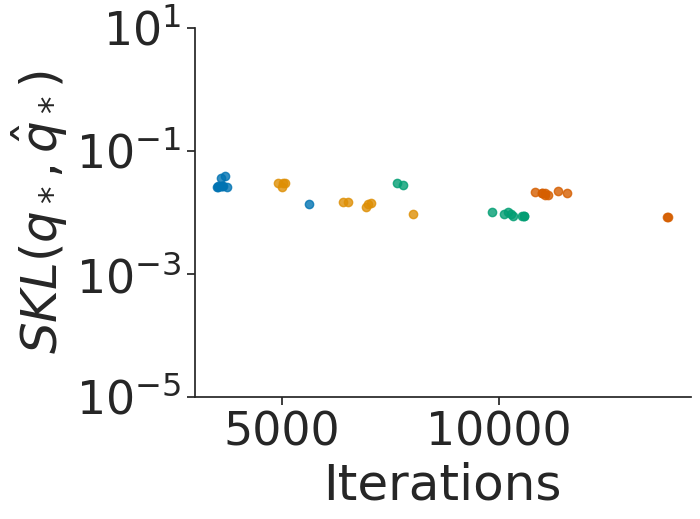}
\caption{\footnotesize banded correlated $\paramdim = 100$} 
\end{subfigure} 
\caption{
Robustness to minimum window size $W_{min}$,, using Gaussian targets.
\textbf{(top)} Iterations versus symmetrized KL divergence between iterate average and optimal variational approximation. 
The distinct lines represent repeated experiments and the vertical lines indicate the termination rule trigger points.
\textbf{(bottom)} Iterations versus symmetrized KL divergence between iterate average and optimal variational approximation
at the termination rule trigger point.}
\label{fig:Robustness-minimum-window-size-gaussian}
\end{center}
\end{figure}

\begin{figure}[tbp]
\begin{center}
\begin{subfigure}[t]{.24\textwidth}
\centering
\includegraphics[width=\textwidth,height=0.198\textheight]{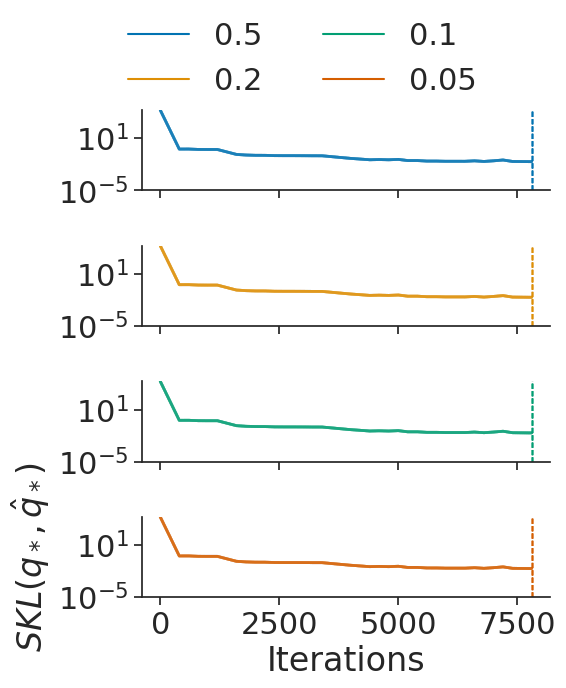} \\
\includegraphics[width=\textwidth,height=0.128\textheight]{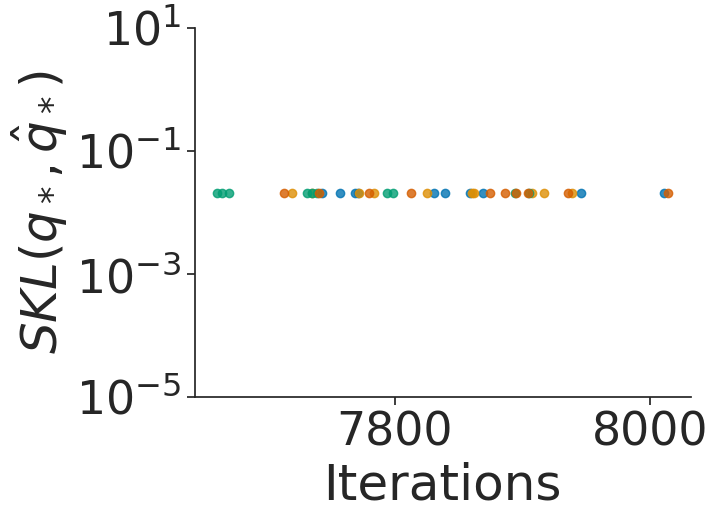}
\caption{\footnotesize uncorrelated $\paramdim = 100$} 
\end{subfigure}  
\begin{subfigure}[t]{.24\textwidth}
\centering
\includegraphics[width=\textwidth]{Hyper_parameter_robust/avgadam/gaussian_500D_uncorrelated_var_constant_check_of_mcse_threshold_lines_avgadam.png} \\
\includegraphics[width=\textwidth]{Hyper_parameter_robust/avgadam/gaussian_500D_uncorrelated_var_constant_check_of_mcse_threshold_scatter_avgadam.png}
\caption{\footnotesize uncorrelated $\paramdim = 500$} 
\end{subfigure} 
\begin{subfigure}[t]{.24\textwidth}
\centering
\includegraphics[width=\textwidth]{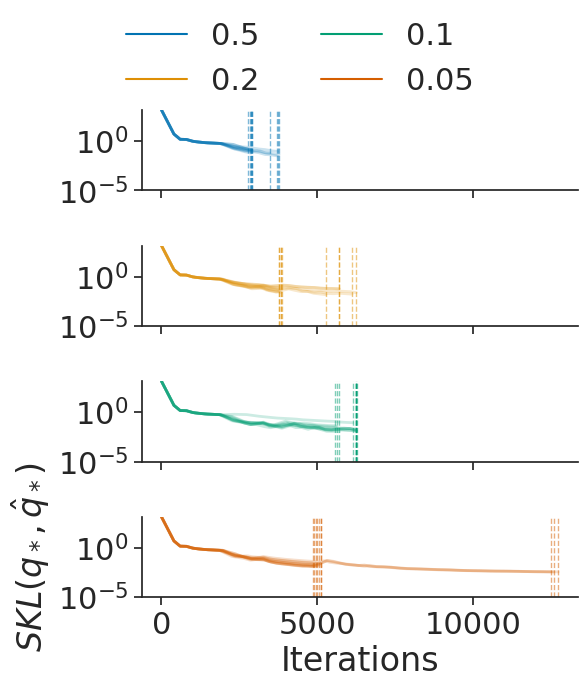} \\
\includegraphics[width=\textwidth]{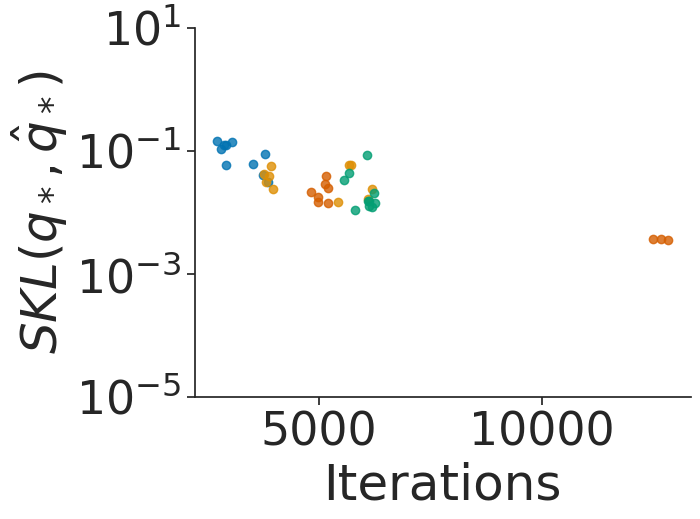}
\caption{\footnotesize uniform correlated $\paramdim = 100$} 
\end{subfigure}  
 \begin{subfigure}[t]{.24\textwidth}
\centering
\includegraphics[width=\textwidth]{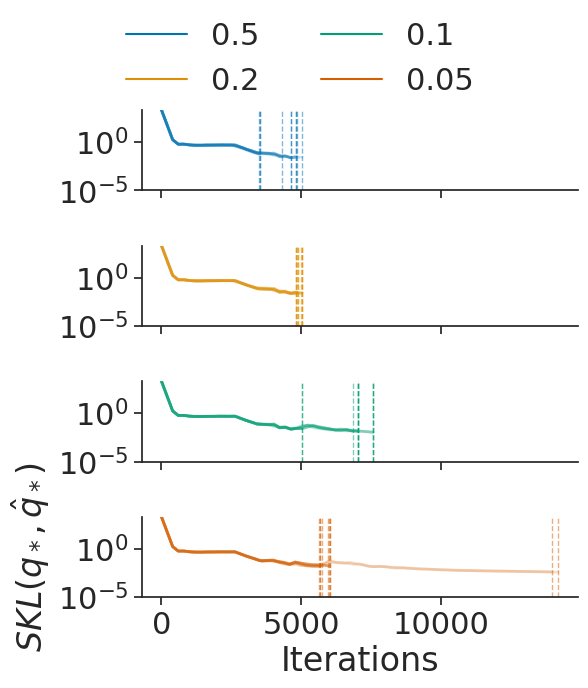} \\
\includegraphics[width=\textwidth]{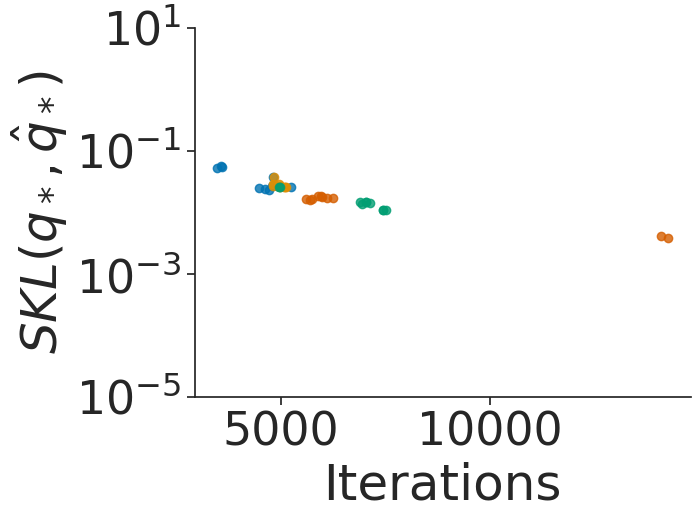}
\caption{\footnotesize banded correlated $\paramdim = 100$} 
\end{subfigure}  
\caption{
Robustness to initial iterate average relative error threshold $\varepsilon$ using Gaussian targets.
\textbf{(top)} Iterations versus symmetrized KL divergence between iterate average and optimal variational approximation. 
The distinct lines represent repeated experiments and the vertical lines indicate the termination rule trigger points.
\textbf{(bottom)} Iterations versus symmetrized KL divergence between iterate average and optimal variational approximation
at the termination rule trigger point.}
\label{fig:Robustness-mcse-threshold-gaussian}
\end{center}
\end{figure}

\begin{figure}[tbp]
\begin{center}
\begin{subfigure}[t]{.24\textwidth}
\centering
\includegraphics[width=\textwidth,height=0.198\textheight]{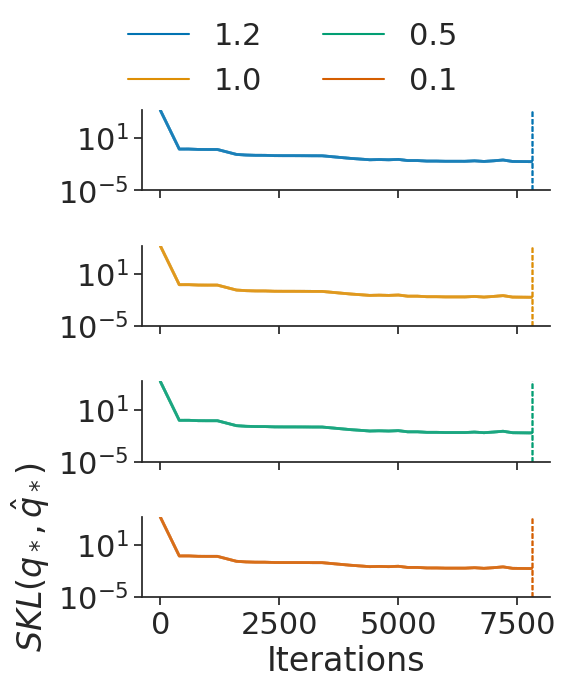} \\
\includegraphics[width=\textwidth,height=0.128\textheight]{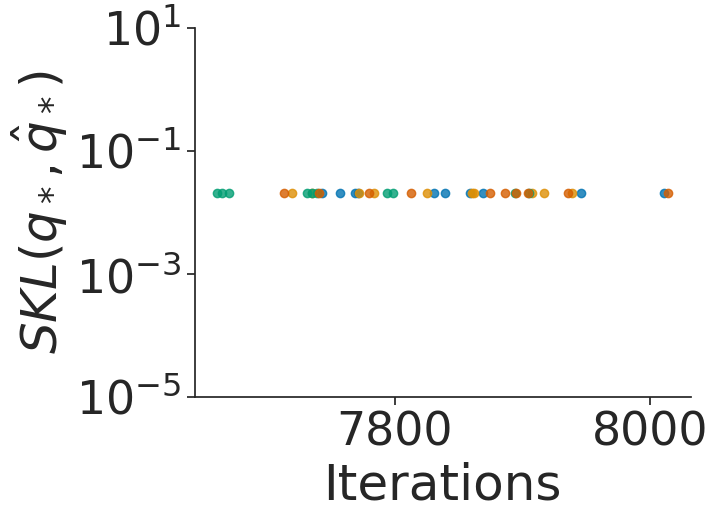}
\caption{\footnotesize uncorrelated $\paramdim = 100$} 
\end{subfigure}  
\begin{subfigure}[t]{.24\textwidth}
\centering
\includegraphics[width=\textwidth,height=0.198\textheight]{Hyper_parameter_robust/avgadam/gaussian_500D_uncorrelated_var_constant_check_of_inefficiency_threshold_lines_avgadam.png} \\
\includegraphics[width=\textwidth,height=0.128\textheight]{Hyper_parameter_robust/avgadam/gaussian_500D_uncorrelated_var_constant_check_of_inefficiency_threshold_scatter_avgadam.png}
\caption{\footnotesize uncorrelated $\paramdim = 500$} 
\end{subfigure}
\begin{subfigure}[t]{.24\textwidth}
\centering
\includegraphics[width=\textwidth]{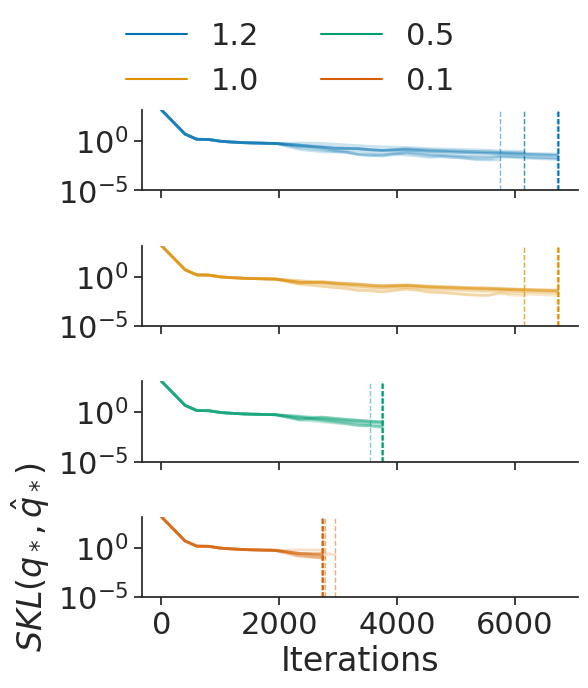} \\
\includegraphics[width=\textwidth]{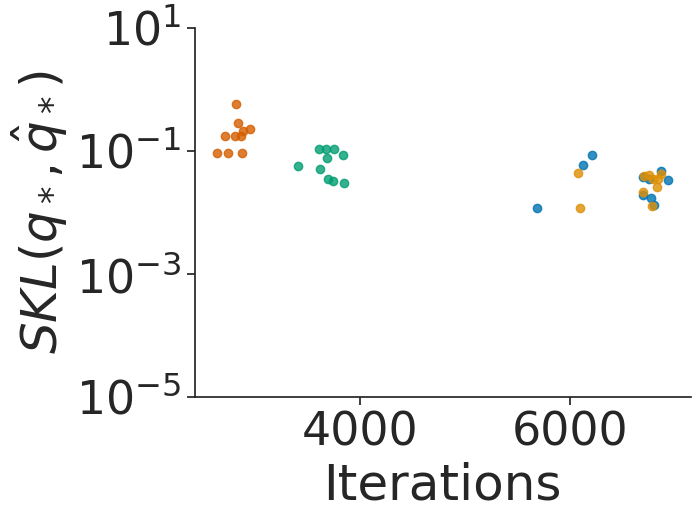}
\caption{\footnotesize uniform correlated $\paramdim = 100$} 
\end{subfigure}  
\begin{subfigure}[t]{.24\textwidth}
\centering
\includegraphics[width=\textwidth]{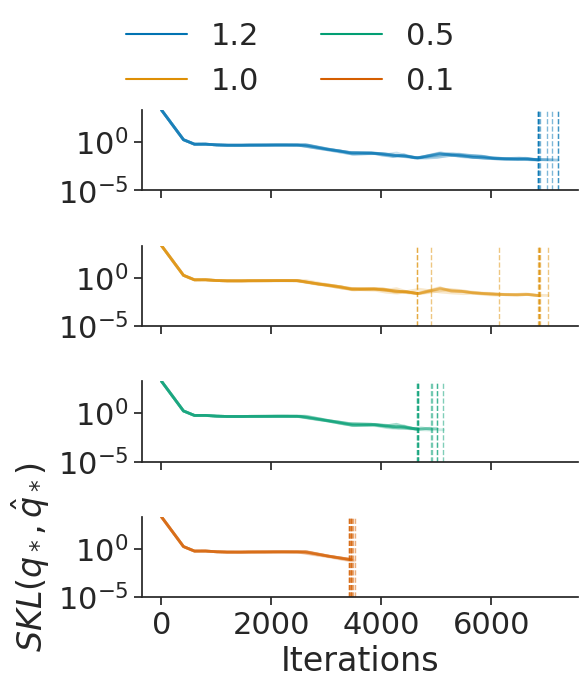} \\
\includegraphics[width=\textwidth]{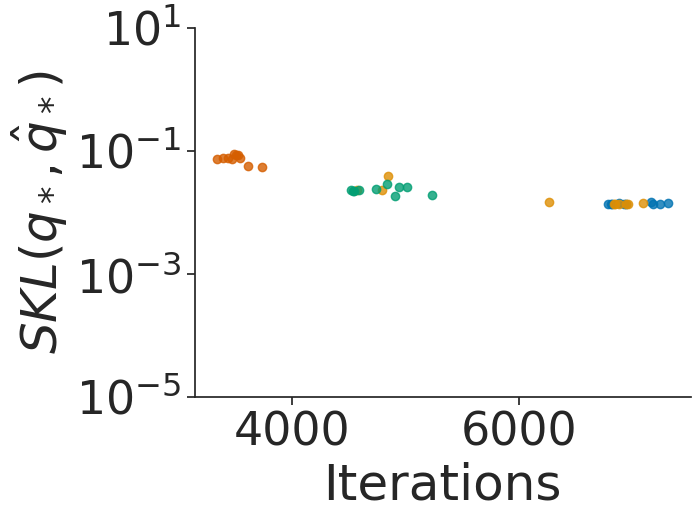}
\caption{\footnotesize banded correlated $\paramdim = 100$} 
\end{subfigure}  
\caption{
Robustness to inefficiency threshold $\tau$ using Gaussian targets.
\textbf{(top)} Iterations versus symmetrized KL divergence between iterate average and optimal variational approximation. 
The distinct lines represent repeated experiments and the vertical lines indicate the termination rule trigger points.
\textbf{(bottom)} Iterations versus symmetrized KL divergence between iterate average and optimal variational approximation
at the termination rule trigger point.}
\label{fig:Robustness-inefficiency-threshold-gaussian}
\end{center}
\end{figure}

\begin{figure}[tbp]
\begin{center}
\begin{subfigure}[t]{.24\textwidth}
\centering
\includegraphics[width=\textwidth]{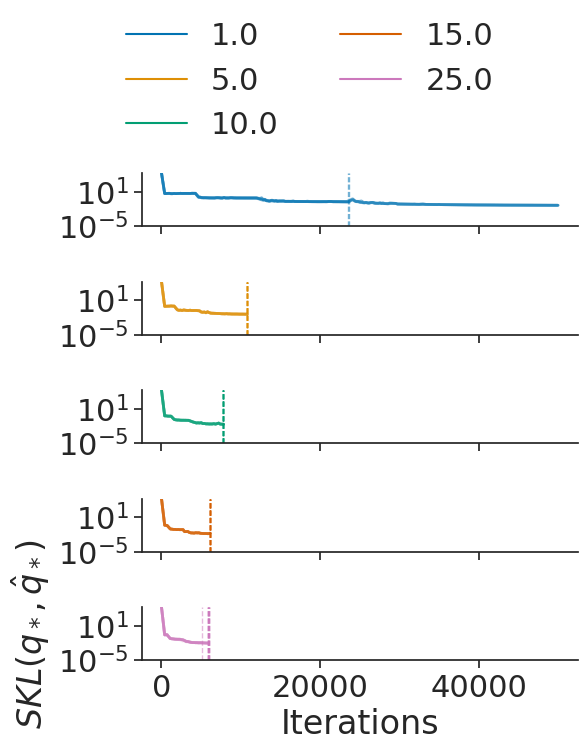} \\
\includegraphics[width=\textwidth]{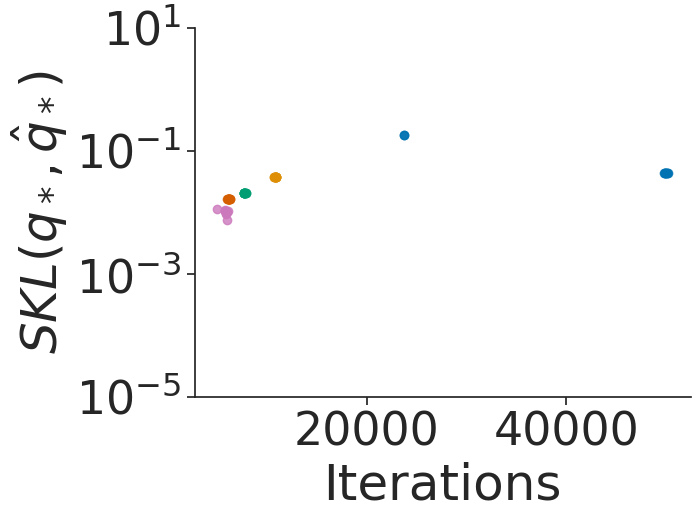}
\caption{\footnotesize uncorrelated $\paramdim = 100$} 
\end{subfigure}  
\begin{subfigure}[t]{.24\textwidth}
\centering
\includegraphics[width=\textwidth]{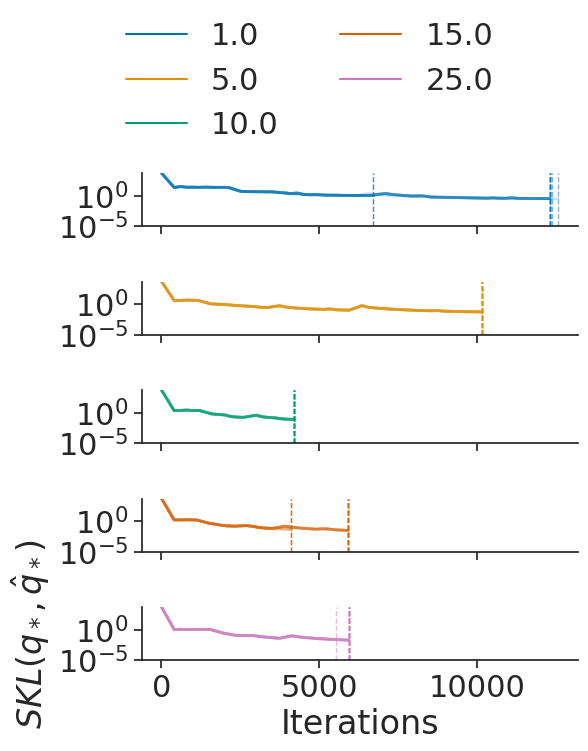} \\
\includegraphics[width=\textwidth]{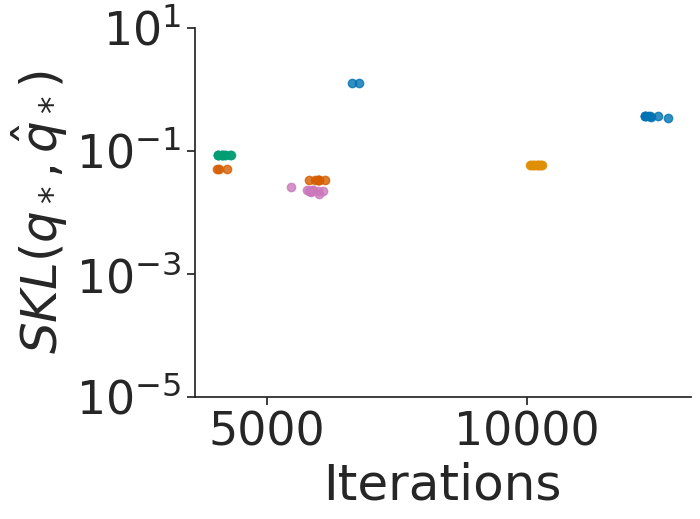}
\caption{\footnotesize uncorrelated $\paramdim = 500$} 
\end{subfigure}
\begin{subfigure}[t]{.24\textwidth}
\centering
\includegraphics[width=\textwidth]{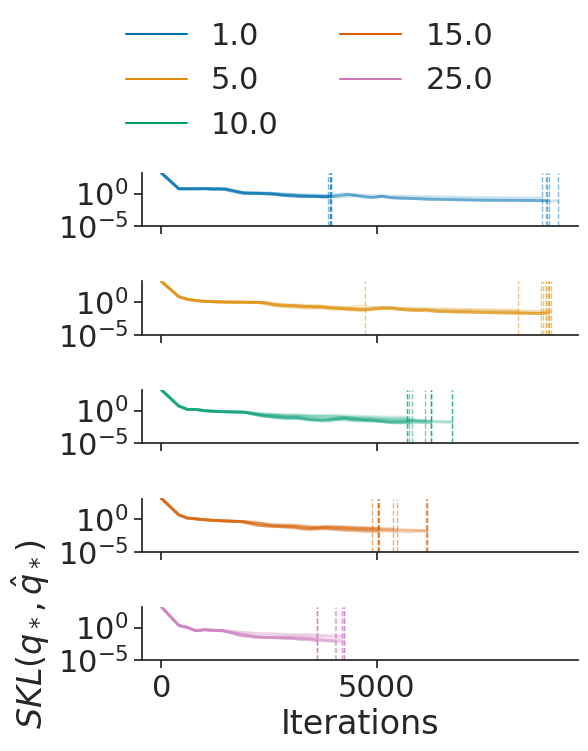} \\
\includegraphics[width=\textwidth]{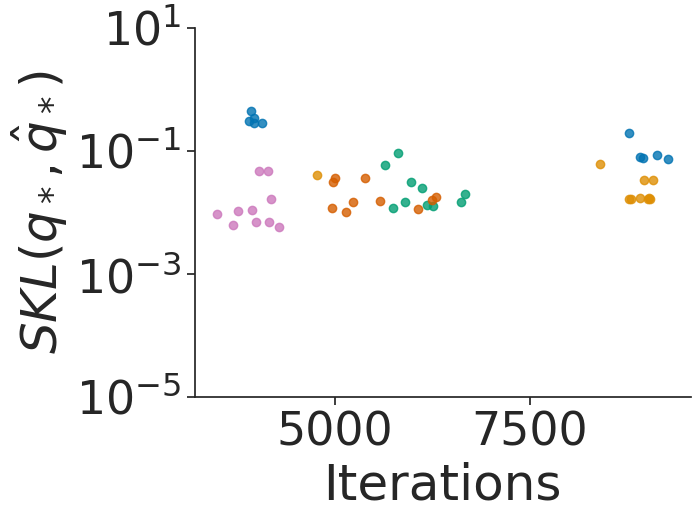}
\caption{\footnotesize uniform correlated $\paramdim = 100$} 
\end{subfigure}  
\begin{subfigure}[t]{.24\textwidth}
\centering
\includegraphics[width=\textwidth]{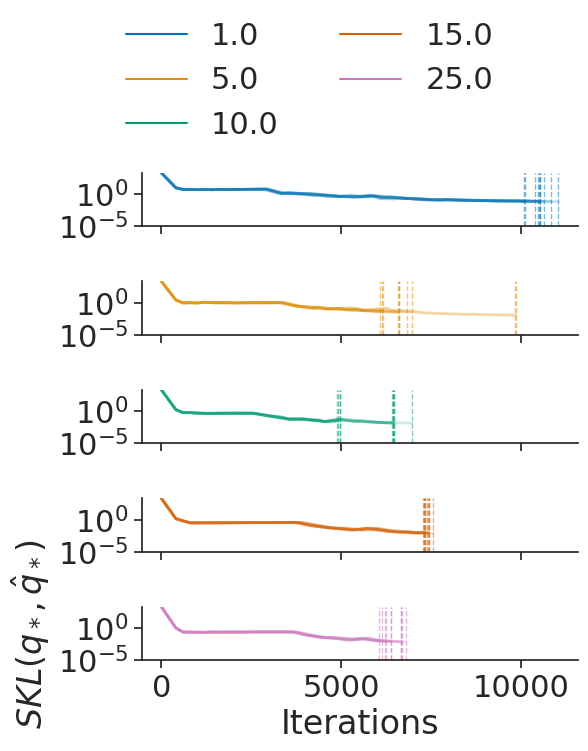} \\
\includegraphics[width=\textwidth]{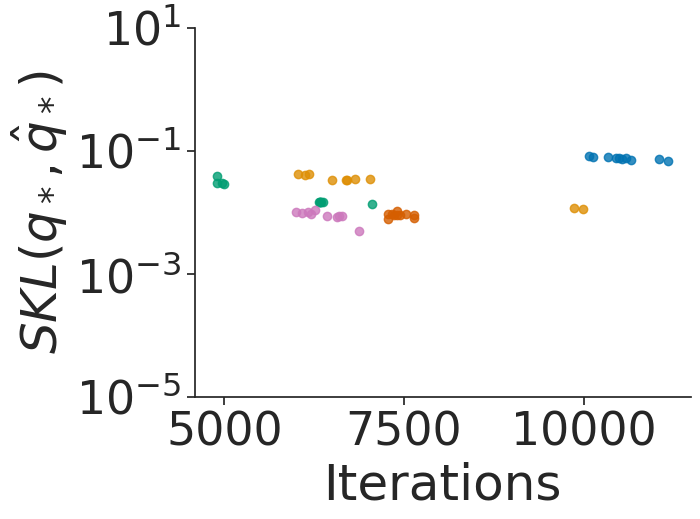}
\caption{\footnotesize banded correlated $\paramdim = 100$} 
\end{subfigure}  
\caption{
Robustness to Monte Carlo samples $M$ using Gaussian targets.
\textbf{(top)} Iterations versus symmetrized KL divergence between iterate average and optimal variational approximation. 
The distinct lines represent repeated experiments and the vertical lines indicate the termination rule trigger points.
\textbf{(bottom)} Iterations versus symmetrized KL divergence between iterate average and optimal variational approximation
at the termination rule trigger point.}
\label{fig:Robustness-mc-samples-gaussian}
\end{center}
\end{figure}

\begin{figure}[tbp]
\begin{center}
\begin{subfigure}[t]{.24\textwidth}
\centering
\includegraphics[width=\textwidth]{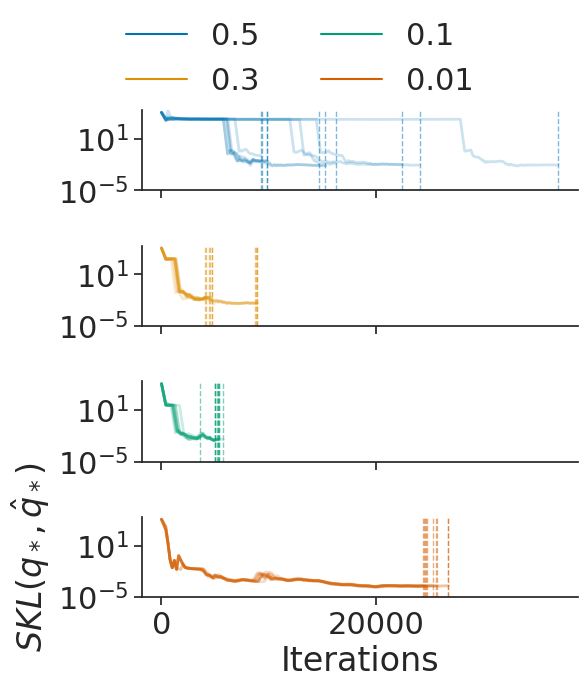}\\
\includegraphics[width=\textwidth]{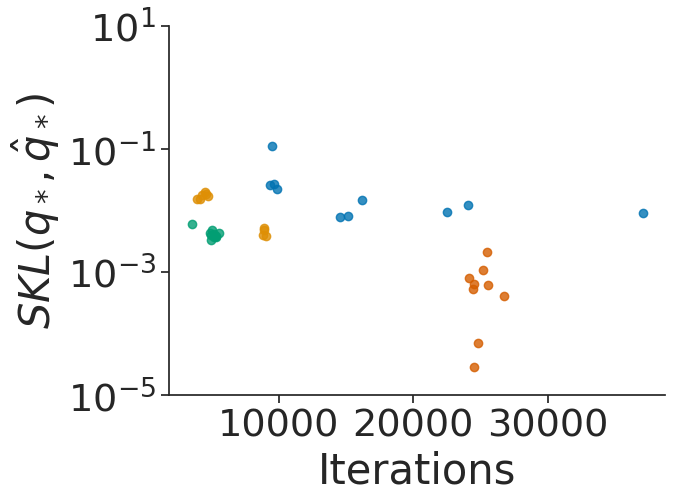}
\caption{$\learningrate_{0}$} 
\end{subfigure} 
\begin{subfigure}[t]{.24\textwidth}
\centering
\includegraphics[width=\textwidth]{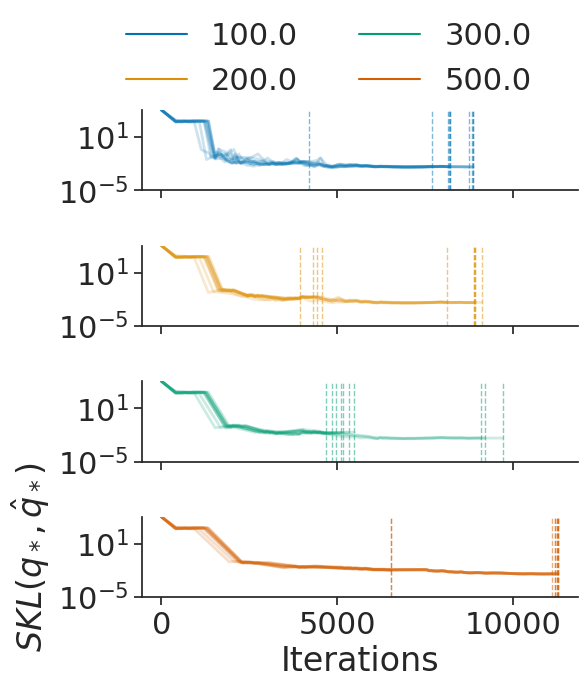} \\
\includegraphics[width=\textwidth]{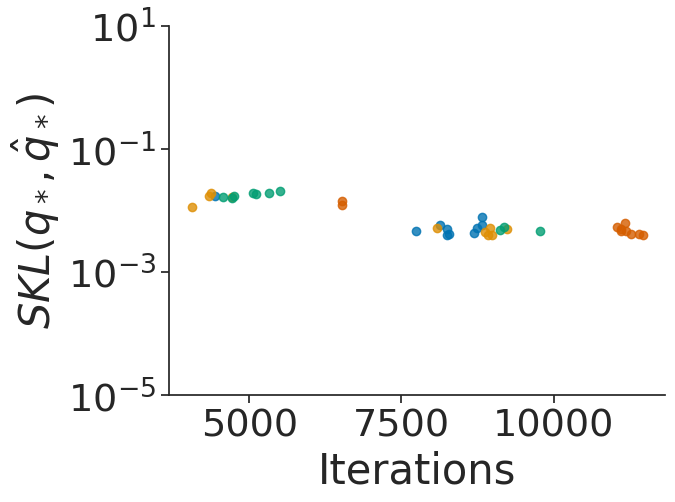}
\caption{$W_{min}$} 
\end{subfigure} 
\begin{subfigure}[t]{.24\textwidth}
\centering
\includegraphics[width=\textwidth]{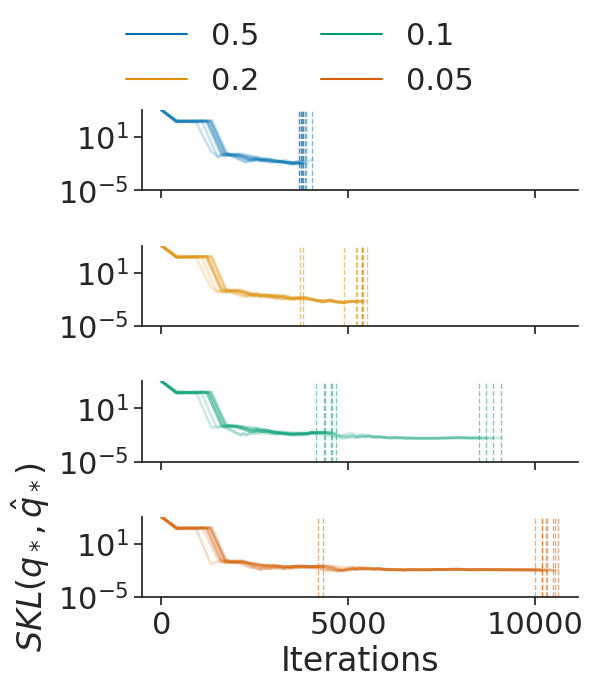} \\
\includegraphics[width=\textwidth]{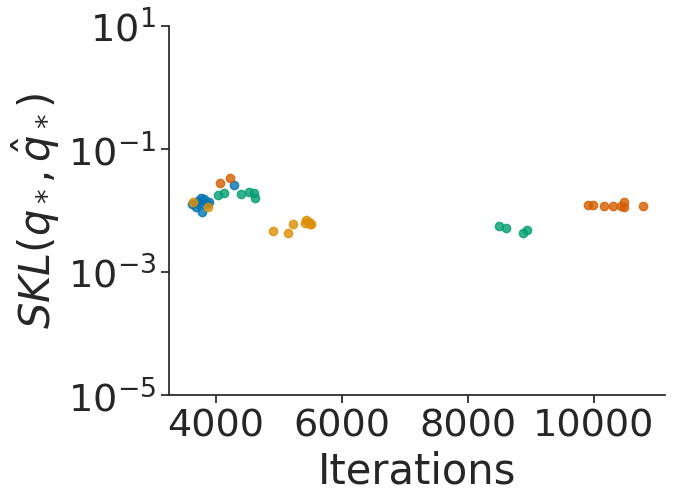}
\caption{$\varepsilon_{0}$} 
\end{subfigure} 
\begin{subfigure}[t]{.24\textwidth}
\centering
\includegraphics[width=\textwidth,height=0.198\textheight]{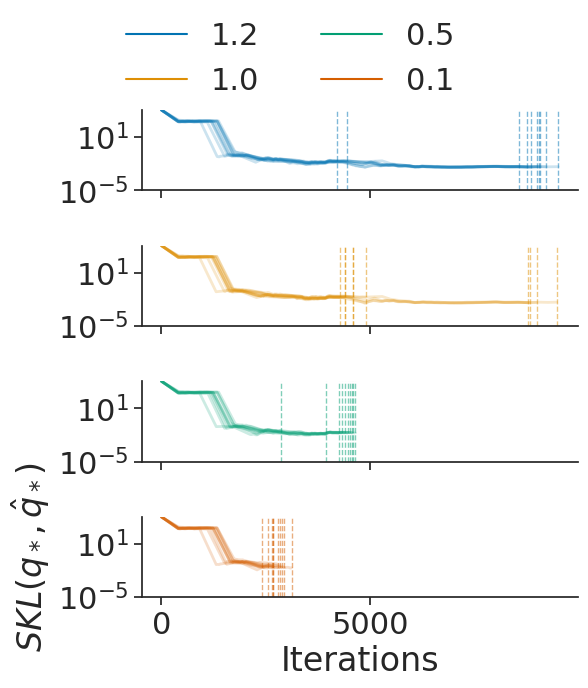} \\
\includegraphics[width=\textwidth,height=0.128\textheight]{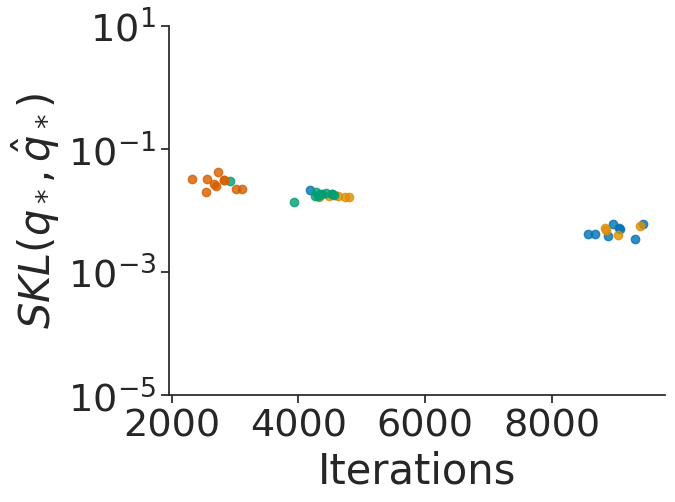}
\caption{$\tau$} 
\end{subfigure}
\caption{
Robustness to tuning parameters (a) initial learning rate $\learningrate_{0}$, (b) minimum window size $W_{\min}$, (c) initial iterate average relative error threshold $\varepsilon_{0}$, and (d)  inefficiency threshold $\tau$ using \textit{arK}  dataset from \texttt{posteriordb} package.
\textbf{(top)} Iterations versus symmetrized KL divergence between iterate average and optimal variational approximation. 
The transparent lines represent repeated experiments and the vertical lines indicate the termination rule trigger points.
\textbf{(bottom)} Iterations versus symmetrized KL divergence between iterate average and optimal variational approximation
at the termination rule trigger point.}
\label{fig:Robustness-tuning-parameters-ark-posteriordb}
\end{center}
\end{figure}

\end{document}